\documentclass[11pt]{article}

\usepackage[preprint]{acl}

\usepackage[utf8]{inputenc}
\usepackage[T1]{fontenc}
\usepackage{times}
\usepackage{latexsym}
\usepackage{inconsolata}
\usepackage{tikz}

\usepackage{amsmath}
\usepackage{amssymb}
\usepackage{amsthm}
\usepackage{amsfonts}

\usepackage{graphicx}
\usepackage{booktabs}
\usepackage{multirow}
\usepackage{enumitem}
\usepackage{subcaption}

\usepackage{url}
\usepackage[nameinlink,noabbrev,capitalise]{cleveref}

\usepackage{microtype}
\usepackage{xcolor}
\usepackage{tcolorbox}
\usepackage{fancyvrb}
\DeclareUnicodeCharacter{00AC}{\ensuremath{\neg}}
\DeclareUnicodeCharacter{00B2}{\ensuremath{^{2}}}
\DeclareUnicodeCharacter{00B7}{\ensuremath{\cdot}}
\DeclareUnicodeCharacter{02E3}{\ensuremath{^{x}}}
\DeclareUnicodeCharacter{0393}{\ensuremath{\Gamma}}
\DeclareUnicodeCharacter{03B6}{\ensuremath{\zeta}}
\DeclareUnicodeCharacter{03BE}{\ensuremath{\xi}}
\DeclareUnicodeCharacter{03C3}{\ensuremath{\sigma}}
\DeclareUnicodeCharacter{03C4}{\ensuremath{\tau}}
\DeclareUnicodeCharacter{2014}{\textemdash{}}
\DeclareUnicodeCharacter{2026}{\ldots{}}
\DeclareUnicodeCharacter{2080}{\ensuremath{_0}}
\DeclareUnicodeCharacter{2081}{\ensuremath{_1}}
\DeclareUnicodeCharacter{2115}{\ensuremath{\mathbb{N}}}
\DeclareUnicodeCharacter{211A}{\ensuremath{\mathbb{Q}}}
\DeclareUnicodeCharacter{211D}{\ensuremath{\mathbb{R}}}
\DeclareUnicodeCharacter{2124}{\ensuremath{\mathbb{Z}}}
\DeclareUnicodeCharacter{2192}{\ensuremath{\rightarrow}}
\DeclareUnicodeCharacter{2194}{\ensuremath{\leftrightarrow}}
\DeclareUnicodeCharacter{21D2}{\ensuremath{\Rightarrow}}
\DeclareUnicodeCharacter{2200}{\ensuremath{\forall}}
\DeclareUnicodeCharacter{2203}{\ensuremath{\exists}}
\DeclareUnicodeCharacter{2205}{\ensuremath{\emptyset}}
\DeclareUnicodeCharacter{2208}{\ensuremath{\in}}
\DeclareUnicodeCharacter{220F}{\ensuremath{\textstyle\prod}}
\DeclareUnicodeCharacter{2211}{\ensuremath{\textstyle\sum}}
\DeclareUnicodeCharacter{2212}{\ensuremath{-}}
\DeclareUnicodeCharacter{2223}{\ensuremath{\mid}}
\DeclareUnicodeCharacter{2227}{\ensuremath{\wedge}}
\DeclareUnicodeCharacter{2228}{\ensuremath{\vee}}
\DeclareUnicodeCharacter{2229}{\ensuremath{\cap}}
\DeclareUnicodeCharacter{222B}{\ensuremath{\textstyle\int}}
\DeclareUnicodeCharacter{2260}{\ensuremath{\ne}}
\DeclareUnicodeCharacter{2264}{\ensuremath{\le}}
\DeclareUnicodeCharacter{22C0}{\ensuremath{\textstyle\bigwedge}}
\DeclareUnicodeCharacter{22C2}{\ensuremath{\textstyle\bigcap}}
\DeclareUnicodeCharacter{22C3}{\ensuremath{\textstyle\bigcup}}
\DeclareUnicodeCharacter{2713}{\ensuremath{\checkmark}}
\DeclareUnicodeCharacter{2717}{\ensuremath{\times}}
\DeclareUnicodeCharacter{27F6}{\ensuremath{\longrightarrow}}
\DeclareUnicodeCharacter{27F7}{\ensuremath{\longleftrightarrow}}
\DeclareUnicodeCharacter{27F9}{\ensuremath{\Longrightarrow}}
\newtcolorbox{promptbox}[1]{breakable,enhanced,colback=white,colframe=black,boxrule=0.5pt,arc=0mm,left=4pt,right=4pt,title={#1},fonttitle=\small\bfseries,coltitle=white}
\newlength{\promptchar}
\newenvironment{promptlines}{\par\ttfamily\scriptsize\settowidth{\promptchar}{x}\raggedright\setlength{\parindent}{0pt}\setlength{\parskip}{0pt}\hyphenpenalty=10000\exhyphenpenalty=10000\frenchspacing}{\par}
\newcommand{\pl}[2]{\hangindent=\dimexpr#1\promptchar+2\promptchar\relax\hangafter=1\hspace*{#1\promptchar}#2\par}
\newcommand{\pe}{\vspace{0.5\baselineskip}}

\tcbuselibrary{skins,breakable}

\newtheoremstyle{boxed}{0pt}{0pt}{\normalfont}{}{\bfseries\color{tcbcolframe}}{}{0.5em}{\thmname{#1}\thmnumber{ #2}\thmnote{ (#3)}}
\theoremstyle{boxed}
\newtheorem{theorem}{Theorem}[section]
\newtheorem{proposition}[theorem]{Proposition}
\newtheorem{lemma}[theorem]{Lemma}
\newtheorem{corollary}[theorem]{Corollary}
\theoremstyle{definition}
\newtheorem{definition}[theorem]{Definition}

\theoremstyle{remark}

\usepackage{etoolbox}
\AtBeginEnvironment{proposition}{\crefalias{theorem}{proposition}}
\AtBeginEnvironment{lemma}{\crefalias{theorem}{lemma}}
\AtBeginEnvironment{corollary}{\crefalias{theorem}{corollary}}
\AtBeginEnvironment{definition}{\crefalias{theorem}{definition}}
\AtBeginEnvironment{assumption}{\crefalias{theorem}{assumption}}
\AtBeginEnvironment{remark}{\crefalias{theorem}{remark}}
\crefname{assumption}{Assumption}{Assumptions}
\Crefname{assumption}{Assumption}{Assumptions}

\newcommand{\px}[1]{\textsc{#1}}
\newcommand{\sd}[1]{{\scriptsize$\pm$#1}}
\newcommand{\msd}[2]{#1\sd{#2}}
\newcommand{\sig}{\rlap{$^{*}$}}
\newcommand{\Prob}{\mathbb{P}}
\definecolor{ThmAcc}{HTML}{2E7D6B}\definecolor{ThmAccBg}{HTML}{EEF7F4}
\definecolor{ThmAdd}{HTML}{6A3D9A}\definecolor{ThmAddBg}{HTML}{F4F0FA}
\definecolor{ThmFix}{HTML}{A3263B}\definecolor{ThmFixBg}{HTML}{FBEFF1}
\definecolor{ThmPlat}{HTML}{8A6D00}\definecolor{ThmPlatBg}{HTML}{FAF6E6}
\newtcolorbox{thmbox}[2]{colframe=#1, colback=#2, boxrule=0.4pt, arc=1.5pt, boxsep=0pt,
  left=2.5pt, right=2.5pt, top=2pt, bottom=2pt, grow to left by=2.9pt, grow to right by=2.9pt,
  before skip=4pt, after skip=6pt,
  before upper={\setlength{\topsep}{0pt}\setlength{\partopsep}{0pt}}}
\DeclareRobustCommand{\thmcref}[2]{{\hypersetup{linkcolor=#1}\textcolor{#1}{\cref{#2}}}}
\DeclareRobustCommand{\thmCref}[2]{{\hypersetup{linkcolor=#1}\textcolor{#1}{\Cref{#2}}}}
\newcommand{\refacc}[1]{\thmcref{ThmAcc}{#1}}
\newcommand{\refadd}[1]{\thmcref{ThmAdd}{#1}}\newcommand{\Refadd}[1]{\thmCref{ThmAdd}{#1}}
\newcommand{\reffix}[1]{\thmcref{ThmFix}{#1}}\newcommand{\Reffix}[1]{\thmCref{ThmFix}{#1}}
\newcommand{\refplat}[1]{\thmcref{ThmPlat}{#1}}
\newif\ifrestating
\newcommand{\stmteq}[1]{\ifrestating\tag{\ref{#1}}\else\label{#1}\fi}
\newenvironment{restated}[2]{\restatingtrue\noindent\textbf{\textcolor{tcbcolframe}{#1~\ref{#2}}}\enspace\ignorespaces}{\par}
\definecolor{AccBlue}{HTML}{0B3289}
\definecolor{FidOrange}{HTML}{E85308}
\newcommand{\accline}{\tikz[baseline=-0.6ex]{\draw[AccBlue, line width=1.5pt] (0,0) -- (0.5,0); \fill[AccBlue] (0.25,0) circle (1.6pt);}}
\newcommand{\fidline}{\tikz[baseline=-0.6ex]{\draw[FidOrange, line width=1.5pt, dashed] (0,0) -- (0.5,0); \fill[FidOrange] (0.22,-0.05) rectangle (0.28,0.05);}}

\title{Reasoning without Gold Standards:\\ A Proxy-Judge Theory of Autoformalization}

\author{
  Lei Xu$^{1,2}$ \quad Xin Quan$^{1}$ \quad André Freitas$^{1,3,4}$ \\
  $^{1}$Idiap Research Institute, Switzerland \\
  $^{2}$École Polytechnique Fédérale de Lausanne (EPFL), Switzerland \\
  $^{3}$Department of Computer Science, University of Manchester, United Kingdom \\
  $^{4}$CRUK National Biomarker Centre, University of Manchester, United Kingdom \\
  \texttt{\{lei.xu, xin.quan, andre.freitas\}@idiap.ch}
}

\begin{document}
\maketitle

\begin{abstract}
Autoformalization (AF) requires both prover acceptance and fidelity to an informal source, but prover acceptance alone cannot establish fidelity, and expert-validated reference formalizations are costly to obtain. We introduce a reference-free proxy-judge framework, which separates correctness into grounding, statement faithfulness, coverage, and inference, and a theory of how imperfect checks shape iterative refinement. We establish three results. First, for a fixed sequence of candidates, an added proxy's false-positive and false-negative rates bound its effect on the pass-but-wrong rate. Second, damage to a correct formalization requires a false rejection that exposes it to further refinement, and harmful revisions are then retained through unguarded changes, missed changes, or prover progress that overrides judge warnings. Third, a decomposition of the final pass rate distinguishes continued refinement, generation of new candidates, and their success in passing the checks. We instantiate the framework with eight proxies and test these predictions on ProofNet and miniF2F in Lean~4 and on e-SNLI and ProntoQA in Isabelle/HOL. Under single-proxy removal, no prediction written before the runs is contradicted, and all traced breaks follow the identified failure routes. The results connect proxy reliability to the benefits, risks, and limits of refinement without reference formalizations.
\end{abstract}

\section{Introduction}\label{sec:intro}
Reasoning benchmarks often assume that an expert reference output specifies correctness in advance, yet complex reasoning tasks increasingly produce outputs whose correctness cannot be judged by exact match against a single reference. Autoformalization (AF) is a representative case, which asks a model to translate an informal mathematical or logical argument into a formally checkable object \citep{wu2022autoformalization,jiang2023draft}.

\begin{figure*}[t]
\centering
\includegraphics[width=\textwidth]{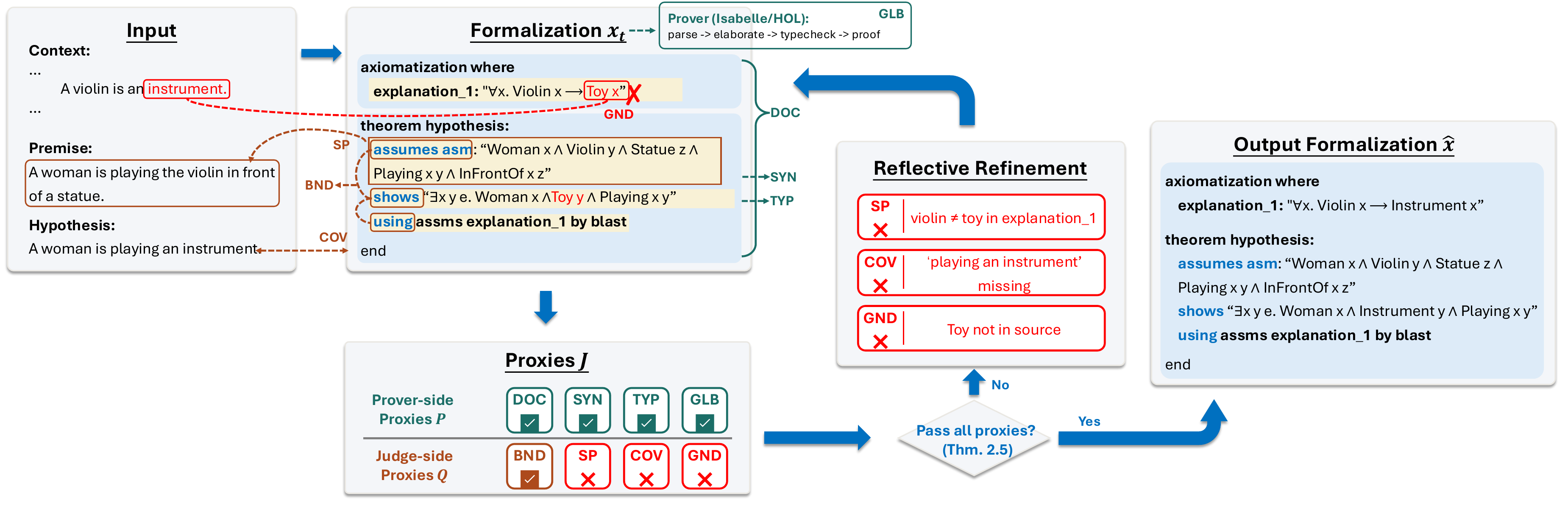}
\caption{An illustrative refinement loop on a problem adapted from e-SNLI, into which we inserted the error \texttt{Toy}. The model formalizes the source into $x_t$, the prover-side proxies $P$ and the judge-side proxies $Q$ check it, and the failing units of \px{sp}, \px{cov}, and \px{gnd} guide refinement. The loop returns $\hat x$ once every proxy passes or the budget ends.}
\label{fig:overview}
\end{figure*}
In mathematics and formal logic, building a reference may itself require solving the task, so gold-standard annotation does not scale beyond simplified cases. A single informal argument also admits many valid renderings, so exact match penalizes correct alternatives and collapses distinct errors into one score \citep{zhang2025beyond}. A proof assistant rejects ill-formed outputs \citep{polu2020generative,yang2023leandojo} but cannot decide whether the output preserves the meaning of the source \citep{zhang2025beyond}, so evaluation must move from reference reproduction to structured evidence.

Although AF is hard to annotate globally, it can be checked locally. Once an argument is rendered as explicit statements, inferential links, and grounding into a shared context, correctness decomposes into finitely many locally checkable constraints. We therefore introduce a reference-free proxy-judge framework for AF, which decomposes correctness into four semantic criteria, namely grounding, statement, coverage, and inference, and checks each with the prover or with judges on small, controlled contexts. Each detected violation leads to a localized refinement rather than an undifferentiated retry (\cref{fig:overview}), an acceptance rule decides which refined candidate the loop keeps, and we analyze this loop under \textit{imperfect proxies}.

Our analysis establishes three results. First, for the same candidates, the increase in incorrect outputs that pass all checks after adding a check is bounded by how often that check rejects a correct formalization that would otherwise have been returned (\refadd{thm:add}). Second, refinement damages a correct formalization only after a false positive and only through an unchecked semantic change, an error missed by the relevant checks, or an acceptance rule that prioritizes prover progress over semantic fidelity (\reffix{thm:fixbreak}). Third, the final pass rate adds to the initial successes the gains of each round, which depend on how many failing problems remain under refinement, how often a step produces a new formalization, and how often it passes all checks (\refplat{lem:plateau}).

This paper addresses the theoretical gap concerning the dynamics of how imperfect judges and critics shape correctness over successive refinement steps. To the best of our knowledge, this is the first theoretical account of how imperfect proxies can lower correctness during refinement. Our contributions can be summarized as follows.
\begin{enumerate}[leftmargin=1.5em,itemsep=1pt,topsep=2pt]
  \item \emph{Theory.} We derive three results that give guarantees for near-perfect proxies and directional predictions for estimated rates (\cref{sec:theory}).
  \item \emph{Framework.} We instantiate the four criteria as eight proxies and close the loop with an acceptance rule, so that each violation maps to an identifiable refinement target (\cref{sec:prelim}).
  \item \emph{Empirical result.} Experiments on ProofNet and miniF2F in Lean~4 and on e-SNLI and ProntoQA in Isabelle/HOL test the theory. Whether a proxy is worth adding depends on its error rates on the target data, so removing the reasoning-chain proxy and the prover errors that it reports lowers both metrics on ProofNet but neither on e-SNLI. The statement-level faithfulness judge filters the most meaning errors among the judges, and a prover-first acceptance rule accepts meaning changes that the judges already flag. Among the compared inference-time methods, our loop reaches the highest prover acceptance and semantic fidelity on both Lean datasets (\cref{sec:exp}).
\end{enumerate}

\newcommand{\stmtpropaccept}{%
If every unit has a fixed verdict, the kept version changes only finitely often. For any verdicts, the prover state never decreases along kept versions, namely
\begin{equation}\stmteq{eq:stmt-prov}
\mathrm{prov}(\tilde x_t)\succeq\mathrm{prov}(\tilde x_{t-1})\qquad\text{for }t\ge1.
\end{equation}}
\newcommand{\stmtthmadd}{%
Let the loops with proxy sets $J$ and $J\cup A$ face the same candidates and verdicts. Then
\begin{enumerate}[label=(\roman*),font=\itshape,leftmargin=*,nosep]
\item \textit{Later stopping}. $\tau_{J\cup A}\ge\tau_J$.
\item \textit{Bounds}. $m_A\,w_J\le w_{J\cup A}\le w_J+c_J\,f_A$.
\end{enumerate}}
\newcommand{\stmtcoradd}{%
If $f_A=0$, $w_{J\cup A}\le w_J$. If $m_A=1$, $w_{J\cup A}\ge w_J$.}
\newcommand{\stmtthmfixbreak}{%
For every refinement step $t\ge1$,
\begin{enumerate}[label=(\roman*),font=\itshape,leftmargin=*,nosep]
\item \textit{Break bound}. $\Prob[\mathsf{Break}_t]\le r_t\,\varepsilon_t$.
\item \textit{End-to-end bound}.
\begin{equation}\stmteq{eq:stmt-e2e}
\begin{split}
&\Prob[R(x_0)=1,\ R(\hat x)=0]\\
&\quad\le\Prob[R(x_0)=1,\ \mathrm{Pass}_J(x_0)=0].
\end{split}
\end{equation}
\end{enumerate}}
\newcommand{\stmtpropbeta}{%
If ${R(\cdot)=V(\cdot)\wedge M(\cdot)}$, ${\beta=0}$.}
\newcommand{\stmtcorfixbreak}{%
If $\Prob[\mathsf{Fix}_t]\ge r_t\,\varepsilon_t$, then
\begin{equation}\stmteq{eq:stmt-monocorr}
\Prob[R(\tilde x_t)=1]\ge\Prob[R(\tilde x_{t-1})=1].
\end{equation}
If $r_t=0$ for all $t$, or if ${R(\cdot)=V(\cdot)\wedge M(\cdot)}$ and every change type has a guard with $q_z=0$, the inequality holds at every $t\ge1$.}
\newcommand{\stmtlemplateau}{%
For any refinement budget,
\begin{equation}\stmteq{eq:stmt-plateau}
\begin{split}
&\Prob[\mathrm{Pass}_J(\hat x)=1]\\
&\quad=\pi_0+(1-\pi_0)\textstyle\sum_{t\ge1}s_t\,\nu_t\,\theta_t .
\end{split}
\end{equation}}

\section{Semantic Criteria and Bounds for Imperfect Proxies}\label{sec:theory}

Formalizing an argument requires both acceptance by the theorem prover and preservation of the source’s meaning. While prover acceptance can be checked directly and deterministically, semantic fidelity must be assessed without a gold-standard formalization. We therefore express semantic fidelity through explicit criteria, evaluate them using local proxy checks, and return an output only when all required checks pass.

\subsection{Proxies and Correctness Labels}\label{ssec:prelim-judges}

Since an informal source admits many valid formalizations, we evaluate a formalization without a gold reference through a finite set of generative or deterministic property checks, which we call \textit{proxies} (\cref{fig:overview}). Let $x$ denote one complete formalization of a problem. A proxy $k$ splits $x$ into units $U_k(x)$ and reports the failing units $F_k(x)\subseteq U_k(x)$ and the undecided units $D_k(x)\subseteq U_k(x)$. The prover-side proxies $P$ read the prover output, the judge-side proxies $Q$ compare $x$ with the source text, and $J=P\cup Q$.

Let $V(x)=1$ if the prover accepts $x$ and $M(x)=1$ if $x$ preserves the meaning of the source, and let both be 0 otherwise. We label $x$ correct by $R(x)=V(x)\wedge M(x)$.

\begin{definition}\label{def:proxy}
We define whether $x$ passes a proxy $k$ and whether it passes every proxy in a set $J$ as $\mathrm{Pass}_k(x)=\mathbf 1[F_k(x)\cup D_k(x)=\varnothing]$ and $\mathrm{Pass}_J(x)=\prod_{k\in J}\mathrm{Pass}_k(x)$. For a correctness label $R(\cdot)$ and a distribution of outputs, the false-positive rate and the false-negative rate of $J$ are
\begin{equation}\label{eq:fm}
\begin{gathered}
f_J=\Prob[\mathrm{Pass}_J(x)=0\mid R(x)=1],\\
m_J=\Prob[\mathrm{Pass}_J(x)=1\mid R(x)=0].
\end{gathered}
\end{equation}
\end{definition}

An output passes a proxy only if no unit fails or stays undecided. Counting a rejection as a positive, the false-positive rate $f_J$ is the fraction of correct outputs that $J$ rejects, and the false-negative rate $m_J$ the fraction of incorrect outputs that it passes.

\subsection{Semantic Criteria}\label{ssec:prelim-gate}

\begin{table}[t]
\centering
\small
\setlength{\tabcolsep}{3pt}
\begin{tabular}{@{}lp{4.3cm}l@{}}
\toprule
Criterion & Property checked & Proxies \\
\midrule
Grounding & Every \textbf{symbol} denotes a concept of the source or a library definition & \px{gnd} \\
Statement & Every formal \textbf{statement} says what its source sentence says & \px{sp} \\
Coverage & Every \textbf{claim of the source} appears in the formalization & \px{cov} \\
Inference & The \textbf{conclusion follows} from the stated premises and steps & \px{bnd} \\
\midrule
Precondition & The output \textbf{parses, type-checks}, and has a valid \textbf{document structure} & \begin{tabular}[t]{@{}l@{}}\px{glb}, \px{doc},\\ \px{syn}, \px{typ}\end{tabular} \\
\bottomrule
\end{tabular}
\caption{Each semantic criterion requires one property of the output, and the listed proxies check it.}
\label{tab:gates}
\end{table}

We assign each proxy to a correctness criterion (\cref{tab:gates}). Grounding concerns what the symbols denote, statement whether each formal statement preserves its source meaning, coverage whether all source claims are represented, and inference whether the conclusion follows from the stated premises. A change that alters correctness alters one of these criteria or breaks well-formedness, their common precondition, and we call grounding, statement, and coverage the meaning criteria. The change type $z$ of a refinement step is the nonempty set of the types of its atomic changes. A guard of a meaning criterion is a judge-side proxy that checks it against the source text, since only judge-side proxies make this comparison.

\begin{definition}\label{def:gate}
The guards of a change type $z$ are $G(z)=\bigcup_{a\in z}G(a)$, where each atomic type $a$ has as guards the set $G(a)\subseteq Q$ of guards of the meaning criteria that $a$ alters.
\end{definition}

 A change type without guards can alter meaning unnoticed (\cref{app:gatekeepers}).

\subsection{The Acceptance Rule}\label{ssec:prelim-loop}

A candidate does not automatically replace the kept version, and \cref{def:accept} decides when it does. Starting from the first draft $\tilde x_0=x_0$, each step $t$ writes a candidate $x_t$ from the kept version $\tilde x_{t-1}$ and feedback on its failing units. The loop stops when the kept version passes all proxies or the budget ends, and returns that version as $\hat x$. To rank candidates that the prover rejects, we record how close each is to acceptance. The prover state of $x$ is the ordered pair $\mathrm{prov}(x)=(V(x),-\mathrm{err}(x))$, where $\mathrm{err}(x)$ counts the classes of prover-reported errors in $x$ (\cref{app:acceptance}). We compare prover states lexicographically, first by acceptance and then by the number of error classes, so a refinement step that clears one class of prover error counts as progress before the prover accepts.

\begin{definition}\label{def:accept}
The candidate $x_t$ replaces $\tilde x_{t-1}$ if and only if one of the following holds.
\begin{enumerate}[label=(\roman*),leftmargin=*,nosep]
\item \textit{Full pass}. $\mathrm{Pass}_J(x_t)=1$.
\item \textit{Prover progress}. $\mathrm{prov}(x_t)\succ\mathrm{prov}(\tilde x_{t-1})$.
\item \textit{Judge progress}. $\mathrm{prov}(x_t)=\mathrm{prov}(\tilde x_{t-1})$, $\Delta_k(\tilde x_{t-1},x_t)=\varnothing$ for all $k\in Q$, and $F_k(\tilde x_{t-1})\not\subseteq F_k(x_t)$ for some $k\in Q$.
\end{enumerate}
\end{definition}

Rule (i) keeps any candidate that passes every proxy, and rule (ii) keeps any gain in prover state even when judges flag new units. Rule (iii) keeps a candidate with the same prover state only if it removes at least one failing unit and adds no effective new failing unit, which is a unit that fails on the candidate but neither fails on the kept version nor stays flagged when the kept version is re-judged. For a judge-side proxy $k$, let $X_k(\tilde x)$ be the units of $\tilde x$ that re-judging still flags. The effective new failing units of a candidate $x$ against the kept version $\tilde x$ are then $\Delta_k(\tilde x,x)=F_k(x)\setminus(F_k(\tilde x)\cup X_k(\tilde x))$.

\begin{thmbox}{ThmAcc}{ThmAccBg}
\begin{proposition}\label{prop:accept}
\stmtpropaccept
\end{proposition}
\end{thmbox}

A gain in the prover state $\mathrm{prov}(\cdot)$ is never undone, and \cref{app:proofs} proves the proposition.

\subsection{Pass-but-Wrong Rate under Added Proxies}\label{ssec:theory-add}

An added proxy filters out incorrect outputs but can also reject correct ones, which then stay in the loop and are exposed to further revisions. To isolate this filtering effect, we hold the candidate sequence $x_0,x_1,\dots$ fixed and vary only the proxy set. A loop with proxy set $J$ stops at $\tau_J=\min\{t:\mathrm{Pass}_J(x_t)=1\}$ and returns $\hat x_J=x_{\tau_J}$. The fractions of problems on which it returns a correct or incorrect passing candidate are
\begin{equation}\label{eq:cw}
\begin{gathered}
c_J=\Prob[\tau_J<\infty,\ R(\hat x_J)=1],\\
w_J=\Prob[\tau_J<\infty,\ R(\hat x_J)=0],
\end{gathered}
\end{equation}
and we call $w_J$ the pass-but-wrong rate. An added set of proxies $A$ acts on the candidates that the $J$-loop returns, so we measure its false-positive and false-negative rates there,
\begin{equation}\label{eq:fm-added}
\begin{gathered}
f_A\!=\!\Prob[\mathrm{Pass}_A(\hat x_J)\!=\!0\mkern1mu|\mkern1mu\tau_J\!<\!\infty,R(\hat x_J)\!=\!1],\\
m_A\!=\!\Prob[\mathrm{Pass}_A(\hat x_J)\!=\!1\mkern1mu|\mkern1mu\tau_J\!<\!\infty,R(\hat x_J)\!=\!0].
\end{gathered}
\end{equation}

\begin{thmbox}{ThmAdd}{ThmAddBg}
\begin{theorem}\label{thm:add}
\stmtthmadd
\end{theorem}
\end{thmbox}

By \refadd{thm:add}(i), a stricter loop never stops earlier. By \refadd{thm:add}(ii), the added proxies raise $w$ by at most the fraction $c_J f_A$ of problems whose correct return they reject and keep at least the fraction $m_A$ of $w_J$. An added proxy without false positives ($f_A=0$) never raises $w$, and one that misses every incorrect output ($m_A=1$) never lowers it (\refadd{cor:add}). \Cref{app:proofs} proves both results.

\subsection{Bounding Breaks}\label{ssec:theory-fixbreak}

Although \refacc{prop:accept} keeps the prover state from falling, correctness can fall along kept versions, because a step that satisfies one proxy can change meaning that another proxy guards, so we split the change in correctness into fixes and breaks. Let $\mathsf{Fix}_t$ be the set of problems on which step $t$ turns an incorrect kept version into a correct one, and $\mathsf{Break}_t$ the set on which it turns a correct kept version into an incorrect one, so correctness can fall only through breaks. Let $\mathsf{Act}_t$ be the set of problems on which step $t$ is issued, and $z_t$ its change type. A break needs a correct version under refinement, and the fraction of such problems is $r_t=\Prob[\mathsf{Act}_t,R(\tilde x_{t-1})=1]$. There the step has type $z$ and an incorrect candidate with probability
\begin{equation}\label{eq:alpha}
\alpha_z\!=\!\Prob[z_t\!=\!z,R(x_t)\!=\!0\mkern1mu|\mkern1mu\mathsf{Act}_t,R(\tilde x_{t-1})\!=\!1],
\end{equation}
and the probability that rule (ii) of \cref{def:accept} then accepts an incorrect candidate is
\begin{multline}\label{eq:beta}
\beta=\Prob\bigl[x_t\text{ accepted by rule (ii)},\ R(x_t)=0\ \big|\\
\mathsf{Act}_t,\ R(\tilde x_{t-1})=1\bigr].
\end{multline}
A guard $k$ is silent at step $t$ if it reports no effective new failing unit on $x_t$, so an undecided item leaves it silent. Let $\mathsf{Silent}_{t,z}$ be the set of problems on which every guard in $G(z)$ is silent at step $t$. The probability that they stay silent on a type-$z$ step from a correct version to an incorrect candidate is
\begin{multline}\label{eq:qz}
q_z=\Prob\bigl[\mathsf{Silent}_{t,z}\ \big|\ \mathsf{Act}_t,\\
R(\tilde x_{t-1})=1,\ z_t=z,\ R(x_t)=0\bigr].
\end{multline}
The leak of refinement step $t$ adds up the three ways an incorrect candidate replaces a correct version,
\begin{equation}\label{eq:eps}
\varepsilon_t=\sum_{z:\,G(z)=\varnothing}\alpha_z+\sum_{z:\,G(z)\neq\varnothing}\alpha_z\,q_z+\beta .
\end{equation}

\begin{thmbox}{ThmFix}{ThmFixBg}
\begin{theorem}\label{thm:fixbreak}
\stmtthmfixbreak
\end{theorem}
\end{thmbox}

By \reffix{thm:fixbreak}(i), $\mathsf{Break}_t$ needs a correct version that a false positive kept in the loop, which $r_t$ counts, and a candidate that takes one of the three routes of $\varepsilon_t$, namely a change type without a guard, a silent guard, or acceptance by rule (ii). \Reffix{thm:fixbreak}(ii) bounds the correct first drafts that end incorrect by those that fail $J$, and we also call such an end-to-end change a break. Under ${R(\cdot)=V(\cdot)\wedge M(\cdot)}$, rule (ii) never replaces a correct version and $\beta=0$, whereas under ${R(\cdot)=M(\cdot)}$ a refinement step may change the meaning to satisfy the prover (\reffix{prop:beta}). Correctness cannot fall at refinement step $t$ with $\Prob[\mathsf{Fix}_t]\ge r_t\,\varepsilon_t$ (\reffix{cor:fixbreak}). \Cref{app:proofs} proves these results and analyzes each route.

\subsection{Where the Refinement Budget Goes}\label{ssec:theory-plateau}

Since the budget may end before an output passes all proxies, we factor the final pass rate at each step by whether refinement continues, whether it produces a new candidate, and whether that candidate passes. Let $\pi_0=\Prob[\mathrm{Pass}_J(x_0)=1]$ be the fraction of first drafts that pass all proxies, and $\mathsf N_0$ the set of problems whose first draft fails. Let $\mathsf{New}_t$ be the set of problems on which step $t$ returns a candidate $x_t$ that differs from all earlier candidates. At step $t$, $s_t$ is the fraction of failing first drafts still under refinement, $\nu_t$ the fraction of these steps that return a new candidate, and $\theta_t$ the fraction of new candidates that pass all proxies,
\begin{equation}\label{eq:layers}
\begin{aligned}
s_t&=\Prob[\mathsf{Act}_t\mid\mathsf N_0],\\
\nu_t&=\Prob[\mathsf{New}_t\mid\mathsf{Act}_t],\\
\theta_t&=\Prob[\mathrm{Pass}_J(x_t)=1\mid\mathsf{Act}_t,\ \mathsf{New}_t].
\end{aligned}
\end{equation}

\begin{thmbox}{ThmPlat}{ThmPlatBg}
\begin{lemma}\label{lem:plateau}
\stmtlemplateau
\end{lemma}
\end{thmbox}

The factors $s_t$, $\nu_t$, and $\theta_t$ distinguish three limits on further progress: fewer problems remain under refinement, fewer refinement steps produce previously unseen candidates, or fewer new candidates pass all proxies. Their evolution across refinement steps indicates which of these limits constrains the gains from additional refinement steps. \Cref{app:proofs} proves the lemma.

\section{Instantiating Criteria for Refinement}\label{sec:prelim}

\begin{table}[t]
\centering
\small
\setlength{\tabcolsep}{2pt}
\begin{tabular}{@{}llll@{}}
\toprule
Proxy & Criterion & Property checked & Unit \\
\midrule
\px{glb} & Precondition & No unclassified error & File \\
\px{doc} & Precondition & Valid document structure & File \\
\px{syn} & Precondition & Parsable declarations & Module \\
\px{typ} & Precondition & Consistent types & Module \\
\midrule
\px{bnd} & Inference & Reasoning chain & Module \\
\px{gnd} & Grounding & Grounding & Module \\
\px{sp} & Statement & Statement-level faithfulness & Statement \\
\px{cov} & Coverage & Coverage & Claim \\
\bottomrule
\end{tabular}
\caption{Each of the eight proxies checks one property of a criterion and reports its failures per unit.}
\label{tab:proxies}
\end{table}

We instantiate the criteria of \cref{ssec:prelim-gate} as the eight proxies of \cref{tab:proxies} and close the loop with the acceptance rule of \cref{def:accept}. The prover-side proxies $P=\{\px{glb},\px{doc},\px{syn},\px{typ}\}$ check well-formedness, and an LLM judge decides the judge-side proxies $Q=\{\px{bnd},\px{gnd},\px{sp},\px{cov}\}$. \px{bnd} also flags prover errors from unresolved names or unfinished proofs and \px{glb} all other prover errors, so an output that passes all proxies is accepted by the prover. A module is a Lean declaration or an Isabelle section. The feedback names each failing unit, its proxy, and the unit to revise, and the loop runs at most five refinement steps and stops early after three steps without progress (\cref{app:acceptance}).

\section{Experiments}\label{sec:exp}
The experiments focus on the theory of \cref{sec:theory}. We test each of its results on four datasets and two provers, and read the refinement logs to see how proxies behave. We give the selection criteria for the models and datasets in \cref{app:selection}.

\begin{table*}[t]
\centering
\small
\setlength{\tabcolsep}{4pt}
\begin{tabular}{@{}lcccccccc@{}}
\toprule
& \multicolumn{2}{c}{ProofNet} & \multicolumn{2}{c}{miniF2F} & \multicolumn{2}{c}{e-SNLI} & \multicolumn{2}{c}{ProntoQA} \\
\cmidrule(lr){2-3}\cmidrule(lr){4-5}\cmidrule(lr){6-7}\cmidrule(l){8-9}
Method & Acc.$\uparrow$ & Fid.$\uparrow$ & Acc.$\uparrow$ & Fid.$\uparrow$ & Acc.$\uparrow$ & Fid.$\uparrow$ & Acc.$\uparrow$ & Fid.$\uparrow$ \\
\midrule
Single-shot & \msd{51.1}{1.1}\sig & \msd{43.6}{1.4}\sig & \msd{82.7}{1.3}\sig & \msd{65.9}{1.4}\sig & \msd{85.3}{0.6}\sig & \msd{57.3}{2.1} & \msd{\textbf{100.0}}{0.0} & \msd{99.8}{0.3} \\
pass@$k$ & \msd{61.0}{1.5}\sig & \msd{53.7}{1.1}\sig & \msd{88.5}{0.7}\sig & \msd{72.5}{1.6}\sig &\msd{84.0}{1.0}\sig & \msd{59.7}{2.5} & \msd{\textbf{100.0}}{0.0} & \msd{99.8}{0.3} \\
Self-Refine & \msd{56.8}{1.1}\sig & \msd{50.9}{1.1}\sig & \msd{82.1}{1.3}\sig & \msd{77.1}{1.7} & \msd{77.7}{1.2}\sig & \msd{58.7}{0.6} & \msd{\textbf{100.0}}{0.0} & \msd{99.8}{0.3} \\
Reflexion & \msd{86.3}{1.1}\sig & \msd{72.0}{1.9}\sig & \msd{95.9}{0.4} & \msd{75.0}{1.0}\sig & \msd{\textbf{99.7}}{0.6} & \msd{\textbf{66.3}}{2.1} & \msd{\textbf{100.0}}{0.0} & \msd{\textbf{100.0}}{0.0} \\
Ours & \msd{\textbf{91.4}}{1.9} & \msd{\textbf{75.6}}{3.0} & \msd{\textbf{97.8}}{0.2} & \msd{\textbf{80.7}}{1.6} & \msd{\textbf{99.7}}{0.6} & \msd{63.3}{2.3} & \msd{\textbf{100.0}}{0.0} & \msd{99.7}{0.3} \\
\bottomrule
\end{tabular}
\caption{We compare the methods by the prover acceptance (Acc.) and semantic fidelity (Fid.) of the returned output, in percent, as mean $\pm$ sample standard deviation. Bold marks the best value in each column, and an asterisk marks a significant difference from our loop (pooled paired sign test, $p<0.05$).}
\label{tab:main}
\end{table*}
\subsection{Experimental Setup}\label{ssec:exp-setup}

We evaluate on the autoformalization benchmarks ProofNet \citep{azerbayev2023proofnet} and miniF2F \citep{zheng2021minif2f} in Lean~4 with Mathlib \citep{moura2021lean4}, and on the natural-language reasoning benchmarks e-SNLI \citep{camburu2018esnli} and ProntoQA \citep{saparov2023prontoqa} in Isabelle/HOL \citep{nipkow2002isabelle}. \Cref{app:data} gives the data and the protocol.

We use DeepSeek-V4-Flash to write the first draft and every refinement step, and we use GPT-5.4-mini to run the four judge-side proxies. \Cref{app:details} gives the remaining settings.

We report two metrics on the returned output $\hat x$. (i) \textbf{Prover acceptance} is the fraction of problems whose output the prover accepts. (ii) \textbf{Semantic fidelity} is the fraction of problems with $R(\hat x)=1$, where the prover accepts the output and the semantic-fidelity judge rates it faithful to the source (\cref{app:sfj}). Requiring prover acceptance makes this metric more reliable than the judge alone, since a meaning verdict is informative only for outputs that the prover accepts. DeepSeek-V4-Flash runs this judge outside the loop.

\begin{table}[t]
\centering
\small
\setlength{\tabcolsep}{1.6pt}
\begin{tabular}{@{}llcccc@{}}
\toprule
Judge & Metric & ProofNet & miniF2F & e-SNLI & ProntoQA \\
\midrule
\px{sp} & Agreement\,$\uparrow$ & 90.5 & 94.0 & 83.0 & 95.6 \\
 & $f\,\downarrow$ & 0.6 & 5.3 & 6.8 & 4.6 \\
 & $m\,\downarrow$ & 73.3 & 9.8 & 24.7 & 0.0 \\
\midrule
\px{cov} & Agreement\,$\uparrow$ & 99.2 & 98.0 & 51.8 & 100.0 \\
 & $f\,\downarrow$ & 0.3 & 0.4 & 47.5 & 0.0 \\
 & $m\,\downarrow$ & 100.0 & 10.6 & 100.0 & -- \\
\midrule
\px{gnd} & Agreement\,$\uparrow$ & 100.0 & 86.2 & -- & 100.0 \\
 & $f\,\downarrow$ & 0.0 & 0.4 & -- & 0.0 \\
 & $m\,\downarrow$ & -- & 100.0 & -- & -- \\
\midrule
Fidelity & Agreement\,$\uparrow$ & 88.7 & 97.6 & 70.2 & 100.0 \\
judge & $f\,\downarrow$ & 2.9 & 2.5 & 16.4 & 0.0 \\
 & $m\,\downarrow$ & 72.1 & 2.1 & 39.2 & 0.0 \\
\bottomrule
\end{tabular}
\caption{We compare each judge with manual annotation by agreement, false-positive rate $f$, and false-negative rate $m$, in percent, and a dash marks a rate that the annotation cannot estimate (\cref{app:annotation}).}
\label{tab:judge-human}
\end{table}

\subsection{Main Results}\label{ssec:exp-main}

\Cref{tab:main} compares our refinement loop with three generation strategies to check that it remains competitive: producing a single draft, generating and reranking $k$ drafts (pass@$k$), and repeatedly revising a draft using Self-Refine \citep{madaan2023selfrefine} or Reflexion \citep{shinn2023reflexion}. \Cref{app:methods} specifies the settings for each method.

On ProntoQA, all methods are near the maximum and none differs significantly from our loop. On the other three datasets, our loop is significantly better in 17 of 24 comparisons and reaches the highest prover acceptance and semantic fidelity on both Lean datasets, while on e-SNLI its semantic fidelity does not differ significantly from any baseline. Reflexion is the strongest baseline and the only one whose prover acceptance does not differ significantly from our loop on miniF2F and e-SNLI. Despite ten drafts per problem, pass@$k$ reaches lower prover acceptance than our loop on all three datasets below the ceiling. Self-Refine, without prover feedback, is significantly below Reflexion in prover acceptance there (\cref{tab:extra-p}).

\subsection{Gains across Refinement Rounds}\label{ssec:exp-rounds}

To test the predicted dynamics of refinement, we evaluate the kept version after each step (\cref{fig:rounds}). Prover acceptance never decreases, consistent with \refacc{prop:accept}, and on ProofNet barely increases after the fourth step. Semantic fidelity improves alongside prover acceptance on the Lean datasets.

\begin{figure*}[t]
\centering
\begin{subfigure}[t]{0.245\textwidth}\centering\includegraphics{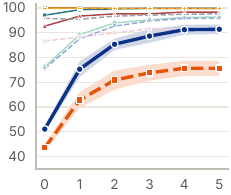}\caption{ProofNet}\label{fig:rounds-proofnet}\end{subfigure}\hfill
\begin{subfigure}[t]{0.245\textwidth}\centering\includegraphics{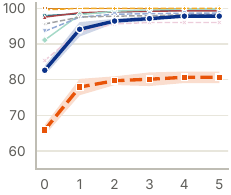}\caption{miniF2F}\label{fig:rounds-minif2f}\end{subfigure}\hfill
\begin{subfigure}[t]{0.245\textwidth}\centering\includegraphics{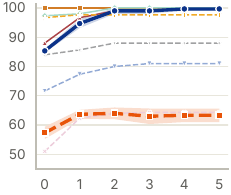}\caption{e-SNLI}\label{fig:rounds-esnli}\end{subfigure}\hfill
\begin{subfigure}[t]{0.245\textwidth}\centering\includegraphics{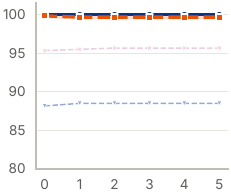}\caption{ProntoQA}\label{fig:rounds-prontoqa}\end{subfigure}\\[4pt]
\includegraphics{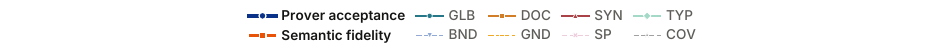}
\caption{Each panel plots prover acceptance (\protect\accline), semantic fidelity (\protect\fidline), and the pass rate of each proxy (thin lines, solid for prover-side and dashed for judge-side proxies) after each refinement round, in percent. Round 0 is the first draft, and shaded bands give one sample standard deviation around the two metrics.}
\label{fig:rounds}
\end{figure*}

\paragraph{Prover progress does not guarantee correctness.} The kept version can become incorrect although its prover state never worsens (\refacc{prop:accept}). On e-SNLI, the third step turns some correct versions into incorrect ones and corrects none, reducing semantic fidelity from 64.0\% to 63.0\%. This illustrates the mechanism bounded by \reffix{thm:fixbreak}(i), in which correct versions still under refinement are exposed to harmful revisions, and over the first three steps their fraction $r_t$ reaches 32.3\% on e-SNLI but at most 5.7\% on ProofNet.

\paragraph{Bounded budgets expose distinct limits to refinement.} The decomposition in \refplat{lem:plateau} explains why further steps yield little progress (\cref{tab:plateau}). On ProofNet, 51.5\% of steps produce a new candidate that still fails, so $\theta_t$ is the main limit. On e-SNLI, $\nu_t$ falls from 99.1\% to 39.2\% at the second step, as the loop increasingly repeats earlier candidates. Five refinement steps per problem without early stopping yield nearly the same final pass rates as our loop on both datasets, so further attempts contribute little within the tested budget.

\subsection{Fixes and Breaks under Manual Annotation}\label{ssec:exp-align}

We test whether proxy errors have the consequences predicted by \reffix{thm:fixbreak}, with human annotations of overall correctness as an independent reference (\cref{tab:align}) and the semantic-fidelity judge on all problems (\cref{app:align-judge}). On ProofNet, e-SNLI, and ProntoQA, 16.0\%, 24.0\%, and 24.0\% of annotated problems begin with a correct draft that a proxy falsely rejects, and these false positives keep correct drafts in the loop, where by \reffix{thm:fixbreak} a break can start. Refinement increases annotated overall correctness on ProofNet, miniF2F, and e-SNLI, breaks no correct output on the Lean datasets, and on e-SNLI fixes 16.0\% of problems and breaks 8.0\%. Observed break rates remain below the end-to-end bound in \reffix{thm:fixbreak}(ii), except on miniF2F, where both are zero, which supports the predicted bound on breaks.

\begin{table}[t]
\centering
\small
\setlength{\tabcolsep}{3pt}
\begin{tabular}{@{}lcccc@{}}
\toprule
Dataset & Correct & Fixed & Broken & \shortstack{Correct but\\rejected} \\
\midrule
ProofNet & 52.0 $\to$ 60.0 & 8.0 & 0.0 & 16.0 \\
miniF2F & 60.0 $\to$ 76.0 & 16.0 & 0.0 & 0.0 \\
e-SNLI & 44.0 $\to$ 52.0 & 16.0 & 8.0 & 24.0 \\
ProntoQA & 100.0 $\to$ 100.0 & 0.0 & 0.0 & 24.0 \\
\bottomrule
\end{tabular}
\caption{Under manual annotation, we count the correct problems before and after refinement, the fixed and broken problems, and the correct first drafts that fail some proxy, in percent of the annotated problems.}
\label{tab:align}
\end{table}

\paragraph{Prover progress can override warnings about meaning changes.} We test the failure routes in \reffix{thm:fixbreak}(i) by tracing each loss of semantic fidelity to its refinement step under ${R(\cdot)=M(\cdot)}$ (\cref{app:break-trace}). On the Lean datasets, rule (ii) retained 3 of the 5 traced breaks because they improved the prover result, 2 of them although judges flagged new failures, so a judge's warning does not prevent a harmful revision when prover progress takes priority. The remaining 2 Lean breaks and all 8 e-SNLI breaks passed judges that reported no effective new failure, as on e-SNLI problem 30, where a step targeting \px{bnd} added the premise \texttt{Human x} and \px{sp} did not flag it. Every traced break therefore followed the prover-progress route $\beta$ or the silent-guard route $q_z$, and none involved a change type without a guard.

\subsection{Removing Individual Proxies}\label{ssec:exp-lesion}

We test whether a proxy's measured error rates predict the effect of removing it, which in \refadd{thm:add} compares the remaining proxies $J$ with $J\cup\{k\}$, using predictions recorded before the runs (\cref{app:predictions}) and lesions on ProofNet and e-SNLI (\cref{tab:lesion}). Removing a proxy also changes the feedback and hence the candidates, beyond the fixed-candidate setting. No removal reduces the fraction of outputs passing the remaining checks, removing \px{syn} or \px{typ} raises $w$ or changes it by at most half a percentage point, as their zero false-positive rates predict, and under \eqref{eq:cw} no prediction is contradicted. On ProofNet, passing fewer checks can conceal a loss of formal validity, since removing \px{typ} leaves 87.4\% passing the remaining proxies but only 66.3\% passing all eight, mostly because the loop stops on outputs with type errors.

\begin{table*}[t]
\centering
\scriptsize
\setlength{\tabcolsep}{3.6pt}
\begin{tabular}{@{}lccccccccc@{}}
\toprule
& \multicolumn{2}{c}{Pass rate\,$\uparrow$} & \multicolumn{2}{c}{Pass-but-wrong rate $w\,\downarrow$} & \multicolumn{2}{c}{Prover acceptance\,$\uparrow$} & \multicolumn{2}{c}{Semantic fidelity\,$\uparrow$} & Predicted $w$ \\
\cmidrule(lr){2-3}\cmidrule(lr){4-5}\cmidrule(lr){6-7}\cmidrule(lr){8-9}\cmidrule(l){10-10}
Removed & ProofNet & e-SNLI & ProofNet & e-SNLI & ProofNet & e-SNLI & ProofNet & e-SNLI & ProofNet / e-SNLI \\
\midrule
None (ours) & \msd{81.0}{0.8} & \msd{50.3}{3.2} & \msd{9.0}{1.6} & \msd{10.7}{2.5} & \msd{\textbf{91.4}}{1.9} & \msd{99.7}{0.6} & \msd{75.6}{3.0} & \msd{63.3}{2.3} &  \\
\px{sp} & \msd{88.1}{0.3} & \msd{\textbf{84.7}}{0.6} & \msd{15.2}{0.6} & \msd{28.0}{2.6} & \msd{90.1}{0.5} & \msd{99.0}{1.0} & \msd{73.6}{1.0} & \msd{62.0}{1.7} & $>$ / $>$ \\
\px{cov} & \msd{81.9}{1.5} & \msd{56.0}{6.6} & \msd{9.7}{1.4} & \msd{11.0}{1.0} & \msd{91.0}{2.1} & \msd{\textbf{100.0}}{0.0} & \msd{\textbf{76.0}}{3.2} & \msd{64.0}{1.7} & $\geq$ / -- \\
\px{gnd} & \msd{81.1}{1.1} & \msd{50.7}{2.9} & \msd{\textbf{8.6}}{1.3} & \msd{11.0}{2.6} & \msd{\textbf{91.4}}{1.9} & \msd{99.7}{0.6} & \msd{\textbf{76.0}}{2.7} & \msd{62.7}{2.9} & $=$ / $=$ \\
\px{bnd} & \msd{\textbf{88.8}}{0.8} & \msd{64.3}{3.1} & \msd{30.4}{2.6} & \msd{13.0}{2.0} & \msd{71.2}{1.7}\sig & \msd{99.7}{0.6} & \msd{60.4}{2.9}\sig & \msd{\textbf{66.0}}{1.0} & -- / -- \\
\px{syn} & \msd{83.7}{2.2} & \msd{52.0}{5.6} & \msd{15.0}{1.3} & \msd{15.3}{3.2} & \msd{85.7}{2.5}\sig & \msd{88.3}{1.2}\sig & \msd{72.0}{2.9}\sig & \msd{56.0}{1.0}\sig & $\geq$ / $\geq$ \\
\px{typ} & \msd{87.4}{1.0} & \msd{50.3}{4.2} & \msd{28.4}{4.6} & \msd{\textbf{10.3}}{2.5} & \msd{72.5}{2.0}\sig & \msd{96.0}{0.0}\sig & \msd{61.2}{4.2}\sig & \msd{62.3}{2.5} & $\geq$ / $\geq$ \\
\bottomrule
\end{tabular}
\caption{Each row removes one proxy from our loop, and values are in percent, as mean $\pm$ sample standard deviation. The pass rate and $w$ are over the remaining proxies. Bold marks the best value in each column with an arrow. The last column gives the predicted direction of $w$, where $>$ means a rise, $\geq$ no fall, $=$ no change, and -- no prediction. An asterisk marks a significant difference from our loop (pooled paired sign test, $p<0.05$).}
\label{tab:lesion}
\end{table*}

\paragraph{Statement-level checks filter out semantically incorrect outputs.} Among the judges that compare the output with the source, removing \px{sp} raises $w$ most, from 9.0\% to 15.2\% on ProofNet and from 10.7\% to 28.0\% on e-SNLI, so statement-level checks keep a substantial fraction of incorrect outputs from being returned as passing.

\paragraph{The value of \px{bnd} depends on the reliability of its feedback.} The effect of removing \px{bnd} differs across datasets, consistent with the role of detections and false positives in \refadd{thm:add}. On ProofNet, 95.0\% of the \px{bnd} failures resolved by retained steps are prover-reported errors, which never reject correct outputs, and removing \px{bnd} raises $w$ more than any other removal and lowers prover acceptance. On e-SNLI, the LLM judge decides every \px{bnd} failure and rejects many correct first drafts, so removing \px{bnd} raises the pass rate with only a small increase in $w$. The same proxy thus detects errors on one dataset and keeps correct drafts in refinement on another.

\begin{figure}[t]
\centering
\begin{subfigure}[t]{0.49\linewidth}\centering\includegraphics[width=\linewidth]{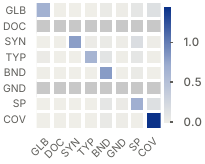}\caption{Units fixed}\label{fig:heatmap-fixed}\end{subfigure}\hfill
\begin{subfigure}[t]{0.49\linewidth}\centering\includegraphics[width=\linewidth]{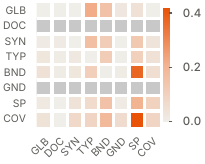}\caption{Units broken}\label{fig:heatmap-broken}\end{subfigure}
\caption{We average the failing units that a proposal fixes (a) and breaks (b) over all proposals of the four datasets, including rejected ones. Rows give the targeted proxy and columns the affected proxy. Gray rows have too few proposals for an average (none for \px{gnd}).}
\label{fig:heatmap}
\end{figure}
\subsection{Individual Proxies}\label{ssec:exp-proxies}

Since the bounds of \refadd{thm:add} depend on target-data error rates, we compare judges with unit-level human annotations (\cref{tab:judge-human}, \cref{app:annotation,app:proxy-readings,app:verdicts}). The false-negative rate of \px{sp} is 73.3\% on ProofNet but 9.8\% on miniF2F, and the false-positive rate of \px{cov} is 47.5\% on e-SNLI. \px{cov} misses 92\% of the unfaithful outputs that \px{sp} misses, so it barely lowers $q_z$ if change types behave alike (\cref{app:qz-chain}). With one proxy alone, prover acceptance on ProofNet is at most 63.6\% against 91.4\% with all proxies (\cref{app:feedback}).

\paragraph{Targeted refinement steps can introduce errors in meaning.} Refinement usually edits the statement or module that the feedback names, in 98.4\% of \px{sp} steps and 94.4\% of \px{typ} steps on ProofNet and in 98.9\% of \px{sp} steps on e-SNLI. Resolved failures concentrate on the diagonal of \cref{fig:heatmap} and newly introduced failures in the \px{sp} column, and across the four datasets all newly failing \px{sp} units in retained revisions arise from steps targeting \px{typ}, \px{bnd}, or \px{syn} (51.0\%, 26.5\%, and 22.4\%), so steps that target types, inference, or syntax can alter statement-level meaning.

\section{Related Work}\label{sec:related}

\paragraph{Autoformalization and neural theorem proving.}
LLM-guided AF \citep{wu2022autoformalization,jiang2023draft} and neural proof search \citep{polu2020generative,yang2023leandojo} use proof assistants as a binary signal. LLM judges compare the formalization with its source to filter candidates \citep{xuejun2025mathesis}, consistency checks rerank samples \citep{li2024autoformalize}, and repair loops pair prover errors with an overall consistency verdict \citep{jana2026proofbridge,guo2025atf}, whereas our judges report failing units per criterion. \citet{zhang2026monotonic} formalize reference-free refinement for full-theorem AF with four scalar criteria that improve monotonically under an acceptance policy. We give a similar guarantee for the prover state and show that correctness can still fall under imperfect proxies. Prover-driven repair improves explanation-style AF on Isabelle natural-language inference \citep{quan2024verification,quan2025faithful}, and Pseudo-Formalization and its Block Verification algorithm check natural-language proofs module by module with an LLM \citep{barkallah2026pseudo}, whereas in our loop the prover decides well-formedness and judges cover the four criteria.

\paragraph{Inference-time refinement.}
Self-Refine and Reflexion use verbal feedback without weight updates \citep{madaan2023selfrefine,shinn2023reflexion}, whereas our loop feeds back the failing units of each proxy and keeps a candidate only under the acceptance rule of \cref{def:accept}. Self-correction can break correct answers \citep{huang2024selfcorrect}, and optimizing a proxy can drift from the intended goal \citep{manheim2018categorizing,gao2023scaling}, while we bound how adding a proxy changes the pass-but-wrong rate. \citet{yang2025probabilistic} model per-round accuracy from the probabilities of keeping correct and fixing incorrect answers, and we add the routes by which breaks pass the proxies.

\section{Conclusion}\label{sec:conclusion}
We formulate AF evaluation and refinement around explicit checks of four semantic criteria. The theory bounds how adding imperfect proxies changes the pass-but-wrong rate and how errors missed by judges or retained despite their warnings survive refinement, and it shows that a correct formalization can become incorrect only if a false rejection exposes it to further revision. A decomposition of the final pass rate identifies whether refinement continues, produces new candidates, and yields outputs that pass all checks. Experiments on four datasets and two provers support these predictions and show that proxy reliability must be measured on the target data. Collectively, the results explain how proxy reliability and the rules for retaining revisions shape the benefits and risks of refinement.

\section*{Limitations}
The limitations of this work concern what passing every proxy guarantees, where refinement gains, how far the theory and the acceptance rule reach, and how the verdicts could serve beyond inference time. First, passing every proxy is audit-sufficient rather than fully semantic evidence of correctness. A passing output satisfies only the properties that the proxies check, and a change of meaning passes unnoticed when it has no guard or all its guards stay silent. The statement-level faithfulness judge can miss a premise that is plausible by common sense but absent from the source, such as \texttt{Human x} in \cref{ssec:exp-align}, and a check that traces each premise to the source would guard this change type.

Second, the gain from refinement depends on where the bottleneck lies. Refinement gains most where the prover rejects many first drafts and less where the remaining errors concern meaning. On e-SNLI, refinement soon stops producing new candidates, so more diverse refinement proposals could raise the plateau. Whether the proxy decomposition generalizes as a dataset-agnostic refinement signal requires evaluation on benchmarks with broader topic coverage and first drafts below saturation.

Third, the theory and the acceptance rule leave room for improvement. The bounds of \refadd{thm:add} hold for a fixed sequence of candidates and leave out how a change of the proxy set alters the refinement feedback and hence the candidates, and extending them to feedback-dependent candidates is open. Breaks start only from false positives on correct versions (\reffix{thm:fixbreak}), so lowering judge false positives on correct drafts, for example on implicit reasoning steps, lowers the bound on breaks. The prover-first acceptance rule keeps refinement steps that improve the prover state even when judges flag new meaning errors, and a rule that also weighs these flags would close this route.

Finally, the framework is purely an inference-time refinement procedure. The per-criterion verdict vector is in principle a richer training signal than a scalar reward, and integrating it into reinforcement-learning fine-tuning of the proposer or the judge would connect calibration with policy improvement. Beyond AF, we expect the same unit-level decomposition to apply wherever correctness can be approximated by a calibrated family of property checks, including program synthesis and multi-hop natural-language reasoning.

\section*{Acknowledgments}
This work was partially funded by the Swiss National Science Foundation (SNSF) projects RATIONAL and M-RATIONAL. We used AI assistants (Claude Code with Anthropic Claude models) for coding and writing support: they implemented parts of the experiment code following the authors' specifications and helped draft and edit text under the authors' direction. The authors reviewed all code and text and take full responsibility for the content of the paper.

\bibliography{references}

@inproceedings{wu2022autoformalization,
  title     = {Autoformalization with Large Language Models},
  author    = {Wu, Yuhuai and Jiang, Albert Q. and Li, Wenda and Rabe, Markus N. and Staats, Charles and Jamnik, Mateja and Szegedy, Christian},
  booktitle = {Advances in Neural Information Processing Systems},
  year      = {2022},
  url       = {https://proceedings.neurips.cc/paper_files/paper/2022/file/d0c6bc641a56bebee9d985b937307367-Paper-Conference.pdf}
}

@article{polu2020generative,
  title   = {Generative Language Modeling for Automated Theorem Proving},
  author  = {Polu, Stanislas and Sutskever, Ilya},
  journal = {arXiv preprint arXiv:2009.03393},
  year    = {2020},
  url     = {https://arxiv.org/abs/2009.03393}
}

@inproceedings{zheng2021minif2f,
  title     = {{MiniF2F}: A Cross-System Benchmark for Formal Olympiad-Level Mathematics},
  author    = {Zheng, Kunhao and Han, Jesse Michael and Polu, Stanislas},
  booktitle = {International Conference on Learning Representations (ICLR)},
  year      = {2022},
  url       = {https://arxiv.org/abs/2109.00110}
}

@inproceedings{yang2023leandojo,
  title     = {{LeanDojo}: Theorem Proving with Retrieval-Augmented Language Models},
  author    = {Yang, Kaiyu and Swope, Aidan M. and Gu, Alex and Chalamala, Rahul and Song, Peiyang and Yu, Shixing and Godil, Saad and Prenger, Ryan and Anandkumar, Anima},
  booktitle = {Advances in Neural Information Processing Systems (Datasets and Benchmarks Track)},
  year      = {2023},
  url       = {https://arxiv.org/abs/2306.15626}
}

@inproceedings{madaan2023selfrefine,
  title     = {Self-Refine: Iterative Refinement with Self-Feedback},
  author    = {Madaan, Aman and Tandon, Niket and Gupta, Prakhar and Hallinan, Skyler and Gao, Luyu and Wiegreffe, Sarah and Alon, Uri and Dziri, Nouha and Prabhumoye, Shrimai and Yang, Yiming and Gupta, Shashank and Majumder, Bodhisattwa Prasad and Hermann, Katherine and Welleck, Sean and Yazdanbakhsh, Amir and Clark, Peter},
  booktitle = {Advances in Neural Information Processing Systems},
  year      = {2023},
  url       = {https://arxiv.org/abs/2303.17651}
}

@inproceedings{shinn2023reflexion,
  title     = {Reflexion: Language Agents with Verbal Reinforcement Learning},
  author    = {Shinn, Noah and Cassano, Federico and Gopinath, Ashwin and Narasimhan, Karthik and Yao, Shunyu},
  booktitle = {Advances in Neural Information Processing Systems},
  year      = {2023},
  url       = {https://arxiv.org/abs/2303.11366}
}

@inproceedings{jiang2023draft,
  title     = {Draft, Sketch, and Prove: Guiding Formal Theorem Provers with Informal Proofs},
  author    = {Jiang, Albert Q. and Welleck, Sean and Zhou, Jin Peng and Li, Wenda and Liu, Jiacheng and Jamnik, Mateja and Lacroix, Timoth\'ee and Wu, Yuhuai and Lample, Guillaume},
  booktitle = {International Conference on Learning Representations (ICLR)},
  year      = {2023},
  url       = {https://arxiv.org/abs/2210.12283}
}

@article{azerbayev2023proofnet,
  title   = {ProofNet: Autoformalizing and Formally Proving Undergraduate-Level Mathematics},
  author  = {Azerbayev, Zhangir and Piotrowski, Bartosz and Schoelkopf, Hailey and Ayers, Edward W. and Radev, Dragomir and Avigad, Jeremy},
  journal = {arXiv preprint arXiv:2302.12433},
  year    = {2023},
  url     = {https://arxiv.org/abs/2302.12433}
}

@article{zhang2025beyond,
  title   = {Beyond Gold Standards: Epistemic Ensemble of {LLM} Judges for Formal Mathematical Reasoning},
  author  = {Zhang, Lan and Valentino, Marco and Meadows, Jordan and Freitas, Andr\'e},
  journal = {arXiv preprint arXiv:2506.10903},
  year    = {2025},
  url     = {https://arxiv.org/abs/2506.10903}
}

@inproceedings{camburu2018esnli,
  title     = {e-{SNLI}: Natural Language Inference with Natural Language Explanations},
  author    = {Camburu, Oana-Maria and Rockt{\"{a}}schel, Tim and Lukasiewicz, Thomas and Blunsom, Phil},
  booktitle = {Advances in Neural Information Processing Systems},
  volume    = {31},
  year      = {2018}
}

@inproceedings{saparov2023prontoqa,
  title     = {Language Models Are Greedy Reasoners: A Systematic Formal Analysis of Chain-of-Thought},
  author    = {Saparov, Abulhair and He, He},
  booktitle = {International Conference on Learning Representations (ICLR)},
  year      = {2023},
  url       = {https://arxiv.org/abs/2210.01240}
}

@inproceedings{valentino2021nli,
  title     = {Do Natural Language Explanations Represent Valid Logical Arguments? Verifying Entailment in Explainable {NLI} Gold Standards},
  author    = {Valentino, Marco and Pratt-Hartmann, Ian and Freitas, Andr{\'e}},
  booktitle = {Proceedings of the 14th International Conference on Computational Semantics (IWCS)},
  pages     = {76--86},
  year      = {2021},
  url       = {https://aclanthology.org/2021.iwcs-1.8/},
}

@article{zhang2026monotonic,
  title     = {Monotonic Reference-Free Refinement for Autoformalization},
  author    = {Zhang, Lan and Valentino, Marco and Freitas, Andr{\'e}},
  journal   = {arXiv preprint arXiv:2601.23166},
  year      = {2026},
  url       = {https://arxiv.org/abs/2601.23166}
}

@inproceedings{quan2025faithful,
  title     = {Faithful and Robust {LLM}-Driven Theorem Proving for {NLI} Explanations},
  author    = {Quan, Xin and Valentino, Marco and Dennis, Louise A. and Freitas, Andr{\'e}},
  booktitle = {Proceedings of the 63rd Annual Meeting of the Association for Computational Linguistics (Volume 1: Long Papers)},
  pages     = {17734--17755},
  year      = {2025},
  publisher = {Association for Computational Linguistics},
  doi       = {10.18653/v1/2025.acl-long.867},
  url       = {https://aclanthology.org/2025.acl-long.867/}
}

@inproceedings{moura2021lean4,
  title     = {The {Lean 4} Theorem Prover and Programming Language},
  author    = {de Moura, Leonardo and Ullrich, Sebastian},
  editor    = {Platzer, Andr{\'e} and Sutcliffe, Geoff},
  booktitle = {Automated Deduction -- CADE 28},
  pages     = {625--635},
  year      = {2021},
  publisher = {Springer International Publishing},
  address   = {Cham}
}

@book{nipkow2002isabelle,
  title     = {{Isabelle/HOL}: A Proof Assistant for Higher-Order Logic},
  author    = {Nipkow, Tobias and Paulson, Lawrence C. and Wenzel, Markus},
  series    = {Lecture Notes in Computer Science},
  volume    = {2283},
  publisher = {Springer},
  year      = {2002}
}

@article{barkallah2026pseudo,
  title   = {Pseudo-Formalization for Automatic Proof Verification},
  author  = {Barkallah, Slim and Bailey, Luke and Wen, Kaiyue and Abouzaid, Mohammed and Ma, Tengyu},
  journal = {arXiv preprint arXiv:2605.20531},
  year    = {2026}
}

@inproceedings{xuejun2025mathesis,
  title   = {Mathesis: Towards Formal Theorem Proving from Natural Languages},
  author  = {Yu, Xuejun and Zhong, Jianyuan and Feng, Zijin and Zhai, Pengyi and Yousefzadeh, Roozbeh and Ng, Wei Chong and Liu, Haoxiong and Shou, Ziyi and Xiong, Jing and Zhou, Yudong and Ong, Claudia Beth and Sugiarto, Austen Jeremy and Zhang, Yaoxi and Tai, Wai Ming and Cao, Huan and Lu, Dongcai and Sun, Jiacheng and Xu, Qiang and Xin, Shen and Li, Zhenguo},
  booktitle = {International Conference on Learning Representations (ICLR)},
  year    = {2026},
  url     = {https://openreview.net/forum?id=CJdX82odge}
}

@inproceedings{gao2023scaling,
  title     = {Scaling Laws for Reward Model Overoptimization},
  author    = {Gao, Leo and Schulman, John and Hilton, Jacob},
  booktitle = {Proceedings of the 40th International Conference on Machine Learning},
  series    = {Proceedings of Machine Learning Research},
  volume    = {202},
  year      = {2023}
}

@article{manheim2018categorizing,
  title   = {Categorizing Variants of {Goodhart's} Law},
  author  = {Manheim, David and Garrabrant, Scott},
  journal = {arXiv preprint arXiv:1803.04585},
  year    = {2018}
}

@inproceedings{quan2024verification,
  title     = {Verification and Refinement of Natural Language Explanations through {LLM}-Symbolic Theorem Proving},
  author    = {Quan, Xin and Valentino, Marco and Dennis, Louise A. and Freitas, Andr{\'e}},
  booktitle = {Proceedings of the 2024 Conference on Empirical Methods in Natural Language Processing},
  year      = {2024}
}

@inproceedings{li2024autoformalize,
  title     = {Autoformalize Mathematical Statements by Symbolic Equivalence and Semantic Consistency},
  author    = {Li, Zenan and Wu, Yifan and Li, Zhaoyu and Wei, Xinming and Zhang, Xian and Yang, Fan and Ma, Xiaoxing},
  booktitle = {Advances in Neural Information Processing Systems},
  year      = {2024}
}

@inproceedings{huang2024selfcorrect,
  title     = {Large Language Models Cannot Self-Correct Reasoning Yet},
  author    = {Huang, Jie and Chen, Xinyun and Mishra, Swaroop and Zheng, Huaixiu Steven and Yu, Adams Wei and Song, Xinying and Zhou, Denny},
  booktitle = {The Twelfth International Conference on Learning Representations},
  year      = {2024},
  url       = {https://arxiv.org/abs/2310.01798}
}

@inproceedings{yang2025probabilistic,
  title     = {A Probabilistic Inference Scaling Theory for {LLM} Self-Correction},
  author    = {Yang, Zhe and Zhang, Yichang and Wang, Yudong and Xu, Ziyao and Lin, Junyang and Sui, Zhifang},
  booktitle = {Proceedings of the 2025 Conference on Empirical Methods in Natural Language Processing},
  year      = {2025},
  doi       = {10.18653/v1/2025.emnlp-main.685}
}

@inproceedings{jana2026proofbridge,
  title     = {{ProofBridge}: Auto-Formalization of Natural Language Proofs in {Lean} via Joint Embeddings},
  author    = {Jana, Prithwish and Kale, Kaan and Tanriverdi, Ahmet Ege and Song, Cruise and Vishwanath, Sriram and Ganesh, Vijay},
  booktitle = {International Conference on Learning Representations (ICLR)},
  year      = {2026},
  url       = {https://openreview.net/forum?id=U2jxHXuOX9}
}

@article{guo2025atf,
  title   = {Autoformalizer with Tool Feedback},
  author  = {Guo, Qi and Wang, Jianing and Zhang, Jianfei and Kong, Deyang and Huang, Xiangzhou and Xi, Xiangyu and Wang, Wei and Wang, Jingang and Cai, Xunliang and Zhang, Shikun and Ye, Wei},
  journal = {arXiv preprint arXiv:2510.06857},
  year    = {2025},
  url     = {https://arxiv.org/abs/2510.06857}
}

\appendix
\crefalias{section}{appendix}
\crefalias{subsection}{subappendix}

\section{Proofs}\label{app:proofs}

Probabilities are over a uniformly drawn problem and include the randomness of the refinement model and of the LLM judges. A conditional probability with a null conditioning event is set to zero. A product of a conditional rate and a prior fraction denotes the probability of the joint event, as in $c_J f_A$ and $r_t\alpha_z q_z$. Every event is a set of problems, and $\{\cdot\}$ denotes the set of problems on which the condition inside holds, so that, for example, $\mathsf{Fix}_t=\{R(\tilde x_{t-1})=0,\ R(\tilde x_t)=1\}$ and $\mathsf{Break}_t=\{R(\tilde x_{t-1})=1,\ R(\tilde x_t)=0\}$. The other events are those of \cref{ssec:theory-fixbreak,ssec:theory-plateau}.

\begin{thmbox}{ThmAcc}{ThmAccBg}
\begin{restated}{Proposition}{prop:accept}
\stmtpropaccept
\end{restated}
\end{thmbox}

\begin{proof}[Proof of \refacc{prop:accept}]
We show that the prover state never decreases and then that, with fixed verdicts, it cannot stay constant forever.

\emph{Step 1.} A candidate accepted by rule (i) passes all proxies, so the prover accepts it with $\mathrm{err}(x_t)=0$, which is the maximal state $(1,0)$. Rule (ii) raises the state and rule (iii) keeps it, hence
\begin{equation}\label{eq:prov-mono}
\mathrm{prov}(\tilde x_t)\succeq\mathrm{prov}(\tilde x_{t-1}).
\end{equation}

\emph{Step 2.} At a step accepted by rule (iii), $\Delta_k(\tilde x_{t-1},x_t)=\varnothing$ for every $k\in Q$, that is,
\begin{equation}\label{eq:no-new}
F_k(\tilde x_t)\subseteq F_k(\tilde x_{t-1})\cup X_k(\tilde x_{t-1})\qquad(k\in Q).
\end{equation}
Units are matched across versions by their identifiers (\cref{app:acceptance}). With fixed verdicts, $X_k(\tilde x)\subseteq F_k(\tilde x)$, so \eqref{eq:no-new} gives $F_k(\tilde x_t)\subseteq F_k(\tilde x_{t-1})$, and rule (iii) requires one inclusion to be strict. Hence
\begin{equation}\label{eq:potential}
\sum_{k\in Q}|F_k(\tilde x_t)|<\sum_{k\in Q}|F_k(\tilde x_{t-1})|.
\end{equation}

\emph{Step 3.} By \eqref{eq:prov-mono} the prover state takes finitely many values and never decreases, and by \eqref{eq:potential} the nonnegative integer $\sum_{k\in Q}|F_k(\tilde x_t)|$ strictly decreases while it stays constant. The kept version therefore changes only finitely often.
\end{proof}

\Refadd{thm:add}(ii) bounds the new pass-but-wrong rate on both sides. The upper bound adds to $w_J$ the fraction $c_J f_A$ of problems on which $J$ returns a correct output that $A$ rejects, because only these problems can turn into new wrong returns. The lower bound keeps the fraction $m_A$ of the old wrong returns that $A$ passes, because $A$ can remove at most the others. The bounds therefore grow with the false-positive rate and the false-negative rate of $A$, respectively, and both reduce to $w_J$ in the two limiting cases of \refadd{cor:add}.

\begin{thmbox}{ThmAdd}{ThmAddBg}
\begin{restated}{Theorem}{thm:add}
\stmtthmadd
\end{restated}
\end{thmbox}

\begin{proof}[Proof of \refadd{thm:add}]
We prove (i) by comparing the stopping steps of the two loops and (ii) by splitting $w_{J\cup A}-w_J$ into the problems that enter and leave the pass-but-wrong event.

(i) Both loops read the same verdicts, so for every $t$
\begin{equation}\label{eq:pass-sub}
\begin{split}
\mathrm{Pass}_{J\cup A}(x_t)&=\mathrm{Pass}_J(x_t)\,\mathrm{Pass}_A(x_t)\\
&\le\mathrm{Pass}_J(x_t),
\end{split}
\end{equation}
so the first $t$ with $\mathrm{Pass}_{J\cup A}(x_t)=1$ is no earlier than the first $t$ with $\mathrm{Pass}_J(x_t)=1$, that is,
\begin{equation}\label{eq:tau}
\tau_{J\cup A}\ge\tau_J .
\end{equation}

(ii) Let $\mathsf E_J=\{\tau_J<\infty,\ R(x_{\tau_J})=0\}$ and define $\mathsf E_{J\cup A}$ likewise, so that $w_J=\Prob[\mathsf E_J]$ by \eqref{eq:cw} and
\begin{equation}\label{eq:w-diff}
w_{J\cup A}-w_J=\Prob[\mathsf E_{J\cup A}\setminus\mathsf E_J]-\Prob[\mathsf E_J\setminus\mathsf E_{J\cup A}].
\end{equation}

\emph{Step 1 (first term).} On $\mathsf E_{J\cup A}\setminus\mathsf E_J$, \eqref{eq:tau} gives $\tau_J\le\tau_{J\cup A}<\infty$ and $R(x_{\tau_J})=1$. The two returned candidates differ in correctness, so $\tau_J<\tau_{J\cup A}$ and $x_{\tau_J}$ fails $A$. Hence
\begin{equation}\label{eq:first-term}
\begin{split}
\Prob[\mathsf E_{J\cup A}\setminus\mathsf E_J]&\le\Prob\bigl[\tau_J<\infty,\ R(x_{\tau_J})=1,\\
&\qquad\mathrm{Pass}_A(x_{\tau_J})=0\bigr]=c_J\,f_A.
\end{split}
\end{equation}

\emph{Step 2 (second term).} If a problem lies in $\mathsf E_J$ and $\mathrm{Pass}_A(x_{\tau_J})=1$, then $x_{\tau_J}$ passes $J\cup A$ while every earlier candidate fails $J$, so both loops return the same incorrect candidate and the problem lies in $\mathsf E_{J\cup A}$. Hence
\begin{equation}\label{eq:second-term}
\begin{split}
\Prob[\mathsf E_J\setminus\mathsf E_{J\cup A}]&\le\Prob[\mathsf E_J,\ \mathrm{Pass}_A(x_{\tau_J})=0]\\
&=w_J\,(1-m_A).
\end{split}
\end{equation}

\emph{Step 3.} Inserting \eqref{eq:first-term} into \eqref{eq:w-diff} and dropping the nonnegative second term gives
\begin{equation}\label{eq:upper}
w_{J\cup A}\le w_J+c_J\,f_A .
\end{equation}
Inserting \eqref{eq:second-term} into \eqref{eq:w-diff} and dropping the nonnegative first term gives
\begin{equation}\label{eq:lower}
w_{J\cup A}\ge w_J-w_J\,(1-m_A)=m_A\,w_J .
\end{equation}
\end{proof}

\begin{thmbox}{ThmAdd}{ThmAddBg}
\begin{corollary}\label{cor:add}
\stmtcoradd
\end{corollary}
\end{thmbox}

\begin{proof}[Proof of \refadd{cor:add}]
Set $f_A=0$ in the upper and $m_A=1$ in the lower bound of \refadd{thm:add}(ii).
\end{proof}

Between these two cases, whether $w$ rises or falls depends on $f_A$ and $m_A$ on the target data.

\phantomsection\label{app:qz-chain}By the chain rule, $q_z$ in \eqref{eq:qz} is the product of the successive conditional false-negative rates of the guards in $G(z)$, without assuming independence. The bound of \reffix{thm:fixbreak}(ii) is available before refinement, and removing a proxy with many false positives but no change type to guard can lower it.

\begin{thmbox}{ThmFix}{ThmFixBg}
\begin{restated}{Theorem}{thm:fixbreak}
\stmtthmfixbreak
\end{restated}
\end{thmbox}

\begin{proof}[Proof of \reffix{thm:fixbreak}]
We show (i) by covering $\mathsf{Break}_t$ with events that match the terms of \eqref{eq:eps}, and (ii) from the stopping rule. For (i), let $\mathsf C_t=\mathsf{Act}_t\cap\{R(\tilde x_{t-1})=1\}$, so that $\Prob[\mathsf C_t]=r_t$.

\emph{Step 1.} On $\mathsf{Break}_t$ the kept version changes, so refinement step $t$ is issued, $x_t$ is accepted, and $\tilde x_t=x_t$ with $R(x_t)=0$.

\emph{Step 2.} If rule (i) accepts $x_t$, every $\Delta_k(\tilde x_{t-1},x_t)$ is empty. If rule (iii) accepts it, $\Delta_k(\tilde x_{t-1},x_t)=\varnothing$ for all $k\in Q\supseteq G(z_t)$. In both cases every guard of $z_t$ is silent. The remaining case is rule (ii). Hence
\begin{equation}\label{eq:break-cover}
\begin{split}
\mathsf{Break}_t\subseteq{}&\textstyle\bigcup_{z}\bigl(\mathsf C_t\cap\{z_t=z,\ R(x_t)=0\}\\
&\qquad\cap\,\mathsf{Silent}_{t,z}\bigr)\\
&\cup\bigl(\mathsf C_t\cap\{\text{rule (ii)},\ R(x_t)=0\}\bigr).
\end{split}
\end{equation}

\emph{Step 3.} The types are disjoint. A type without guards has $\mathsf{Silent}_{t,z}$ always true and contributes $r_t\alpha_z$. A type with guards contributes $r_t\alpha_zq_z$ by \eqref{eq:qz}. The last set has probability $r_t\beta$. Taking probabilities in \eqref{eq:break-cover} and using \eqref{eq:eps},
\begin{equation}\label{eq:break-sum}
\begin{split}
\Prob[\mathsf{Break}_t]&\le\textstyle\sum_{z:\,G(z)=\varnothing}r_t\alpha_z\\
&\quad+\textstyle\sum_{z:\,G(z)\neq\varnothing}r_t\alpha_zq_z+r_t\beta\\
&=r_t\,\varepsilon_t .
\end{split}
\end{equation}

(ii) A first draft that passes $J$ is returned unchanged. A problem with $R(x_0)=1$ and $R(\hat x)=0$ has a first draft that fails $J$, which gives (ii).
\end{proof}

The third route of $\varepsilon_t$ depends on $R(\cdot)$.

\begin{thmbox}{ThmFix}{ThmFixBg}
\begin{proposition}\label{prop:beta}
\stmtpropbeta
\end{proposition}
\end{thmbox}

\begin{proof}[Proof of \reffix{prop:beta}]
Under ${R(\cdot)=V(\cdot)\wedge M(\cdot)}$, a correct version is accepted by the prover with $\mathrm{err}(\cdot)=0$, so its prover state $(1,0)$ is maximal and rule (ii) cannot accept any candidate on it. Hence $\beta=0$.
\end{proof}

Under ${R(\cdot)=M(\cdot)}$, the route stays open because a step may change the meaning to satisfy the prover. Since correctness changes by fixes minus breaks, the break bound of \reffix{thm:fixbreak} gives conditions under which correctness never falls.

\begin{thmbox}{ThmFix}{ThmFixBg}
\begin{corollary}\label{cor:fixbreak}
\stmtcorfixbreak
\end{corollary}
\end{thmbox}

\begin{proof}[Proof of \reffix{cor:fixbreak}]
Since $R(\cdot)$ takes values in $\{0,1\}$,
\begin{equation}\label{eq:fix-break}
R(\tilde x_t)-R(\tilde x_{t-1})=\mathbf 1[\mathsf{Fix}_t]-\mathbf 1[\mathsf{Break}_t].
\end{equation}
Taking expectations and using \eqref{eq:break-sum} gives
\begin{equation}\label{eq:mono}
\begin{split}
&\Prob[R(\tilde x_t)=1]-\Prob[R(\tilde x_{t-1})=1]\\
&\qquad\ge\Prob[\mathsf{Fix}_t]-r_t\,\varepsilon_t ,
\end{split}
\end{equation}
which is nonnegative under the first condition. If $r_t=0$ for all $t$, or if ${R(\cdot)=V(\cdot)\wedge M(\cdot)}$ and every type has a guard with $q_z=0$, $r_t\varepsilon_t=0$ by \reffix{prop:beta}, and the first condition holds.
\end{proof}

Correctness therefore cannot fall at a refinement step whose fixes outweigh the leak on correct versions under refinement, and it never falls when no correct version stays under refinement or when, under ${R(\cdot)=V(\cdot)\wedge M(\cdot)}$, every change type has guards that never all stay silent on a break.

In \refplat{lem:plateau}, false positives and conflicts between proxies lower $\theta_t$, rejected candidates trigger early stopping and lower $s_t$, and the ability of the model to produce new candidates sets $\nu_t$.

\begin{thmbox}{ThmPlat}{ThmPlatBg}
\begin{restated}{Lemma}{lem:plateau}
\stmtlemplateau
\end{restated}
\end{thmbox}

\begin{proof}[Proof of \refplat{lem:plateau}]
We split the final pass rate by the refinement step at which a failing first draft first passes and then factor each term by the chain rule.

Let $\mathsf H_t=\mathsf{Act}_t\cap\{\mathrm{Pass}_J(x_t)=1\}$ be the event that the candidate of step $t$ passes all proxies.

\emph{Step 1.} Refinement step 1 is issued only when the first draft fails, and step $t$ only after step $t-1$, so every $\mathsf H_t$ and every $\mathsf{Act}_t$ lies in $\mathsf N_0$. A passing candidate is accepted by rule (i) and stops the loop, so the $\mathsf H_t$ are disjoint, and
\begin{equation}\label{eq:pass-split}
\begin{split}
\Prob[\mathrm{Pass}_J(\hat x)=1]={}&\pi_0+(1-\pi_0)\\
&\cdot\textstyle\sum_{t\ge1}\Prob[\mathsf H_t\mid\mathsf N_0].
\end{split}
\end{equation}

\emph{Step 2.} A candidate that was measured before has already failed, so $\mathsf H_t\subseteq\mathsf{New}_t$. By the chain rule and $\mathsf{Act}_t\subseteq\mathsf N_0$,
\begin{equation}\label{eq:chain}
\Prob[\mathsf H_t\mid\mathsf N_0]=s_t\,\nu_t\,\theta_t ,
\end{equation}
and inserting \eqref{eq:chain} into \eqref{eq:pass-split} gives the identity.
\end{proof}

\section{Guards and Error Profiles}\label{app:notation}
The theory of \cref{sec:theory} leaves two choices to the system, namely which proxy guards which change type and which errors each proxy can detect, and we give both for our system.

\subsection{Guard Assignment}\label{app:gatekeepers}
\Cref{tab:gatekeepers} lists the atomic types and their guards. A step may make several atomic changes, its change type is the set of their atomic types, and its guards are the union of theirs.

\begin{table}[h]
\centering
\small
\setlength{\tabcolsep}{3pt}
\begin{tabular}{@{}p{1.75cm}p{1.4cm}p{1.75cm}p{2.0cm}@{}}
\toprule
Atomic type & Criteria & Guards & Checked property \\
\midrule
Deleted content & Coverage & \px{cov} & Presence of every source claim \\
Merged or renamed statements & Grounding, statement & \px{gnd}, \px{sp} & Grounded names; faithful merged statements \\
Added premise & Statement & \px{sp} & Source sentence for each premise \\
Changed conclusion & Statement, coverage & \px{sp}, \px{cov} & Faithful and present conclusion \\
Rewritten statement & Statement; grounding for new names & \px{sp}; \px{gnd} for new names & Faithful rewrite; grounded new names \\
Syntax or type fix only & None & None & -- \\
Proof-only change & Inference & None & -- \\
Unextractable statements & Statement & None & -- \\
\bottomrule
\end{tabular}
\caption{Each atomic type alters the listed criteria, and its guards check the listed property. Inference is not a meaning criterion, and the judges cannot read unextractable statements, so these two types have no guard.}
\label{tab:gatekeepers}
\end{table}

\subsection{Error Profiles by Criterion}\label{app:classes}
The error profile of a proxy depends on the criterion that it checks. The deterministic verifiers of well-formedness (\px{glb}, \px{doc}, \px{syn}, \px{typ}) never change their verdict and report a location, yet a well-formed output with a changed meaning passes them. The guards of the statement and grounding criteria (\px{sp}, \px{gnd}) localize a verdict to a statement or a module and miss changes that span units. The guard of the coverage criterion (\px{cov}) checks that each source claim is present and ignores whether it is rendered correctly. The proxy of the inference criterion (\px{bnd}) names the missing step and can produce many false positives when the chain is valid but implicit. A holistic judge scores the whole output and gives no location for refinement. The semantic-fidelity judge of \cref{sec:exp} is of this kind, and we use it only for evaluation.

\section{Implementation of the Acceptance Rule}\label{app:acceptance}

The acceptance rule of \cref{def:accept} is the comparison that our loop implements. A unit keeps its identity across versions through its identifier. A unit of \px{gnd} or \px{bnd} is a declaration, which its kind and name identify, or its kind and order when it has no name. A unit of \px{sp} is a statement, which its declaration and position identify, and a unit of \px{cov} is a claim of the source, which no step changes.

Each judge verdict is itself a majority of three samples. When a candidate shows judge-flagged units that the kept version does not have, the kept version is re-judged on those items until it has five verdicts, and a unit that the majority still flags on the kept version is exempt. A new unit absent from the kept version is never exempt.

Two further cases make rule (iii) reject a candidate, and they do not count as effective new failing units. A failing unit of the kept version that disappears from the candidate triggers rejection unless it belongs to \px{cov} or was flagged for lacking a source sentence, and so does a statement that \px{sp} flags on both versions if its text changed. The error count $\mathrm{err}(x)$ counts five classes of prover-reported errors, one for each of \px{glb}, \px{doc}, \px{syn}, \px{typ}, and \px{bnd}, and the prover accepts $x$ exactly when $\mathrm{err}(x)=0$. Prover-reported errors enter the comparison only through the prover state, and this includes the errors that \px{bnd} counts. The judge-side comparison uses judge-flagged units only. An item that the judge cannot decide on the candidate produces no failing unit, so the corresponding failing units of the kept version count as removed.

While the prover still rejects the kept version, judge feedback is withheld from the refinement prompt and judge-side proxies act only through the acceptance rule. A candidate that passes all proxies is always accepted. When the prover still rejects the kept version and a new candidate is not accepted, the next refinement step continues from that candidate, and \cref{def:accept} simplifies this to always refining the kept version.

The loop runs at most five refinement steps. It stops early after three consecutive steps that are rejected, repeat an earlier candidate, or return no code, and a step that continues from an unaccepted candidate while the prover still rejects the kept version does not count toward the three.

\section{Experimental Details}\label{app:details}

\subsection{Datasets, Runs, and Significance}\label{app:data}
\Cref{tab:datasets} lists the datasets, and the e-SNLI set of \citet{quan2025faithful} follows the selection protocol of \citet{valentino2021nli}. Our loop, its ablations, and the compared methods run three times on every problem, and we report the mean over the three runs with the sample standard deviation. For significance, we average the outcome of each problem over the three runs and compare the paired averages with a two-sided exact sign test. \Cref{tab:extra-p} gives the $p$-values of the comparisons in \cref{sec:exp} that no table marks. The manual annotation of whole outputs, the tracing of breaks to refinement steps in \cref{ssec:exp-align}, the run with five refinement steps on every problem in \cref{ssec:exp-rounds}, and the plateau predictions of \cref{app:results} use the first run. In the run with five refinement steps on every problem, 81.3\% of ProofNet and 53.0\% of e-SNLI problems pass all proxies, against 81.9\% and 54.0\% in the first run of our loop.

\begin{table}[h]
\centering
\small
\setlength{\tabcolsep}{4pt}
\begin{tabular}{@{}lccc@{}}
\toprule
Comparison & ProofNet & miniF2F & e-SNLI \\
\midrule
Self-Refine vs.\ Reflexion & $<$0.001 & $<$0.001 & $<$0.001 \\
Self-Refine vs.\ single-shot & 0.201 & 0.125 & 0.003 \\
Only \px{bnd} vs.\ ours & $<$0.001 & -- & $<$0.001 \\
Only \px{sp} vs.\ ours & $<$0.001 & -- & $<$0.001 \\
Only \px{typ} vs.\ ours & $<$0.001 & -- & $<$0.001 \\
\bottomrule
\end{tabular}
\caption{The pooled paired sign test gives $p$-values for prover acceptance in the comparisons of \cref{ssec:exp-main,ssec:exp-proxies} that no table marks. Rows Only $k$ are the feedback ablations of \cref{app:feedback}, which skip miniF2F.}
\label{tab:extra-p}
\end{table}

\subsection{Selection of Models and Datasets}\label{app:selection}
We fix one generation model for all methods, so that differences between methods come from the loop alone. The judges inside the loop run on a model from another family than the generator, so that no model grades its own outputs. The evaluation judge is not a judge of the loop, because the loop revises its outputs according to those judges and an evaluation by them would favor our method. \Cref{tab:models} lists the models.

\begin{table*}[t]
\centering
\small
\setlength{\tabcolsep}{3pt}
\begin{tabular}{@{}p{1.9cm}p{3.2cm}p{1.2cm}p{2.1cm}ccp{3.5cm}@{}}
\toprule
\raggedright Model & \raggedright Role & \raggedright Access & \raggedright Version & \raggedright Temp. & \raggedright \shortstack{Max output\\tokens} & \raggedright Selection reason \tabularnewline
\midrule
\raggedright DeepSeek-V4-Flash & \raggedright First drafts, refinement steps, and the generation of every compared method & \raggedright Hosted API & \raggedright Not recorded & \raggedright 0.2 & \raggedright 2000 & \raggedright One generator for all methods, so differences come from the loop; another family than the loop judges \tabularnewline
\raggedright GPT-5.4-mini & \raggedright Judge-side proxies, claim extraction for \px{cov}, and the selection of pass@$k$ & \raggedright Hosted API & \raggedright \texttt{gpt-5.4-mini-\allowbreak 2026-03-17} & \raggedright 0.2 & \raggedright 1500 & \raggedright Dated version for reproducibility; low price (USD 0.75 input and 4.50 output per million tokens); another family than the generator \tabularnewline
\raggedright DeepSeek-V4-Flash & \raggedright Semantic-fidelity judge, for evaluation only & \raggedright Hosted API & \raggedright Not recorded & \raggedright 0.2 & \raggedright 4000 & \raggedright Not a judge of the loop, whose outputs follow those judges; agrees with GPT-5.4-mini on development candidates \tabularnewline
\bottomrule
\end{tabular}
\caption{We list the role, access, version, temperature, output limit, and reason for each model.}
\label{tab:models}
\end{table*}

The four datasets are common public benchmarks (\cref{tab:datasets}). We formalize the two mathematics benchmarks in Lean~4 with Mathlib and the two natural-language inference benchmarks in Isabelle/HOL, and we use e-SNLI and ProntoQA following previous work that uses them \citep{valentino2021nli,quan2024verification,quan2025faithful}. We take the full test split where one exists, and we sample 200 ProntoQA problems so that its size is comparable to the mathematics benchmarks.

\begin{table*}[t]
\centering
\small
\setlength{\tabcolsep}{3pt}
\begin{tabular}{@{}p{2.6cm}p{2.4cm}p{1.8cm}p{2.8cm}p{1.4cm}p{2.4cm}@{}}
\toprule
\raggedright Dataset & \raggedright Domain & \raggedright Prover & \raggedright Source & \raggedright Split and size & \raggedright Input \tabularnewline
\midrule
\raggedright ProofNet \citep{azerbayev2023proofnet} & \raggedright Undergraduate mathematics & \raggedright Lean~4, Mathlib & \raggedright Lean~4 version & \raggedright Test, 182 & \raggedright Problem statement \tabularnewline
\raggedright miniF2F \citep{zheng2021minif2f} & \raggedright Competition mathematics & \raggedright Lean~4, Mathlib & \raggedright Lean~4 version & \raggedright Test, 244 & \raggedright Problem statement \tabularnewline
\raggedright e-SNLI \citep{camburu2018esnli} & \raggedright NLI explanations & \raggedright Isabelle/HOL & \raggedright Release of \citet{quan2025faithful} & \raggedright 100 & \raggedright Premise, hypothesis, explanation \tabularnewline
\raggedright ProntoQA \citep{saparov2023prontoqa} & \raggedright Synthetic logical reasoning & \raggedright Isabelle/HOL & \raggedright Development split, fixed seed & \raggedright 200 & \raggedright Premise, hypothesis, rules \tabularnewline
\bottomrule
\end{tabular}
\caption{We list the domain, prover, source, split and size, and input of each dataset.}
\label{tab:datasets}
\end{table*}

\subsection{Compared Methods}\label{app:methods}
Single-shot generation returns our first draft unchanged, and pass@$k$ samples ten drafts and keeps the one that our proxies rank highest. Self-Refine \citep{madaan2023selfrefine} critiques and rewrites its own output without prover feedback, and Reflexion \citep{shinn2023reflexion} feeds back the raw prover errors with its recent reflections. Self-Refine and Reflexion start from our first draft and receive the same refinement budget. None of the three alternatives receives the library-name and error hints that our refinement prompt adds (\cref{app:repair-prompt}), and \cref{tab:settings} lists their settings.

\subsection{Demonstrations in the Prompts}
Each generation prompt contains three demonstrations. For miniF2F and ProofNet they come from the validation splits. No public Isabelle corpus exists for e-SNLI and ProntoQA, so their demonstrations are taken from first-round outputs of our system that the prover accepts and that pass every proxy, and all methods share them. These six problems also belong to the evaluated sets.

\subsection{Content of the Refinement Prompt}\label{app:repair-prompt}
A refinement prompt lists the failing proxies with prover errors first, the module of each failure, the unit that this step should target, and up to three recently rejected candidates with the reason for each rejection. On Lean it also adds library-name hints for the reported errors, which come from a fixed table of renamed or commonly mistyped Mathlib names. When Lean cannot resolve a name, the hints also list the five closest names in the installed Mathlib. On Isabelle it adds fixed reminders for frequent syntax and type errors instead. The prompt also quotes the prover errors of the last rejected candidate. These hints appear neither in the first draft, nor in the single-critique ablation, nor in Self-Refine, Reflexion, or pass@$k$. The model first writes a short diagnosis of each failure and then the new formalization, and the next prompt includes up to three recent diagnoses. While the prover rejects the kept version, judge feedback is left out of the prompt. \Cref{app:prompts} gives the full messages.

\subsection{Procedure of the Semantic-Fidelity Judge}\label{app:sfj}
Following the decomposition of LeanScorer \citep{xuejun2025mathesis}, the judge first copies the conditions and conclusions of the source as separate items without rewriting them, once per problem. It then removes the proofs, compares each item with the remaining formalization, grades it as consistent, slightly inconsistent, or seriously inconsistent, and checks for missing or added conditions. One serious inconsistency makes the output unfaithful, and otherwise the grades are combined into a score with threshold 0.6. Narrative sentences without content need no counterpart, a question must be answered in the conclusion, an added condition or a weaker conclusion is serious, and Lean semantics such as natural numbers that include zero are applied. The judge runs three times at temperature 0.2, and two unfaithful votes make the output unfaithful.

\subsection{Annotation Procedure}\label{app:annotation}
We annotate 25 problems per dataset in each of two sets. The per-item set labels first drafts item by item for \px{sp}, \px{cov}, and \px{gnd}, and the whole-output set labels the first draft and the returned output of our loop as a whole. Where the whole-output set labels an output, this label decides its meaning, and the semantic-fidelity judge decides that of every other output. Both sets oversample problems on which two judges disagree, so we report fractions of the annotated sample or reweight them by sampling stratum, as for all rates in \cref{tab:judge-human}. We do not use the per-item \px{gnd} labels on e-SNLI.

For each output, annotators see the source problem, a reference formalization that only helps them understand the problem, hints for its terms, and the formalization to label, which for Lean is the theorem statement alone. They see no verdict of any proxy or judge and no sign of which method or step produced the output, and the outputs appear in random order. \px{sp} asks whether each formal statement says what its source sentence says, \px{cov} whether every key claim of the source has an explicit counterpart, and \px{gnd} whether every symbol can be traced to the source, its context, or a library definition. Each label is pass or fail, and \px{sp} and \px{cov} also allow unsure. An output is correct when all three labels pass and incorrect otherwise, and in the whole-output set an undecided annotator marks it unsure and writes a one-sentence reason.

Annotators judge only whether the formalization states what the source says, not whether its conclusion follows, so an output that faithfully renders a gap in the source explanation passes. A statement fails \px{sp} if it swaps arguments or roles, reverses a relation, misplaces the scope of a quantifier or negation, changes a connective, changes a number or bound, swaps premise and conclusion, states as fact what the source states as possible, changes the meaning of a predicate, or uses a bridging fact that the source does not give. Purely verbal differences pass. A defect belongs to \px{sp} when the annotator can point to the statement that renders the claim and to \px{cov} when no statement renders it. For \px{gnd}, library definitions, bound variables, the types and variables needed to state the objects of the source, and the role words of the event notation, such as \texttt{Agent} and \texttt{Patient}, have a source. In the whole-output set, an output is also incorrect when the library allows a more faithful formalization of the source. Annotators always compare the output with the source problem, not with the reference formalization. \Cref{tab:annotation} counts the per-item labels.

\begin{table}[h]
\centering
\small
\setlength{\tabcolsep}{3pt}
\begin{tabular}{@{}lccccc@{}}
\toprule
Dataset & Outputs & \px{sp} fail & \px{cov} fail & \px{gnd} fail & Incorrect \\
\midrule
ProofNet & 25 & 13 & 1 & 0 & 13 \\
miniF2F & 25 & 6 & 4 & 1 & 7 \\
e-SNLI & 25 & 12 & 1 & -- & 13 \\
ProntoQA & 25 & 1 & 0 & 0 & 1 \\
\bottomrule
\end{tabular}
\caption{Among the 25 per-item drafts of each dataset, we count those that fail each label and those that are incorrect. \Cref{tab:align} covers the whole-output set.}
\label{tab:annotation}
\end{table}

The prompt of the semantic-fidelity judge was finalized after its agreement was measured, so its agreement rates are in-sample.

\subsection{Hyperparameters and Lesion Reruns}
\Cref{tab:settings} lists the models and settings. Each problem has a time limit of 2400 seconds.

Each lesion of \cref{ssec:exp-lesion} reuses the main-run trajectory on every problem where the removed proxy never acted on any measured version, that is, on the first draft or on any measured proposal. A judge-side proxy has not acted when it flagged no unit and left no item undecided, a prover-side proxy has not acted when the prover never reported its class of error, and \px{bnd} must meet both conditions. On such a problem, removing the proxy changes no feedback, acceptance, or stopping decision. We rerun only the remaining problems, which make up 19.7\% of all combinations of problem and lesion on average over three runs (16.8\% on ProofNet and 25.1\% on e-SNLI), so on the reused problems each lesion equals our loop by construction.

\begin{table}[h]
\centering
\small
\setlength{\tabcolsep}{3pt}
\begin{tabular}{@{}p{3.0cm}p{4.2cm}@{}}
\toprule
Component & Setting \\
\midrule
Generation and\newline refinement & DeepSeek-V4-Flash; temperature 0.2 \\
Judge-side proxies & GPT-5.4-mini; temperature 0.2; majority of three samples \\
Semantic-fidelity judge & DeepSeek-V4-Flash; temperature 0.2; majority of three samples \\
Refinement budget & At most 5 refinement steps; early stop after 3 non-improving steps (\cref{app:acceptance}) \\
Re-judging & Up to five verdicts\newline on the kept version \\
pass@$k$ & $k=10$; temperature 0.2; ranked by our proxies \\
Self-Refine, Reflexion & First draft and at most 5 rounds \\
\bottomrule
\end{tabular}
\caption{All methods use the listed models and settings.}
\label{tab:settings}
\end{table}

\subsection{Call Dates}\label{app:models}
\Cref{tab:dates} lists the dates of the model calls.

\begin{table}[h]
\centering
\small
\setlength{\tabcolsep}{3pt}
\begin{tabular}{@{}p{4.6cm}l@{}}
\toprule
Experiment & Dates \\
\midrule
Shared first drafts & 21--23 Sep \\
Our loop, refinement and judges & 6--7 Oct \\
Self-Refine and Reflexion & 23 Sep--6 Oct \\
pass@$k$ sampling & 23--24 Sep \\
pass@$k$ selection and judge readings of the compared methods & 5--6 Oct \\
Lesions & 9--10 Oct \\
Feedback ablations & 9 Oct \\
Semantic-fidelity judge & 23 Sep--10 Oct \\
Second and third runs of every experiment & 10--11 Oct \\
\bottomrule
\end{tabular}
\caption{Each experiment ran on these dates in 2026 (UTC), for the first run unless noted. Our loop, Self-Refine, Reflexion, and the ablations share first drafts.}
\label{tab:dates}
\end{table}

\subsection{Calls and Time per Problem}
\Cref{tab:cost} reports the average calls and wall-clock time per problem. For pass@$k$, judge calls are those that rerank its candidates with our proxies. Time for our loop, Self-Refine, and Reflexion excludes generating the shared first draft.

\begin{table}[h]
\centering
\small
\setlength{\tabcolsep}{3pt}
\begin{tabular}{@{}llccc@{}}
\toprule
& & \multicolumn{2}{c}{Calls} & \\
\cmidrule(lr){3-4}
Dataset & Method & Generation & Judge & Time (s) \\
\midrule
ProofNet & pass@$k$ & \msd{10.0}{0.0} & \msd{70.5}{0.4} & \msd{206.0}{54.1} \\
 & Self-Refine & \msd{6.4}{0.1} & \msd{0.0}{0.0} & \msd{112.6}{45.8} \\
 & Reflexion & \msd{3.7}{0.1} & \msd{0.0}{0.0} & \msd{77.1}{36.9} \\
 & Ours & \msd{2.5}{0.1} & \msd{28.4}{1.2} & \msd{129.4}{45.9} \\
\midrule
miniF2F & pass@$k$ & \msd{10.0}{0.0} & \msd{62.0}{0.8} & \msd{187.0}{23.7} \\
 & Self-Refine & \msd{4.8}{0.0} & \msd{0.0}{0.0} & \msd{66.7}{16.0} \\
 & Reflexion & \msd{1.8}{0.0} & \msd{0.0}{0.0} & \msd{27.4}{2.2} \\
 & Ours & \msd{1.6}{0.0} & \msd{20.0}{0.3} & \msd{94.2}{39.4} \\
\midrule
e-SNLI & pass@$k$ & \msd{10.0}{0.0} & \msd{152.7}{1.7} & \msd{134.7}{22.9} \\
 & Self-Refine & \msd{6.9}{0.1} & \msd{0.0}{0.0} & \msd{50.1}{11.8} \\
 & Reflexion & \msd{1.5}{0.1} & \msd{0.0}{0.0} & \msd{10.3}{0.7} \\
 & Ours & \msd{3.0}{0.1} & \msd{77.3}{4.4} & \msd{181.6}{8.2} \\
\midrule
ProntoQA & pass@$k$ & \msd{10.0}{0.0} & \msd{113.6}{2.6} & \msd{158.2}{32.6} \\
 & Self-Refine & \msd{2.2}{0.1} & \msd{0.0}{0.0} & \msd{12.1}{3.0} \\
 & Reflexion & \msd{1.0}{0.0} & \msd{0.0}{0.0} & \msd{7.1}{0.6} \\
 & Ours & \msd{1.5}{0.0} & \msd{110.8}{6.1} & \msd{265.8}{26.5} \\
\bottomrule
\end{tabular}
\caption{Each method spends the listed calls and time per problem on average. Generation counts generation and refinement calls, Judge counts the judge calls, and Time gives wall-clock seconds. Values are the mean $\pm$ sample standard deviation over three runs.}
\label{tab:cost}
\end{table}

\section{Additional Results}\label{app:results}

\subsection{Fixes and Breaks under the Semantic-Fidelity Judge}\label{app:align-judge}
\Cref{tab:align-judge} repeats \cref{tab:align} on all problems and three runs with the semantic-fidelity judge deciding meaning. On e-SNLI, 32.3\% of problems have a correct first draft that some proxy rejects and 3.7\% end broken, all from this group. On ProntoQA, 15.5\% have such a draft and 0.2\% end broken.

\begin{table}[h]
\centering
\small
\setlength{\tabcolsep}{3pt}
\begin{tabular}{@{}lccc@{}}
\toprule
Dataset & Fixed & Broken & Correct but rejected \\
\midrule
ProofNet & \msd{32.1}{1.8} & \msd{0.0}{0.0} & \msd{3.7}{0.8} \\
miniF2F & \msd{14.9}{0.3} & \msd{0.3}{0.5} & \msd{3.4}{1.0} \\
e-SNLI & \msd{9.7}{2.5} & \msd{3.7}{3.1} & \msd{32.3}{0.6} \\
ProntoQA & \msd{0.0}{0.0} & \msd{0.2}{0.3} & \msd{15.5}{1.3} \\
\bottomrule
\end{tabular}
\caption{When the semantic-fidelity judge decides meaning, we count fixes, breaks, and correct first drafts that fail some proxy, as percentages of all problems. Each cell gives the mean $\pm$ sample standard deviation over three runs, and problems that the judge cannot decide are left out of the denominator.}
\label{tab:align-judge}
\end{table}

\subsection{Tracing Breaks to Refinement Steps}\label{app:break-trace}
For the routes of breaks in \cref{ssec:exp-align}, we take ${R(\cdot)=M(\cdot)}$, which ignores the prover, and trace to a refinement step each output that the semantic-fidelity judge rates faithful at first and unfaithful at the end (\cref{tab:break-routes}). We assign each traced break to a route by the rule that accepted the step that changed the meaning. A step accepted by rule (ii) takes the route of $\beta$, and a step accepted by rule (i) or (iii) is split by whether the guards of its change type stay silent, flag it, or do not exist. An output is untraceable when the judge changes its verdict although no step on the refinement chain changed the part in question, and we do not count it as a break. On ProofNet, miniF2F, and e-SNLI, the loop accepted every proposal that improved the prover state while the judges flagged new units. On ProofNet, the recorded verdicts leave 57.9\% of such proposals undecided, and re-judging confirms the flags on all others. Under the annotation, the e-SNLI outputs of problems 30 and 97 are also broken, and in both a step added content to the premises while the guard \px{sp} stayed silent.

\begin{table}[h]
\centering
\footnotesize
\setlength{\tabcolsep}{2.5pt}
\begin{tabular}{@{}lcccc@{}}
\toprule
Route & ProofNet & miniF2F & e-SNLI & ProntoQA \\
\midrule
Outputs & 5 & 3 & 8 & 0 \\
\midrule
Rule (ii) & 2 & 1 & 0 & 0 \\
\quad Flagged & 1 & 1 & 0 & 0 \\
Silent & 1 & 1 & 8 & 0 \\
None & 0 & 0 & 0 & 0 \\
Overruled & 0 & 0 & 0 & 0 \\
Untraced & 2 & 1 & 0 & 0 \\
\bottomrule
\end{tabular}
\caption{We count the traced outputs of the first run by the route of the step that changed their meaning. Outputs counts these outputs, Rule (ii) the breaks that rule (ii) accepted, and Flagged those among them accepted while the judges flagged new units. Among the breaks accepted by rule (i) or (iii), Silent, None, and Overruled count those whose guards stayed silent, did not exist, or flagged them. Untraced counts the untraceable outputs.}
\label{tab:break-routes}
\end{table}

\subsection{Predictions Written before the Lesions}\label{app:predictions}
Without stratum weights, \px{sp} detects 54\% of meaning errors on the annotated ProofNet sample and 75\% on e-SNLI, with false-positive rates of 17\% and 15\%, so removing it should raise both the pass rate and $w$ (\refadd{thm:add}). Correctness should drop as well, because \px{sp} is the only guard of added premises and rewritten statements (\reffix{thm:fixbreak}). The judge \px{gnd} produced no false positive on the annotated ProofNet and e-SNLI samples and almost never fires in the main runs, so removing it should change nothing. Prover errors never fire on an output that is correct, so removing \px{syn} or \px{typ} should not lower $w$. For \px{cov} on e-SNLI we predicted only a small rise in pass rate, because its false positives and its role as a guard push $w$, breaks, and correctness in opposite directions. From \reffix{thm:fixbreak} we also predicted that datasets with fewer correct first drafts that fail some proxy have a lower bound on breaks, and that a second guard lowers the bound only when its false-negative rate, conditional on the first guard missing, is clearly below one.

We judge each prediction on the mean outcome over three runs and test it with the pooled sign test of \cref{app:data}. Under the directional standard, a prediction agrees when the outcome moves in the predicted direction and is contradicted when it moves the other way with a paired two-sided sign test below $p=0.05$, and otherwise it cannot be decided. The symmetric standard also requires $p<0.05$ for agreement. \Cref{tab:pred-count} gives the outcome of all predictions under both standards. Two lesion arms that each removed an earlier sub-check of \px{bnd} were withdrawn after these sub-checks were merged into \px{bnd}, and we use their readings only to check the predictions. We wrote no prediction for removing \px{bnd} as a whole. The only contradiction concerns ProofNet with \px{typ} alone, where we expected prover acceptance near that of our loop and observed 63.6\% against 91.4\%.

\begin{table}[h]
\centering
\footnotesize
\setlength{\tabcolsep}{2.5pt}
\begin{tabular}{@{}lcccccc@{}}
\toprule
& \multicolumn{3}{c}{Directional} & \multicolumn{3}{c}{Symmetric} \\
\cmidrule(lr){2-4}\cmidrule(l){5-7}
Group & Ag. & Co. & Un. & Ag. & Co. & Un. \\
\midrule
One proxy removed & 62.5 & 0.0 & 37.5 & 32.5 & 0.0 & 67.5 \\
\px{bnd} sub-checks & 81.2 & 0.0 & 18.8 & 25.0 & 0.0 & 75.0 \\
Feedback & 85.7 & 7.1 & 7.1 & 71.4 & 7.1 & 21.4 \\
\midrule
All & 71.4 & 1.4 & 27.1 & 38.6 & 1.4 & 60.0 \\
\bottomrule
\end{tabular}
\caption{We judge the predictions written before the runs under the directional and the symmetric standard on the mean over three runs with the pooled sign test and give their outcome in percent. Percentages are over the 40 predictions for removing one proxy, the 16 for the withdrawn sub-checks of \px{bnd}, the 14 for the feedback ablations, and all 70 together. Ag., Co., and Un. stand for agree, contradicted, and undecided. The row of \px{bnd} sub-checks gives the withdrawn arms that removed an earlier sub-check of \px{bnd}, which ran once.}
\label{tab:pred-count}
\end{table}

We committed these predictions to our repository on 9 October 2026, before the lesion runs. \Cref{tab:pred-list} lists the predictions for removing one proxy. Breaks are at most 0.2\% on ProofNet in every configuration, and on e-SNLI they are 3.7\% with all proxies, 5.0\% without \px{sp}, and 3.0\% without \px{cov}.

\begin{table}[h]
\centering
\footnotesize
\setlength{\tabcolsep}{1.2pt}
\begin{tabular}{@{}llllll@{}}
\toprule
& & \multicolumn{2}{c}{ProofNet} & \multicolumn{2}{c}{e-SNLI} \\
\cmidrule(lr){3-4}\cmidrule(l){5-6}
Removed & Measure & Predicted & Result & Predicted & Result \\
\midrule
\px{sp} & Pass & $>$ & agree* & $>$ & agree* \\
 & $w$ & $>$ & agree* & $>$ & agree* \\
 & Breaks & $\geq$ & agree & $>$ & agree \\
 & Fidelity & $<$ & agree & $<$ & agree \\
\midrule
\px{cov} & Pass & $>$ & agree & $>$ & agree* \\
 & $w$ & $\geq$ & agree & -- & undecided \\
 & Breaks & -- & undecided & -- & undecided \\
 & Fidelity & $\leq$ & undecided & -- & undecided \\
\midrule
\px{gnd} & Pass & $=$ & undecided & $=$ & undecided \\
 & $w$ & $=$ & undecided & $=$ & undecided \\
 & Breaks & $=$ & agree & $=$ & undecided \\
 & Fidelity & $=$ & undecided & $=$ & undecided \\
\midrule
\px{syn} & Pass & $>$ & agree* & $>$ & agree \\
 & $w$ & $\geq$ & agree* & $\geq$ & agree* \\
 & Breaks & $=$ & agree & $=$ & undecided \\
 & Fidelity & $\leq$ & agree* & $\leq$ & agree* \\
\midrule
\px{typ} & Pass & $>$ & agree* & $>$ & undecided \\
 & $w$ & $\geq$ & agree* & $\geq$ & undecided \\
 & Breaks & $=$ & agree & $=$ & agree \\
 & Fidelity & $\leq$ & agree* & $\leq$ & agree \\
\bottomrule
\end{tabular}
\caption{We list the predictions for removing one proxy as we committed them before the lesion runs, with their outcome on the mean over three runs. Pass is the pass rate over the remaining proxies, and breaks are correct first drafts whose returned output is incorrect. A prediction is a rise ($>$), a fall ($<$), no fall ($\geq$), no rise ($\leq$), no change ($=$), or none (--), and the rises predicted for \px{cov} on the pass rate are small. An asterisk marks an agreement that is significant ($p<0.05$), as the symmetric standard requires. \Cref{tab:lesion} gives the values.}
\label{tab:pred-list}
\end{table}

\subsection{Paired Sign Tests for the Lesions}
\Cref{tab:lesion-p} gives the pooled paired sign-test $p$-values of each lesion against our loop. When we wrote the predictions, the pass-but-wrong rate was defined more narrowly, as passing outputs that the prover accepts and the judge rates unfaithful. On the first run, that definition gives slightly fewer agreeing predictions than \eqref{eq:cw} and no contradiction.

\begin{table}[h]
\centering
\small
\setlength{\tabcolsep}{4pt}
\begin{tabular}{@{}lcccc@{}}
\toprule
& \multicolumn{2}{c}{ProofNet} & \multicolumn{2}{c}{e-SNLI} \\
\cmidrule(lr){2-3}\cmidrule(l){4-5}
Removed & Acc. & Fid. & Acc. & Fid. \\
\midrule
\px{sp} & 0.289 & 0.078 & 0.625 & 0.648 \\
\px{cov} & 1.000 & 0.688 & 1.000 & 1.000 \\
\px{gnd} & 1.000 & 0.500 & 1.000 & 0.500 \\
\px{bnd} & $<$0.001 & $<$0.001 & 1.000 & 0.302 \\
\px{syn} & $<$0.001 & 0.002 & $<$0.001 & 0.002 \\
\px{typ} & $<$0.001 & $<$0.001 & 0.008 & 0.250 \\
\bottomrule
\end{tabular}
\caption{The pooled paired sign test compares each lesion of \cref{tab:lesion} with our loop and gives these $p$-values for prover acceptance (Acc.) and semantic fidelity (Fid.), where the outcome of each problem is averaged over three runs before the test.}
\label{tab:lesion-p}
\end{table}

\subsection{Agreement of \px{bnd} with Annotation}\label{app:judges}
\Cref{tab:judge-human} in \cref{ssec:exp-proxies} reports per-dataset agreement, false-positive, and false-negative rates on the per-item set. On the whole-output set, \px{bnd} agrees with the annotated judgment of meaning on 64\%, 72\%, 48\%, and 76\% of first drafts on ProofNet, miniF2F, e-SNLI, and ProntoQA. In the first run, a prover check of the ProntoQA conclusions covers the 11.0\% of problems on which the judge rejects \px{bnd} or cannot decide it. All of them are derivable, so these are false positives.

\subsection{Stability and Locality of Verdicts}\label{app:verdicts}
Re-judged \px{sp} verdicts differ from the five-verdict majority on at most 2.9\% of units per dataset. Verdict changes after a step stay inside the edited statements and declarations on ProofNet and miniF2F.

\subsection{Per-Proxy Error Rates and Shares}\label{app:proxy-readings}
Pooled over datasets, \px{sp} has the lowest false-negative rate of the judges and the largest share of fixed units on miniF2F and e-SNLI (\cref{tab:proxies-props}).

\begin{table}[h]
\centering
\small
\setlength{\tabcolsep}{2.5pt}
\begin{tabular}{@{}lcccccc@{}}
\toprule
& & & \multicolumn{4}{c}{Share of fixed units} \\
\cmidrule(l){4-7}
Proxy & $f$ & $m$ & ProofNet & miniF2F & e-SNLI & ProntoQA \\
\midrule
\px{sp} & 15 & 34 & \msd{17.2}{3.9} & \msd{34.6}{4.7} & \msd{35.8}{4.3} & -- \\
\px{cov} & 7 & 67 & \msd{4.2}{1.2} & \msd{16.2}{2.1} & \msd{14.7}{4.5} & -- \\
\px{gnd} & 1 & 100 & \msd{0.0}{0.0} & \msd{1.2}{1.1} & \msd{2.7}{3.2} & -- \\
\px{bnd} & -- & -- & \msd{33.6}{3.2} & \msd{14.4}{4.5} & \msd{18.0}{2.6} & -- \\
\bottomrule
\end{tabular}
\caption{We give the false-positive and false-negative rates $f$ and $m$ of the judge-side proxies against manual annotation and their share of the units that accepted steps fix, in percent, as mean $\pm$ sample standard deviation over three runs. The rates come from the per-item set, pooled over datasets without weights, on 68, 94, and 77 correct items and on 32, 6, and 1 incorrect items for \px{sp}, \px{cov}, and \px{gnd}. For \px{gnd} on e-SNLI we do not use the per-item labels, and \cref{app:judges} gives the rates of \px{bnd}. ProntoQA has too few accepted steps for shares.}
\label{tab:proxies-props}
\end{table}

\subsection{Feedback Ablations}\label{app:feedback}
The full loop reaches the highest prover acceptance and semantic fidelity on both datasets (\cref{tab:feedback}).

\begin{table}[h]
\centering
\small
\setlength{\tabcolsep}{4pt}
\begin{tabular}{@{}lcccc@{}}
\toprule
& \multicolumn{2}{c}{ProofNet} & \multicolumn{2}{c}{e-SNLI} \\
\cmidrule(lr){2-3}\cmidrule(l){4-5}
Feedback & Acc.\,$\uparrow$ & Fid.\,$\uparrow$ & Acc.\,$\uparrow$ & Fid.\,$\uparrow$ \\
\midrule
Only \px{bnd} & \msd{62.3}{1.1} & \msd{52.2}{2.0} & \msd{85.7}{0.6} & \msd{55.0}{1.0} \\
Only \px{sp} & \msd{52.7}{1.1} & \msd{46.3}{1.8} & \msd{86.3}{0.6} & \msd{57.0}{3.0} \\
Only \px{typ} & \msd{63.6}{3.7} & \msd{53.5}{3.0} & \msd{88.0}{0.0} & \msd{58.3}{3.1} \\
Single critique & \msd{66.8}{2.3} & \msd{58.8}{0.5} & \msd{88.0}{0.0} & \msd{58.3}{2.5} \\
All proxies (ours) & \msd{91.4}{1.9} & \msd{75.6}{3.0} & \msd{99.7}{0.6} & \msd{63.3}{2.3} \\
\bottomrule
\end{tabular}
\caption{The feedback ablations change the feedback that drives the loop, and values are in percent of all problems, as mean $\pm$ sample standard deviation over three runs. In the rows Only $k$, proxy $k$ alone drives feedback, acceptance, and stopping. The single critique is one unstructured critique without the library-name and error hints of our refinement prompt. Acc. is prover acceptance and Fid. is semantic fidelity.}
\label{tab:feedback}
\end{table}

\subsection{Layers of the Refinement Budget}\label{app:plateau}
\Cref{tab:plateau} splits the final pass rate of each dataset into the factors of \refplat{lem:plateau} at every refinement step, which \cref{ssec:exp-rounds} summarizes.

\begin{table*}[t]
\centering
\footnotesize
\setlength{\tabcolsep}{3pt}
\begin{tabular}{@{}llccccccccc@{}}
\toprule
& & \multicolumn{5}{c}{Refinement step $t$} & & & & \\
\cmidrule(lr){3-7}
Dataset & Factor & 1 & 2 & 3 & 4 & 5 & $\pi_0$ & No new & New fails & Passes \\
\midrule
ProofNet & $s_t$ & \msd{100.0}{0.0} & \msd{65.1}{4.1} & \msd{45.9}{2.8} & \msd{30.6}{2.1} & \msd{17.3}{8.0} & \multirow{3}{*}{\msd{43.8}{0.6}} & \multirow{3}{*}{\msd{22.9}{1.8}} & \multirow{3}{*}{\msd{51.5}{3.3}} & \multirow{3}{*}{\msd{81.0}{0.8}} \\
 & $\nu_t$ & \msd{97.1}{1.9} & \msd{82.6}{2.8} & \msd{61.5}{6.3} & \msd{45.8}{3.9} & \msd{35.2}{9.5} &  &  &  &  \\
 & $\theta_t$ & \msd{35.9}{3.7} & \msd{35.6}{10.2} & \msd{30.1}{3.9} & \msd{19.6}{18.1} & \msd{12.2}{10.7} &  &  &  &  \\
\midrule
miniF2F & $s_t$ & \msd{100.0}{0.0} & \msd{45.4}{7.7} & \msd{32.8}{6.8} & \msd{16.2}{2.8} & \msd{8.6}{2.9} & \multirow{3}{*}{\msd{69.5}{0.9}} & \multirow{3}{*}{\msd{18.7}{3.0}} & \multirow{3}{*}{\msd{44.2}{4.9}} & \multirow{3}{*}{\msd{92.2}{2.1}} \\
 & $\nu_t$ & \msd{96.9}{2.0} & \msd{75.7}{5.9} & \msd{57.9}{10.2} & \msd{57.5}{7.2} & \msd{61.3}{10.8} &  &  &  &  \\
 & $\theta_t$ & \msd{56.4}{7.5} & \msd{36.8}{3.1} & \msd{22.2}{4.8} & \msd{35.9}{23.0} & \msd{0.0}{0.0} &  &  &  &  \\
\midrule
e-SNLI & $s_t$ & \msd{100.0}{0.0} & \msd{77.5}{3.4} & \msd{70.1}{2.9} & \msd{14.8}{4.5} & \msd{6.4}{0.9} & \multirow{3}{*}{\msd{27.3}{2.3}} & \multirow{3}{*}{\msd{46.2}{1.9}} & \multirow{3}{*}{\msd{42.0}{1.5}} & \multirow{3}{*}{\msd{50.3}{3.2}} \\
 & $\nu_t$ & \msd{99.1}{0.8} & \msd{39.2}{2.8} & \msd{16.4}{5.9} & \msd{12.2}{11.3} & \msd{28.3}{10.4} &  &  &  &  \\
 & $\theta_t$ & \msd{22.7}{3.3} & \msd{24.2}{1.7} & \msd{16.4}{4.7} & -- & \msd{0.0}{0.0} &  &  &  &  \\
\midrule
ProntoQA & $s_t$ & \msd{100.0}{0.0} & \msd{97.9}{3.6} & \msd{96.9}{3.1} & \msd{0.0}{0.0} & \msd{0.0}{0.0} & \multirow{3}{*}{\msd{84.5}{1.3}} & \multirow{3}{*}{\msd{66.5}{5.9}} & \multirow{3}{*}{\msd{32.4}{6.6}} & \multirow{3}{*}{\msd{85.0}{1.0}} \\
 & $\nu_t$ & \msd{47.3}{11.8} & \msd{36.1}{3.1} & \msd{16.6}{2.9} & -- & -- &  &  &  &  \\
 & $\theta_t$ & \msd{6.1}{10.5} & \msd{2.6}{4.4} & \msd{0.0}{0.0} & -- & -- &  &  &  &  \\
\bottomrule
\end{tabular}
\caption{We list the factors of \refplat{lem:plateau} at every step for each dataset. No new and New fails are the shares of all steps that return no new candidate or a new candidate that fails, and Passes is the fraction of problems whose output passes all proxies. Values are in percent, as mean $\pm$ sample standard deviation over three runs, and a dash marks a fraction that is undefined in some run because no problem is under refinement or no step returns a new candidate.}
\label{tab:plateau}
\end{table*}

\subsection{Predicting the Plateau across Datasets}
From the first run, we predict the fraction of problems that pass all proxies on one dataset from the rate at which each first-draft failure combination is resolved on the other. For miniF2F from ProofNet the range is 88.8\% to 90.9\% against 91.0\%, and for ProofNet from miniF2F it is 80.6\% to 86.1\% against 81.9\%. From e-SNLI to ProntoQA the prediction is 94.2\% against 86.0\%. From ProntoQA to e-SNLI the range is 30.0\% to 59.0\% against 54.0\%.

\subsection{Per-Dataset Fixes and Breaks}
\Cref{fig:heatmap-per} shows each dataset, where fixed units lie on the diagonal as in the pooled \cref{fig:heatmap}.

\begin{figure*}[h]
\centering
\begin{subfigure}[t]{0.245\textwidth}\centering\includegraphics[width=\linewidth]{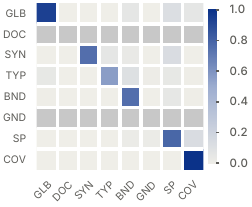}\caption{ProofNet, fixed}\label{fig:heat-proofnet-fixed}\end{subfigure}\hfill
\begin{subfigure}[t]{0.245\textwidth}\centering\includegraphics[width=\linewidth]{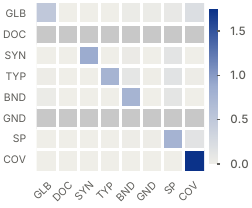}\caption{miniF2F, fixed}\label{fig:heat-minif2f-fixed}\end{subfigure}\hfill
\begin{subfigure}[t]{0.245\textwidth}\centering\includegraphics[width=\linewidth]{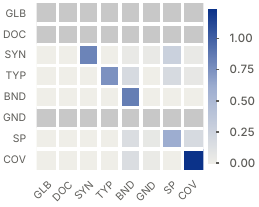}\caption{e-SNLI, fixed}\label{fig:heat-esnli-fixed}\end{subfigure}\hfill
\begin{subfigure}[t]{0.245\textwidth}\centering\includegraphics[width=\linewidth]{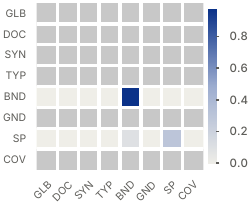}\caption{ProntoQA, fixed}\label{fig:heat-prontoqa-fixed}\end{subfigure}\\[2pt]
\begin{subfigure}[t]{0.245\textwidth}\centering\includegraphics[width=\linewidth]{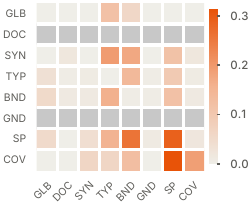}\caption{ProofNet, broken}\label{fig:heat-proofnet-broken}\end{subfigure}\hfill
\begin{subfigure}[t]{0.245\textwidth}\centering\includegraphics[width=\linewidth]{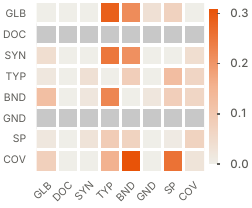}\caption{miniF2F, broken}\label{fig:heat-minif2f-broken}\end{subfigure}\hfill
\begin{subfigure}[t]{0.245\textwidth}\centering\includegraphics[width=\linewidth]{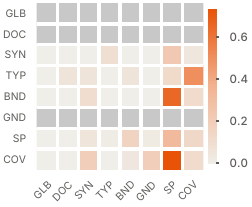}\caption{e-SNLI, broken}\label{fig:heat-esnli-broken}\end{subfigure}\hfill
\begin{subfigure}[t]{0.245\textwidth}\centering\includegraphics[width=\linewidth]{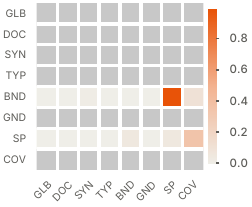}\caption{ProntoQA, broken}\label{fig:heat-prontoqa-broken}\end{subfigure}
\caption{Panels (a) to (d) count the failing units that proposals fix on each dataset, and panels (e) to (h) count those that they break, summed over three runs. Rows give the proxy that the refinement prompt targets, and columns give the affected proxy. The color bars give the average number of units fixed or broken per proposal that targets the row proxy, and their scales differ across panels. Rows targeted by fewer than ten proposals are shown in gray.}
\label{fig:heatmap-per}
\end{figure*}

\onecolumn
\section{Prompts of All Methods}\label{app:prompts}
We give the messages of every prompt that our loop, the semantic-fidelity judge, and the compared methods send, verbatim except for shortened local file paths in prover messages, for one Lean and one Isabelle example. Messages that the runs did not log are rebuilt with the code that produced the runs and checked against the prompt hashes and the first-round cache that the runs recorded. Each judge call is sent three times with the same messages. The first-draft prompts of pass@$k$ equal those of our loop, and its selection uses our judge prompts for \px{gnd}, \px{sp}, and \px{cov}. Each box holds the messages of one prompt in order, with the role of each message. Lines longer than the box wrap inside it, and continued lines are indented.

\subsection{Our Loop on Lean}
\begin{promptbox}{First draft \textmd{(example: ProofNet problem \texttt{exercise\_4\_4})}}
\noindent{\small\bfseries SYSTEM:}\par
\begin{promptlines}
\pl{0}{You are an expert in formal language Lean4.\allowbreak{}}
\pl{0}{You will be given a mathematical statement written in natural language and LaTeX symbols.\allowbreak{}}
\pl{0}{Your task is to provide the formal code of the given natural language mathematical statement in Lean4 with the following instructions:\allowbreak{}}
\pl{0}{1.\allowbreak{} You should give the formal code directly without any additional explanation.\allowbreak{} If the given mathematical statement is a theorem or lemma,\allowbreak{} omit the proof and use the default \textquotesingle{}sorry\textquotesingle{} mode in the formal code.\allowbreak{}2.\allowbreak{} In case that you need to import any necessary preambles,\allowbreak{} you should not import any fake (\allowbreak{}non-{}exist) preambles.\allowbreak{}}
\pl{0}{3.\allowbreak{} You should wrap the formal code in a way illustrated as the following:\allowbreak{}}
\pl{0}{\char0{}\char0{}\char0{}lean}
\pl{0}{Your Formal Code}
\pl{0}{\char0{}\char0{}\char0{}}
\pl{0}{4.\allowbreak{} Lean 4 /\allowbreak{} Mathlib notation constraint — for \char0{}Finset\char0{} summation and product notation,\allowbreak{} use the MEMBERSHIP form with \char0{}∈\char0{} (\allowbreak{}Unicode U+2208),\allowbreak{} NOT the legacy \char0{}in\char0{} keyword form.\allowbreak{} Mathlib\textquotesingle{}s current BigOperators notation requires:\allowbreak{}}
\pl{3}{✓ \char0{}∑ k ∈ Finset.\allowbreak{}range 12,\allowbreak{} f k\char0{}\ \ \ \ \ -{}-{} correct}
\pl{3}{✗ \char0{}∑ k in Finset.\allowbreak{}range 12,\allowbreak{} f k\char0{}\ \ \ \ -{}-{} INVALID,\allowbreak{} fails to parse with \char0{}unexpected token \textquotesingle{}in\textquotesingle{}\char0{}}
\pl{0}{Same rule applies to \char0{}∏\char0{},\allowbreak{} \char0{}⋃\char0{},\allowbreak{} \char0{}⋂\char0{},\allowbreak{} and any other indexed-{}set binder:\allowbreak{} use \char0{}∈\char0{},\allowbreak{} not \char0{}in\char0{}.\allowbreak{}}
\pl{0}{Strictly follow the instructions that have been claimed.\allowbreak{}}
\pe
\end{promptlines}

\noindent{\small\bfseries USER:}\par
\begin{promptlines}
\pl{0}{Natural language statement:\allowbreak{} Prove that there is no rational number whose square is \$12\$.\allowbreak{}}
\pl{0}{Give me the Lean4 formal code of the statement:\allowbreak{}}
\pe
\end{promptlines}

\noindent{\small\bfseries ASSISTANT:}\par
\begin{promptlines}
\pl{0}{\char0{}\char0{}\char0{}lean}
\pl{0}{import Mathlib}
\pe
\pl{0}{open Topology Filter Real Complex TopologicalSpace Finset}
\pl{0}{open scoped BigOperators}
\pe
\pl{0}{theorem exercise\_\allowbreak{}1\_\allowbreak{}2 :\allowbreak{} ¬ ∃ (\allowbreak{}x :\allowbreak{} ℚ),\allowbreak{} (\allowbreak{} x \textasciicircum{} 2 =\allowbreak{} 12 ) :\allowbreak{}=\allowbreak{} sorry}
\pl{0}{\char0{}\char0{}\char0{}}
\pe
\end{promptlines}

\noindent{\small\bfseries USER:}\par
\begin{promptlines}
\pl{0}{Natural language statement:\allowbreak{} For all odd \$n\$ show that \$8 \textbackslash{}mid n\textasciicircum{}\{\allowbreak{}2\}-{}1\$.\allowbreak{}}
\pl{0}{Give me the Lean4 formal code of the statement:\allowbreak{}}
\pe
\end{promptlines}

\noindent{\small\bfseries ASSISTANT:}\par
\begin{promptlines}
\pl{0}{\char0{}\char0{}\char0{}lean}
\pl{0}{import Mathlib}
\pe
\pl{0}{open Real}
\pl{0}{open scoped BigOperators}
\pe
\pl{0}{theorem exercise\_\allowbreak{}1\_\allowbreak{}27 \{\allowbreak{}n :\allowbreak{} ℕ\} (\allowbreak{}hn :\allowbreak{} Odd n) :\allowbreak{} 8 ∣ (\allowbreak{}n\textasciicircum{}2 -{} 1) :\allowbreak{}=\allowbreak{} sorry}
\pl{0}{\char0{}\char0{}\char0{}}
\pe
\end{promptlines}

\noindent{\small\bfseries USER:}\par
\begin{promptlines}
\pl{0}{Natural language statement:\allowbreak{} Prove that \$x\textasciicircum{}3 -{} 3x -{} 1\$ is irreducible over \$\textbackslash{}mathbb\{\allowbreak{}Q\}\$.\allowbreak{}}
\pl{0}{Give me the Lean4 formal code of the statement:\allowbreak{}}
\pe
\end{promptlines}

\noindent{\small\bfseries ASSISTANT:}\par
\begin{promptlines}
\pl{0}{\char0{}\char0{}\char0{}lean}
\pl{0}{import Mathlib}
\pe
\pl{0}{open Fintype Set Real Ideal Polynomial}
\pl{0}{open scoped BigOperators}
\pe
\pl{0}{theorem exercise\_\allowbreak{}5\_\allowbreak{}5\_\allowbreak{}2 :\allowbreak{} Irreducible (\allowbreak{}X\textasciicircum{}3 -{} 3*X -{} 1 :\allowbreak{} Polynomial ℚ) :\allowbreak{}=\allowbreak{} sorry}
\pl{0}{\char0{}\char0{}\char0{}}
\pe
\end{promptlines}

\noindent{\small\bfseries USER:}\par
\begin{promptlines}
\pl{0}{Natural language statement:\allowbreak{} Consider a prime \$p\$ of the form \$4 t+1\$.\allowbreak{} Show that \$a\$ is a primitive root modulo \$p\$ iff \$-{}a\$ is a primitive root modulo \$p\$.\allowbreak{}}
\pl{0}{Give me the Lean4 formal code of the statement:\allowbreak{}}
\pe
\end{promptlines}

\end{promptbox}

\begin{promptbox}{Refinement (fourth refinement step) \textmd{(example: ProofNet problem \texttt{exercise\_4\_4})}}
\noindent{\small\bfseries SYSTEM:}\par
\begin{promptlines}
\pl{0}{You are an expert in formal language Lean4.\allowbreak{}}
\pl{0}{You will be given a mathematical statement written in natural language and LaTeX symbols.\allowbreak{}}
\pl{0}{You will also be given a formal code which attempted to describe the given mathematical statement in Lean4.\allowbreak{}}
\pl{0}{Your task is to refine the given formal code to make it syntactically correct while maintaining the alignment with the given natural language mathematical statement.\allowbreak{}}
\pl{0}{Here are some instructions for your task:\allowbreak{}}
\pl{0}{1.\allowbreak{} Before the code,\allowbreak{} write one short paragraph that starts with \char0{}DIAGNOSIS:\allowbreak{}\char0{} (\allowbreak{}at most about 100 words,\allowbreak{} no code blocks):\allowbreak{} for each reported problem,\allowbreak{} say its root cause and the exact name,\allowbreak{} notation or change you will use to fix it.\allowbreak{} Then give the formal code.\allowbreak{} If the given mathematical statement is a theorem or lemma,\allowbreak{} omit the proof and use the default \textquotesingle{}sorry\textquotesingle{} mode in the formal code.\allowbreak{}}
\pl{0}{2.\allowbreak{} In case that you need to import any necessary preambles,\allowbreak{} you should not import any fake (\allowbreak{}non-{}exist) preambles.\allowbreak{}}
\pl{0}{3.\allowbreak{} You should wrap the formal code in a way illustrated as the following:\allowbreak{}}
\pl{0}{\char0{}\char0{}\char0{}lean}
\pl{0}{Your Formal Code}
\pl{0}{\char0{}\char0{}\char0{}}
\pl{0}{4.\allowbreak{} Lean 4 /\allowbreak{} Mathlib notation constraint — for \char0{}Finset\char0{} summation and product notation,\allowbreak{} use the MEMBERSHIP form with \char0{}∈\char0{} (\allowbreak{}Unicode U+2208),\allowbreak{} NOT the legacy \char0{}in\char0{} keyword form.\allowbreak{} Mathlib\textquotesingle{}s current BigOperators notation requires:\allowbreak{}}
\pl{3}{✓ \char0{}∑ k ∈ Finset.\allowbreak{}range 12,\allowbreak{} f k\char0{}\ \ \ \ \ -{}-{} correct}
\pl{3}{✗ \char0{}∑ k in Finset.\allowbreak{}range 12,\allowbreak{} f k\char0{}\ \ \ \ -{}-{} INVALID,\allowbreak{} fails to parse with \char0{}unexpected token \textquotesingle{}in\textquotesingle{}\char0{}}
\pl{0}{Same rule applies to \char0{}∏\char0{},\allowbreak{} \char0{}⋃\char0{},\allowbreak{} \char0{}⋂\char0{},\allowbreak{} and any other indexed-{}set binder:\allowbreak{} use \char0{}∈\char0{},\allowbreak{} not \char0{}in\char0{}.\allowbreak{} If the previous code uses the correct \char0{}∈\char0{} form,\allowbreak{} PRESERVE it — do not rewrite \char0{}∈\char0{} to \char0{}in\char0{}.\allowbreak{}}
\pl{0}{Strictly follow the instructions that have been claimed.\allowbreak{}}
\pe
\pl{0}{HOW TO USE THE FEEDBACK (\allowbreak{}these override your own judgement):\allowbreak{}}
\pl{0}{-{} Every violation listed below must be addressed.\allowbreak{} Do not fix one and leave the rest.\allowbreak{}}
\pl{0}{-{} Handing back the same code unchanged counts as no answer.\allowbreak{} If you believe the code is already correct,\allowbreak{} change what the feedback names anyway,\allowbreak{} or state the fix in code.\allowbreak{}}
\pl{0}{-{} When the feedback and your own reading disagree,\allowbreak{} follow the feedback.\allowbreak{} The prover errors are hard constraints and come first;\allowbreak{} the judge findings rank above your own opinion.\allowbreak{}}
\pe
\end{promptlines}

\noindent{\small\bfseries USER:}\par
\begin{promptlines}
\pl{0}{FOCUS:\allowbreak{} the strongest violation is \char0{}ξ\_\allowbreak{}sp\char0{} inside \char0{}theorem:\allowbreak{}exercise\_\allowbreak{}9\_\allowbreak{}18\char0{}.\allowbreak{} Repair that first;\allowbreak{} leave the rest of the file unchanged unless fixing it requires a change elsewhere.\allowbreak{}}
\pe
\pl{0}{NOTE ON YOUR LAST ATTEMPT:\allowbreak{} your previous edit,\allowbreak{} made to fix \char0{}ξ\_\allowbreak{}sp\char0{} inside \char0{}theorem:\allowbreak{}exercise\_\allowbreak{}9\_\allowbreak{}18\char0{} (\allowbreak{}the FOCUS in this message),\allowbreak{} made the code worse,\allowbreak{} so it was discarded and you are back on the earlier version.\allowbreak{} It broke ξ\_\allowbreak{}bnd.\allowbreak{} Keep what worked and fix only what the feedback names.\allowbreak{}}
\pe
\pl{0}{The discarded version got these prover errors:\allowbreak{}}
\pe
\pl{0}{[\allowbreak{}ξ\_\allowbreak{}bnd] binding integrity (\allowbreak{}free variables /\allowbreak{} undefined references)}
\pl{3}{-{} at module theorem:\allowbreak{}exercise\_\allowbreak{}9\_\allowbreak{}18 line 6,\allowbreak{} column 5:\allowbreak{} Unknown constant \char0{}ZMod.\allowbreak{}IsPrimitiveRoot\char0{}}
\pl{3}{-{} at module theorem:\allowbreak{}exercise\_\allowbreak{}9\_\allowbreak{}18 line 6,\allowbreak{} column 32:\allowbreak{} Unknown constant \char0{}ZMod.\allowbreak{}IsPrimitiveRoot\char0{}}
\pl{3}{-{} repair hint:\allowbreak{} Repair unknown symbol names,\allowbreak{} proof steps,\allowbreak{} binding links,\allowbreak{} or unresolved goals.\allowbreak{}}
\pe
\pl{0}{PROVER ERROR LOCATIONS — these lines (\allowbreak{}\textgreater{}\textgreater{}) in the previous submission triggered the diagnostic.\allowbreak{} Inspect them closely:\allowbreak{}}
\pl{0}{\char0{}\char0{}\char0{}}
\pl{3}{L5:\allowbreak{} theorem exercise\_\allowbreak{}9\_\allowbreak{}18 \{\allowbreak{}p :\allowbreak{} ℕ\} (\allowbreak{}hp :\allowbreak{} p.\allowbreak{}Prime) (\allowbreak{}h :\allowbreak{} p \% 4 =\allowbreak{} 1) (\allowbreak{}a :\allowbreak{} ZMod p) :\allowbreak{}}
\pl{0}{\textgreater{}\textgreater{} L6:\allowbreak{}\ \ \ \ \ (\allowbreak{}ZMod.\allowbreak{}IsPrimitiveRoot a) ↔ (\allowbreak{}ZMod.\allowbreak{}IsPrimitiveRoot (\allowbreak{}-{}a)) :\allowbreak{}=\allowbreak{} sorry}
\pl{12}{\textasciicircum{}\ \ \ \ \ \ \ \ \ \ \ \ \ \ \ \ \ \ \ \ \ \ \ \ \ \ \textasciicircum{}}
\pl{0}{\char0{}\char0{}\char0{}}
\pe
\pl{0}{MATHLIB NAMING REMINDERS — the previous attempt used an identifier that Lean cannot resolve.\allowbreak{} Consider these canonical forms before retrying:\allowbreak{}}
\pl{0}{-{} Lean cannot resolve \char0{}ZMod.\allowbreak{}IsPrimitiveRoot\char0{}.\allowbreak{} The closest names that exist in the current Mathlib:\allowbreak{} \char0{}IsPrimitiveRoot\char0{},\allowbreak{} \char0{}primitiveRoots\char0{},\allowbreak{} \char0{}Complex.\allowbreak{}isPrimitiveRoot\_\allowbreak{}iff\char0{},\allowbreak{} \char0{}Complex.\allowbreak{}isPrimitiveRoot\_\allowbreak{}exp\char0{},\allowbreak{} \char0{}IsPrimitiveRoot.\allowbreak{}isPrimitiveRoot\_\allowbreak{}iff\textquotesingle{}\char0{}.\allowbreak{} Use one of them only if it means what the problem states;\allowbreak{} otherwise express the notion with definitions that exist.\allowbreak{}}
\pe
\pl{0}{YOUR EARLIER DIAGNOSES ON THIS TASK (\allowbreak{}oldest first,\allowbreak{} most recent last).\allowbreak{} Keep following the same line of repair unless the new errors show it was wrong:\allowbreak{}}
\pl{0}{-{} iteration 2:\allowbreak{} The main issue is that \char0{}p -{} a\char0{} in \char0{}ℕ\char0{} is truncated subtraction and does not represent \char0{}-{}a\char0{} modulo \char0{}p\char0{}.\allowbreak{} The natural language statement says \textquotedbl{}−a is a primitive root modulo p\textquotedbl{},\allowbreak{} which means the additive inverse of \char0{}a\char0{} in the ring \char0{}ℤ/\allowbreak{}pℤ\char0{}.\allowbreak{} To faithfully model this,\allowbreak{} we should work in \char0{}ZMod p\char0{} (\allowbreak{}the ring of integers modulo \char0{}p\char0{}).\allowbreak{} Also,\allowbreak{} the condition \char0{}p =\allowbreak{} 4t+1\char0{} is given as \char0{}p \% 4 =\allowbreak{} 1\char0{},\allowbreak{} which is equivalent and acceptable.\allowbreak{} I will change the theorem to use \char0{}a :\allowbreak{} ZMod p\char0{} and \char0{}IsPrimitiveRoot\char0{} in \char0{}ZMod p\char0{},\allowbreak{} and replace \char0{}(\allowbreak{}p -{} a)\char0{} with \char0{}(\allowbreak{}-{}a)\char0{} (\allowbreak{}the additive inverse in \char0{}ZMod p\char0{}).\allowbreak{} The \char0{}hp\char0{} prime condition and \char0{}h :\allowbreak{} p \% 4 =\allowbreak{} 1\char0{} remain.\allowbreak{}}
\pl{0}{-{} iteration 3:\allowbreak{} The root cause is that \char0{}IsPrimitiveRoot\char0{} in \char0{}Nat\char0{} expects two arguments:\allowbreak{} \char0{}IsPrimitiveRoot a p\char0{} where \char0{}a\char0{} and \char0{}p\char0{} are \char0{}ℕ\char0{}.\allowbreak{} The previous attempt used \char0{}IsPrimitiveRoot a\char0{} with only one argument,\allowbreak{} which is a type mismatch.\allowbreak{} The natural language statement says \textquotedbl{}a is a primitive root modulo p\textquotedbl{},\allowbreak{} which in Mathlib is expressed as \char0{}IsPrimitiveRoot a p\char0{} (\allowbreak{}with \char0{}a\char0{} and \char0{}p\char0{} as \char0{}ℕ\char0{}).\allowbreak{} The statement also says \textquotedbl{}−a is a primitive root modulo p\textquotedbl{},\allowbreak{} which in \char0{}ℕ\char0{} is not directly expressible.\allowbreak{} However,\allowbreak{} the standard Mathlib convention for primitive roots modulo a prime uses \char0{}Nat\char0{} with \char0{}IsPrimitiveRoot a p\char0{}.\allowbreak{} The correct formalization of \char0{}-{}a\char0{} modulo \char0{}p\char0{} in \char0{}ℕ\char0{} is \char0{}(\allowbreak{}p -{} a) \% p\char0{} or simply \char0{}p -{} a\char0{} when \char0{}a ≠ 0\char0{}.\allowbreak{} But the natural language says \textquotedbl{}−a\textquotedbl{},\allowbreak{} meaning the additive inverse modulo \char0{}p\char0{}.\allowbreak{} To faithfully represent this,\allowbreak{} we should work in \char0{}ZMod p\char0{} where \char0{}-{}a\char0{} is well-{}defined.\allowbreak{} The earlier attempt with \char0{}ZMod p\char0{} failed because \char0{}IsPrimitiveRoot\char0{} in \char0{}ZMod p\char0{} expects a single argument \char0{}IsPrimitiveRoot a\char0{} (\allowbreak{}type \char0{}Prop\char0{}),\allowbreak{} not two arguments.\allowbreak{} The error was that \char0{}IsPrimitiveRoot a\char0{} was being interpreted as a function \char0{}ℕ → Prop\char0{} because \char0{}a\char0{} is of type \char0{}ZMod p\char0{} but \char0{}IsPrimitiveRoot\char0{} expects a \char0{}ℕ\char0{} argument.\allowbreak{} Actually,\allowbreak{} \char0{}IsPrimitiveRoot\char0{} }
\pe
\pl{0}{ALREADY-{}TRIED ATTEMPTS ON THIS TASK (\allowbreak{}oldest first,\allowbreak{} most recent last).\allowbreak{} Every one of these was submitted to the Lean kernel and REJECTED on the axes shown.\allowbreak{} Do not resubmit any of them,\allowbreak{} and do not re-{}apply an edit you have already tried — if an edit is listed here,\allowbreak{} it did not work:\allowbreak{}}
\pe
\pl{0}{attempt at iteration 2 — rejected on [\allowbreak{}ξ\_\allowbreak{}typ]:\allowbreak{}}
\pl{0}{\char0{}\char0{}\char0{}lean}
\pl{0}{import Mathlib}
\pe
\pl{0}{open Nat}
\pe
\pl{0}{theorem exercise\_\allowbreak{}9\_\allowbreak{}18 \{\allowbreak{}p :\allowbreak{} ℕ\} (\allowbreak{}hp :\allowbreak{} p.\allowbreak{}Prime) (\allowbreak{}h :\allowbreak{} p \% 4 =\allowbreak{} 1) (\allowbreak{}a :\allowbreak{} ZMod p) :\allowbreak{}}
\pl{4}{(\allowbreak{}IsPrimitiveRoot a) ↔ (\allowbreak{}IsPrimitiveRoot (\allowbreak{}-{}a)) :\allowbreak{}=\allowbreak{} sorry}
\pl{0}{\char0{}\char0{}\char0{}}
\pe
\pl{0}{attempt at iteration 3 — rejected on [\allowbreak{}ξ\_\allowbreak{}bnd]:\allowbreak{}}
\pl{0}{\char0{}\char0{}\char0{}lean}
\pl{0}{import Mathlib}
\pe
\pl{0}{open Nat}
\pe
\pl{0}{theorem exercise\_\allowbreak{}9\_\allowbreak{}18 \{\allowbreak{}p :\allowbreak{} ℕ\} (\allowbreak{}hp :\allowbreak{} p.\allowbreak{}Prime) (\allowbreak{}h :\allowbreak{} p \% 4 =\allowbreak{} 1) (\allowbreak{}a :\allowbreak{} ZMod p) :\allowbreak{}}
\pl{4}{(\allowbreak{}ZMod.\allowbreak{}IsPrimitiveRoot a) ↔ (\allowbreak{}ZMod.\allowbreak{}IsPrimitiveRoot (\allowbreak{}-{}a)) :\allowbreak{}=\allowbreak{} sorry}
\pl{0}{\char0{}\char0{}\char0{}}
\pe
\pl{0}{Natural language statement:\allowbreak{} Consider a prime \$p\$ of the form \$4 t+1\$.\allowbreak{} Show that \$a\$ is a primitive root modulo \$p\$ iff \$-{}a\$ is a primitive root modulo \$p\$.\allowbreak{}}
\pe
\pl{0}{Here is the complete .\allowbreak{}lean file currently submitted to the prover:\allowbreak{}}
\pl{0}{\char0{}\char0{}\char0{}lean}
\pl{0}{import Mathlib}
\pe
\pl{0}{open Nat}
\pe
\pl{0}{theorem exercise\_\allowbreak{}9\_\allowbreak{}18 \{\allowbreak{}p :\allowbreak{} ℕ\} (\allowbreak{}hp :\allowbreak{} p.\allowbreak{}Prime) (\allowbreak{}h :\allowbreak{} p \% 4 =\allowbreak{} 1) (\allowbreak{}a :\allowbreak{} ℕ) :\allowbreak{}}
\pl{4}{(\allowbreak{}IsPrimitiveRoot a p) ↔ (\allowbreak{}IsPrimitiveRoot (\allowbreak{}p -{} a) p) :\allowbreak{}=\allowbreak{} sorry}
\pe
\pl{0}{\char0{}\char0{}\char0{}}
\pe
\pl{0}{The provided code has the following errors detected by an 8-{}proxy critique stack (\allowbreak{}parse → type-{}check → bind → ground → document → global → semantic-{}preservation → coverage):\allowbreak{}}
\pe
\pl{0}{[\allowbreak{}ξ\_\allowbreak{}sp] semantic preservation (\allowbreak{}formal → NL faithfulness)}
\pl{3}{-{} theorem:\allowbreak{}exercise\_\allowbreak{}9\_\allowbreak{}18\#1:\allowbreak{} The Lean statement uses \char0{}p -{} a\char0{} in \char0{}ℕ\char0{},\allowbreak{} which is truncated subtraction and does not express \char0{}-{}a\char0{} modulo \char0{}p\char0{};\allowbreak{} it also quantifies over \char0{}a :\allowbreak{} ℕ\char0{} without the modular/\allowbreak{}negative interpretation in the informal claim.\allowbreak{};\allowbreak{} The formal statement uses \char0{}p -{} a\char0{} in \char0{}ℕ\char0{},\allowbreak{} which is truncated subtraction and does not express \char0{}-{}a\char0{} modulo \char0{}p\char0{};\allowbreak{} it also omits the existential form \char0{}p =\allowbreak{} 4t+1\char0{} in favor }
\pl{3}{-{} failing modules:\allowbreak{} theorem:\allowbreak{}exercise\_\allowbreak{}9\_\allowbreak{}18}
\pl{3}{-{} repair hint:\allowbreak{} A formal statement says something its source sentence does not say.\allowbreak{} Rewrite that statement so it says exactly what the source sentence says:\allowbreak{} keep every qualifier,\allowbreak{} do not strengthen,\allowbreak{} weaken,\allowbreak{} generalise or reverse it.\allowbreak{} If the failing statement has no source sentence at all (\allowbreak{}a bridge rule you added yourself),\allowbreak{} delete it instead of rewording it.\allowbreak{}}
\pe
\pl{0}{[\allowbreak{}ξ\_\allowbreak{}gnd] reference grounding to NL/\allowbreak{}Γ symbols}
\pl{3}{-{} theorem:\allowbreak{}exercise\_\allowbreak{}9\_\allowbreak{}18:\allowbreak{} The formal statement uses \char0{}p -{} a\char0{} rather than the informal \char0{}-{}a\char0{},\allowbreak{} introducing a different symbol/\allowbreak{}expression on the formal side that is not sourced by the given NL statement.\allowbreak{};\allowbreak{} The formal statement introduces the symbol \char0{}p -{} a\char0{} (\allowbreak{}using subtraction modulo the natural numbers) rather than the informal \char0{}-{}a\char0{},\allowbreak{} so the module contains a formal-{}side expression not directly sourced by th}
\pl{3}{-{} failing modules:\allowbreak{} theorem:\allowbreak{}exercise\_\allowbreak{}9\_\allowbreak{}18}
\pl{3}{-{} repair hint:\allowbreak{} A symbol on the formal side has no source in the text.\allowbreak{} Replace it with a symbol named after words that actually occur in the text,\allowbreak{} or remove it.\allowbreak{}}
\pe
\pl{0}{Refine the formal code to fix the issues above while staying faithful to the natural language statement.\allowbreak{} Output the COMPLETE refined .\allowbreak{}lean file.\allowbreak{}}
\pe
\pl{0}{STRICT RULES for the imports/\allowbreak{}opens header:\allowbreak{}}
\pl{0}{-{} Copy every \char0{}import .\allowbreak{}.\allowbreak{}.\allowbreak{}\char0{} and \char0{}open .\allowbreak{}.\allowbreak{}.\allowbreak{}\char0{} line from the previous file VERBATIM,\allowbreak{} in the\ \ \ same order,\allowbreak{} unless the feedback explicitly identifies a header problem.\allowbreak{} Do NOT\ \ \ change \char0{}import Mathlib\char0{} to fine-{}grained paths like\ \ \ \char0{}Mathlib.\allowbreak{}Algebra.\allowbreak{}BigOperators.\allowbreak{}Basic\char0{};\allowbreak{} those are legacy paths and DO NOT EXIST in\ \ \ current Mathlib (\allowbreak{}the build will fail with \char0{}object file .\allowbreak{}.\allowbreak{}.\allowbreak{} does not exist\char0{}).\allowbreak{}}
\pl{0}{-{} All \char0{}import\char0{} lines must appear before any \char0{}open\char0{} or \char0{}theorem\char0{}.\allowbreak{}}
\pe
\pe
\end{promptlines}

\end{promptbox}

\begin{promptbox}{Claim extraction for \px{cov} \textmd{(example: ProofNet problem \texttt{exercise\_4\_4})}}
\noindent{\small\bfseries SYSTEM:}\par
\begin{promptlines}
\pl{0}{You are decomposing an informal natural-{}language mathematical statement into atomic claim units.\allowbreak{} A claim unit is the smallest self-{}contained assertion or condition that a faithful formalization has to represent in some form.\allowbreak{}}
\pe
\pl{0}{EXTRACT ONLY WHAT BEARS ON TRUTH CONDITIONS:\allowbreak{} each hypothesis,\allowbreak{} each conclusion,\allowbreak{} each constraint that changes what makes the statement true.\allowbreak{} Do not extract narrative colour,\allowbreak{} motivation,\allowbreak{} or restatements of the same condition in other words — one condition is one claim unit,\allowbreak{} however many sentences describe it.\allowbreak{}}
\pe
\pl{0}{Output format — one claim unit per line,\allowbreak{} exactly:\allowbreak{}}
\pl{0}{CLAIM:\allowbreak{} \textless{}one short atomic NL clause\textgreater{}}
\pl{0}{Output nothing else (\allowbreak{}no numbering,\allowbreak{} no commentary).\allowbreak{}}
\end{promptlines}

\noindent{\small\bfseries USER:}\par
\begin{promptlines}
\pl{0}{Informal:\allowbreak{}}
\pl{0}{Consider a prime \$p\$ of the form \$4 t+1\$.\allowbreak{} Show that \$a\$ is a primitive root modulo \$p\$ iff \$-{}a\$ is a primitive root modulo \$p\$.\allowbreak{}}
\end{promptlines}

\end{promptbox}

\begin{promptbox}{Judge \px{sp} \textmd{(example: ProofNet problem \texttt{exercise\_4\_4})}}
\noindent{\small\bfseries SYSTEM:}\par
\begin{promptlines}
\pl{0}{You are evaluating semantic preservation (\allowbreak{}xi\_\allowbreak{}sp) per explicit statement of a Lean 4 formalization.\allowbreak{}}
\pe
\pl{0}{The audit unit is the STATEMENT:\allowbreak{} each top-{}level declaration\textquotesingle{}s proposition (\allowbreak{}a theorem /\allowbreak{} lemma /\allowbreak{} def statement) is one unit.\allowbreak{} Read each statement on its own,\allowbreak{} as a proposition about the objects it names,\allowbreak{} and ask whether it says what the informal NL says:\allowbreak{} which relations it asserts,\allowbreak{} of which objects,\allowbreak{} in which order and direction,\allowbreak{} with which quantifiers,\allowbreak{} connectives and numeric bounds.\allowbreak{} A variable name is only a placeholder for an object.\allowbreak{} Whether helper declarations and binders elsewhere in the module are wired to this statement is not part of its content.\allowbreak{}}
\pe
\pl{0}{READ EACH STATEMENT AS LEAN READS IT.\allowbreak{} A Lean term means what its definition says,\allowbreak{} which can differ from the usual reading of similar notation:\allowbreak{} \char0{}ℕ\char0{} contains 0,\allowbreak{} subtraction on \char0{}ℕ\char0{} stops at 0,\allowbreak{} division on \char0{}ℕ\char0{} and \char0{}ℤ\char0{} rounds down,\allowbreak{} division by 0 gives 0,\allowbreak{} the integral of a function that is not integrable on the set is 0 (\allowbreak{}so an improper integral that converges only conditionally is not the Lean integral),\allowbreak{} and \char0{}≤\char0{} between structures follows Mathlib\textquotesingle{}s order (\allowbreak{}between topologies,\allowbreak{} \char0{}τ ≤ σ\char0{} means τ has at least the open sets of σ).\allowbreak{} When such a difference changes what the statement says for values it allows,\allowbreak{} for example a subtraction on \char0{}ℕ\char0{} that can go below 0,\allowbreak{} or a quantity the NL lets be negative or fractional declared in \char0{}ℕ\char0{} or \char0{}ℤ\char0{},\allowbreak{} the statement fails.\allowbreak{} The user message lists the definitions of the Mathlib names the statement uses;\allowbreak{} read each of those names by its definition there,\allowbreak{} not by how its name sounds.\allowbreak{} When a name has several definitions listed,\allowbreak{} Lean uses the one whose arguments fit.\allowbreak{}}
\pe
\pl{0}{You will see the full .\allowbreak{}lean file as context AND a list of statement identifiers with their line numbers and text.\allowbreak{}}
\pe
\pl{0}{CHECK EACH STATEMENT FOR (\allowbreak{}this is the whole list;\allowbreak{} work through it in order):\allowbreak{}}
\pl{0}{1.\allowbreak{} argument order — are the arguments of every relation and function in the order the NL gives them?}
\pl{0}{2.\allowbreak{} direction of relations — \char0{}\textless{}\char0{} vs \char0{}\textgreater{}\char0{},\allowbreak{} \char0{}≤\char0{} vs \char0{}\textless{}\char0{},\allowbreak{} divides vs is-{}divided-{}by,\allowbreak{} subset vs superset,\allowbreak{} and the direction of any implication;\allowbreak{}}
\pl{0}{3.\allowbreak{} negation and quantifier scope — what exactly does \char0{}¬\char0{} /\allowbreak{} \char0{}∀\char0{} /\allowbreak{} \char0{}∃\char0{} range over,\allowbreak{} and is an alternation \char0{}∀ x,\allowbreak{} ∃ y\char0{} kept in that order? A property the NL assumes of all elements,\allowbreak{} stated only for the particular elements named in the goal,\allowbreak{} has a narrower scope and fails;\allowbreak{}}
\pl{0}{4.\allowbreak{} connectives — \char0{}∧\char0{} vs \char0{}∨\char0{},\allowbreak{} implication vs biconditional,\allowbreak{} and how they are parenthesised;\allowbreak{}}
\pl{0}{5.\allowbreak{} numeric bounds — the constants themselves,\allowbreak{} and open vs closed bounds;\allowbreak{} for \textquotedbl{}the smallest (\allowbreak{}largest) value is v\textquotedbl{},\allowbreak{} a goal stating that every admissible value is at least (\allowbreak{}at most) v renders the claim,\allowbreak{} with or without the statement that v is attained;\allowbreak{}}
\pl{0}{6.\allowbreak{} hypothesis vs conclusion — is what the NL assumes stated as a hypothesis,\allowbreak{} and what the NL concludes stated as the goal,\allowbreak{} rather than swapped;\allowbreak{}}
\pl{0}{7.\allowbreak{} added or stronger conditions — every hypothesis,\allowbreak{} every condition on a bound variable and every condition inside the goal must be stated in the NL or hold in every case the NL covers.\allowbreak{} A condition fails when it rules out a case the NL covers:\allowbreak{} a premise the NL does not give,\allowbreak{} a property stronger than the one the NL gives (\allowbreak{}for \textquotedbl{}f is differentiable at 0\textquotedbl{},\allowbreak{} \char0{}Differentiable ℝ f\char0{} requires differentiability everywhere and fails),\allowbreak{} or a property required on a larger set than the NL says.\allowbreak{} A condition that holds in every case the NL covers passes even when the NL does not write it out:\allowbreak{} for \textquotedbl{}Let p be a prime\textquotedbl{},\allowbreak{} the hypothesis \char0{}2 ≤ p\char0{} passes.\allowbreak{} Declaring an object the NL talks about,\allowbreak{} with the type or type class that makes it the kind of object the NL says it is,\allowbreak{} passes.\allowbreak{}}
\pe
\pl{0}{JUDGE ONLY WHAT IS WRITTEN.\allowbreak{} A statement that is absent is not your concern (\allowbreak{}missing content is coverage,\allowbreak{} xi\_\allowbreak{}cov,\allowbreak{} not semantic preservation).\allowbreak{}}
\pl{0}{NO OUTSIDE KNOWLEDGE.\allowbreak{} Judge only against the informal NL text you are given.\allowbreak{} If the formalization relies on a bridging fact that the NL never states,\allowbreak{} that statement fails,\allowbreak{} however true the fact is mathematically.\allowbreak{} Mathlib definitions and standard Lean library facts are part of the language,\allowbreak{} not outside knowledge.\allowbreak{}}
\pl{0}{AN EMPTY FORMALIZATION FAILS:\allowbreak{} if the unit carries no proposition,\allowbreak{} it fails.\allowbreak{}}
\pl{0}{A DEFINITION IS COMPARED WITH THE TEXT THAT DEFINES ITS OBJECT.\allowbreak{} A \char0{}def\char0{} or \char0{}abbrev\char0{} unit defines an object:\allowbreak{} compare it only with the NL text that defines or describes that object,\allowbreak{} and quote that text on the SOURCE line.\allowbreak{} It passes when it defines the object as the NL does;\allowbreak{} the problem\textquotesingle{}s claim belongs to the theorem,\allowbreak{} not to the definition.\allowbreak{}}
\pl{0}{WORDING ALONE IS NEVER A FAILURE:\allowbreak{} synonymous predicates,\allowbreak{} equivalent rewritings,\allowbreak{} different variable /\allowbreak{} theorem names,\allowbreak{} LaTeX-{}vs-{}Lean notation and whitespace all pass.\allowbreak{}}
\pl{0}{OUT OF SCOPE:\allowbreak{} proof body completeness (\allowbreak{}\char0{}:\allowbreak{}=\allowbreak{} sorry\char0{} is acceptable),\allowbreak{} tactic correctness,\allowbreak{} mathematical truth.\allowbreak{}}
\pe
\pl{0}{Output format — for each statement,\allowbreak{} two lines,\allowbreak{} in this order:\allowbreak{}}
\pl{0}{SOURCE \textless{}id\textgreater{}:\allowbreak{} \textless{}verbatim quote from the informal NL that this statement renders,\allowbreak{} or NONE\textgreater{}}
\pl{0}{STATEMENT \textless{}id\textgreater{} \textbar{} VERDICT:\allowbreak{} pass\textbar{}fail\textbar{}unknown \textbar{} REASON:\allowbreak{} \textless{}one sentence\textgreater{}}
\pe
\pl{0}{The SOURCE line comes first and must be a quote,\allowbreak{} not a paraphrase.\allowbreak{} Write NONE when the NL contains no such passage — a unit with SOURCE NONE is recorded as a failure.\allowbreak{} Use VERDICT:\allowbreak{} unknown only when you have a real quote but cannot decide.\allowbreak{}}
\pl{0}{Output nothing other than these lines.\allowbreak{}}
\end{promptlines}

\noindent{\small\bfseries USER:}\par
\begin{promptlines}
\pl{0}{Informal:\allowbreak{}}
\pl{0}{Consider a prime \$p\$ of the form \$4 t+1\$.\allowbreak{} Show that \$a\$ is a primitive root modulo \$p\$ iff \$-{}a\$ is a primitive root modulo \$p\$.\allowbreak{}}
\pe
\pl{0}{You are evaluating a SINGLE Lean 4 statement (\allowbreak{}no surrounding file context beyond its module):\allowbreak{}}
\pl{0}{-{} STATEMENT theorem:\allowbreak{}exercise\_\allowbreak{}9\_\allowbreak{}18\#1\ \ (\allowbreak{}module theorem:\allowbreak{}exercise\_\allowbreak{}9\_\allowbreak{}18,\allowbreak{} line 5):\allowbreak{} theorem exercise\_\allowbreak{}9\_\allowbreak{}18 \{\allowbreak{}p :\allowbreak{} ℕ\} (\allowbreak{}hp :\allowbreak{} p.\allowbreak{}Prime) (\allowbreak{}h :\allowbreak{} p \% 4 =\allowbreak{} 1) (\allowbreak{}a :\allowbreak{} ℕ) :\allowbreak{} (\allowbreak{}IsPrimitiveRoot a p) ↔ (\allowbreak{}IsPrimitiveRoot (\allowbreak{}p -{} a) p) :\allowbreak{}=\allowbreak{} sorry}
\pl{0}{\char0{}\char0{}\char0{}}
\pl{0}{theorem exercise\_\allowbreak{}9\_\allowbreak{}18 \{\allowbreak{}p :\allowbreak{} ℕ\} (\allowbreak{}hp :\allowbreak{} p.\allowbreak{}Prime) (\allowbreak{}h :\allowbreak{} p \% 4 =\allowbreak{} 1) (\allowbreak{}a :\allowbreak{} ℕ) :\allowbreak{}}
\pl{4}{(\allowbreak{}IsPrimitiveRoot a p) ↔ (\allowbreak{}IsPrimitiveRoot (\allowbreak{}p -{} a) p) :\allowbreak{}=\allowbreak{} sorry}
\pl{0}{\char0{}\char0{}\char0{}}
\pe
\pl{0}{DEFINITIONS OF THE MATHLIB NAMES THIS STATEMENT USES — signature and docstring from the current Mathlib:\allowbreak{}}
\pl{0}{-{} \char0{}IsPrimitiveRoot\char0{} :\allowbreak{} \{\allowbreak{}M :\allowbreak{} Type u\_\allowbreak{}1\} → [\allowbreak{}CommMonoid M] → M → ℕ → Prop — An element \char0{}ζ\char0{} is a primitive \char0{}k\char0{}-{}th root of unity if \char0{}ζ \textasciicircum{} k =\allowbreak{} 1\char0{},\allowbreak{} and if \char0{}l\char0{} satisfies \char0{}ζ \textasciicircum{} l =\allowbreak{} 1\char0{} then \char0{}k ∣ l\char0{}.\allowbreak{}}
\pe
\pl{0}{Emit exactly ONE verdict line for this statement.\allowbreak{}}
\end{promptlines}

\end{promptbox}

\begin{promptbox}{Judge \px{cov} \textmd{(example: ProofNet problem \texttt{exercise\_4\_4})}}
\noindent{\small\bfseries SYSTEM:}\par
\begin{promptlines}
\pl{0}{You are evaluating coverage (\allowbreak{}xi\_\allowbreak{}cov) of a Lean 4 formalization.\allowbreak{}}
\pe
\pl{0}{For each claim unit extracted from the informal NL,\allowbreak{} decide ONE thing only:\allowbreak{} is there something in the formalization that recognisably corresponds to this claim (\allowbreak{}a hypothesis,\allowbreak{} conclusion,\allowbreak{} definition,\allowbreak{} or similar)?}
\pe
\pl{0}{PRESENT COUNTS AS COVERED,\allowbreak{} EVEN IF IT IS WRONG.\allowbreak{} A claim rendered with the wrong direction,\allowbreak{} the wrong bound or the wrong argument order is still covered — getting it wrong is semantic preservation (\allowbreak{}xi\_\allowbreak{}sp),\allowbreak{} not coverage.\allowbreak{} Synonymous predicates count as corresponding,\allowbreak{} and a claim split across several hypotheses or definitions counts as covered.\allowbreak{}}
\pl{0}{Only a claim with nothing corresponding to it in the formalization fails.\allowbreak{}}
\pl{0}{QUALIFIERS ARE PART OF THE CLAIM.\allowbreak{} When a claim qualifies an object (\allowbreak{}positive,\allowbreak{} nonzero,\allowbreak{} distinct,\allowbreak{} an integer,\allowbreak{} the smallest,\allowbreak{} the largest),\allowbreak{} the claim is covered only if the formalization carries that qualifier too,\allowbreak{} in a hypothesis,\allowbreak{} in a type (\allowbreak{}such as \char0{}ℕ+\char0{} for positive integers) or in the goal;\allowbreak{} \char0{}(\allowbreak{}n :\allowbreak{} ℕ)\char0{} alone does not carry \textquotedbl{}positive\textquotedbl{}.\allowbreak{} A claim that v is the smallest (\allowbreak{}largest) value with a property is covered by a statement that every value with the property is at least (\allowbreak{}at most) v,\allowbreak{} with or without the statement that v has the property;\allowbreak{} a statement that only gives v,\allowbreak{} or only says that v has the property,\allowbreak{} does not cover it.\allowbreak{}}
\pl{0}{NOTATION AND SETTING.\allowbreak{} A claim that only introduces notation or the setting (\allowbreak{}fixes a symbol,\allowbreak{} recalls the standard meaning of a term,\allowbreak{} or says what kind of object a variable is) is covered when the formalization uses that symbol,\allowbreak{} term or object with that meaning,\allowbreak{} or states what the definition amounts to;\allowbreak{} it needs no statement of its own.\allowbreak{} For example,\allowbreak{} \textquotedbl{}let n! denote 1·2·…·n\textquotedbl{} is covered by \char0{}Nat.\allowbreak{}factorial n\char0{},\allowbreak{} \textquotedbl{}a,\allowbreak{} b,\allowbreak{} c form an arithmetic progression\textquotedbl{} is covered by \char0{}b -{} a =\allowbreak{} c -{} b\char0{},\allowbreak{} and in \textquotedbl{}f(\allowbreak{}x) =\allowbreak{} x² + 1;\allowbreak{} find f(\allowbreak{}3)\textquotedbl{} the definition of f is covered by the goal \char0{}(\allowbreak{}3 :\allowbreak{} ℝ) \textasciicircum{} 2 + 1 =\allowbreak{} 10\char0{}.\allowbreak{} Quote the formal text that uses it on the FORMAL line.\allowbreak{}}
\pl{0}{STORY OF A WORD PROBLEM.\allowbreak{} A claim that tells the story of a word problem (\allowbreak{}what the objects are and what happens to them) is covered when the quantities and operations it contributes appear in the formalization:\allowbreak{} in \textquotedbl{}a tank holds 50 liters and 8 liters leak out;\allowbreak{} how much is left\textquotedbl{},\allowbreak{} the expression \char0{}50 -{} 8\char0{} covers both the tank and the leak.\allowbreak{}}
\pl{0}{Notation /\allowbreak{} variable naming differences never make a claim missing.\allowbreak{}}
\pl{0}{OUT OF SCOPE:\allowbreak{} proof body,\allowbreak{} mathematical truth.\allowbreak{}}
\pe
\pl{0}{Each claim line must also name one formal module from the module list given with the formalization.\allowbreak{} For a covered claim,\allowbreak{} MODULE is the module that represents it.\allowbreak{} For a missing claim,\allowbreak{} MODULE is the module where the missing content belongs.\allowbreak{}}
\pe
\pl{0}{Output format — for each claim unit,\allowbreak{} two lines,\allowbreak{} in this order:\allowbreak{}}
\pl{0}{FORMAL \#\textless{}index\textgreater{}:\allowbreak{} \textless{}verbatim quote from the formalization that corresponds to this claim,\allowbreak{} or NONE\textgreater{}}
\pl{0}{CLAIM \#\textless{}index\textgreater{} \textbar{} VERDICT:\allowbreak{} pass\textbar{}fail\textbar{}unknown \textbar{} MODULE:\allowbreak{} \textless{}module id\textgreater{} \textbar{} REASON:\allowbreak{} \textless{}one sentence\textgreater{}}
\pe
\pl{0}{The FORMAL line comes first and must be a quote from the .\allowbreak{}lean file,\allowbreak{} not a paraphrase.\allowbreak{} Write NONE when nothing corresponds — a claim with FORMAL NONE is recorded as not covered.\allowbreak{} Use VERDICT:\allowbreak{} unknown only when you have a real quote but cannot decide.\allowbreak{}}
\pl{0}{Output nothing else.\allowbreak{}}
\end{promptlines}

\noindent{\small\bfseries USER:}\par
\begin{promptlines}
\pl{0}{Informal:\allowbreak{}}
\pl{0}{Consider a prime \$p\$ of the form \$4 t+1\$.\allowbreak{} Show that \$a\$ is a primitive root modulo \$p\$ iff \$-{}a\$ is a primitive root modulo \$p\$.\allowbreak{}}
\pe
\pl{0}{Informal claim units to check:\allowbreak{}}
\pl{0}{CLAIM \#1:\allowbreak{} p is a prime of the form 4t+1}
\pl{0}{CLAIM \#2:\allowbreak{} a is a primitive root modulo p}
\pl{0}{CLAIM \#3:\allowbreak{} -{}a is a primitive root modulo p}
\pe
\pl{0}{Formal modules:\allowbreak{}}
\pl{0}{-{} MODULE theorem:\allowbreak{}exercise\_\allowbreak{}9\_\allowbreak{}18\ \ (\allowbreak{}lines 5-{}6,\allowbreak{} kind=\allowbreak{}theorem)}
\pe
\pl{0}{Lean 4 formalization (\allowbreak{}judge whether each claim is represented anywhere inside it):\allowbreak{}}
\pl{0}{\char0{}\char0{}\char0{}}
\pl{0}{import Mathlib}
\pe
\pl{0}{open Nat}
\pe
\pl{0}{theorem exercise\_\allowbreak{}9\_\allowbreak{}18 \{\allowbreak{}p :\allowbreak{} ℕ\} (\allowbreak{}hp :\allowbreak{} p.\allowbreak{}Prime) (\allowbreak{}h :\allowbreak{} p \% 4 =\allowbreak{} 1) (\allowbreak{}a :\allowbreak{} ℕ) :\allowbreak{}}
\pl{4}{(\allowbreak{}IsPrimitiveRoot a p) ↔ (\allowbreak{}IsPrimitiveRoot (\allowbreak{}p -{} a) p) :\allowbreak{}=\allowbreak{} sorry}
\pl{0}{\char0{}\char0{}\char0{}}
\end{promptlines}

\end{promptbox}

\begin{promptbox}{Judge \px{gnd} \textmd{(example: ProofNet problem \texttt{exercise\_4\_4})}}
\noindent{\small\bfseries SYSTEM:}\par
\begin{promptlines}
\pl{0}{You are evaluating grounding (\allowbreak{}xi\_\allowbreak{}gnd) per module of a Lean 4 formalization.\allowbreak{}}
\pe
\pl{0}{xi\_\allowbreak{}gnd asks ONE question at SYMBOL granularity:\allowbreak{} for every predicate,\allowbreak{} constant or symbol that appears on the formal side of a module,\allowbreak{} does it trace back to something introduced by the informal NL,\allowbreak{} or by the per-{}instance fact context? A formal-{}side concept invented out of nowhere -{}-{} mentioned in neither -{}-{} is a grounding failure.\allowbreak{}}
\pe
\pl{0}{THIS IS INDEPENDENT OF EVERY PROVER-{}SIDE CHECK.\allowbreak{} An out-{}of-{}nowhere predicate is perfectly well-{}typed,\allowbreak{} parses fine,\allowbreak{} and the prover can even close a proof using it.\allowbreak{} It is still a grounding violation.\allowbreak{} Conversely,\allowbreak{} a symbol that simply does not exist in the library is a binding failure,\allowbreak{} NOT your concern -{}-{} do not report it.\allowbreak{}}
\pe
\pl{0}{Aggregation is strict-{}AND at module level:\allowbreak{} a single unsourced symbol inside a module sets that module\textquotesingle{}s verdict to fail.\allowbreak{}}
\pe
\pl{0}{DO NOT CONFUSE WITH NEIGHBOURING AXES:\allowbreak{}}
\pl{0}{-{} xi\_\allowbreak{}sp is STATEMENT granularity (\allowbreak{}does the proposition match what nl says?).\allowbreak{} A symbol can be grounded but used in a misaligned statement -{}-{} that is xi\_\allowbreak{}sp\textquotesingle{}s problem,\allowbreak{} not yours.\allowbreak{}}
\pl{0}{-{} xi\_\allowbreak{}cov traces NL -{}\textgreater{} formal (\allowbreak{}is every NL claim represented?).\allowbreak{} You trace formal -{}\textgreater{} NL (\allowbreak{}does every formal symbol have a source?).\allowbreak{}}
\pe
\pl{0}{Ignore:\allowbreak{} variable naming,\allowbreak{} whitespace,\allowbreak{} notation choices,\allowbreak{} theorem names,\allowbreak{} and standard library machinery that the informal statement obviously presupposes (\allowbreak{}arithmetic,\allowbreak{} ordering,\allowbreak{} the ambient number types).\allowbreak{}}
\pl{0}{If uncertain -{}\textgreater{} verdict=\allowbreak{}pass.\allowbreak{}}
\pe
\pl{0}{Output format -{}-{} one line per module,\allowbreak{} exactly:\allowbreak{}}
\pl{0}{MODULE \textless{}id\textgreater{} \textbar{} VERDICT:\allowbreak{} pass\textbar{}fail \textbar{} REASON:\allowbreak{} \textless{}one sentence\textgreater{}}
\pe
\pl{0}{Output nothing other than these lines.\allowbreak{}}
\end{promptlines}

\noindent{\small\bfseries USER:}\par
\begin{promptlines}
\pl{0}{Informal:\allowbreak{}}
\pl{0}{Consider a prime \$p\$ of the form \$4 t+1\$.\allowbreak{} Show that \$a\$ is a primitive root modulo \$p\$ iff \$-{}a\$ is a primitive root modulo \$p\$.\allowbreak{}}
\pe
\pl{0}{You are evaluating a SINGLE Lean 4 module (\allowbreak{}no surrounding file context beyond its module):\allowbreak{}}
\pl{0}{-{} MODULE theorem:\allowbreak{}exercise\_\allowbreak{}9\_\allowbreak{}18\ \ (\allowbreak{}lines 5-{}6,\allowbreak{} kind=\allowbreak{}theorem)}
\pl{0}{\char0{}\char0{}\char0{}}
\pl{0}{theorem exercise\_\allowbreak{}9\_\allowbreak{}18 \{\allowbreak{}p :\allowbreak{} ℕ\} (\allowbreak{}hp :\allowbreak{} p.\allowbreak{}Prime) (\allowbreak{}h :\allowbreak{} p \% 4 =\allowbreak{} 1) (\allowbreak{}a :\allowbreak{} ℕ) :\allowbreak{}}
\pl{4}{(\allowbreak{}IsPrimitiveRoot a p) ↔ (\allowbreak{}IsPrimitiveRoot (\allowbreak{}p -{} a) p) :\allowbreak{}=\allowbreak{} sorry}
\pl{0}{\char0{}\char0{}\char0{}}
\pe
\pl{0}{Emit exactly ONE verdict line for this module.\allowbreak{}}
\end{promptlines}

\end{promptbox}

\begin{promptbox}{Judge \px{bnd} \textmd{(example: ProofNet problem \texttt{exercise\_4\_4})}}
\noindent{\small\bfseries SYSTEM:}\par
\begin{promptlines}
\pl{0}{You are evaluating inferential binding (\allowbreak{}xi\_\allowbreak{}bnd) per module of a Lean 4 formalization.\allowbreak{}}
\pe
\pl{0}{xi\_\allowbreak{}bnd asks ONE question at MODULE granularity:\allowbreak{} are the declarations,\allowbreak{} binders and hypotheses of the module wired to the goal the way the informal statement wires its objects? The object of your verdict is the wiring (\allowbreak{}where each object the goal uses is introduced,\allowbreak{} and whether the pieces that constrain it actually reach it),\allowbreak{} never whether a single proposition says the right thing.\allowbreak{}}
\pe
\pl{0}{Work through the module in this order:\allowbreak{}}
\pl{0}{1.\allowbreak{} TRACE THE OBJECTS.\allowbreak{} For each variable,\allowbreak{} function or set the goal mentions,\allowbreak{} find where the module introduces it:\allowbreak{} a binder,\allowbreak{} a \char0{}let\char0{},\allowbreak{} a \char0{}def\char0{} or \char0{}abbrev\char0{},\allowbreak{} or an earlier declaration.\allowbreak{} Library names count as introduced.\allowbreak{} The wiring is broken if the goal uses a name that nothing introduces and that would only be accepted as an auto-{}bound implicit.\allowbreak{}}
\pl{0}{2.\allowbreak{} TRACE THE CONSTRAINTS.\allowbreak{} Every hypothesis,\allowbreak{} \char0{}let\char0{} and helper declaration the module states must reach the goal through shared objects.\allowbreak{} The wiring is broken if a helper declaration or \char0{}let\char0{} is defined but the goal refers to a separately introduced object instead,\allowbreak{} so the helper never constrains the goal,\allowbreak{} or if a later declaration relies on an earlier one with its parameters instantiated by unrelated objects.\allowbreak{}}
\pe
\pl{0}{Example (\allowbreak{}not from the evaluated data).\allowbreak{} Informal:\allowbreak{} \textquotedbl{}Let h(\allowbreak{}t) =\allowbreak{} t\textasciicircum{}2 + 4.\allowbreak{} Show that h(\allowbreak{}5) =\allowbreak{} 29.\allowbreak{}\textquotedbl{}}
\pl{0}{-{} broken:\allowbreak{} \char0{}def h (\allowbreak{}t :\allowbreak{} ℕ) :\allowbreak{} ℕ :\allowbreak{}=\allowbreak{} t \textasciicircum{} 2 + 4\char0{} followed by \char0{}theorem t (\allowbreak{}g :\allowbreak{} ℕ → ℕ) :\allowbreak{} g 5 =\allowbreak{} 29 :\allowbreak{}=\allowbreak{} sorry\char0{}.\allowbreak{} The goal is about g,\allowbreak{} which nothing ties to h.\allowbreak{}}
\pl{0}{-{} connected:\allowbreak{} \char0{}theorem t :\allowbreak{} h 5 =\allowbreak{} 29 :\allowbreak{}=\allowbreak{} sorry\char0{}.\allowbreak{}}
\pe
\pl{0}{JUDGE ONLY THE WIRING.\allowbreak{} Whether a hypothesis or the goal states the right condition,\allowbreak{} whether an object has the right property,\allowbreak{} whether something is missing,\allowbreak{} and whether the theorem is true are judged on other axes.\allowbreak{} Do not judge the proof (\allowbreak{}\char0{}sorry\char0{} is acceptable).\allowbreak{} Names the prover cannot resolve are not your concern.\allowbreak{}}
\pe
\pl{0}{Output format,\allowbreak{} for each module two lines in this order:\allowbreak{}}
\pl{0}{WIRING \textless{}id\textgreater{}:\allowbreak{} \textless{}the object or declaration that is not wired,\allowbreak{} or ALL WIRED\textgreater{}}
\pl{0}{MODULE \textless{}id\textgreater{} \textbar{} VERDICT:\allowbreak{} pass\textbar{}fail\textbar{}unknown \textbar{} REASON:\allowbreak{} \textless{}one sentence\textgreater{}}
\pe
\pl{0}{The WIRING line comes first.\allowbreak{} If it names something that is not wired,\allowbreak{} the verdict is fail.\allowbreak{} Use VERDICT:\allowbreak{} unknown only when you cannot tell how an object is introduced.\allowbreak{}}
\pl{0}{Output nothing other than these lines.\allowbreak{}}
\end{promptlines}

\noindent{\small\bfseries USER:}\par
\begin{promptlines}
\pl{0}{Informal:\allowbreak{}}
\pl{0}{Consider a prime \$p\$ of the form \$4 t+1\$.\allowbreak{} Show that \$a\$ is a primitive root modulo \$p\$ iff \$-{}a\$ is a primitive root modulo \$p\$.\allowbreak{}}
\pe
\pl{0}{Full Lean 4 file (\allowbreak{}context):\allowbreak{}}
\pl{0}{\char0{}\char0{}\char0{}}
\pl{0}{import Mathlib}
\pe
\pl{0}{open Nat}
\pe
\pl{0}{theorem exercise\_\allowbreak{}9\_\allowbreak{}18 \{\allowbreak{}p :\allowbreak{} ℕ\} (\allowbreak{}hp :\allowbreak{} p.\allowbreak{}Prime) (\allowbreak{}h :\allowbreak{} p \% 4 =\allowbreak{} 1) (\allowbreak{}a :\allowbreak{} ℕ) :\allowbreak{}}
\pl{4}{(\allowbreak{}IsPrimitiveRoot a p) ↔ (\allowbreak{}IsPrimitiveRoot (\allowbreak{}p -{} a) p) :\allowbreak{}=\allowbreak{} sorry}
\pl{0}{\char0{}\char0{}\char0{}}
\pe
\pl{0}{Modules to evaluate (\allowbreak{}one verdict line per module):\allowbreak{}}
\pl{0}{-{} MODULE theorem:\allowbreak{}exercise\_\allowbreak{}9\_\allowbreak{}18\ \ (\allowbreak{}lines 5-{}6,\allowbreak{} kind=\allowbreak{}theorem)}
\end{promptlines}

\end{promptbox}

\subsection{Our Loop on Isabelle}
\begin{promptbox}{First draft \textmd{(example: e-SNLI problem 72)}}
\noindent{\small\bfseries SYSTEM:}\par
\begin{promptlines}
\pl{0}{You are an expert in formal language Isabelle/\allowbreak{}HOL (\allowbreak{}Isabelle 2024).\allowbreak{}}
\pl{0}{You will be given a natural-{}language entailment task:\allowbreak{} a premise,\allowbreak{} a hypothesis,\allowbreak{} and one or more explanations bridging them.\allowbreak{}}
\pl{0}{Your task is to produce a complete Isabelle 2024 theory file that formalizes the premise (\allowbreak{}as the \char0{}assumes\char0{} clause),\allowbreak{} the explanations (\allowbreak{}as \char0{}axiomatization\char0{} blocks),\allowbreak{} and the hypothesis (\allowbreak{}as the \char0{}shows\char0{} clause).\allowbreak{} The proof body MUST be exactly \char0{}sorry\char0{}.\allowbreak{} Proof completion is out of scope;\allowbreak{} we are validating the formalization only.\allowbreak{}}
\pe
\pl{0}{STRICT CONSTRAINTS (\allowbreak{}Isabelle 2024 source format):\allowbreak{}}
\pl{0}{1.\allowbreak{} Use ASCII Isabelle escapes for logical operators,\allowbreak{} NOT Unicode.\allowbreak{} Required substitutions:\allowbreak{}}
\pl{3}{∧\ \ →\ \ \textbackslash{}\textless{}and\textgreater{}}
\pl{3}{∨\ \ →\ \ \textbackslash{}\textless{}or\textgreater{}}
\pl{3}{¬\ \ →\ \ \textbackslash{}\textless{}not\textgreater{}}
\pl{3}{⟹\ \ →\ \ \textbackslash{}\textless{}Longrightarrow\textgreater{}}
\pl{3}{⟶\ \ →\ \ \textbackslash{}\textless{}longrightarrow\textgreater{}}
\pl{3}{⟷\ \ →\ \ \textbackslash{}\textless{}longleftrightarrow\textgreater{}}
\pl{3}{↔\ \ →\ \ \textbackslash{}\textless{}longleftrightarrow\textgreater{}}
\pl{3}{∀\ \ →\ \ \textbackslash{}\textless{}forall\textgreater{}}
\pl{3}{∃\ \ →\ \ \textbackslash{}\textless{}exists\textgreater{}}
\pl{3}{⋀\ \ →\ \ \textbackslash{}\textless{}And\textgreater{}}
\pl{3}{⇒\ \ →\ \ \textbackslash{}\textless{}Rightarrow\textgreater{}}
\pl{3}{→\ \ →\ \ \textbackslash{}\textless{}longrightarrow\textgreater{}}
\pl{0}{2.\allowbreak{} The proof body MUST be the single token \char0{}sorry\char0{} (\allowbreak{}no \char0{}proof -{} qed\char0{} block,\allowbreak{} no \char0{}by .\allowbreak{}.\allowbreak{}.\allowbreak{}\char0{} tactic,\allowbreak{} no \char0{}oops\char0{},\allowbreak{} no \char0{}admit\char0{}).\allowbreak{} The file must compile under Isabelle\textquotesingle{}s \char0{}quick\_\allowbreak{}and\_\allowbreak{}dirty\char0{} mode.\allowbreak{}}
\pl{0}{3.\allowbreak{} Structure the file as:\allowbreak{}}
\pl{3}{theory \textless{}theory\_\allowbreak{}name\textgreater{}}
\pl{3}{imports Main}
\pl{3}{begin}
\pe
\pl{3}{typedecl entity}
\pl{3}{typedecl event}
\pe
\pl{3}{consts}
\pl{5}{\textless{}predicate\_\allowbreak{}1\textgreater{} :\allowbreak{}:\allowbreak{} \textquotedbl{}\textless{}type signature\textgreater{}\textquotedbl{}}
\pl{5}{\textless{}predicate\_\allowbreak{}2\textgreater{} :\allowbreak{}:\allowbreak{} \textquotedbl{}\textless{}type signature\textgreater{}\textquotedbl{}}
\pl{5}{.\allowbreak{}.\allowbreak{}.\allowbreak{}}
\pe
\pl{3}{axiomatization where}
\pl{5}{explanation\_\allowbreak{}1:\allowbreak{} \textquotedbl{}\textless{}formalized explanation 1\textgreater{}\textquotedbl{}}
\pe
\pl{3}{theorem hypothesis:\allowbreak{}}
\pl{5}{assumes asm:\allowbreak{} \textquotedbl{}\textless{}formalized premise\textgreater{}\textquotedbl{}}
\pl{5}{shows \textquotedbl{}\textless{}formalized hypothesis\textgreater{}\textquotedbl{}}
\pl{5}{sorry}
\pe
\pl{3}{end}
\pl{0}{FAITHFULNESS RULES (\allowbreak{}the file is checked against these after it compiles):\allowbreak{}}
\pl{0}{-{} One rule per explanation sentence:\allowbreak{} every \char0{}axiomatization\char0{} entry must render exactly one explanation sentence,\allowbreak{} saying what that sentence says and nothing more.\allowbreak{} Do NOT add any other axioms or bridge rules.\allowbreak{}}
\pl{0}{-{} The premise goes into \char0{}assumes\char0{} and the hypothesis into \char0{}shows\char0{},\allowbreak{} with every qualifier kept and nothing added,\allowbreak{} dropped,\allowbreak{} strengthened or weakened.\allowbreak{}}
\pl{0}{-{} The statements must connect:\allowbreak{} write the premise in \char0{}assumes\char0{} at a fine enough grain that it introduces the same entities and predicates the explanation rule uses in its antecedent,\allowbreak{} so that the rule applies to the assumptions directly and yields what \char0{}shows\char0{} needs.\allowbreak{} Reuse one predicate name for one concept across premise,\allowbreak{} explanation and hypothesis.\allowbreak{}}
\pl{0}{-{} Name every symbol after words that occur in the text.\allowbreak{}}
\pl{0}{4.\allowbreak{} The \char0{}\textless{}theory\_\allowbreak{}name\textgreater{}\char0{} MUST be EXACTLY the value provided in the user prompt — Isabelle requires that the theory name match the filename.\allowbreak{} Do not invent or modify it.\allowbreak{}}
\pl{0}{5.\allowbreak{} Wrap your output in fences exactly as:\allowbreak{}}
\pl{0}{\char0{}\char0{}\char0{}isabelle}
\pl{0}{\textless{}your theory file content\textgreater{}}
\pl{0}{\char0{}\char0{}\char0{}}
\pl{0}{Output nothing outside the fences.\allowbreak{}}
\pe
\end{promptlines}

\noindent{\small\bfseries USER:}\par
\begin{promptlines}
\pl{0}{Theory name (\allowbreak{}use this exact name):\allowbreak{} demo\_\allowbreak{}javelin\_\allowbreak{}competition}
\pe
\pl{0}{Premise:\allowbreak{} Javelin competition}
\pl{0}{Hypothesis:\allowbreak{} There is a competition.\allowbreak{}}
\pl{0}{Explanation:\allowbreak{} There must be a competition for a Javelin competition to take place.\allowbreak{}}
\end{promptlines}

\noindent{\small\bfseries ASSISTANT:}\par
\begin{promptlines}
\pl{0}{\char0{}\char0{}\char0{}isabelle}
\pl{0}{theory demo\_\allowbreak{}javelin\_\allowbreak{}competition}
\pl{0}{imports Main}
\pl{0}{begin}
\pe
\pl{0}{typedecl entity}
\pl{0}{typedecl event}
\pe
\pl{0}{consts}
\pl{2}{JavelinCompetition :\allowbreak{}:\allowbreak{} \textquotedbl{}event \textbackslash{}\textless{}Rightarrow\textgreater{} bool\textquotedbl{}}
\pl{2}{Competition :\allowbreak{}:\allowbreak{} \textquotedbl{}event \textbackslash{}\textless{}Rightarrow\textgreater{} bool\textquotedbl{}}
\pe
\pl{0}{axiomatization where}
\pl{2}{explanation\_\allowbreak{}1:\allowbreak{} \textquotedbl{}\textbackslash{}\textless{}forall\textgreater{}e.\allowbreak{} JavelinCompetition e \textbackslash{}\textless{}longrightarrow\textgreater{} Competition e\textquotedbl{}}
\pe
\pl{0}{theorem hypothesis:\allowbreak{}}
\pl{2}{assumes asm:\allowbreak{} \textquotedbl{}JavelinCompetition e\textquotedbl{}}
\pl{2}{shows \textquotedbl{}\textbackslash{}\textless{}exists\textgreater{}e.\allowbreak{} Competition e\textquotedbl{}}
\pl{2}{sorry}
\pe
\pl{0}{end}
\pl{0}{\char0{}\char0{}\char0{}}
\pe
\end{promptlines}

\noindent{\small\bfseries USER:}\par
\begin{promptlines}
\pl{0}{Theory name (\allowbreak{}use this exact name):\allowbreak{} demo\_\allowbreak{}group\_\allowbreak{}costumes}
\pe
\pl{0}{Premise:\allowbreak{} a group of people are wearing costumes}
\pl{0}{Hypothesis:\allowbreak{} People are wearing costumes.\allowbreak{}}
\pl{0}{Explanation:\allowbreak{} A group means that there are people.\allowbreak{}}
\end{promptlines}

\noindent{\small\bfseries ASSISTANT:}\par
\begin{promptlines}
\pl{0}{\char0{}\char0{}\char0{}isabelle}
\pl{0}{theory demo\_\allowbreak{}group\_\allowbreak{}costumes}
\pl{0}{imports Main}
\pl{0}{begin}
\pe
\pl{0}{typedecl entity}
\pl{0}{typedecl event}
\pe
\pl{0}{consts}
\pl{2}{Group :\allowbreak{}:\allowbreak{} \textquotedbl{}entity \textbackslash{}\textless{}Rightarrow\textgreater{} bool\textquotedbl{}}
\pl{2}{People :\allowbreak{}:\allowbreak{} \textquotedbl{}entity \textbackslash{}\textless{}Rightarrow\textgreater{} bool\textquotedbl{}}
\pl{2}{WearingCostumes :\allowbreak{}:\allowbreak{} \textquotedbl{}entity \textbackslash{}\textless{}Rightarrow\textgreater{} bool\textquotedbl{}}
\pe
\pl{0}{axiomatization where}
\pl{2}{explanation\_\allowbreak{}1:\allowbreak{} \textquotedbl{}\textbackslash{}\textless{}forall\textgreater{}x.\allowbreak{} Group x \textbackslash{}\textless{}longrightarrow\textgreater{} People x\textquotedbl{}}
\pe
\pl{0}{theorem hypothesis:\allowbreak{}}
\pl{2}{assumes asm:\allowbreak{} \textquotedbl{}Group x \textbackslash{}\textless{}and\textgreater{} WearingCostumes x\textquotedbl{}}
\pl{2}{shows \textquotedbl{}\textbackslash{}\textless{}exists\textgreater{}x.\allowbreak{} People x \textbackslash{}\textless{}and\textgreater{} WearingCostumes x\textquotedbl{}}
\pl{2}{sorry}
\pe
\pl{0}{end}
\pl{0}{\char0{}\char0{}\char0{}}
\pe
\end{promptlines}

\noindent{\small\bfseries USER:}\par
\begin{promptlines}
\pl{0}{Theory name (\allowbreak{}use this exact name):\allowbreak{} demo\_\allowbreak{}group\_\allowbreak{}beach}
\pe
\pl{0}{Premise:\allowbreak{} A group of people on the beach with cameras.\allowbreak{}}
\pl{0}{Hypothesis:\allowbreak{} People are on a beach.\allowbreak{}}
\pl{0}{Explanation:\allowbreak{} Group of people are people.\allowbreak{}}
\end{promptlines}

\noindent{\small\bfseries ASSISTANT:}\par
\begin{promptlines}
\pl{0}{\char0{}\char0{}\char0{}isabelle}
\pl{0}{theory demo\_\allowbreak{}group\_\allowbreak{}beach}
\pl{0}{imports Main}
\pl{0}{begin}
\pe
\pl{0}{typedecl entity}
\pl{0}{typedecl event}
\pe
\pl{0}{consts}
\pl{2}{GroupOfPeople :\allowbreak{}:\allowbreak{} \textquotedbl{}entity \textbackslash{}\textless{}Rightarrow\textgreater{} bool\textquotedbl{}}
\pl{2}{People :\allowbreak{}:\allowbreak{} \textquotedbl{}entity \textbackslash{}\textless{}Rightarrow\textgreater{} bool\textquotedbl{}}
\pl{2}{OnBeach :\allowbreak{}:\allowbreak{} \textquotedbl{}entity \textbackslash{}\textless{}Rightarrow\textgreater{} bool\textquotedbl{}}
\pl{2}{WithCameras :\allowbreak{}:\allowbreak{} \textquotedbl{}entity \textbackslash{}\textless{}Rightarrow\textgreater{} bool\textquotedbl{}}
\pe
\pl{0}{axiomatization where}
\pl{2}{explanation\_\allowbreak{}1:\allowbreak{} \textquotedbl{}\textbackslash{}\textless{}forall\textgreater{}x.\allowbreak{} GroupOfPeople x \textbackslash{}\textless{}longrightarrow\textgreater{} People x\textquotedbl{}}
\pe
\pl{0}{theorem hypothesis:\allowbreak{}}
\pl{2}{assumes asm:\allowbreak{} \textquotedbl{}GroupOfPeople x \textbackslash{}\textless{}and\textgreater{} OnBeach x \textbackslash{}\textless{}and\textgreater{} WithCameras x\textquotedbl{}}
\pl{2}{shows \textquotedbl{}\textbackslash{}\textless{}exists\textgreater{}x.\allowbreak{} People x \textbackslash{}\textless{}and\textgreater{} OnBeach x\textquotedbl{}}
\pl{2}{sorry}
\pe
\pl{0}{end}
\pl{0}{\char0{}\char0{}\char0{}}
\pe
\end{promptlines}

\noindent{\small\bfseries USER:}\par
\begin{promptlines}
\pl{0}{Theory name (\allowbreak{}use this exact name):\allowbreak{} eval\_\allowbreak{}72}
\pe
\pl{0}{Premise:\allowbreak{} Two women are wet while holding hands with a long building and buses in the background.\allowbreak{}}
\pl{0}{Hypothesis:\allowbreak{} Women are in the city.\allowbreak{}}
\pl{0}{Explanation:\allowbreak{}\ \ There are women because there are two women.\allowbreak{}}
\pe
\pl{0}{Produce the complete Isabelle 2024 theory file with a \char0{}sorry\char0{} proof body.\allowbreak{}}
\end{promptlines}

\end{promptbox}

\begin{promptbox}{Refinement (third refinement step) \textmd{(example: e-SNLI problem 72)}}
\noindent{\small\bfseries SYSTEM:}\par
\begin{promptlines}
\pl{0}{You are an expert in formal language Isabelle/\allowbreak{}HOL (\allowbreak{}Isabelle 2024).\allowbreak{}}
\pl{0}{You will be given an entailment task (\allowbreak{}premise,\allowbreak{} hypothesis,\allowbreak{} explanation) and a formal Isabelle theory file that attempted to capture it.\allowbreak{} The prover has reported errors detected by an 8-{}proxy critique stack.\allowbreak{}}
\pe
\pl{0}{Your task is to refine the theory file to make it pass Isabelle\textquotesingle{}s type/\allowbreak{}syntax check while staying faithful to the entailment task.\allowbreak{} Output the COMPLETE refined .\allowbreak{}thy file.\allowbreak{}}
\pe
\pl{0}{STRICT CONSTRAINTS:\allowbreak{}}
\pl{0}{1.\allowbreak{} ASCII Isabelle escapes ONLY (\allowbreak{}\char0{}\textbackslash{}\textless{}and\textgreater{}\char0{},\allowbreak{} \char0{}\textbackslash{}\textless{}longrightarrow\textgreater{}\char0{},\allowbreak{} \char0{}\textbackslash{}\textless{}forall\textgreater{}\char0{},\allowbreak{} \char0{}\textbackslash{}\textless{}exists\textgreater{}\char0{},\allowbreak{} \char0{}\textbackslash{}\textless{}Rightarrow\textgreater{}\char0{},\allowbreak{} etc.\allowbreak{}).\allowbreak{} Do NOT use the Unicode forms (\allowbreak{}∧,\allowbreak{} ⟶,\allowbreak{} ∀,\allowbreak{} ∃,\allowbreak{} ⇒,\allowbreak{} etc.\allowbreak{}) — they will not parse on the Isabelle 2024 build we use.\allowbreak{}}
\pl{0}{2.\allowbreak{} The proof body MUST remain exactly \char0{}sorry\char0{}.\allowbreak{} Do NOT write \char0{}proof -{} qed\char0{},\allowbreak{} \char0{}by .\allowbreak{}.\allowbreak{}.\allowbreak{}\char0{},\allowbreak{} \char0{}oops\char0{},\allowbreak{} or any other proof attempt.\allowbreak{} The task is autoformalization,\allowbreak{} not theorem proving.\allowbreak{}}
\pl{0}{3.\allowbreak{} Preserve the \char0{}theory \textless{}name\textgreater{}\char0{} header verbatim — the filename must match the theory name or Isabelle rejects the load.\allowbreak{}}
\pe
\pl{0}{HOW TO USE THE FEEDBACK (\allowbreak{}these override your own judgement):\allowbreak{}}
\pl{0}{-{} Every violation listed below must be addressed.\allowbreak{} Do not fix one and leave the rest.\allowbreak{}}
\pl{0}{-{} Handing back the same code unchanged counts as no answer.\allowbreak{} If you believe the code is already correct,\allowbreak{} change what the feedback names anyway,\allowbreak{} or state the fix in code.\allowbreak{}}
\pl{0}{-{} When the feedback and your own reading disagree,\allowbreak{} follow the feedback.\allowbreak{} The prover errors are hard constraints and come first;\allowbreak{} the judge findings rank above your own opinion.\allowbreak{}}
\pl{0}{4.\allowbreak{} The file MUST have shape \char0{}theory \textless{}name\textgreater{} imports Main begin .\allowbreak{}.\allowbreak{}.\allowbreak{} end\char0{} with \char0{}sorry\char0{} as the proof body.\allowbreak{}}
\pl{0}{5.\allowbreak{} Wrap your output in fences exactly as:\allowbreak{}}
\pl{0}{\char0{}\char0{}\char0{}isabelle}
\pl{0}{\textless{}your theory file content\textgreater{}}
\pl{0}{\char0{}\char0{}\char0{}}
\pl{0}{Before the code,\allowbreak{} write one short paragraph that starts with \char0{}DIAGNOSIS:\allowbreak{}\char0{} (\allowbreak{}at most about 100 words,\allowbreak{} no code blocks):\allowbreak{} for each reported problem,\allowbreak{} say its root cause and the exact name,\allowbreak{} notation or change you will use to fix it.\allowbreak{} Output nothing outside the fences except that DIAGNOSIS paragraph.\allowbreak{}}
\pe
\end{promptlines}

\noindent{\small\bfseries USER:}\par
\begin{promptlines}
\pl{0}{FOCUS:\allowbreak{} the strongest violation is \char0{}ξ\_\allowbreak{}sp\char0{} inside \char0{}theorem:\allowbreak{}hypothesis\char0{}.\allowbreak{} Repair that first;\allowbreak{} leave the rest of the file unchanged unless fixing it requires a change elsewhere.\allowbreak{}}
\pe
\pl{0}{NOTE ON YOUR LAST ATTEMPT:\allowbreak{} your previous edit,\allowbreak{} made to fix \char0{}ξ\_\allowbreak{}sp\char0{} inside \char0{}theorem:\allowbreak{}hypothesis\char0{} (\allowbreak{}the FOCUS in this message),\allowbreak{} made the code worse,\allowbreak{} so it was discarded and you are back on the earlier version.\allowbreak{} It broke ξ\_\allowbreak{}syn.\allowbreak{} Keep what worked and fix only what the feedback names.\allowbreak{}}
\pe
\pl{0}{ISABELLE/\allowbreak{}HOL SYNTAX REMINDERS — the previous attempt produced Isabelle that the kernel could not parse or type-{}check.\allowbreak{} Apply these conventions before retrying:\allowbreak{}}
\pl{0}{-{} Every variable used inside the body of \char0{}\textbackslash{}\textless{}exists\textgreater{}x y z.\allowbreak{} .\allowbreak{}.\allowbreak{}.\allowbreak{}\char0{} or \char0{}\textbackslash{}\textless{}forall\textgreater{}x y z.\allowbreak{} .\allowbreak{}.\allowbreak{}.\allowbreak{}\char0{} must appear in the binder list.\allowbreak{} If the body references \char0{}TShirt t\char0{},\allowbreak{} \char0{}Stirring e\char0{},\allowbreak{} \char0{}Agent e m\char0{},\allowbreak{} etc.\allowbreak{},\allowbreak{} then every one of \char0{}t\char0{},\allowbreak{} \char0{}e\char0{},\allowbreak{} \char0{}m\char0{},\allowbreak{} .\allowbreak{}.\allowbreak{}.\allowbreak{} must be listed in the binder before the \char0{}.\allowbreak{}\char0{}.\allowbreak{} Example:\allowbreak{} \char0{}\textbackslash{}\textless{}exists\textgreater{}m t e p c o.\allowbreak{} Man m \textbackslash{}\textless{}and\textgreater{} TShirt t \textbackslash{}\textless{}and\textgreater{} Stirring e \textbackslash{}\textless{}and\textgreater{} .\allowbreak{}.\allowbreak{}.\allowbreak{}\char0{}.\allowbreak{} Missing binders cause inner syntax errors and semantic-{}preservation failures.\allowbreak{}}
\pe
\pl{0}{YOUR EARLIER DIAGNOSES ON THIS TASK (\allowbreak{}oldest first,\allowbreak{} most recent last).\allowbreak{} Keep following the same line of repair unless the new errors show it was wrong:\allowbreak{}}
\pl{0}{-{} iteration 2:\allowbreak{} The main problem is that \char0{}InCity x\char0{} in the conclusion has no support from the premise.\allowbreak{} The premise says the women are \textquotedbl{}with a long building and buses in the background\textquotedbl{} — this implies they are in a city setting.\allowbreak{} I will add a rule \char0{}background\_\allowbreak{}in\_\allowbreak{}city\char0{} that states:\allowbreak{} if an entity has a long building and buses in the background,\allowbreak{} then it is in the city.\allowbreak{} This connects the premise to the conclusion without adding unsupported facts.\allowbreak{} Also,\allowbreak{} the binder \char0{}x\char0{} in \char0{}theorem hypothesis\char0{} must be quantified — I will add \char0{}\textbackslash{}\textless{}exists\textgreater{}x.\allowbreak{}\char0{} before the assumption to make it a proper existential premise.\allowbreak{}}
\pe
\pl{0}{ALREADY-{}TRIED ATTEMPTS ON THIS TASK (\allowbreak{}oldest first,\allowbreak{} most recent last).\allowbreak{} Every one of these was submitted to the Isabelle kernel and REJECTED on the axes shown.\allowbreak{} Do not resubmit any of them,\allowbreak{} and do not re-{}apply an edit you have already tried — if an edit is listed here,\allowbreak{} it did not work:\allowbreak{}}
\pe
\pl{0}{attempt at iteration 2 — rejected on [\allowbreak{}ξ\_\allowbreak{}syn]:\allowbreak{}}
\pl{0}{\char0{}\char0{}\char0{}isabelle}
\pl{0}{theory eval\_\allowbreak{}72}
\pl{0}{imports Main}
\pl{0}{begin}
\pe
\pl{0}{typedecl entity}
\pl{0}{typedecl event}
\pe
\pl{0}{consts}
\pl{2}{TwoWomen :\allowbreak{}:\allowbreak{} \textquotedbl{}entity \textbackslash{}\textless{}Rightarrow\textgreater{} bool\textquotedbl{}}
\pl{2}{Women :\allowbreak{}:\allowbreak{} \textquotedbl{}entity \textbackslash{}\textless{}Rightarrow\textgreater{} bool\textquotedbl{}}
\pl{2}{Wet :\allowbreak{}:\allowbreak{} \textquotedbl{}entity \textbackslash{}\textless{}Rightarrow\textgreater{} bool\textquotedbl{}}
\pl{2}{HoldingHands :\allowbreak{}:\allowbreak{} \textquotedbl{}entity \textbackslash{}\textless{}Rightarrow\textgreater{} bool\textquotedbl{}}
\pl{2}{LongBuilding :\allowbreak{}:\allowbreak{} \textquotedbl{}entity \textbackslash{}\textless{}Rightarrow\textgreater{} bool\textquotedbl{}}
\pl{2}{Buses :\allowbreak{}:\allowbreak{} \textquotedbl{}entity \textbackslash{}\textless{}Rightarrow\textgreater{} bool\textquotedbl{}}
\pl{2}{InBackground :\allowbreak{}:\allowbreak{} \textquotedbl{}entity \textbackslash{}\textless{}Rightarrow\textgreater{} bool\textquotedbl{}}
\pl{2}{InCity :\allowbreak{}:\allowbreak{} \textquotedbl{}entity \textbackslash{}\textless{}Rightarrow\textgreater{} bool\textquotedbl{}}
\pe
\pl{0}{axiomatization where}
\pl{2}{explanation\_\allowbreak{}1:\allowbreak{} \textquotedbl{}\textbackslash{}\textless{}forall\textgreater{}x.\allowbreak{} TwoWomen x \textbackslash{}\textless{}longrightarrow\textgreater{} Women x\textquotedbl{}}
\pl{2}{background\_\allowbreak{}in\_\allowbreak{}city:\allowbreak{} \textquotedbl{}\textbackslash{}\textless{}forall\textgreater{}x y z.\allowbreak{} (\allowbreak{}LongBuilding y \textbackslash{}\textless{}and\textgreater{} Buses z \textbackslash{}\textless{}and\textgreater{} InBackground x) \textbackslash{}\textless{}longrightarrow\textgreater{} InCity x\textquotedbl{}}
\pe
\pl{0}{theorem hypothesis:\allowbreak{}}
\pl{2}{assumes asm:\allowbreak{} \textquotedbl{}\textbackslash{}\textless{}exists\textgreater{}x y z.\allowbreak{} TwoWomen x \textbackslash{}\textless{}and\textgreater{} Wet x \textbackslash{}\textless{}and\textgreater{} HoldingHands x \textbackslash{}\textless{}and\textgreater{} LongBuilding y \textbackslash{}\textless{}and\textgreater{} Buses z \textbackslash{}\textless{}and\textgreater{} InBackground x\textquotedbl{}}
\pl{2}{shows \textquotedbl{}\textbackslash{}\textless{}exists\textgreater{}x.\allowbreak{} Women x \textbackslash{}\textless{}and\textgreater{} InCity x\textquotedbl{}}
\pl{2}{sorry}
\pe
\pl{0}{end}
\pl{0}{\char0{}\char0{}\char0{}}
\pe
\pl{0}{Premise:\allowbreak{} Two women are wet while holding hands with a long building and buses in the background.\allowbreak{}}
\pl{0}{Hypothesis:\allowbreak{} Women are in the city.\allowbreak{}}
\pl{0}{Explanation:\allowbreak{}\ \ There are women because there are two women.\allowbreak{}}
\pe
\pl{0}{Here is the complete .\allowbreak{}thy file currently submitted to the prover:\allowbreak{}}
\pl{0}{\char0{}\char0{}\char0{}isabelle}
\pl{0}{theory eval\_\allowbreak{}72}
\pl{0}{imports Main}
\pl{0}{begin}
\pe
\pl{0}{typedecl entity}
\pl{0}{typedecl event}
\pe
\pl{0}{consts}
\pl{2}{TwoWomen :\allowbreak{}:\allowbreak{} \textquotedbl{}entity \textbackslash{}\textless{}Rightarrow\textgreater{} bool\textquotedbl{}}
\pl{2}{Women :\allowbreak{}:\allowbreak{} \textquotedbl{}entity \textbackslash{}\textless{}Rightarrow\textgreater{} bool\textquotedbl{}}
\pl{2}{Wet :\allowbreak{}:\allowbreak{} \textquotedbl{}entity \textbackslash{}\textless{}Rightarrow\textgreater{} bool\textquotedbl{}}
\pl{2}{HoldingHands :\allowbreak{}:\allowbreak{} \textquotedbl{}entity \textbackslash{}\textless{}Rightarrow\textgreater{} bool\textquotedbl{}}
\pl{2}{LongBuilding :\allowbreak{}:\allowbreak{} \textquotedbl{}entity \textbackslash{}\textless{}Rightarrow\textgreater{} bool\textquotedbl{}}
\pl{2}{Buses :\allowbreak{}:\allowbreak{} \textquotedbl{}entity \textbackslash{}\textless{}Rightarrow\textgreater{} bool\textquotedbl{}}
\pl{2}{InBackground :\allowbreak{}:\allowbreak{} \textquotedbl{}entity \textbackslash{}\textless{}Rightarrow\textgreater{} bool\textquotedbl{}}
\pl{2}{InCity :\allowbreak{}:\allowbreak{} \textquotedbl{}entity \textbackslash{}\textless{}Rightarrow\textgreater{} bool\textquotedbl{}}
\pe
\pl{0}{axiomatization where}
\pl{2}{explanation\_\allowbreak{}1:\allowbreak{} \textquotedbl{}\textbackslash{}\textless{}forall\textgreater{}x.\allowbreak{} TwoWomen x \textbackslash{}\textless{}longrightarrow\textgreater{} Women x\textquotedbl{}}
\pe
\pl{0}{theorem hypothesis:\allowbreak{}}
\pl{2}{assumes asm:\allowbreak{} \textquotedbl{}TwoWomen x \textbackslash{}\textless{}and\textgreater{} Wet x \textbackslash{}\textless{}and\textgreater{} HoldingHands x \textbackslash{}\textless{}and\textgreater{} LongBuilding y \textbackslash{}\textless{}and\textgreater{} Buses z \textbackslash{}\textless{}and\textgreater{} InBackground x\textquotedbl{}}
\pl{2}{shows \textquotedbl{}\textbackslash{}\textless{}exists\textgreater{}x.\allowbreak{} Women x \textbackslash{}\textless{}and\textgreater{} InCity x\textquotedbl{}}
\pl{2}{sorry}
\pe
\pl{0}{end}
\pe
\pl{0}{\char0{}\char0{}\char0{}}
\pe
\pl{0}{The provided code has the following errors detected by an 8-{}proxy critique stack (\allowbreak{}parse → type-{}check → bind → ground → document → global → semantic-{}preservation → coverage):\allowbreak{}}
\pe
\pl{0}{[\allowbreak{}ξ\_\allowbreak{}sp] semantic preservation (\allowbreak{}formal → NL faithfulness)}
\pl{3}{-{} theorem:\allowbreak{}hypothesis\#1:\allowbreak{} The formalization adds a background relation only for x and omits the scene structure that the women are with the long building and buses in the background,\allowbreak{} so it does not cleanly match the premise as stated.\allowbreak{};\allowbreak{} The premise says the women are in the background of the scene,\allowbreak{} but the formalization adds separate unrelated entities y and z for the building and buses and does not ca}
\pl{3}{-{} failing modules:\allowbreak{} theorem:\allowbreak{}hypothesis}
\pl{3}{-{} repair hint:\allowbreak{} A formal statement says something its source sentence does not say.\allowbreak{} Rewrite that statement so it says exactly what the source sentence says:\allowbreak{} keep every qualifier,\allowbreak{} do not strengthen,\allowbreak{} weaken,\allowbreak{} generalise or reverse it.\allowbreak{} If the failing statement has no source sentence at all (\allowbreak{}a bridge rule you added yourself),\allowbreak{} delete it instead of rewording it.\allowbreak{}}
\pe
\pl{0}{[\allowbreak{}ξ\_\allowbreak{}bnd] cross-{}statement binding (\allowbreak{}the statements connect the way the NL reasoning connects)}
\pl{3}{-{} theorem:\allowbreak{}hypothesis:\allowbreak{} The chain supplies Women from TwoWomen,\allowbreak{} but nothing supplies InCity for the object in the shows clause.\allowbreak{};\allowbreak{} the chain supplies Women from TwoWomen,\allowbreak{} but nothing in the module supplies InCity for the object required by the shows clause.\allowbreak{};\allowbreak{} the chain supplies Women from TwoWomen,\allowbreak{} but nothing supplies InCity for the object required by the shows clause.\allowbreak{}}
\pl{3}{-{} failing modules:\allowbreak{} theorem:\allowbreak{}hypothesis}
\pl{3}{-{} repair hint:\allowbreak{} The statements do not connect:\allowbreak{} the conclusion needs a fact that no assumption or rule supplies for the same variable.\allowbreak{} Do NOT add new rules that are not sentences of the text.\allowbreak{} Instead re-{}express the premise in the assumptions at a finer grain,\allowbreak{} introducing the same entities and predicates that the existing rule uses in its antecedent,\allowbreak{} so that the existing rule applies directly.\allowbreak{} Writing the premise at a finer grain may only split what the premise sentence already says:\allowbreak{} never add to the assumptions a fact,\allowbreak{} event or condition that the premise sentence does not state.\allowbreak{}}
\pe
\pl{0}{Refine the formal code to fix the issues above while staying faithful to the entailment task.\allowbreak{} Output the COMPLETE refined .\allowbreak{}thy file with the \char0{}theory\char0{} header,\allowbreak{} \char0{}imports\char0{},\allowbreak{} \char0{}begin\char0{},\allowbreak{} body,\allowbreak{} and \char0{}end\char0{} lines all present and the proof body still \char0{}sorry\char0{}.\allowbreak{}}
\pe
\end{promptlines}

\end{promptbox}

\begin{promptbox}{Claim extraction for \px{cov} \textmd{(example: e-SNLI problem 72)}}
\noindent{\small\bfseries SYSTEM:}\par
\begin{promptlines}
\pl{0}{You are decomposing an informal natural-{}language entailment task (\allowbreak{}premise + hypothesis + explanation) into atomic claim units.\allowbreak{} A claim unit is the smallest self-{}contained assertion that a faithful formalization has to represent in some form.\allowbreak{}}
\pe
\pl{0}{EXTRACT ONLY WHAT BEARS ON TRUTH CONDITIONS:\allowbreak{} each assertion made by the premise,\allowbreak{} each rule stated in the explanation,\allowbreak{} and the hypothesis itself.\allowbreak{} The explanation sentences are in scope — they are part of what the formalization must carry.\allowbreak{} Do not extract narrative colour.\allowbreak{} Within one sentence,\allowbreak{} do not extract restatements of the same assertion in other words.\allowbreak{}}
\pe
\pl{0}{EXTRACT EVERY SENTENCE ON ITS OWN.\allowbreak{} Extract the claims of each sentence separately,\allowbreak{} even when an explanation sentence repeats content of the premise or the hypothesis:\allowbreak{} an explanation claim is never merged into a premise or hypothesis claim.\allowbreak{}}
\pe
\pl{0}{Output format — one claim unit per line,\allowbreak{} exactly:\allowbreak{}}
\pl{0}{CLAIM:\allowbreak{} (\allowbreak{}premise\textbar{}hypothesis\textbar{}expla\allowbreak{}nation) \textless{}one short atomic NL clause\textgreater{}}
\pl{0}{Output nothing else (\allowbreak{}no numbering,\allowbreak{} no commentary).\allowbreak{}}
\end{promptlines}

\noindent{\small\bfseries USER:}\par
\begin{promptlines}
\pl{0}{Informal:\allowbreak{}}
\pl{0}{Premise:\allowbreak{} Two women are wet while holding hands with a long building and buses in the background.\allowbreak{}}
\pl{0}{Hypothesis:\allowbreak{} Women are in the city.\allowbreak{}}
\pl{0}{Explanation:\allowbreak{}\ \ There are women because there are two women.\allowbreak{}}
\end{promptlines}

\end{promptbox}

\begin{promptbox}{Judge \px{sp} \textmd{(example: e-SNLI problem 72)}}
{\small\itshape This judge makes 3 calls on this candidate, and we show the first. The other calls differ only in the audited unit.}

\noindent{\small\bfseries SYSTEM:}\par
\begin{promptlines}
\pl{0}{You are evaluating semantic preservation (\allowbreak{}xi\_\allowbreak{}sp) per explicit statement of an Isabelle/\allowbreak{}HOL formalization.\allowbreak{} The task gives a premise,\allowbreak{} a hypothesis and an explanation.\allowbreak{} The formalization renders the premise as the \char0{}assumes\char0{} clause of the theorem (\allowbreak{}or as named axioms),\allowbreak{} the hypothesis as its \char0{}shows\char0{} clause,\allowbreak{} and each explanation sentence as one named axiom.\allowbreak{}}
\pe
\pl{0}{UNITS AND THE SENTENCE EACH IS COMPARED WITH.\allowbreak{} Each named axiom is one statement.\allowbreak{} The theorem is split into two statements,\allowbreak{} its \char0{}assumes\char0{} part and its \char0{}shows\char0{} part.\allowbreak{} Declarations (\allowbreak{}typedecl,\allowbreak{} consts) are not statements and are not listed.\allowbreak{} Every listed statement comes with a note that names the one sentence it is compared with:\allowbreak{} the premise for the \char0{}assumes\char0{} part,\allowbreak{} the hypothesis for the \char0{}shows\char0{} part,\allowbreak{} and the explanation sentence it renders for an axiom.\allowbreak{} Compare the statement with that sentence and with nothing else.\allowbreak{} The theorem\textquotesingle{}s name (\allowbreak{}often \char0{}hypothesis\char0{}) and the statement ids are only labels;\allowbreak{} the note alone says which sentence a statement is compared with.\allowbreak{} Judge the \char0{}assumes\char0{} part only against the premise and the \char0{}shows\char0{} part only against the hypothesis,\allowbreak{} each as a description of its own sentence;\allowbreak{} whether the \char0{}shows\char0{} part follows from the \char0{}assumes\char0{} part,\allowbreak{} and what the other part of the theorem says,\allowbreak{} are not judged on this statement.\allowbreak{} You will see the statement together with its module (\allowbreak{}in some settings,\allowbreak{} the full .\allowbreak{}thy file) as context.\allowbreak{}}
\pe
\pl{0}{THE QUESTION.\allowbreak{} Read the statement as a proposition and ask whether it says what its sentence says:\allowbreak{} which predicates it asserts,\allowbreak{} of which objects,\allowbreak{} with which polarity,\allowbreak{} quantifiers,\allowbreak{} connectives,\allowbreak{} modality and numbers.\allowbreak{}}
\pe
\pl{0}{HOW TO READ VARIABLES (\allowbreak{}Isabelle semantics).\allowbreak{}}
\pl{0}{-{} A variable left free in an axiom is universally quantified:\allowbreak{} the axiom is a rule about every object.\allowbreak{}}
\pl{0}{-{} A variable left free in the theorem stands for an arbitrary individual:\allowbreak{} the theorem claims its \char0{}shows\char0{} for every individual that satisfies its \char0{}assumes\char0{}.\allowbreak{} A variable that occurs only in \char0{}assumes\char0{} therefore amounts to an existential,\allowbreak{} so an \char0{}assumes\char0{} part over free variables says that there are individuals and events as described.\allowbreak{} For example,\allowbreak{} \char0{}GroupOfPeople x \textbackslash{}\textless{}and\textgreater{} OnBeach x \textbackslash{}\textless{}and\textgreater{} WithCameras x\char0{} renders \textquotedbl{}A group of people on the beach with cameras.\allowbreak{}\textquotedbl{}}
\pl{0}{-{} A variable that occurs in both \char0{}assumes\char0{} and \char0{}shows\char0{} is an individual the premise introduced,\allowbreak{} and the \char0{}shows\char0{} part speaks about that same individual.\allowbreak{} A \char0{}shows\char0{} part may render the hypothesis with its own existential,\allowbreak{} as \char0{}\textbackslash{}\textless{}exists\textgreater{}x.\allowbreak{} People x \textbackslash{}\textless{}and\textgreater{} OnBeach x\char0{} renders \textquotedbl{}People are on a beach.\allowbreak{}\textquotedbl{},\allowbreak{} or about individuals the premise introduced;\allowbreak{} both are faithful.\allowbreak{}}
\pl{0}{-{} Variable names are placeholders and carry no meaning.\allowbreak{}}
\pe
\pl{0}{ACCEPTED WAYS OF WRITING.\allowbreak{} The formalization was instructed to write this way;\allowbreak{} each is faithful when its condition holds.\allowbreak{}}
\pl{0}{-{} One sentence becomes a conjunction of predicates over individuals and events:\allowbreak{} what the sentence says of an individual or an event becomes predicates on that individual or event,\allowbreak{} as \textquotedbl{}A group of people on the beach with cameras.\allowbreak{}\textquotedbl{} becomes \char0{}GroupOfPeople x \textbackslash{}\textless{}and\textgreater{} OnBeach x \textbackslash{}\textless{}and\textgreater{} WithCameras x\char0{} and \textquotedbl{}Javelin competition\textquotedbl{} becomes \char0{}JavelinCompetition e\char0{}.\allowbreak{} Splitting a sentence this way adds no content,\allowbreak{} provided every predicate comes from the sentence.\allowbreak{}}
\pl{0}{-{} A group the sentence names as one thing may be written as one individual with a predicate for the group,\allowbreak{} as \textquotedbl{}a group of people\textquotedbl{} is written \char0{}GroupOfPeople x\char0{};\allowbreak{} a number or quantity word describing the group may be part of that predicate.\allowbreak{}}
\pl{0}{-{} The premise may be written as named axioms instead of an \char0{}assumes\char0{} clause.\allowbreak{}}
\pl{0}{-{} An explanation sentence that states a link between two descriptions (\allowbreak{}\textquotedbl{}Group of people are people.\allowbreak{}\textquotedbl{},\allowbreak{} \textquotedbl{}There must be a competition for a Javelin competition to take place.\allowbreak{}\textquotedbl{}) states a general rule,\allowbreak{} even when it is worded about the individuals of this scene.\allowbreak{} It is faithfully rendered as a universally quantified rule whose antecedent is exactly the first description and whose conclusion is exactly the second,\allowbreak{} as in \char0{}\textbackslash{}\textless{}forall\textgreater{}x.\allowbreak{} GroupOfPeople x \textbackslash{}\textless{}longrightarrow\textgreater{} People x\char0{} and \char0{}\textbackslash{}\textless{}forall\textgreater{}e.\allowbreak{} JavelinCompetition e \textbackslash{}\textless{}longrightarrow\textgreater{} Competition e\char0{}.\allowbreak{} Such a rule is faithful only if its antecedent carries nothing beyond the first description (\allowbreak{}no detail of the scene that only the premise mentions) and its conclusion does not already appear among its own conditions.\allowbreak{}}
\pl{0}{-{} Wording is not content:\allowbreak{} synonymous predicates,\allowbreak{} equivalent rewritings,\allowbreak{} different predicate /\allowbreak{} theory /\allowbreak{} variable names and ASCII-{}vs-{}Unicode all pass.\allowbreak{}}
\pe
\pl{0}{WHAT FAILS.\allowbreak{} Each item says which statements it applies to;\allowbreak{} work through the items that apply,\allowbreak{} in order.\allowbreak{}}
\pl{0}{1.\allowbreak{} Argument order and attachment (\allowbreak{}every statement):\allowbreak{} the arguments of a predicate in the wrong order,\allowbreak{} such as agent and patient swapped;\allowbreak{} a property the sentence gives to one entity attached to the variable of another entity;\allowbreak{} one variable used for two entities the sentence keeps apart.\allowbreak{}}
\pl{0}{2.\allowbreak{} Direction (\allowbreak{}every statement):\allowbreak{} an implication written the wrong way round;\allowbreak{} \char0{}P \textbackslash{}\textless{}longrightarrow\textgreater{} Q\char0{} is not \char0{}Q \textbackslash{}\textless{}longrightarrow\textgreater{} P\char0{}.\allowbreak{}}
\pl{0}{3.\allowbreak{} Negation and quantifiers (\allowbreak{}every statement):\allowbreak{} what \char0{}\textbackslash{}\textless{}not\textgreater{}\char0{} applies to,\allowbreak{} and what \char0{}\textbackslash{}\textless{}forall\textgreater{}\char0{} and \char0{}\textbackslash{}\textless{}exists\textgreater{}\char0{} range over,\allowbreak{} including the order of an alternation.\allowbreak{} A sentence that reports particular individuals,\allowbreak{} such as the premise,\allowbreak{} rendered with \char0{}\textbackslash{}\textless{}forall\textgreater{}\char0{} over all such individuals fails.\allowbreak{}}
\pl{0}{4.\allowbreak{} Connectives (\allowbreak{}every statement):\allowbreak{} \char0{}\textbackslash{}\textless{}and\textgreater{}\char0{} against \char0{}\textbackslash{}\textless{}or\textgreater{}\char0{},\allowbreak{} implication against \char0{}\textbackslash{}\textless{}longleftrightarrow\textgreater{}\char0{},\allowbreak{} and their bracketing.\allowbreak{}}
\pl{0}{5.\allowbreak{} Numbers and counts (\allowbreak{}every statement):\allowbreak{} any number or cardinality the sentence states.\allowbreak{}}
\pl{0}{6.\allowbreak{} Modality (\allowbreak{}every statement):\allowbreak{} what the sentence states only as possible,\allowbreak{} as a capacity or as a purpose,\allowbreak{} rendered as actually holding.\allowbreak{}}
\pl{0}{7.\allowbreak{} Added,\allowbreak{} stronger or changed content (\allowbreak{}every statement):\allowbreak{} a predicate that renders no words of the sentence,\allowbreak{} or a relation rendered as a more specific,\allowbreak{} stronger or different relation.\allowbreak{} A predicate that renders words of the sentence is never added,\allowbreak{} and a predicate that holds in every case the sentence describes passes:\allowbreak{} for \textquotedbl{}A man sleeps.\allowbreak{}\textquotedbl{},\allowbreak{} \char0{}Man x \textbackslash{}\textless{}and\textgreater{} Person x \textbackslash{}\textless{}and\textgreater{} Sleeps x\char0{} passes,\allowbreak{} and \char0{}Man x \textbackslash{}\textless{}and\textgreater{} Old x \textbackslash{}\textless{}and\textgreater{} Sleeps x\char0{} fails because the sentence also covers young men.\allowbreak{} For an axiom,\allowbreak{} judge it by whether it changes which cases the rule covers:\allowbreak{} a condition in its antecedent fails when it rules out a case the sentence covers,\allowbreak{} and its conclusion fails when it asserts more than the sentence.\allowbreak{} Content of the sentence that the statement leaves out is judged on coverage (\allowbreak{}xi\_\allowbreak{}cov),\allowbreak{} not here.\allowbreak{}}
\pl{0}{8.\allowbreak{} Rules (\allowbreak{}axioms only):\allowbreak{} an axiom fails if it renders none of the explanation sentences,\allowbreak{} however true it is in the world,\allowbreak{} which includes a category or synonym link that no sentence states;\allowbreak{} if its sentence only reports what particular individuals in this scene are or do,\allowbreak{} since such a sentence states no general rule;\allowbreak{} if its antecedent or its conclusion carries anything its sentence does not state,\allowbreak{} such as a detail of the scene that only the premise mentions;\allowbreak{} or if its conclusion already appears among its own conditions.\allowbreak{} A sentence that links two descriptions with words such as \textquotedbl{}implies\textquotedbl{},\allowbreak{} \textquotedbl{}means\textquotedbl{},\allowbreak{} \textquotedbl{}is\textquotedbl{},\allowbreak{} \textquotedbl{}is a kind of\textquotedbl{} or \textquotedbl{}must be .\allowbreak{}.\allowbreak{}.\allowbreak{} for\textquotedbl{} states a link even when it mentions \textquotedbl{}her\textquotedbl{},\allowbreak{} \textquotedbl{}him\textquotedbl{} or \textquotedbl{}the boy\textquotedbl{};\allowbreak{} it is not a report about particular individuals.\allowbreak{}}
\pl{0}{9.\allowbreak{} Roles (\allowbreak{}the theorem):\allowbreak{} what the task gives as premise written in \char0{}shows\char0{},\allowbreak{} or the hypothesis written in \char0{}assumes\char0{}.\allowbreak{}}
\pe
\pl{0}{NO OUTSIDE KNOWLEDGE.\allowbreak{} Judge only against the sentence the statement is compared with.\allowbreak{}}
\pl{0}{AN EMPTY FORMALIZATION FAILS:\allowbreak{} if the unit carries no proposition,\allowbreak{} it fails.\allowbreak{}}
\pl{0}{OUT OF SCOPE:\allowbreak{} the proof body (\allowbreak{}\char0{}sorry\char0{} is fine);\allowbreak{} whether the hypothesis follows from the premise and the explanation;\allowbreak{} whether the statements connect into an argument;\allowbreak{} statements you cannot see;\allowbreak{} and statements that are missing (\allowbreak{}missing content is coverage,\allowbreak{} xi\_\allowbreak{}cov).\allowbreak{}}
\pe
\pl{0}{Output format — for each statement,\allowbreak{} two lines,\allowbreak{} in this order:\allowbreak{}}
\pl{0}{SOURCE \textless{}id\textgreater{}:\allowbreak{} \textless{}verbatim quote of the part of the sentence named in this statement\textquotesingle{}s note that the statement renders,\allowbreak{} or NONE\textgreater{}}
\pl{0}{STATEMENT \textless{}id\textgreater{} \textbar{} VERDICT:\allowbreak{} pass\textbar{}fail\textbar{}unknown \textbar{} REASON:\allowbreak{} \textless{}one sentence\textgreater{}}
\pe
\pl{0}{The SOURCE line comes first:\allowbreak{} the statement id,\allowbreak{} a colon,\allowbreak{} and then the quote in double quotes,\allowbreak{} copied word for word from the sentence,\allowbreak{} for example}
\pl{0}{SOURCE theorem:\allowbreak{}hypothesis\#2:\allowbreak{} \textquotedbl{}People are on a beach.\allowbreak{}\textquotedbl{}}
\pl{0}{A SOURCE line written any other way is not read as a quote.\allowbreak{} Write NONE instead of a quote only when the statement renders no part of the sentence it is compared with — a unit with SOURCE NONE is recorded as a failure.\allowbreak{} Use VERDICT:\allowbreak{} unknown only when you have a real quote but cannot decide.\allowbreak{}}
\pl{0}{Output nothing other than these lines.\allowbreak{}}
\end{promptlines}

\noindent{\small\bfseries USER:}\par
\begin{promptlines}
\pl{0}{Informal:\allowbreak{}}
\pl{0}{Premise:\allowbreak{} Two women are wet while holding hands with a long building and buses in the background.\allowbreak{}}
\pl{0}{Hypothesis:\allowbreak{} Women are in the city.\allowbreak{}}
\pl{0}{Explanation:\allowbreak{}\ \ There are women because there are two women.\allowbreak{}}
\pe
\pl{0}{You are evaluating a SINGLE Isabelle .\allowbreak{}thy statement (\allowbreak{}no surrounding file context beyond its module):\allowbreak{}}
\pl{0}{-{} STATEMENT axiomatization:\allowbreak{}1\#1\ \ (\allowbreak{}module axiomatization:\allowbreak{}1,\allowbreak{} line 19):\allowbreak{} explanation\_\allowbreak{}1:\allowbreak{} \textquotedbl{}\textbackslash{}\textless{}forall\textgreater{}x.\allowbreak{} TwoWomen x \textbackslash{}\textless{}longrightarrow\textgreater{} Women x\textquotedbl{}}
\pl{2}{This statement is an axiom;\allowbreak{} compare it with this explanation sentence:\allowbreak{} There are women because there are two women.\allowbreak{}}
\pl{0}{\char0{}\char0{}\char0{}}
\pl{0}{axiomatization where}
\pl{2}{explanation\_\allowbreak{}1:\allowbreak{} \textquotedbl{}\textbackslash{}\textless{}forall\textgreater{}x.\allowbreak{} TwoWomen x \textbackslash{}\textless{}longrightarrow\textgreater{} Women x\textquotedbl{}}
\pe
\pl{0}{\char0{}\char0{}\char0{}}
\pe
\pl{0}{Emit exactly ONE verdict line for this statement.\allowbreak{}}
\end{promptlines}

\end{promptbox}

\begin{promptbox}{Judge \px{cov} \textmd{(example: e-SNLI problem 72)}}
{\small\itshape This judge makes 2 calls on this candidate, and we show the first. The other calls differ only in the audited unit.}

\noindent{\small\bfseries SYSTEM:}\par
\begin{promptlines}
\pl{0}{You are evaluating coverage (\allowbreak{}xi\_\allowbreak{}cov) of an Isabelle 2024 .\allowbreak{}thy formalization.\allowbreak{}}
\pe
\pl{0}{For each claim unit extracted from the informal NL (\allowbreak{}premise + hypothesis + explanation),\allowbreak{} decide ONE thing only:\allowbreak{} does the formalization say this claim,\allowbreak{} including the link the claim expresses? Where the statement that says it sits (\allowbreak{}an axiom,\allowbreak{} a clause of \char0{}assumes\char0{},\allowbreak{} or the \char0{}shows\char0{} clause) does not matter.\allowbreak{} The parts of a claim may be said by several statements,\allowbreak{} but the link the claim expresses (\allowbreak{}the relation between its entities,\allowbreak{} or the if-{}then of a rule) must be said by one statement.\allowbreak{} When the words at the two ends of the link appear in different statements that each say something else,\allowbreak{} the link is not said and the claim is not covered.\allowbreak{} How the statement is written (\allowbreak{}as a fact or as a rule,\allowbreak{} with which quantifier) does not matter here;\allowbreak{} a wrong form is semantic preservation (\allowbreak{}xi\_\allowbreak{}sp).\allowbreak{}}
\pe
\pl{0}{PRESENT COUNTS AS COVERED,\allowbreak{} EVEN IF IT IS WRONG.\allowbreak{} A claim rendered with the implication reversed,\allowbreak{} the wrong predicate arity or the wrong scope is still covered — getting it wrong is semantic preservation (\allowbreak{}xi\_\allowbreak{}sp),\allowbreak{} not coverage.\allowbreak{} Synonymous predicates count as corresponding.\allowbreak{}}
\pl{0}{A claim fails when the formalization does not render all of it,\allowbreak{} as specified by the rule below.\allowbreak{}}
\pl{0}{A CLAIM NEEDS ITS WHOLE CONTENT.\allowbreak{} A claim is covered only if everything it states is rendered:\allowbreak{} each entity,\allowbreak{} each property and the relation between them.\allowbreak{} If the formalization renders only part of the claim,\allowbreak{} the claim is not covered:\allowbreak{} for example,\allowbreak{} the claim says a teacher is grading a paper and only the teacher is rendered,\allowbreak{} or a claim links two things and no statement states that link.\allowbreak{} A single predicate whose name carries several parts of the claim renders those parts,\allowbreak{} and a claim whose words are only part of a longer predicate name is rendered by that predicate:\allowbreak{} \char0{}OldFisherman x\char0{} renders both \textquotedbl{}a fisherman\textquotedbl{} and \textquotedbl{}the fisherman is old\textquotedbl{}.\allowbreak{} A claim needs no predicate or statement of its own.\allowbreak{} A part rendered wrongly still counts as present (\allowbreak{}that is xi\_\allowbreak{}sp).\allowbreak{}}
\pl{0}{QUALIFIERS ARE PART OF THE CLAIM.\allowbreak{} A number word (\allowbreak{}two,\allowbreak{} three,\allowbreak{} several),\allowbreak{} a group word (\allowbreak{}a group of,\allowbreak{} a crowd of) and participants the sentence keeps apart must be carried too:\allowbreak{} by a predicate whose name carries them (\allowbreak{}\char0{}TwoCats x\char0{},\allowbreak{} \char0{}CrowdOfFans x\char0{}),\allowbreak{} or by one variable for each participant.\allowbreak{} For \textquotedbl{}Two cats sleep on a sofa.\allowbreak{}\textquotedbl{},\allowbreak{} \char0{}Cat x \textbackslash{}\textless{}and\textgreater{} Sleeps x \textbackslash{}\textless{}and\textgreater{} OnSofa x\char0{} does not carry \textquotedbl{}two\textquotedbl{},\allowbreak{} and \char0{}Cat x \textbackslash{}\textless{}and\textgreater{} Cat y\char0{} carries it;\allowbreak{} the two variables need not be stated distinct.\allowbreak{}}
\pl{0}{SCENE DETAILS.\allowbreak{} A claim that describes the circumstances of the scene (\allowbreak{}where,\allowbreak{} when,\allowbreak{} with what) is covered by predicates or relations on the individual or event it describes,\allowbreak{} whether they are written as one predicate or several:\allowbreak{} in \textquotedbl{}A boy kicks a ball in a field at night.\allowbreak{}\textquotedbl{},\allowbreak{} \char0{}InField e \textbackslash{}\textless{}and\textgreater{} AtNight e\char0{} on the kicking event covers \textquotedbl{}in a field at night\textquotedbl{}.\allowbreak{}}
\pl{0}{An explanation that only comments on the sentences (\allowbreak{}for example \textquotedbl{}the hypothesis restates the premise\textquotedbl{} or \textquotedbl{}the two sentences say the same thing\textquotedbl{}) makes no claim of its own beyond the content it points to.\allowbreak{} It needs no statement saying how the sentences relate:\allowbreak{} it is covered when the content it points to is rendered.\allowbreak{} On its FORMAL line,\allowbreak{} quote the statement that renders that content,\allowbreak{} and give VERDICT:\allowbreak{} pass.\allowbreak{}}
\pl{0}{Predicate /\allowbreak{} variable naming differences never make a claim missing.\allowbreak{}}
\pl{0}{OUT OF SCOPE:\allowbreak{} proof body,\allowbreak{} entailment validity.\allowbreak{}}
\pe
\pl{0}{Each claim line must also name one formal module from the module list given with the formalization.\allowbreak{} For a covered claim,\allowbreak{} MODULE is the module that represents it.\allowbreak{} For a missing claim,\allowbreak{} MODULE is the module where the missing content belongs.\allowbreak{}}
\pe
\pl{0}{Output format — for each claim unit,\allowbreak{} two lines,\allowbreak{} in this order:\allowbreak{}}
\pl{0}{FORMAL \#\textless{}index\textgreater{}:\allowbreak{} \textless{}verbatim quote from the .\allowbreak{}thy file that corresponds to this claim,\allowbreak{} or NONE\textgreater{}}
\pl{0}{CLAIM \#\textless{}index\textgreater{} \textbar{} VERDICT:\allowbreak{} pass\textbar{}fail\textbar{}unknown \textbar{} MODULE:\allowbreak{} \textless{}module id\textgreater{} \textbar{} REASON:\allowbreak{} \textless{}one sentence\textgreater{}}
\pe
\pl{0}{The FORMAL line comes first and must be a quote from the .\allowbreak{}thy file,\allowbreak{} not a paraphrase.\allowbreak{} Write NONE when nothing corresponds — a claim with FORMAL NONE is recorded as not covered.\allowbreak{}}
\pl{0}{Use VERDICT:\allowbreak{} unknown only when you have a real quote but cannot decide.\allowbreak{}}
\pl{0}{Output nothing else.\allowbreak{}}
\end{promptlines}

\noindent{\small\bfseries USER:}\par
\begin{promptlines}
\pl{0}{Informal:\allowbreak{}}
\pl{0}{Premise:\allowbreak{} Two women are wet while holding hands with a long building and buses in the background.\allowbreak{}}
\pl{0}{Hypothesis:\allowbreak{} Women are in the city.\allowbreak{}}
\pl{0}{Explanation:\allowbreak{}\ \ There are women because there are two women.\allowbreak{}}
\pe
\pl{0}{Informal claim units to check:\allowbreak{}}
\pl{0}{CLAIM \#1:\allowbreak{} premise Two women are wet}
\pl{0}{CLAIM \#2:\allowbreak{} premise Two women are holding hands}
\pl{0}{CLAIM \#3:\allowbreak{} premise A long building is in the background}
\pl{0}{CLAIM \#4:\allowbreak{} premise Buses are in the background}
\pl{0}{CLAIM \#5:\allowbreak{} hypothesis Women are in the city}
\pl{0}{CLAIM \#6:\allowbreak{} explanation There are women}
\pe
\pl{0}{Formal modules:\allowbreak{}}
\pl{0}{-{} MODULE typedecl:\allowbreak{}entity\ \ (\allowbreak{}lines 5-{}5,\allowbreak{} kind=\allowbreak{}typedecl)}
\pl{0}{-{} MODULE typedecl:\allowbreak{}event\ \ (\allowbreak{}lines 6-{}7,\allowbreak{} kind=\allowbreak{}typedecl)}
\pl{0}{-{} MODULE consts:\allowbreak{}1\ \ (\allowbreak{}lines 8-{}17,\allowbreak{} kind=\allowbreak{}consts)}
\pl{0}{-{} MODULE axiomatization:\allowbreak{}1\ \ (\allowbreak{}lines 18-{}20,\allowbreak{} kind=\allowbreak{}axiomatization)}
\pl{0}{-{} MODULE theorem:\allowbreak{}hypothesis\ \ (\allowbreak{}lines 21-{}25,\allowbreak{} kind=\allowbreak{}theorem)}
\pe
\pl{0}{Isabelle .\allowbreak{}thy formalization (\allowbreak{}judge whether each claim is represented anywhere inside it):\allowbreak{}}
\pl{0}{\char0{}\char0{}\char0{}}
\pl{0}{theory eval\_\allowbreak{}72}
\pl{0}{imports Main}
\pl{0}{begin}
\pe
\pl{0}{typedecl entity}
\pl{0}{typedecl event}
\pe
\pl{0}{consts}
\pl{2}{TwoWomen :\allowbreak{}:\allowbreak{} \textquotedbl{}entity \textbackslash{}\textless{}Rightarrow\textgreater{} bool\textquotedbl{}}
\pl{2}{Women :\allowbreak{}:\allowbreak{} \textquotedbl{}entity \textbackslash{}\textless{}Rightarrow\textgreater{} bool\textquotedbl{}}
\pl{2}{Wet :\allowbreak{}:\allowbreak{} \textquotedbl{}entity \textbackslash{}\textless{}Rightarrow\textgreater{} bool\textquotedbl{}}
\pl{2}{HoldingHands :\allowbreak{}:\allowbreak{} \textquotedbl{}entity \textbackslash{}\textless{}Rightarrow\textgreater{} bool\textquotedbl{}}
\pl{2}{LongBuilding :\allowbreak{}:\allowbreak{} \textquotedbl{}entity \textbackslash{}\textless{}Rightarrow\textgreater{} bool\textquotedbl{}}
\pl{2}{Buses :\allowbreak{}:\allowbreak{} \textquotedbl{}entity \textbackslash{}\textless{}Rightarrow\textgreater{} bool\textquotedbl{}}
\pl{2}{InBackground :\allowbreak{}:\allowbreak{} \textquotedbl{}entity \textbackslash{}\textless{}Rightarrow\textgreater{} bool\textquotedbl{}}
\pl{2}{InCity :\allowbreak{}:\allowbreak{} \textquotedbl{}entity \textbackslash{}\textless{}Rightarrow\textgreater{} bool\textquotedbl{}}
\pe
\pl{0}{axiomatization where}
\pl{2}{explanation\_\allowbreak{}1:\allowbreak{} \textquotedbl{}\textbackslash{}\textless{}forall\textgreater{}x.\allowbreak{} TwoWomen x \textbackslash{}\textless{}longrightarrow\textgreater{} Women x\textquotedbl{}}
\pe
\pl{0}{theorem hypothesis:\allowbreak{}}
\pl{2}{assumes asm:\allowbreak{} \textquotedbl{}TwoWomen x \textbackslash{}\textless{}and\textgreater{} Wet x \textbackslash{}\textless{}and\textgreater{} HoldingHands x \textbackslash{}\textless{}and\textgreater{} LongBuilding y \textbackslash{}\textless{}and\textgreater{} Buses z \textbackslash{}\textless{}and\textgreater{} InBackground x\textquotedbl{}}
\pl{2}{shows \textquotedbl{}\textbackslash{}\textless{}exists\textgreater{}x.\allowbreak{} Women x \textbackslash{}\textless{}and\textgreater{} InCity x\textquotedbl{}}
\pl{2}{sorry}
\pe
\pl{0}{end}
\pl{0}{\char0{}\char0{}\char0{}}
\end{promptlines}

\end{promptbox}

\begin{promptbox}{Judge \px{gnd} \textmd{(example: e-SNLI problem 72)}}
{\small\itshape This judge makes 5 calls on this candidate, and we show the first. The other calls differ only in the audited unit.}

\noindent{\small\bfseries SYSTEM:}\par
\begin{promptlines}
\pl{0}{You are evaluating grounding (\allowbreak{}xi\_\allowbreak{}gnd) per module of a Isabelle .\allowbreak{}thy formalization.\allowbreak{}}
\pe
\pl{0}{xi\_\allowbreak{}gnd asks ONE question at SYMBOL granularity:\allowbreak{} for every predicate,\allowbreak{} constant or symbol that appears on the formal side of a module,\allowbreak{} does it trace back to something introduced by the informal NL,\allowbreak{} or by the per-{}instance fact context? A formal-{}side concept invented out of nowhere -{}-{} mentioned in neither -{}-{} is a grounding failure.\allowbreak{}}
\pe
\pl{0}{THIS IS INDEPENDENT OF EVERY PROVER-{}SIDE CHECK.\allowbreak{} An out-{}of-{}nowhere predicate is perfectly well-{}typed,\allowbreak{} parses fine,\allowbreak{} and the prover can even close a proof using it.\allowbreak{} It is still a grounding violation.\allowbreak{} Conversely,\allowbreak{} a symbol that simply does not exist in the library is a binding failure,\allowbreak{} NOT your concern -{}-{} do not report it.\allowbreak{}}
\pe
\pl{0}{Aggregation is strict-{}AND at module level:\allowbreak{} a single unsourced symbol inside a module sets that module\textquotesingle{}s verdict to fail.\allowbreak{}}
\pe
\pl{0}{DO NOT CONFUSE WITH NEIGHBOURING AXES:\allowbreak{}}
\pl{0}{-{} xi\_\allowbreak{}sp is STATEMENT granularity (\allowbreak{}does the proposition match what nl says?).\allowbreak{} A symbol can be grounded but used in a misaligned statement -{}-{} that is xi\_\allowbreak{}sp\textquotesingle{}s problem,\allowbreak{} not yours.\allowbreak{}}
\pl{0}{-{} xi\_\allowbreak{}cov traces NL -{}\textgreater{} formal (\allowbreak{}is every NL claim represented?).\allowbreak{} You trace formal -{}\textgreater{} NL (\allowbreak{}does every formal symbol have a source?).\allowbreak{}}
\pe
\pl{0}{Ignore:\allowbreak{} variable naming,\allowbreak{} whitespace,\allowbreak{} notation choices,\allowbreak{} theorem names,\allowbreak{} and standard library machinery that the informal statement obviously presupposes (\allowbreak{}HOL connectives,\allowbreak{} equality,\allowbreak{} the entailment scaffolding).\allowbreak{}}
\pl{0}{The informal NL of an entailment problem is all three texts:\allowbreak{} the premise,\allowbreak{} the hypothesis and the explanation.\allowbreak{} A symbol named after words of any one of them has a source;\allowbreak{} in particular the symbols of the conclusion come from the hypothesis.\allowbreak{}}
\pl{0}{If uncertain -{}\textgreater{} verdict=\allowbreak{}pass.\allowbreak{}}
\pe
\pl{0}{Output format -{}-{} one line per module,\allowbreak{} exactly:\allowbreak{}}
\pl{0}{MODULE \textless{}id\textgreater{} \textbar{} VERDICT:\allowbreak{} pass\textbar{}fail \textbar{} REASON:\allowbreak{} \textless{}one sentence\textgreater{}}
\pe
\pl{0}{Output nothing other than these lines.\allowbreak{}}
\end{promptlines}

\noindent{\small\bfseries USER:}\par
\begin{promptlines}
\pl{0}{Informal:\allowbreak{}}
\pl{0}{Premise:\allowbreak{} Two women are wet while holding hands with a long building and buses in the background.\allowbreak{}}
\pl{0}{Hypothesis:\allowbreak{} Women are in the city.\allowbreak{}}
\pl{0}{Explanation:\allowbreak{}\ \ There are women because there are two women.\allowbreak{}}
\pe
\pl{0}{You are evaluating a SINGLE Isabelle .\allowbreak{}thy module (\allowbreak{}no surrounding file context beyond its module):\allowbreak{}}
\pl{0}{-{} MODULE typedecl:\allowbreak{}entity\ \ (\allowbreak{}lines 5-{}5,\allowbreak{} kind=\allowbreak{}typedecl)}
\pl{0}{\char0{}\char0{}\char0{}}
\pl{0}{typedecl entity}
\pl{0}{\char0{}\char0{}\char0{}}
\pe
\pl{0}{Emit exactly ONE verdict line for this module.\allowbreak{}}
\end{promptlines}

\end{promptbox}

\begin{promptbox}{Judge \px{bnd} \textmd{(example: e-SNLI problem 72)}}
\noindent{\small\bfseries SYSTEM:}\par
\begin{promptlines}
\pl{0}{You are evaluating inferential binding (\allowbreak{}xi\_\allowbreak{}bnd) per module of an Isabelle/\allowbreak{}HOL formalization of an inference item (\allowbreak{}premise,\allowbreak{} hypothesis,\allowbreak{} explanation).\allowbreak{}}
\pe
\pl{0}{xi\_\allowbreak{}bnd asks ONE question at MODULE granularity:\allowbreak{} do the statements of the module connect to one another the way the reasoning in the NL connects? The object of your verdict is the wiring between statements (\allowbreak{}which statement supplies what another one needs,\allowbreak{} and through which variables),\allowbreak{} never the content of a single statement.\allowbreak{}}
\pe
\pl{0}{Work through the module in this order:\allowbreak{}}
\pl{0}{1.\allowbreak{} TRACE THE CHAIN.\allowbreak{} List each predicate the \char0{}shows\char0{} clause needs.\allowbreak{} For each,\allowbreak{} find the statement that supplies it:\allowbreak{} the \char0{}assumes\char0{} clause,\allowbreak{} a premise axiom,\allowbreak{} or the conclusion of a rule.\allowbreak{} A rule supplies its conclusion only if every condition in its antecedent is in turn supplied.\allowbreak{} A rule only works forward:\allowbreak{} \char0{}\textbackslash{}\textless{}forall\textgreater{}x.\allowbreak{} A x \textbackslash{}\textless{}longrightarrow\textgreater{} B x\char0{} can supply B for an object that has A,\allowbreak{} and never supplies A.\allowbreak{} The chain is broken if a needed predicate has no supplier,\allowbreak{} or if a rule the chain relies on has an antecedent condition that nothing supplies.\allowbreak{}}
\pl{0}{2.\allowbreak{} CHECK THE SHARED OBJECTS.\allowbreak{} A supplier counts only if it lands on the object that needs it:\allowbreak{} a rule\textquotesingle{}s conclusion must reach,\allowbreak{} through its variables,\allowbreak{} the same object that the \char0{}shows\char0{} clause talks about.\allowbreak{} The chain is broken if the only supplier gives the predicate to a different object.\allowbreak{}}
\pl{0}{3.\allowbreak{} CHECK THE BINDINGS.\allowbreak{} A free variable inside an axiom is read as universally quantified,\allowbreak{} so an axiom that is meant to speak about particular individuals but leaves them free breaks the binding.\allowbreak{} A variable left free in the theorem stands for an arbitrary individual:\allowbreak{} the theorem claims its \char0{}shows\char0{} for every individual that satisfies its \char0{}assumes\char0{},\allowbreak{} so \char0{}assumes\char0{} and \char0{}shows\char0{} may share free variables,\allowbreak{} and such variables are acceptable.\allowbreak{}}
\pe
\pl{0}{Example (\allowbreak{}not from the evaluated data).\allowbreak{} Premise:\allowbreak{} \textquotedbl{}A chef slices bread.\allowbreak{}\textquotedbl{} Hypothesis:\allowbreak{} \textquotedbl{}A cook cuts food.\allowbreak{}\textquotedbl{} Explanation:\allowbreak{} \textquotedbl{}A chef is a cook,\allowbreak{} slicing is cutting,\allowbreak{} and bread is food.\allowbreak{}\textquotedbl{}}
\pl{0}{-{} broken:\allowbreak{} axioms \char0{}\textbackslash{}\textless{}forall\textgreater{}x.\allowbreak{} Chef x \textbackslash{}\textless{}longrightarrow\textgreater{} Cook x\char0{} and \char0{}\textbackslash{}\textless{}forall\textgreater{}x y.\allowbreak{} Slices x y \textbackslash{}\textless{}longrightarrow\textgreater{} Cuts x y\char0{},\allowbreak{} assumes \char0{}\textbackslash{}\textless{}exists\textgreater{}x y.\allowbreak{} Chef x \textbackslash{}\textless{}and\textgreater{} Bread y \textbackslash{}\textless{}and\textgreater{} Slices x y\char0{},\allowbreak{} shows \char0{}\textbackslash{}\textless{}exists\textgreater{}x y.\allowbreak{} Cook x \textbackslash{}\textless{}and\textgreater{} Food y \textbackslash{}\textless{}and\textgreater{} Cuts x y\char0{}.\allowbreak{} Nothing supplies Food.\allowbreak{}}
\pl{0}{-{} connected:\allowbreak{} the same module with the axiom \char0{}\textbackslash{}\textless{}forall\textgreater{}y.\allowbreak{} Bread y \textbackslash{}\textless{}longrightarrow\textgreater{} Food y\char0{} added.\allowbreak{}}
\pe
\pl{0}{JUDGE ONLY THE CONNECTIONS.\allowbreak{} A module whose chain connects passes even if one of its statements says the wrong thing,\allowbreak{} is not licensed by the explanation,\allowbreak{} or leaves out part of the NL;\allowbreak{} those are judged on other axes.\allowbreak{} Do not decide whether a prover would close the proof.\allowbreak{} Names the prover cannot resolve are not your concern.\allowbreak{}}
\pe
\pl{0}{Output format,\allowbreak{} for each module two lines in this order:\allowbreak{}}
\pl{0}{CHAIN \textless{}id\textgreater{}:\allowbreak{} \textless{}for the module that holds the \char0{}shows\char0{} clause,\allowbreak{} each predicate the \char0{}shows\char0{} clause needs,\allowbreak{} written as \char0{}Pred \textless{}-{} supplier\char0{},\allowbreak{} where the supplier is the \char0{}assumes\char0{} clause,\allowbreak{} a premise axiom,\allowbreak{} or a rule written as \char0{}rule name (\allowbreak{}concludes Pred)\char0{} whose conditions are listed the same way;\allowbreak{} a rule that does not conclude Pred cannot be its supplier;\allowbreak{} write \char0{}Pred \textless{}-{} NONE\char0{} when nothing supplies it;\allowbreak{} for every other module,\allowbreak{} ALL SUPPLIED\textgreater{}}
\pl{0}{MODULE \textless{}id\textgreater{} \textbar{} VERDICT:\allowbreak{} pass\textbar{}fail\textbar{}unknown \textbar{} REASON:\allowbreak{} \textless{}one sentence\textgreater{}}
\pe
\pl{0}{The CHAIN line comes first.\allowbreak{} If it contains \char0{}\textless{}-{} NONE\char0{},\allowbreak{} the verdict is fail.\allowbreak{} Use VERDICT:\allowbreak{} unknown only when you cannot tell whether a supplier lands on the right object.\allowbreak{}}
\pl{0}{Output nothing other than these lines.\allowbreak{}}
\end{promptlines}

\noindent{\small\bfseries USER:}\par
\begin{promptlines}
\pl{0}{Informal:\allowbreak{}}
\pl{0}{Premise:\allowbreak{} Two women are wet while holding hands with a long building and buses in the background.\allowbreak{}}
\pl{0}{Hypothesis:\allowbreak{} Women are in the city.\allowbreak{}}
\pl{0}{Explanation:\allowbreak{}\ \ There are women because there are two women.\allowbreak{}}
\pe
\pl{0}{Full Isabelle .\allowbreak{}thy file (\allowbreak{}context):\allowbreak{}}
\pl{0}{\char0{}\char0{}\char0{}}
\pl{0}{theory eval\_\allowbreak{}72}
\pl{0}{imports Main}
\pl{0}{begin}
\pe
\pl{0}{typedecl entity}
\pl{0}{typedecl event}
\pe
\pl{0}{consts}
\pl{2}{TwoWomen :\allowbreak{}:\allowbreak{} \textquotedbl{}entity \textbackslash{}\textless{}Rightarrow\textgreater{} bool\textquotedbl{}}
\pl{2}{Women :\allowbreak{}:\allowbreak{} \textquotedbl{}entity \textbackslash{}\textless{}Rightarrow\textgreater{} bool\textquotedbl{}}
\pl{2}{Wet :\allowbreak{}:\allowbreak{} \textquotedbl{}entity \textbackslash{}\textless{}Rightarrow\textgreater{} bool\textquotedbl{}}
\pl{2}{HoldingHands :\allowbreak{}:\allowbreak{} \textquotedbl{}entity \textbackslash{}\textless{}Rightarrow\textgreater{} bool\textquotedbl{}}
\pl{2}{LongBuilding :\allowbreak{}:\allowbreak{} \textquotedbl{}entity \textbackslash{}\textless{}Rightarrow\textgreater{} bool\textquotedbl{}}
\pl{2}{Buses :\allowbreak{}:\allowbreak{} \textquotedbl{}entity \textbackslash{}\textless{}Rightarrow\textgreater{} bool\textquotedbl{}}
\pl{2}{InBackground :\allowbreak{}:\allowbreak{} \textquotedbl{}entity \textbackslash{}\textless{}Rightarrow\textgreater{} bool\textquotedbl{}}
\pl{2}{InCity :\allowbreak{}:\allowbreak{} \textquotedbl{}entity \textbackslash{}\textless{}Rightarrow\textgreater{} bool\textquotedbl{}}
\pe
\pl{0}{axiomatization where}
\pl{2}{explanation\_\allowbreak{}1:\allowbreak{} \textquotedbl{}\textbackslash{}\textless{}forall\textgreater{}x.\allowbreak{} TwoWomen x \textbackslash{}\textless{}longrightarrow\textgreater{} Women x\textquotedbl{}}
\pe
\pl{0}{theorem hypothesis:\allowbreak{}}
\pl{2}{assumes asm:\allowbreak{} \textquotedbl{}TwoWomen x \textbackslash{}\textless{}and\textgreater{} Wet x \textbackslash{}\textless{}and\textgreater{} HoldingHands x \textbackslash{}\textless{}and\textgreater{} LongBuilding y \textbackslash{}\textless{}and\textgreater{} Buses z \textbackslash{}\textless{}and\textgreater{} InBackground x\textquotedbl{}}
\pl{2}{shows \textquotedbl{}\textbackslash{}\textless{}exists\textgreater{}x.\allowbreak{} Women x \textbackslash{}\textless{}and\textgreater{} InCity x\textquotedbl{}}
\pl{2}{sorry}
\pe
\pl{0}{end}
\pl{0}{\char0{}\char0{}\char0{}}
\pe
\pl{0}{Modules to evaluate (\allowbreak{}one verdict line per module):\allowbreak{}}
\pl{0}{-{} MODULE typedecl:\allowbreak{}entity\ \ (\allowbreak{}lines 5-{}5,\allowbreak{} kind=\allowbreak{}typedecl)}
\pl{0}{-{} MODULE typedecl:\allowbreak{}event\ \ (\allowbreak{}lines 6-{}7,\allowbreak{} kind=\allowbreak{}typedecl)}
\pl{0}{-{} MODULE consts:\allowbreak{}1\ \ (\allowbreak{}lines 8-{}17,\allowbreak{} kind=\allowbreak{}consts)}
\pl{0}{-{} MODULE axiomatization:\allowbreak{}1\ \ (\allowbreak{}lines 18-{}20,\allowbreak{} kind=\allowbreak{}axiomatization)}
\pl{0}{-{} MODULE theorem:\allowbreak{}hypothesis\ \ (\allowbreak{}lines 21-{}25,\allowbreak{} kind=\allowbreak{}theorem)}
\end{promptlines}

\end{promptbox}

\subsection{Semantic-Fidelity Judge}
\begin{promptbox}{Lean, step 1 (listing conditions and conclusions) \textmd{(example: ProofNet problem \texttt{exercise\_4\_4})}}
\noindent{\small\bfseries USER:}\par
\begin{promptlines}
\pl{0}{Help me list the conditions and conclusions in this problem,\allowbreak{} without solving it and without judging whether the conclusions follow.\allowbreak{}}
\pe
\pl{0}{Copy each condition and each conclusion from the problem as it is written,\allowbreak{} keeping its words and its formulas,\allowbreak{} one per line,\allowbreak{} in the order given.\allowbreak{} Do not rewrite a condition in other words or as a different formula,\allowbreak{} do not paraphrase it,\allowbreak{} and do not merge or split sentences.\allowbreak{} Keep every condition,\allowbreak{} including conditions that play no part in reaching the conclusions.\allowbreak{} For an entailment problem,\allowbreak{} which gives a premise,\allowbreak{} an explanation and a hypothesis,\allowbreak{} the conditions are the premise and every sentence of the explanation,\allowbreak{} and the conclusion is the hypothesis.\allowbreak{}}
\pe
\pl{0}{Some problems ask a question and also state its answer,\allowbreak{} for example \textquotedbl{}What is the value of \$x\$? (\allowbreak{}A) 1 (\allowbreak{}B) 2 (\allowbreak{}C) 3 Show that it is (\allowbreak{}B) 2.\allowbreak{}\textquotedbl{} For such a problem,\allowbreak{} the conclusion is the sentence that states the answer.\allowbreak{} The question and the list of answer choices are neither conditions nor conclusions,\allowbreak{} so do not list them.\allowbreak{}}
\pe
\pl{0}{Use exactly this output format and nothing else:\allowbreak{}}
\pl{0}{Conditions:\allowbreak{}}
\pl{0}{C1:\allowbreak{} \textless{}the first condition,\allowbreak{} copied as written\textgreater{}}
\pl{0}{C2:\allowbreak{} \textless{}the next condition\textgreater{}}
\pl{0}{.\allowbreak{}.\allowbreak{}.\allowbreak{}}
\pl{0}{Conclusions:\allowbreak{}}
\pl{0}{H1:\allowbreak{} \textless{}the first conclusion,\allowbreak{} copied as written\textgreater{}}
\pl{0}{.\allowbreak{}.\allowbreak{}.\allowbreak{}}
\pe
\pl{0}{Number the conditions C1,\allowbreak{} C2,\allowbreak{} C3,\allowbreak{} .\allowbreak{}.\allowbreak{}.\allowbreak{} and the conclusions H1,\allowbreak{} H2,\allowbreak{} .\allowbreak{}.\allowbreak{}.\allowbreak{} without gaps.\allowbreak{}}
\pe
\pl{0}{Here are three examples.\allowbreak{}}
\pe
\pl{0}{Example 1:\allowbreak{}}
\pl{0}{[\allowbreak{}Problem]:\allowbreak{} The sequence \$\textbackslash{}\{\allowbreak{}a\_\allowbreak{}n\textbackslash{}\}\$ satisfies \$a\_\allowbreak{}1 =\allowbreak{} 1\$,\allowbreak{} \$a\_\allowbreak{}2 =\allowbreak{} 2\$,\allowbreak{} \$a\_\allowbreak{}\{\allowbreak{}n + 2\}=\allowbreak{}2a\_\allowbreak{}\{\allowbreak{}n + 1\}-{}a\_\allowbreak{}n + 2\$.\allowbreak{} Let \$b\_\allowbreak{}n=\allowbreak{}a\_\allowbreak{}\{\allowbreak{}n + 1\}-{}a\_\allowbreak{}n\$.\allowbreak{} Prove that \$\textbackslash{}\{\allowbreak{}b\_\allowbreak{}n\textbackslash{}\}\$ is an arithmetic sequence.\allowbreak{}}
\pe
\pl{0}{[\allowbreak{}Conditions and Conclusions]:\allowbreak{}}
\pl{0}{Conditions:\allowbreak{}}
\pl{0}{C1:\allowbreak{} \$a\_\allowbreak{}1 =\allowbreak{} 1\$}
\pl{0}{C2:\allowbreak{} \$a\_\allowbreak{}2 =\allowbreak{} 2\$}
\pl{0}{C3:\allowbreak{} \$a\_\allowbreak{}\{\allowbreak{}n + 2\}=\allowbreak{}2a\_\allowbreak{}\{\allowbreak{}n + 1\}-{}a\_\allowbreak{}n + 2\$}
\pl{0}{C4:\allowbreak{} Let \$b\_\allowbreak{}n=\allowbreak{}a\_\allowbreak{}\{\allowbreak{}n + 1\}-{}a\_\allowbreak{}n\$.\allowbreak{}}
\pl{0}{Conclusions:\allowbreak{}}
\pl{0}{H1:\allowbreak{} \$\textbackslash{}\{\allowbreak{}b\_\allowbreak{}n\textbackslash{}\}\$ is an arithmetic sequence.\allowbreak{}}
\pe
\pl{0}{Examples 2 and 3 are made up for illustration;\allowbreak{} neither is taken from any dataset.\allowbreak{}}
\pe
\pl{0}{Example 2:\allowbreak{}}
\pl{0}{[\allowbreak{}Problem]:\allowbreak{}}
\pl{0}{Premise:\allowbreak{} A man in a blue jacket is repairing a bicycle on the sidewalk.\allowbreak{}}
\pl{0}{Explanation:\allowbreak{} A man repairing a bicycle is fixing something.\allowbreak{} A sidewalk is outdoors.\allowbreak{}}
\pl{0}{Hypothesis:\allowbreak{} A man is fixing something outdoors.\allowbreak{}}
\pe
\pl{0}{[\allowbreak{}Conditions and Conclusions]:\allowbreak{}}
\pl{0}{Conditions:\allowbreak{}}
\pl{0}{C1:\allowbreak{} A man in a blue jacket is repairing a bicycle on the sidewalk.\allowbreak{}}
\pl{0}{C2:\allowbreak{} A man repairing a bicycle is fixing something.\allowbreak{}}
\pl{0}{C3:\allowbreak{} A sidewalk is outdoors.\allowbreak{}}
\pl{0}{Conclusions:\allowbreak{}}
\pl{0}{H1:\allowbreak{} A man is fixing something outdoors.\allowbreak{}}
\pe
\pl{0}{Example 3:\allowbreak{}}
\pl{0}{[\allowbreak{}Problem]:\allowbreak{}}
\pl{0}{Premise:\allowbreak{} Quill is a flimpus.\allowbreak{}}
\pl{0}{Explanation:\allowbreak{}}
\pl{0}{E1:\allowbreak{} Every flimpus is a trobus.\allowbreak{}}
\pl{0}{E2:\allowbreak{} Flimpuses are loud.\allowbreak{}}
\pl{0}{E3:\allowbreak{} Each trobus is not sticky.\allowbreak{}}
\pl{0}{E4:\allowbreak{} Every gaskus is sticky.\allowbreak{}}
\pl{0}{Hypothesis:\allowbreak{} Quill is not sticky}
\pe
\pl{0}{[\allowbreak{}Conditions and Conclusions]:\allowbreak{}}
\pl{0}{Conditions:\allowbreak{}}
\pl{0}{C1:\allowbreak{} Quill is a flimpus.\allowbreak{}}
\pl{0}{C2:\allowbreak{} Every flimpus is a trobus.\allowbreak{}}
\pl{0}{C3:\allowbreak{} Flimpuses are loud.\allowbreak{}}
\pl{0}{C4:\allowbreak{} Each trobus is not sticky.\allowbreak{}}
\pl{0}{C5:\allowbreak{} Every gaskus is sticky.\allowbreak{}}
\pl{0}{Conclusions:\allowbreak{}}
\pl{0}{H1:\allowbreak{} Quill is not sticky}
\pe
\pl{0}{Now,\allowbreak{} please list the conditions and conclusions for this problem in the same way,\allowbreak{} without solving it:\allowbreak{}}
\pl{0}{[\allowbreak{}Problem]:\allowbreak{} Consider a prime \$p\$ of the form \$4 t+1\$.\allowbreak{} Show that \$a\$ is a primitive root modulo \$p\$ iff \$-{}a\$ is a primitive root modulo \$p\$.\allowbreak{}}
\pe
\pl{0}{[\allowbreak{}Conditions and Conclusions]:\allowbreak{}}
\pe
\end{promptlines}

\end{promptbox}

\begin{promptbox}{Lean, step 2 (comparison) \textmd{(example: ProofNet problem \texttt{exercise\_4\_4})}}
\noindent{\small\bfseries USER:}\par
\begin{promptlines}
\pl{0}{Here is a problem and a formal statement of it (\allowbreak{}in Lean 4 or in Isabelle/\allowbreak{}HOL).\allowbreak{} Compare the conditions and conclusions of the problem with the formal statement,\allowbreak{} matching them one by one to see if the formal statement is an appropriate translation of each of them by assigning one of three tags (\allowbreak{}Perfectly match;\allowbreak{} Minor inconsistency;\allowbreak{} Major inconsistency).\allowbreak{} Compare every condition C1,\allowbreak{} C2,\allowbreak{} .\allowbreak{}.\allowbreak{}.\allowbreak{} in order,\allowbreak{} including conditions that play no part in reaching the conclusions,\allowbreak{} then compare each conclusion H1,\allowbreak{} H2,\allowbreak{} .\allowbreak{}.\allowbreak{}.\allowbreak{} as its own separate item.\allowbreak{} Then,\allowbreak{} audit for missing/\allowbreak{}implicit conditions.\allowbreak{}}
\pe
\pl{0}{Compare meaning,\allowbreak{} not names.\allowbreak{} Predicate,\allowbreak{} function,\allowbreak{} constant and variable names are only labels:\allowbreak{} a name that differs from the words of the problem,\allowbreak{} or a single predicate that carries several words of the problem (\allowbreak{}for example \char0{}InBlueJacket x\char0{} for \textquotedbl{}in a blue jacket\textquotedbl{}),\allowbreak{} is not an inconsistency by itself.\allowbreak{} What counts is what the formal statement says:\allowbreak{} which properties and relations it asserts of which entities,\allowbreak{} the direction of each implication,\allowbreak{} negation,\allowbreak{} quantifiers,\allowbreak{} connectives and numbers.\allowbreak{} Judge with extremely strict standards on meaning—any minor difference in meaning will be considered a mismatch.\allowbreak{}}
\pe
\pl{0}{What counts as rendering an item.\allowbreak{} Read the whole problem,\allowbreak{} not only the listed item,\allowbreak{} and apply these points:\allowbreak{}}
\pl{0}{-{} Some conditions only tell the story of a word problem:\allowbreak{} who the people are,\allowbreak{} what they are doing,\allowbreak{} what the objects are called.\allowbreak{} Story details that carry no number or relation of their own need no counterpart in the formal statement.\allowbreak{} Tag such a condition Perfectly match when the numbers and relations it does carry are rendered,\allowbreak{} which may be inside a computation rather than in a hypothesis of its own.\allowbreak{} For example,\allowbreak{} in \textquotedbl{}Ann has 12 apples and gives 5 to Ben.\allowbreak{} How many apples does Ann have left? Show that it is 7.\allowbreak{}\textquotedbl{},\allowbreak{} the formal statement \char0{}12 -{} 5 =\allowbreak{} 7\char0{} renders everything;\allowbreak{} the names Ann and Ben and the word apples need no counterpart.\allowbreak{}}
\pl{0}{-{} If an item is the question of the problem (\allowbreak{}a sentence that asks for a value,\allowbreak{} such as \textquotedbl{}How many .\allowbreak{}.\allowbreak{}.\allowbreak{}?\textquotedbl{} or \textquotedbl{}What is the smallest .\allowbreak{}.\allowbreak{}.\allowbreak{}?\textquotedbl{}),\allowbreak{} it is not a hypothesis and needs no hypothesis of its own.\allowbreak{} Tag it,\allowbreak{} and the conclusion that states the answer,\allowbreak{} by whether the conclusion of the formal statement gives that answer together with every requirement the question puts on it (\allowbreak{}smallest,\allowbreak{} largest,\allowbreak{} positive,\allowbreak{} an integer,\allowbreak{} other than 1,\allowbreak{} .\allowbreak{}.\allowbreak{}.\allowbreak{}).\allowbreak{} A requirement of the question that the formal statement leaves out is an inconsistency.\allowbreak{} For a smallest or largest value,\allowbreak{} the usual formal statements are these:\allowbreak{} every value with the required properties is at least (\allowbreak{}at most) the answer;\allowbreak{} or,\allowbreak{} for the minimum or maximum of a function,\allowbreak{} the value at a point where it is attained equals the answer.\allowbreak{} Either may also state that the answer itself has the properties.\allowbreak{} A formal statement that only shows that the answer has the properties,\allowbreak{} and does not say that it is the smallest (\allowbreak{}largest),\allowbreak{} leaves out that requirement.\allowbreak{} For example,\allowbreak{} for \textquotedbl{}What is the smallest positive integer \$n\$ with \$n\textasciicircum{}2 \textgreater{} 50\$? Show that it is 8.\allowbreak{}\textquotedbl{},\allowbreak{} the statement \char0{}(\allowbreak{}n :\allowbreak{} ℕ) (\allowbreak{}h₀ :\allowbreak{} 0 \textless{} n) (\allowbreak{}h₁ :\allowbreak{} 50 \textless{} n \textasciicircum{} 2) :\allowbreak{} 8 ≤ n\char0{} is consistent,\allowbreak{} and \char0{}50 \textless{} 8 \textasciicircum{} 2\char0{} alone is not.\allowbreak{}}
\pl{0}{-{} An assumption that only puts the problem in the standard setting of its subject,\allowbreak{} which the problem and the textbook it comes from take for granted,\allowbreak{} is not an added condition.\allowbreak{} Such an assumption states what kind of structure the objects are,\allowbreak{} for example that the scalars of a vector space form a field,\allowbreak{} or that the functions of a real-{}analysis exercise take real values.\allowbreak{} An assumption that asks more of the problem\textquotesingle{}s own objects than the problem does (\allowbreak{}an extra property of a given object,\allowbreak{} a smaller domain,\allowbreak{} a stronger hypothesis) is an added condition.\allowbreak{} Assuming more than the problem does changes which theorem is stated,\allowbreak{} so tag it Major inconsistency;\allowbreak{} this covers an added hypothesis,\allowbreak{} a hypothesis required on a larger domain than the problem\textquotesingle{}s,\allowbreak{} and a stronger property than the problem\textquotesingle{}s (\allowbreak{}for example \$x \textgreater{} 0\$ where the problem has \$x \textbackslash{}ge 0\$).\allowbreak{} A conclusion that says less than the problem\textquotesingle{}s conclusion is likewise a Major inconsistency.\allowbreak{} Minor inconsistency is for a difference in how something is written that does not change which theorem is stated.\allowbreak{}}
\pl{0}{-{} A formal statement that is written differently but means the same thing given the problem\textquotesingle{}s own conditions is consistent.\allowbreak{} For example,\allowbreak{} for \textquotedbl{}For all odd \$n\$ show that \$8 \textbackslash{}mid n\textasciicircum{}\{\allowbreak{}2\}-{}1\$\textquotedbl{},\allowbreak{} the statement \char0{}\{\allowbreak{}n :\allowbreak{} ℕ\} (\allowbreak{}hn :\allowbreak{} Odd n) :\allowbreak{} 8 ∣ (\allowbreak{}n\textasciicircum{}2 -{} 1)\char0{} is consistent:\allowbreak{} it uses natural-{}number subtraction,\allowbreak{} but since \$n\$ is odd,\allowbreak{} \$n\textasciicircum{}2 \textbackslash{}ge 1\$ and the subtraction is exact.\allowbreak{} Tag each difference once.\allowbreak{} If a difference has already been tagged on a condition or a conclusion,\allowbreak{} the final audit does not tag it again.\allowbreak{}}
\pl{0}{-{} Read the formal statement by the semantics of its language,\allowbreak{} and tag an inconsistency where that reading differs from the problem,\allowbreak{} even when the symbols look like the problem\textquotesingle{}s.\allowbreak{} In Lean 4 and Mathlib,\allowbreak{} the type of a variable is its domain,\allowbreak{} so \char0{}n :\allowbreak{} ℕ\char0{} includes 0 and \textquotedbl{}positive integer\textquotedbl{} needs \char0{}0 \textless{} n\char0{}.\allowbreak{} Subtraction on ℕ stops at 0 (\allowbreak{}\char0{}3 -{} 5 =\allowbreak{} 0\char0{}),\allowbreak{} and \char0{}/\allowbreak{}\char0{} on ℕ and ℤ rounds down (\allowbreak{}\char0{}7 /\allowbreak{} 2 =\allowbreak{} 3\char0{}).\allowbreak{} Division by zero gives 0,\allowbreak{} and the integral of a function that is not integrable is 0.\allowbreak{} A function declared on a whole type is given at every point of that type,\allowbreak{} which says more than a function given only on part of it.\allowbreak{} A Mathlib name means what its definition says,\allowbreak{} including which argument is the base of a logarithm and which way an order between structures goes.\allowbreak{} To apply this,\allowbreak{} look at the type of each variable and ask which values it allows that the problem does not intend,\allowbreak{} in particular 0 and 1 for a natural number,\allowbreak{} and points outside the set on which the problem\textquotesingle{}s function is defined.\allowbreak{} Substitute such a value and check whether the formal statement still says what the problem says.\allowbreak{} If at that value the formal statement makes a claim that the problem does not make,\allowbreak{} for example because a subtraction is cut off at 0,\allowbreak{} that is an inconsistency.\allowbreak{} For an integral,\allowbreak{} check whether the function is integrable on the whole domain:\allowbreak{} an improper integral that converges only conditionally is not a Lebesgue integral,\allowbreak{} and Lean\textquotesingle{}s \char0{}∫\char0{} of such a function is 0.\allowbreak{}}
\pe
\pl{0}{The task that produced this formal statement gave the model the instructions and worked examples quoted below,\allowbreak{} in full and word for word,\allowbreak{} between the two marker lines.\allowbreak{} A formal statement written the way these instructions and examples prescribe is not inconsistent on that account.\allowbreak{} Judge only whether each item says what the problem says.\allowbreak{}}
\pe
\pl{0}{=\allowbreak{}=\allowbreak{}=\allowbreak{}=\allowbreak{}=\allowbreak{} TASK INSTRUCTIONS AND EXAMPLES (\allowbreak{}quoted) =\allowbreak{}=\allowbreak{}=\allowbreak{}=\allowbreak{}=\allowbreak{}}
\pl{0}{You are an expert in formal language Lean4.\allowbreak{}}
\pl{0}{You will be given a mathematical statement written in natural language and LaTeX symbols.\allowbreak{}}
\pl{0}{Your task is to provide the formal code of the given natural language mathematical statement in Lean4 with the following instructions:\allowbreak{}}
\pl{0}{1.\allowbreak{} You should give the formal code directly without any additional explanation.\allowbreak{} If the given mathematical statement is a theorem or lemma,\allowbreak{} omit the proof and use the default \textquotesingle{}sorry\textquotesingle{} mode in the formal code.\allowbreak{}2.\allowbreak{} In case that you need to import any necessary preambles,\allowbreak{} you should not import any fake (\allowbreak{}non-{}exist) preambles.\allowbreak{}}
\pl{0}{3.\allowbreak{} You should wrap the formal code in a way illustrated as the following:\allowbreak{}}
\pl{0}{\char0{}\char0{}\char0{}lean}
\pl{0}{Your Formal Code}
\pl{0}{\char0{}\char0{}\char0{}}
\pl{0}{4.\allowbreak{} Lean 4 /\allowbreak{} Mathlib notation constraint — for \char0{}Finset\char0{} summation and product notation,\allowbreak{} use the MEMBERSHIP form with \char0{}∈\char0{} (\allowbreak{}Unicode U+2208),\allowbreak{} NOT the legacy \char0{}in\char0{} keyword form.\allowbreak{} Mathlib\textquotesingle{}s current BigOperators notation requires:\allowbreak{}}
\pl{3}{✓ \char0{}∑ k ∈ Finset.\allowbreak{}range 12,\allowbreak{} f k\char0{}\ \ \ \ \ -{}-{} correct}
\pl{3}{✗ \char0{}∑ k in Finset.\allowbreak{}range 12,\allowbreak{} f k\char0{}\ \ \ \ -{}-{} INVALID,\allowbreak{} fails to parse with \char0{}unexpected token \textquotesingle{}in\textquotesingle{}\char0{}}
\pl{0}{Same rule applies to \char0{}∏\char0{},\allowbreak{} \char0{}⋃\char0{},\allowbreak{} \char0{}⋂\char0{},\allowbreak{} and any other indexed-{}set binder:\allowbreak{} use \char0{}∈\char0{},\allowbreak{} not \char0{}in\char0{}.\allowbreak{}}
\pl{0}{Strictly follow the instructions that have been claimed.\allowbreak{}}
\pe
\pl{0}{Worked example 1,\allowbreak{} input given to the model:\allowbreak{}}
\pl{0}{Natural language statement:\allowbreak{} Prove that there is no rational number whose square is \$12\$.\allowbreak{}}
\pl{0}{Give me the Lean4 formal code of the statement:\allowbreak{}}
\pe
\pl{0}{Worked example 1,\allowbreak{} answer:\allowbreak{}}
\pl{0}{\char0{}\char0{}\char0{}lean}
\pl{0}{import Mathlib}
\pe
\pl{0}{open Topology Filter Real Complex TopologicalSpace Finset}
\pl{0}{open scoped BigOperators}
\pe
\pl{0}{theorem exercise\_\allowbreak{}1\_\allowbreak{}2 :\allowbreak{} ¬ ∃ (\allowbreak{}x :\allowbreak{} ℚ),\allowbreak{} (\allowbreak{} x \textasciicircum{} 2 =\allowbreak{} 12 ) :\allowbreak{}=\allowbreak{} sorry}
\pl{0}{\char0{}\char0{}\char0{}}
\pe
\pl{0}{Worked example 2,\allowbreak{} input given to the model:\allowbreak{}}
\pl{0}{Natural language statement:\allowbreak{} For all odd \$n\$ show that \$8 \textbackslash{}mid n\textasciicircum{}\{\allowbreak{}2\}-{}1\$.\allowbreak{}}
\pl{0}{Give me the Lean4 formal code of the statement:\allowbreak{}}
\pe
\pl{0}{Worked example 2,\allowbreak{} answer:\allowbreak{}}
\pl{0}{\char0{}\char0{}\char0{}lean}
\pl{0}{import Mathlib}
\pe
\pl{0}{open Real}
\pl{0}{open scoped BigOperators}
\pe
\pl{0}{theorem exercise\_\allowbreak{}1\_\allowbreak{}27 \{\allowbreak{}n :\allowbreak{} ℕ\} (\allowbreak{}hn :\allowbreak{} Odd n) :\allowbreak{} 8 ∣ (\allowbreak{}n\textasciicircum{}2 -{} 1) :\allowbreak{}=\allowbreak{} sorry}
\pl{0}{\char0{}\char0{}\char0{}}
\pe
\pl{0}{Worked example 3,\allowbreak{} input given to the model:\allowbreak{}}
\pl{0}{Natural language statement:\allowbreak{} Prove that \$x\textasciicircum{}3 -{} 3x -{} 1\$ is irreducible over \$\textbackslash{}mathbb\{\allowbreak{}Q\}\$.\allowbreak{}}
\pl{0}{Give me the Lean4 formal code of the statement:\allowbreak{}}
\pe
\pl{0}{Worked example 3,\allowbreak{} answer:\allowbreak{}}
\pl{0}{\char0{}\char0{}\char0{}lean}
\pl{0}{import Mathlib}
\pe
\pl{0}{open Fintype Set Real Ideal Polynomial}
\pl{0}{open scoped BigOperators}
\pe
\pl{0}{theorem exercise\_\allowbreak{}5\_\allowbreak{}5\_\allowbreak{}2 :\allowbreak{} Irreducible (\allowbreak{}X\textasciicircum{}3 -{} 3*X -{} 1 :\allowbreak{} Polynomial ℚ) :\allowbreak{}=\allowbreak{} sorry}
\pl{0}{\char0{}\char0{}\char0{}}
\pl{0}{=\allowbreak{}=\allowbreak{}=\allowbreak{}=\allowbreak{}=\allowbreak{} END OF TASK INSTRUCTIONS AND EXAMPLES =\allowbreak{}=\allowbreak{}=\allowbreak{}=\allowbreak{}=\allowbreak{}}
\pe
\pl{0}{Write every item in the format of the examples below:\allowbreak{} the line of the formal statement that renders it,\allowbreak{} one sentence of interpretation,\allowbreak{} then the Match line.\allowbreak{} Write every tag exactly as \textbackslash{}box\{\allowbreak{}Perfectly match\},\allowbreak{} \textbackslash{}box\{\allowbreak{}Minor inconsistency\} or \textbackslash{}box\{\allowbreak{}Major inconsistency\}.\allowbreak{} A tag written any other way (\allowbreak{}in bold,\allowbreak{} in plain text,\allowbreak{} or outside \textbackslash{}box\{\allowbreak{}\}) is not read.\allowbreak{}}
\pl{0}{**Stop immediately** after evaluating all pairs.\allowbreak{} Do **not** summarize or analyze further.\allowbreak{}}
\pe
\pl{0}{Output Format,\allowbreak{} example 1 (\allowbreak{}a Lean 4 statement):\allowbreak{}}
\pl{0}{Let\textquotesingle{}s compare the conditions and conclusions with the formal statement one by one:\allowbreak{}}
\pe
\pl{0}{C1.\allowbreak{} **\textbackslash{}(\allowbreak{} q \textbackslash{}) is a natural number greater than 1**:\allowbreak{}}
\pl{3}{-{} Formal:\allowbreak{} \char0{}(\allowbreak{}hq :\allowbreak{} 1 \textless{} q)\char0{}.\allowbreak{}}
\pl{3}{-{} Interpretation:\allowbreak{} The hypothesis says \textbackslash{}(\allowbreak{} q \textgreater{} 1 \textbackslash{}) for a natural number \textbackslash{}(\allowbreak{} q \textbackslash{}),\allowbreak{} which is what the condition says.\allowbreak{}}
\pl{3}{-{} Match:\allowbreak{} \textbackslash{}box\{\allowbreak{}Perfectly match\}.\allowbreak{}}
\pe
\pl{0}{C2.\allowbreak{} **\textbackslash{}(\allowbreak{} n \textbackslash{}) is a natural number greater than 1**:\allowbreak{}}
\pl{3}{-{} Formal:\allowbreak{} \char0{}(\allowbreak{}hn :\allowbreak{} 1 \textless{} n)\char0{}.\allowbreak{}}
\pl{3}{-{} Interpretation:\allowbreak{} The hypothesis says \textbackslash{}(\allowbreak{} n \textgreater{} 1 \textbackslash{}) for a natural number \textbackslash{}(\allowbreak{} n \textbackslash{}),\allowbreak{} which is what the condition says.\allowbreak{}}
\pl{3}{-{} Match:\allowbreak{} \textbackslash{}box\{\allowbreak{}Perfectly match\}.\allowbreak{}}
\pe
\pl{0}{C3.\allowbreak{} **Set \textbackslash{}(\allowbreak{} M =\allowbreak{} \textbackslash{}\{\allowbreak{}0,\allowbreak{} 1,\allowbreak{} 2,\allowbreak{} \textbackslash{}cdots,\allowbreak{} q -{} 1\textbackslash{}\} \textbackslash{})**:\allowbreak{}}
\pl{3}{-{} Formal:\allowbreak{} \char0{}(\allowbreak{}M :\allowbreak{} Finset ℕ :\allowbreak{}=\allowbreak{} Finset.\allowbreak{}range q)\char0{}.\allowbreak{}}
\pl{3}{-{} Interpretation:\allowbreak{} \char0{}Finset.\allowbreak{}range q\char0{} is \char0{}\{\allowbreak{}0,\allowbreak{} 1,\allowbreak{} .\allowbreak{}.\allowbreak{}.\allowbreak{},\allowbreak{} q -{} 1\}\char0{}.\allowbreak{}}
\pl{3}{-{} Match:\allowbreak{} \textbackslash{}box\{\allowbreak{}Perfectly match\}.\allowbreak{}}
\pe
\pl{0}{C4.\allowbreak{} **Set \textbackslash{}(\allowbreak{} A \textbackslash{}) definition**:\allowbreak{}}
\pl{3}{-{} Formal:\allowbreak{} \char0{}A :\allowbreak{} Set ℕ :\allowbreak{}=\allowbreak{} \{\allowbreak{}x \textbar{} ∃ (\allowbreak{}x\_\allowbreak{}vec :\allowbreak{} ℕ → ℕ),\allowbreak{} (\allowbreak{}∀ i,\allowbreak{} x\_\allowbreak{}vec i ∈ M) ∧ x =\allowbreak{} ∑ i in Finset.\allowbreak{}range n,\allowbreak{} x\_\allowbreak{}vec (\allowbreak{}i + 1) * q \textasciicircum{} i\}\char0{}.\allowbreak{}}
\pl{3}{-{} Interpretation:\allowbreak{} The Lean definition is technically correct but slightly more abstract than the math,\allowbreak{} since \char0{}x\_\allowbreak{}vec\char0{} is a function on all naturals of which only \char0{}x\_\allowbreak{}vec (\allowbreak{}i + 1)\char0{} for \char0{}i\char0{} in \char0{}Finset.\allowbreak{}range n\char0{} is used.\allowbreak{}}
\pl{3}{-{} Match:\allowbreak{} \textbackslash{}box\{\allowbreak{}Minor inconsistency\}.\allowbreak{}}
\pe
\pl{0}{C5.\allowbreak{} **\textbackslash{}(\allowbreak{} s,\allowbreak{} t \textbackslash{}in A \textbackslash{}) with specific expansions**:\allowbreak{}}
\pl{3}{-{} Formal:\allowbreak{} \char0{}s =\allowbreak{} ∑ i in Finset.\allowbreak{}range n,\allowbreak{} a (\allowbreak{}i + 1) * q \textasciicircum{} i\char0{},\allowbreak{} \char0{}t =\allowbreak{} ∑ i in Finset.\allowbreak{}range n,\allowbreak{} b (\allowbreak{}i + 1) * q \textasciicircum{} i\char0{},\allowbreak{} with \char0{}∀ i,\allowbreak{} a i ∈ M\char0{} and \char0{}∀ i,\allowbreak{} b i ∈ M\char0{}.\allowbreak{}}
\pl{3}{-{} Interpretation:\allowbreak{} The Lean version is correct but not a literal translation,\allowbreak{} since it uses \char0{}a i\char0{} for all \char0{}i\char0{} but only evaluates it at \char0{}i + 1\char0{}.\allowbreak{}}
\pl{3}{-{} Match:\allowbreak{} \textbackslash{}box\{\allowbreak{}Minor inconsistency\}.\allowbreak{}}
\pe
\pl{0}{C6.\allowbreak{} **\textbackslash{}(\allowbreak{} a\_\allowbreak{}n \textless{} b\_\allowbreak{}n \textbackslash{})**:\allowbreak{}}
\pl{3}{-{} Formal:\allowbreak{} \char0{}(\allowbreak{}hab :\allowbreak{} a n \textless{} b n)\char0{}.\allowbreak{}}
\pl{3}{-{} Interpretation:\allowbreak{} The hypothesis says \textbackslash{}(\allowbreak{} a\_\allowbreak{}n \textless{} b\_\allowbreak{}n \textbackslash{}),\allowbreak{} which is what the condition says.\allowbreak{}}
\pl{3}{-{} Match:\allowbreak{} \textbackslash{}box\{\allowbreak{}Perfectly match\}.\allowbreak{}}
\pe
\pl{0}{H1.\allowbreak{} **Conclusion:\allowbreak{} \textbackslash{}(\allowbreak{} s \textless{} t \textbackslash{})**:\allowbreak{}}
\pl{3}{-{} Formal:\allowbreak{} \char0{}s \textless{}=\allowbreak{} t\char0{}.\allowbreak{}}
\pl{3}{-{} Interpretation:\allowbreak{} The goal states a non-{}strict inequality where the conclusion states a strict one.\allowbreak{}}
\pl{3}{-{} Match:\allowbreak{} \textbackslash{}box\{\allowbreak{}Major inconsistency\}.\allowbreak{}}
\pe
\pl{0}{\#\#\# Check for missing conditions /\allowbreak{} implicit conditions:\allowbreak{}}
\pl{3}{-{} No missing conditions /\allowbreak{} implicit conditions}
\pl{3}{-{} Match:\allowbreak{} \textbackslash{}box\{\allowbreak{}Perfectly match\}.\allowbreak{}}
\pe
\pl{0}{Output Format,\allowbreak{} example 2 (\allowbreak{}an Isabelle/\allowbreak{}HOL formalization;\allowbreak{} the example is made up for illustration):\allowbreak{}}
\pl{0}{Let\textquotesingle{}s compare the conditions and conclusions with the formal statement one by one:\allowbreak{}}
\pe
\pl{0}{C1.\allowbreak{} **A man in a blue jacket is repairing a bicycle on the sidewalk.\allowbreak{}**}
\pl{3}{-{} Formal:\allowbreak{} \char0{}assumes asm:\allowbreak{} \textquotedbl{}Man x \textbackslash{}\textless{}and\textgreater{} InBlueJacket x \textbackslash{}\textless{}and\textgreater{} Bicycle y \textbackslash{}\textless{}and\textgreater{} Repairs x y \textbackslash{}\textless{}and\textgreater{} Sidewalk z \textbackslash{}\textless{}and\textgreater{} On x z\textquotedbl{}\char0{}.\allowbreak{}}
\pl{3}{-{} Interpretation:\allowbreak{} The assumption says that a man in a blue jacket repairs a bicycle and is on a sidewalk,\allowbreak{} which is what the premise says.\allowbreak{}}
\pl{3}{-{} Match:\allowbreak{} \textbackslash{}box\{\allowbreak{}Perfectly match\}.\allowbreak{}}
\pe
\pl{0}{C2.\allowbreak{} **A man repairing a bicycle is fixing something.\allowbreak{}**}
\pl{3}{-{} Formal:\allowbreak{} \char0{}explanation\_\allowbreak{}1:\allowbreak{} \textquotedbl{}\textbackslash{}\textless{}forall\textgreater{}x y.\allowbreak{} Repairs x y \textbackslash{}\textless{}longrightarrow\textgreater{} FixesSomething x\textquotedbl{}\char0{}.\allowbreak{}}
\pl{3}{-{} Interpretation:\allowbreak{} The axiom drops the restrictions that \char0{}x\char0{} is a man and \char0{}y\char0{} is a bicycle,\allowbreak{} so it says a little more than the sentence while keeping the same link from repairing to fixing something.\allowbreak{}}
\pl{3}{-{} Match:\allowbreak{} \textbackslash{}box\{\allowbreak{}Minor inconsistency\}.\allowbreak{}}
\pe
\pl{0}{C3.\allowbreak{} **A sidewalk is outdoors.\allowbreak{}**}
\pl{3}{-{} Formal:\allowbreak{} \char0{}explanation\_\allowbreak{}2:\allowbreak{} \textquotedbl{}\textbackslash{}\textless{}forall\textgreater{}z.\allowbreak{} Outdoors z \textbackslash{}\textless{}longrightarrow\textgreater{} Sidewalk z\textquotedbl{}\char0{}.\allowbreak{}}
\pl{3}{-{} Interpretation:\allowbreak{} The implication is reversed:\allowbreak{} the axiom says everything outdoors is a sidewalk.\allowbreak{}}
\pl{3}{-{} Match:\allowbreak{} \textbackslash{}box\{\allowbreak{}Major inconsistency\}.\allowbreak{}}
\pe
\pl{0}{H1.\allowbreak{} **Conclusion:\allowbreak{} A man is fixing something outdoors.\allowbreak{}**}
\pl{3}{-{} Formal:\allowbreak{} \char0{}shows \textquotedbl{}\textbackslash{}\textless{}exists\textgreater{}x z.\allowbreak{} Man x \textbackslash{}\textless{}and\textgreater{} FixesSomething x \textbackslash{}\textless{}and\textgreater{} Outdoors z \textbackslash{}\textless{}and\textgreater{} On x z\textquotedbl{}\char0{}.\allowbreak{}}
\pl{3}{-{} Interpretation:\allowbreak{} The goal says that some man fixes something and is on something outdoors,\allowbreak{} which is what the hypothesis says.\allowbreak{}}
\pl{3}{-{} Match:\allowbreak{} \textbackslash{}box\{\allowbreak{}Perfectly match\}.\allowbreak{}}
\pe
\pl{0}{\#\#\# Check for missing conditions /\allowbreak{} implicit conditions:\allowbreak{}}
\pl{3}{-{} No missing conditions /\allowbreak{} implicit conditions}
\pl{3}{-{} Match:\allowbreak{} \textbackslash{}box\{\allowbreak{}Perfectly match\}.\allowbreak{}}
\pe
\pl{0}{-{}-{}-{}-{}-{}-{}-{}-{}-{}-{}-{}-{}-{}-{}-{}-{}-{}}
\pe
\pl{0}{Problem:\allowbreak{}}
\pl{0}{Consider a prime \$p\$ of the form \$4 t+1\$.\allowbreak{} Show that \$a\$ is a primitive root modulo \$p\$ iff \$-{}a\$ is a primitive root modulo \$p\$.\allowbreak{}}
\pe
\pl{0}{Conditions and conclusions:\allowbreak{}}
\pl{0}{Conditions:\allowbreak{}}
\pl{0}{C1:\allowbreak{} Consider a prime \$p\$ of the form \$4 t+1\$.\allowbreak{}}
\pl{0}{Conclusions:\allowbreak{}}
\pl{0}{H1:\allowbreak{} \$a\$ is a primitive root modulo \$p\$ iff \$-{}a\$ is a primitive root modulo \$p\$.\allowbreak{}}
\pe
\pl{0}{Formal statement:\allowbreak{}}
\pl{0}{import Mathlib}
\pe
\pl{0}{open Nat}
\pe
\pl{0}{theorem exercise\_\allowbreak{}9\_\allowbreak{}18 \{\allowbreak{}p :\allowbreak{} ℕ\} (\allowbreak{}hp :\allowbreak{} p.\allowbreak{}Prime) (\allowbreak{}h :\allowbreak{} p \% 4 =\allowbreak{} 1) (\allowbreak{}a :\allowbreak{} ℕ) :\allowbreak{}}
\pl{4}{(\allowbreak{}IsPrimitiveRoot a p) ↔ (\allowbreak{}IsPrimitiveRoot (\allowbreak{}p -{} a) p) :\allowbreak{}=\allowbreak{} sorry}
\pe
\pl{0}{Output:\allowbreak{}}
\pe
\end{promptlines}

\end{promptbox}

\begin{promptbox}{Isabelle, step 1 (listing conditions and conclusions) \textmd{(example: e-SNLI problem 72)}}
\noindent{\small\bfseries USER:}\par
\begin{promptlines}
\pl{0}{Help me list the conditions and conclusions in this problem,\allowbreak{} without solving it and without judging whether the conclusions follow.\allowbreak{}}
\pe
\pl{0}{Copy each condition and each conclusion from the problem as it is written,\allowbreak{} keeping its words and its formulas,\allowbreak{} one per line,\allowbreak{} in the order given.\allowbreak{} Do not rewrite a condition in other words or as a different formula,\allowbreak{} do not paraphrase it,\allowbreak{} and do not merge or split sentences.\allowbreak{} Keep every condition,\allowbreak{} including conditions that play no part in reaching the conclusions.\allowbreak{} For an entailment problem,\allowbreak{} which gives a premise,\allowbreak{} an explanation and a hypothesis,\allowbreak{} the conditions are the premise and every sentence of the explanation,\allowbreak{} and the conclusion is the hypothesis.\allowbreak{}}
\pe
\pl{0}{Some problems ask a question and also state its answer,\allowbreak{} for example \textquotedbl{}What is the value of \$x\$? (\allowbreak{}A) 1 (\allowbreak{}B) 2 (\allowbreak{}C) 3 Show that it is (\allowbreak{}B) 2.\allowbreak{}\textquotedbl{} For such a problem,\allowbreak{} the conclusion is the sentence that states the answer.\allowbreak{} The question and the list of answer choices are neither conditions nor conclusions,\allowbreak{} so do not list them.\allowbreak{}}
\pe
\pl{0}{Use exactly this output format and nothing else:\allowbreak{}}
\pl{0}{Conditions:\allowbreak{}}
\pl{0}{C1:\allowbreak{} \textless{}the first condition,\allowbreak{} copied as written\textgreater{}}
\pl{0}{C2:\allowbreak{} \textless{}the next condition\textgreater{}}
\pl{0}{.\allowbreak{}.\allowbreak{}.\allowbreak{}}
\pl{0}{Conclusions:\allowbreak{}}
\pl{0}{H1:\allowbreak{} \textless{}the first conclusion,\allowbreak{} copied as written\textgreater{}}
\pl{0}{.\allowbreak{}.\allowbreak{}.\allowbreak{}}
\pe
\pl{0}{Number the conditions C1,\allowbreak{} C2,\allowbreak{} C3,\allowbreak{} .\allowbreak{}.\allowbreak{}.\allowbreak{} and the conclusions H1,\allowbreak{} H2,\allowbreak{} .\allowbreak{}.\allowbreak{}.\allowbreak{} without gaps.\allowbreak{}}
\pe
\pl{0}{Here are three examples.\allowbreak{}}
\pe
\pl{0}{Example 1:\allowbreak{}}
\pl{0}{[\allowbreak{}Problem]:\allowbreak{} The sequence \$\textbackslash{}\{\allowbreak{}a\_\allowbreak{}n\textbackslash{}\}\$ satisfies \$a\_\allowbreak{}1 =\allowbreak{} 1\$,\allowbreak{} \$a\_\allowbreak{}2 =\allowbreak{} 2\$,\allowbreak{} \$a\_\allowbreak{}\{\allowbreak{}n + 2\}=\allowbreak{}2a\_\allowbreak{}\{\allowbreak{}n + 1\}-{}a\_\allowbreak{}n + 2\$.\allowbreak{} Let \$b\_\allowbreak{}n=\allowbreak{}a\_\allowbreak{}\{\allowbreak{}n + 1\}-{}a\_\allowbreak{}n\$.\allowbreak{} Prove that \$\textbackslash{}\{\allowbreak{}b\_\allowbreak{}n\textbackslash{}\}\$ is an arithmetic sequence.\allowbreak{}}
\pe
\pl{0}{[\allowbreak{}Conditions and Conclusions]:\allowbreak{}}
\pl{0}{Conditions:\allowbreak{}}
\pl{0}{C1:\allowbreak{} \$a\_\allowbreak{}1 =\allowbreak{} 1\$}
\pl{0}{C2:\allowbreak{} \$a\_\allowbreak{}2 =\allowbreak{} 2\$}
\pl{0}{C3:\allowbreak{} \$a\_\allowbreak{}\{\allowbreak{}n + 2\}=\allowbreak{}2a\_\allowbreak{}\{\allowbreak{}n + 1\}-{}a\_\allowbreak{}n + 2\$}
\pl{0}{C4:\allowbreak{} Let \$b\_\allowbreak{}n=\allowbreak{}a\_\allowbreak{}\{\allowbreak{}n + 1\}-{}a\_\allowbreak{}n\$.\allowbreak{}}
\pl{0}{Conclusions:\allowbreak{}}
\pl{0}{H1:\allowbreak{} \$\textbackslash{}\{\allowbreak{}b\_\allowbreak{}n\textbackslash{}\}\$ is an arithmetic sequence.\allowbreak{}}
\pe
\pl{0}{Examples 2 and 3 are made up for illustration;\allowbreak{} neither is taken from any dataset.\allowbreak{}}
\pe
\pl{0}{Example 2:\allowbreak{}}
\pl{0}{[\allowbreak{}Problem]:\allowbreak{}}
\pl{0}{Premise:\allowbreak{} A man in a blue jacket is repairing a bicycle on the sidewalk.\allowbreak{}}
\pl{0}{Explanation:\allowbreak{} A man repairing a bicycle is fixing something.\allowbreak{} A sidewalk is outdoors.\allowbreak{}}
\pl{0}{Hypothesis:\allowbreak{} A man is fixing something outdoors.\allowbreak{}}
\pe
\pl{0}{[\allowbreak{}Conditions and Conclusions]:\allowbreak{}}
\pl{0}{Conditions:\allowbreak{}}
\pl{0}{C1:\allowbreak{} A man in a blue jacket is repairing a bicycle on the sidewalk.\allowbreak{}}
\pl{0}{C2:\allowbreak{} A man repairing a bicycle is fixing something.\allowbreak{}}
\pl{0}{C3:\allowbreak{} A sidewalk is outdoors.\allowbreak{}}
\pl{0}{Conclusions:\allowbreak{}}
\pl{0}{H1:\allowbreak{} A man is fixing something outdoors.\allowbreak{}}
\pe
\pl{0}{Example 3:\allowbreak{}}
\pl{0}{[\allowbreak{}Problem]:\allowbreak{}}
\pl{0}{Premise:\allowbreak{} Quill is a flimpus.\allowbreak{}}
\pl{0}{Explanation:\allowbreak{}}
\pl{0}{E1:\allowbreak{} Every flimpus is a trobus.\allowbreak{}}
\pl{0}{E2:\allowbreak{} Flimpuses are loud.\allowbreak{}}
\pl{0}{E3:\allowbreak{} Each trobus is not sticky.\allowbreak{}}
\pl{0}{E4:\allowbreak{} Every gaskus is sticky.\allowbreak{}}
\pl{0}{Hypothesis:\allowbreak{} Quill is not sticky}
\pe
\pl{0}{[\allowbreak{}Conditions and Conclusions]:\allowbreak{}}
\pl{0}{Conditions:\allowbreak{}}
\pl{0}{C1:\allowbreak{} Quill is a flimpus.\allowbreak{}}
\pl{0}{C2:\allowbreak{} Every flimpus is a trobus.\allowbreak{}}
\pl{0}{C3:\allowbreak{} Flimpuses are loud.\allowbreak{}}
\pl{0}{C4:\allowbreak{} Each trobus is not sticky.\allowbreak{}}
\pl{0}{C5:\allowbreak{} Every gaskus is sticky.\allowbreak{}}
\pl{0}{Conclusions:\allowbreak{}}
\pl{0}{H1:\allowbreak{} Quill is not sticky}
\pe
\pl{0}{Now,\allowbreak{} please list the conditions and conclusions for this problem in the same way,\allowbreak{} without solving it:\allowbreak{}}
\pl{0}{[\allowbreak{}Problem]:\allowbreak{} Premise:\allowbreak{} Two women are wet while holding hands with a long building and buses in the background.\allowbreak{}}
\pl{0}{Explanation:\allowbreak{} There are women because there are two women.\allowbreak{}}
\pl{0}{Hypothesis:\allowbreak{} Women are in the city.\allowbreak{}}
\pe
\pl{0}{[\allowbreak{}Conditions and Conclusions]:\allowbreak{}}
\pe
\end{promptlines}

\end{promptbox}

\begin{promptbox}{Isabelle, step 2 (comparison) \textmd{(example: e-SNLI problem 72)}}
\noindent{\small\bfseries USER:}\par
\begin{promptlines}
\pl{0}{Here is a problem and a formal statement of it (\allowbreak{}in Lean 4 or in Isabelle/\allowbreak{}HOL).\allowbreak{} Compare the conditions and conclusions of the problem with the formal statement,\allowbreak{} matching them one by one to see if the formal statement is an appropriate translation of each of them by assigning one of three tags (\allowbreak{}Perfectly match;\allowbreak{} Minor inconsistency;\allowbreak{} Major inconsistency).\allowbreak{} Compare every condition C1,\allowbreak{} C2,\allowbreak{} .\allowbreak{}.\allowbreak{}.\allowbreak{} in order,\allowbreak{} including conditions that play no part in reaching the conclusions,\allowbreak{} then compare each conclusion H1,\allowbreak{} H2,\allowbreak{} .\allowbreak{}.\allowbreak{}.\allowbreak{} as its own separate item.\allowbreak{} Then,\allowbreak{} audit for missing/\allowbreak{}implicit conditions.\allowbreak{}}
\pe
\pl{0}{Compare meaning,\allowbreak{} not names.\allowbreak{} Predicate,\allowbreak{} function,\allowbreak{} constant and variable names are only labels:\allowbreak{} a name that differs from the words of the problem,\allowbreak{} or a single predicate that carries several words of the problem (\allowbreak{}for example \char0{}InBlueJacket x\char0{} for \textquotedbl{}in a blue jacket\textquotedbl{}),\allowbreak{} is not an inconsistency by itself.\allowbreak{} What counts is what the formal statement says:\allowbreak{} which properties and relations it asserts of which entities,\allowbreak{} the direction of each implication,\allowbreak{} negation,\allowbreak{} quantifiers,\allowbreak{} connectives and numbers.\allowbreak{} Judge with extremely strict standards on meaning—any minor difference in meaning will be considered a mismatch.\allowbreak{}}
\pe
\pl{0}{What counts as rendering an item.\allowbreak{} Read the whole problem,\allowbreak{} not only the listed item,\allowbreak{} and apply these points:\allowbreak{}}
\pl{0}{-{} Some conditions only tell the story of a word problem:\allowbreak{} who the people are,\allowbreak{} what they are doing,\allowbreak{} what the objects are called.\allowbreak{} Story details that carry no number or relation of their own need no counterpart in the formal statement.\allowbreak{} Tag such a condition Perfectly match when the numbers and relations it does carry are rendered,\allowbreak{} which may be inside a computation rather than in a hypothesis of its own.\allowbreak{} For example,\allowbreak{} in \textquotedbl{}Ann has 12 apples and gives 5 to Ben.\allowbreak{} How many apples does Ann have left? Show that it is 7.\allowbreak{}\textquotedbl{},\allowbreak{} the formal statement \char0{}12 -{} 5 =\allowbreak{} 7\char0{} renders everything;\allowbreak{} the names Ann and Ben and the word apples need no counterpart.\allowbreak{}}
\pl{0}{-{} If an item is the question of the problem (\allowbreak{}a sentence that asks for a value,\allowbreak{} such as \textquotedbl{}How many .\allowbreak{}.\allowbreak{}.\allowbreak{}?\textquotedbl{} or \textquotedbl{}What is the smallest .\allowbreak{}.\allowbreak{}.\allowbreak{}?\textquotedbl{}),\allowbreak{} it is not a hypothesis and needs no hypothesis of its own.\allowbreak{} Tag it,\allowbreak{} and the conclusion that states the answer,\allowbreak{} by whether the conclusion of the formal statement gives that answer together with every requirement the question puts on it (\allowbreak{}smallest,\allowbreak{} largest,\allowbreak{} positive,\allowbreak{} an integer,\allowbreak{} other than 1,\allowbreak{} .\allowbreak{}.\allowbreak{}.\allowbreak{}).\allowbreak{} A requirement of the question that the formal statement leaves out is an inconsistency.\allowbreak{} For a smallest or largest value,\allowbreak{} the usual formal statements are these:\allowbreak{} every value with the required properties is at least (\allowbreak{}at most) the answer;\allowbreak{} or,\allowbreak{} for the minimum or maximum of a function,\allowbreak{} the value at a point where it is attained equals the answer.\allowbreak{} Either may also state that the answer itself has the properties.\allowbreak{} A formal statement that only shows that the answer has the properties,\allowbreak{} and does not say that it is the smallest (\allowbreak{}largest),\allowbreak{} leaves out that requirement.\allowbreak{} For example,\allowbreak{} for \textquotedbl{}What is the smallest positive integer \$n\$ with \$n\textasciicircum{}2 \textgreater{} 50\$? Show that it is 8.\allowbreak{}\textquotedbl{},\allowbreak{} the statement \char0{}(\allowbreak{}n :\allowbreak{} ℕ) (\allowbreak{}h₀ :\allowbreak{} 0 \textless{} n) (\allowbreak{}h₁ :\allowbreak{} 50 \textless{} n \textasciicircum{} 2) :\allowbreak{} 8 ≤ n\char0{} is consistent,\allowbreak{} and \char0{}50 \textless{} 8 \textasciicircum{} 2\char0{} alone is not.\allowbreak{}}
\pl{0}{-{} An assumption that only puts the problem in the standard setting of its subject,\allowbreak{} which the problem and the textbook it comes from take for granted,\allowbreak{} is not an added condition.\allowbreak{} Such an assumption states what kind of structure the objects are,\allowbreak{} for example that the scalars of a vector space form a field,\allowbreak{} or that the functions of a real-{}analysis exercise take real values.\allowbreak{} An assumption that asks more of the problem\textquotesingle{}s own objects than the problem does (\allowbreak{}an extra property of a given object,\allowbreak{} a smaller domain,\allowbreak{} a stronger hypothesis) is an added condition.\allowbreak{} Assuming more than the problem does changes which theorem is stated,\allowbreak{} so tag it Major inconsistency;\allowbreak{} this covers an added hypothesis,\allowbreak{} a hypothesis required on a larger domain than the problem\textquotesingle{}s,\allowbreak{} and a stronger property than the problem\textquotesingle{}s (\allowbreak{}for example \$x \textgreater{} 0\$ where the problem has \$x \textbackslash{}ge 0\$).\allowbreak{} A conclusion that says less than the problem\textquotesingle{}s conclusion is likewise a Major inconsistency.\allowbreak{} Minor inconsistency is for a difference in how something is written that does not change which theorem is stated.\allowbreak{}}
\pl{0}{-{} A formal statement that is written differently but means the same thing given the problem\textquotesingle{}s own conditions is consistent.\allowbreak{} For example,\allowbreak{} for \textquotedbl{}For all odd \$n\$ show that \$8 \textbackslash{}mid n\textasciicircum{}\{\allowbreak{}2\}-{}1\$\textquotedbl{},\allowbreak{} the statement \char0{}\{\allowbreak{}n :\allowbreak{} ℕ\} (\allowbreak{}hn :\allowbreak{} Odd n) :\allowbreak{} 8 ∣ (\allowbreak{}n\textasciicircum{}2 -{} 1)\char0{} is consistent:\allowbreak{} it uses natural-{}number subtraction,\allowbreak{} but since \$n\$ is odd,\allowbreak{} \$n\textasciicircum{}2 \textbackslash{}ge 1\$ and the subtraction is exact.\allowbreak{} Tag each difference once.\allowbreak{} If a difference has already been tagged on a condition or a conclusion,\allowbreak{} the final audit does not tag it again.\allowbreak{}}
\pl{0}{-{} Read the formal statement by the semantics of its language,\allowbreak{} and tag an inconsistency where that reading differs from the problem,\allowbreak{} even when the symbols look like the problem\textquotesingle{}s.\allowbreak{} In Lean 4 and Mathlib,\allowbreak{} the type of a variable is its domain,\allowbreak{} so \char0{}n :\allowbreak{} ℕ\char0{} includes 0 and \textquotedbl{}positive integer\textquotedbl{} needs \char0{}0 \textless{} n\char0{}.\allowbreak{} Subtraction on ℕ stops at 0 (\allowbreak{}\char0{}3 -{} 5 =\allowbreak{} 0\char0{}),\allowbreak{} and \char0{}/\allowbreak{}\char0{} on ℕ and ℤ rounds down (\allowbreak{}\char0{}7 /\allowbreak{} 2 =\allowbreak{} 3\char0{}).\allowbreak{} Division by zero gives 0,\allowbreak{} and the integral of a function that is not integrable is 0.\allowbreak{} A function declared on a whole type is given at every point of that type,\allowbreak{} which says more than a function given only on part of it.\allowbreak{} A Mathlib name means what its definition says,\allowbreak{} including which argument is the base of a logarithm and which way an order between structures goes.\allowbreak{} To apply this,\allowbreak{} look at the type of each variable and ask which values it allows that the problem does not intend,\allowbreak{} in particular 0 and 1 for a natural number,\allowbreak{} and points outside the set on which the problem\textquotesingle{}s function is defined.\allowbreak{} Substitute such a value and check whether the formal statement still says what the problem says.\allowbreak{} If at that value the formal statement makes a claim that the problem does not make,\allowbreak{} for example because a subtraction is cut off at 0,\allowbreak{} that is an inconsistency.\allowbreak{} For an integral,\allowbreak{} check whether the function is integrable on the whole domain:\allowbreak{} an improper integral that converges only conditionally is not a Lebesgue integral,\allowbreak{} and Lean\textquotesingle{}s \char0{}∫\char0{} of such a function is 0.\allowbreak{}}
\pe
\pl{0}{The task that produced this formal statement gave the model the instructions and worked examples quoted below,\allowbreak{} in full and word for word,\allowbreak{} between the two marker lines.\allowbreak{} A formal statement written the way these instructions and examples prescribe is not inconsistent on that account.\allowbreak{} Judge only whether each item says what the problem says.\allowbreak{}}
\pe
\pl{0}{=\allowbreak{}=\allowbreak{}=\allowbreak{}=\allowbreak{}=\allowbreak{} TASK INSTRUCTIONS AND EXAMPLES (\allowbreak{}quoted) =\allowbreak{}=\allowbreak{}=\allowbreak{}=\allowbreak{}=\allowbreak{}}
\pl{0}{You are an expert in formal language Isabelle/\allowbreak{}HOL (\allowbreak{}Isabelle 2024).\allowbreak{}}
\pl{0}{You will be given a natural-{}language entailment task:\allowbreak{} a premise,\allowbreak{} a hypothesis,\allowbreak{} and one or more explanations bridging them.\allowbreak{}}
\pl{0}{Your task is to produce a complete Isabelle 2024 theory file that formalizes the premise (\allowbreak{}as the \char0{}assumes\char0{} clause),\allowbreak{} the explanations (\allowbreak{}as \char0{}axiomatization\char0{} blocks),\allowbreak{} and the hypothesis (\allowbreak{}as the \char0{}shows\char0{} clause).\allowbreak{} The proof body MUST be exactly \char0{}sorry\char0{}.\allowbreak{} Proof completion is out of scope;\allowbreak{} we are validating the formalization only.\allowbreak{}}
\pe
\pl{0}{STRICT CONSTRAINTS (\allowbreak{}Isabelle 2024 source format):\allowbreak{}}
\pl{0}{1.\allowbreak{} Use ASCII Isabelle escapes for logical operators,\allowbreak{} NOT Unicode.\allowbreak{} Required substitutions:\allowbreak{}}
\pl{3}{∧\ \ →\ \ \textbackslash{}\textless{}and\textgreater{}}
\pl{3}{∨\ \ →\ \ \textbackslash{}\textless{}or\textgreater{}}
\pl{3}{¬\ \ →\ \ \textbackslash{}\textless{}not\textgreater{}}
\pl{3}{⟹\ \ →\ \ \textbackslash{}\textless{}Longrightarrow\textgreater{}}
\pl{3}{⟶\ \ →\ \ \textbackslash{}\textless{}longrightarrow\textgreater{}}
\pl{3}{⟷\ \ →\ \ \textbackslash{}\textless{}longleftrightarrow\textgreater{}}
\pl{3}{↔\ \ →\ \ \textbackslash{}\textless{}longleftrightarrow\textgreater{}}
\pl{3}{∀\ \ →\ \ \textbackslash{}\textless{}forall\textgreater{}}
\pl{3}{∃\ \ →\ \ \textbackslash{}\textless{}exists\textgreater{}}
\pl{3}{⋀\ \ →\ \ \textbackslash{}\textless{}And\textgreater{}}
\pl{3}{⇒\ \ →\ \ \textbackslash{}\textless{}Rightarrow\textgreater{}}
\pl{3}{→\ \ →\ \ \textbackslash{}\textless{}longrightarrow\textgreater{}}
\pl{0}{2.\allowbreak{} The proof body MUST be the single token \char0{}sorry\char0{} (\allowbreak{}no \char0{}proof -{} qed\char0{} block,\allowbreak{} no \char0{}by .\allowbreak{}.\allowbreak{}.\allowbreak{}\char0{} tactic,\allowbreak{} no \char0{}oops\char0{},\allowbreak{} no \char0{}admit\char0{}).\allowbreak{} The file must compile under Isabelle\textquotesingle{}s \char0{}quick\_\allowbreak{}and\_\allowbreak{}dirty\char0{} mode.\allowbreak{}}
\pl{0}{3.\allowbreak{} Structure the file as:\allowbreak{}}
\pl{3}{theory \textless{}theory\_\allowbreak{}name\textgreater{}}
\pl{3}{imports Main}
\pl{3}{begin}
\pe
\pl{3}{typedecl entity}
\pl{3}{typedecl event}
\pe
\pl{3}{consts}
\pl{5}{\textless{}predicate\_\allowbreak{}1\textgreater{} :\allowbreak{}:\allowbreak{} \textquotedbl{}\textless{}type signature\textgreater{}\textquotedbl{}}
\pl{5}{\textless{}predicate\_\allowbreak{}2\textgreater{} :\allowbreak{}:\allowbreak{} \textquotedbl{}\textless{}type signature\textgreater{}\textquotedbl{}}
\pl{5}{.\allowbreak{}.\allowbreak{}.\allowbreak{}}
\pe
\pl{3}{axiomatization where}
\pl{5}{explanation\_\allowbreak{}1:\allowbreak{} \textquotedbl{}\textless{}formalized explanation 1\textgreater{}\textquotedbl{}}
\pe
\pl{3}{theorem hypothesis:\allowbreak{}}
\pl{5}{assumes asm:\allowbreak{} \textquotedbl{}\textless{}formalized premise\textgreater{}\textquotedbl{}}
\pl{5}{shows \textquotedbl{}\textless{}formalized hypothesis\textgreater{}\textquotedbl{}}
\pl{5}{sorry}
\pe
\pl{3}{end}
\pl{0}{FAITHFULNESS RULES (\allowbreak{}the file is checked against these after it compiles):\allowbreak{}}
\pl{0}{-{} One rule per explanation sentence:\allowbreak{} every \char0{}axiomatization\char0{} entry must render exactly one explanation sentence,\allowbreak{} saying what that sentence says and nothing more.\allowbreak{} Do NOT add any other axioms or bridge rules.\allowbreak{}}
\pl{0}{-{} The premise goes into \char0{}assumes\char0{} and the hypothesis into \char0{}shows\char0{},\allowbreak{} with every qualifier kept and nothing added,\allowbreak{} dropped,\allowbreak{} strengthened or weakened.\allowbreak{}}
\pl{0}{-{} The statements must connect:\allowbreak{} write the premise in \char0{}assumes\char0{} at a fine enough grain that it introduces the same entities and predicates the explanation rule uses in its antecedent,\allowbreak{} so that the rule applies to the assumptions directly and yields what \char0{}shows\char0{} needs.\allowbreak{} Reuse one predicate name for one concept across premise,\allowbreak{} explanation and hypothesis.\allowbreak{}}
\pl{0}{-{} Name every symbol after words that occur in the text.\allowbreak{}}
\pl{0}{4.\allowbreak{} The \char0{}\textless{}theory\_\allowbreak{}name\textgreater{}\char0{} MUST be EXACTLY the value provided in the user prompt — Isabelle requires that the theory name match the filename.\allowbreak{} Do not invent or modify it.\allowbreak{}}
\pl{0}{5.\allowbreak{} Wrap your output in fences exactly as:\allowbreak{}}
\pl{0}{\char0{}\char0{}\char0{}isabelle}
\pl{0}{\textless{}your theory file content\textgreater{}}
\pl{0}{\char0{}\char0{}\char0{}}
\pl{0}{Output nothing outside the fences.\allowbreak{}}
\pe
\pl{0}{Worked example 1,\allowbreak{} input given to the model:\allowbreak{}}
\pl{0}{Theory name (\allowbreak{}use this exact name):\allowbreak{} demo\_\allowbreak{}javelin\_\allowbreak{}competition}
\pe
\pl{0}{Premise:\allowbreak{} Javelin competition}
\pl{0}{Hypothesis:\allowbreak{} There is a competition.\allowbreak{}}
\pl{0}{Explanation:\allowbreak{} There must be a competition for a Javelin competition to take place.\allowbreak{}}
\pe
\pl{0}{Worked example 1,\allowbreak{} answer:\allowbreak{}}
\pl{0}{\char0{}\char0{}\char0{}isabelle}
\pl{0}{theory demo\_\allowbreak{}javelin\_\allowbreak{}competition}
\pl{0}{imports Main}
\pl{0}{begin}
\pe
\pl{0}{typedecl entity}
\pl{0}{typedecl event}
\pe
\pl{0}{consts}
\pl{2}{JavelinCompetition :\allowbreak{}:\allowbreak{} \textquotedbl{}event \textbackslash{}\textless{}Rightarrow\textgreater{} bool\textquotedbl{}}
\pl{2}{Competition :\allowbreak{}:\allowbreak{} \textquotedbl{}event \textbackslash{}\textless{}Rightarrow\textgreater{} bool\textquotedbl{}}
\pe
\pl{0}{axiomatization where}
\pl{2}{explanation\_\allowbreak{}1:\allowbreak{} \textquotedbl{}\textbackslash{}\textless{}forall\textgreater{}e.\allowbreak{} JavelinCompetition e \textbackslash{}\textless{}longrightarrow\textgreater{} Competition e\textquotedbl{}}
\pe
\pl{0}{theorem hypothesis:\allowbreak{}}
\pl{2}{assumes asm:\allowbreak{} \textquotedbl{}JavelinCompetition e\textquotedbl{}}
\pl{2}{shows \textquotedbl{}\textbackslash{}\textless{}exists\textgreater{}e.\allowbreak{} Competition e\textquotedbl{}}
\pl{2}{sorry}
\pe
\pl{0}{end}
\pl{0}{\char0{}\char0{}\char0{}}
\pe
\pl{0}{Worked example 2,\allowbreak{} input given to the model:\allowbreak{}}
\pl{0}{Theory name (\allowbreak{}use this exact name):\allowbreak{} demo\_\allowbreak{}group\_\allowbreak{}costumes}
\pe
\pl{0}{Premise:\allowbreak{} a group of people are wearing costumes}
\pl{0}{Hypothesis:\allowbreak{} People are wearing costumes.\allowbreak{}}
\pl{0}{Explanation:\allowbreak{} A group means that there are people.\allowbreak{}}
\pe
\pl{0}{Worked example 2,\allowbreak{} answer:\allowbreak{}}
\pl{0}{\char0{}\char0{}\char0{}isabelle}
\pl{0}{theory demo\_\allowbreak{}group\_\allowbreak{}costumes}
\pl{0}{imports Main}
\pl{0}{begin}
\pe
\pl{0}{typedecl entity}
\pl{0}{typedecl event}
\pe
\pl{0}{consts}
\pl{2}{Group :\allowbreak{}:\allowbreak{} \textquotedbl{}entity \textbackslash{}\textless{}Rightarrow\textgreater{} bool\textquotedbl{}}
\pl{2}{People :\allowbreak{}:\allowbreak{} \textquotedbl{}entity \textbackslash{}\textless{}Rightarrow\textgreater{} bool\textquotedbl{}}
\pl{2}{WearingCostumes :\allowbreak{}:\allowbreak{} \textquotedbl{}entity \textbackslash{}\textless{}Rightarrow\textgreater{} bool\textquotedbl{}}
\pe
\pl{0}{axiomatization where}
\pl{2}{explanation\_\allowbreak{}1:\allowbreak{} \textquotedbl{}\textbackslash{}\textless{}forall\textgreater{}x.\allowbreak{} Group x \textbackslash{}\textless{}longrightarrow\textgreater{} People x\textquotedbl{}}
\pe
\pl{0}{theorem hypothesis:\allowbreak{}}
\pl{2}{assumes asm:\allowbreak{} \textquotedbl{}Group x \textbackslash{}\textless{}and\textgreater{} WearingCostumes x\textquotedbl{}}
\pl{2}{shows \textquotedbl{}\textbackslash{}\textless{}exists\textgreater{}x.\allowbreak{} People x \textbackslash{}\textless{}and\textgreater{} WearingCostumes x\textquotedbl{}}
\pl{2}{sorry}
\pe
\pl{0}{end}
\pl{0}{\char0{}\char0{}\char0{}}
\pe
\pl{0}{Worked example 3,\allowbreak{} input given to the model:\allowbreak{}}
\pl{0}{Theory name (\allowbreak{}use this exact name):\allowbreak{} demo\_\allowbreak{}group\_\allowbreak{}beach}
\pe
\pl{0}{Premise:\allowbreak{} A group of people on the beach with cameras.\allowbreak{}}
\pl{0}{Hypothesis:\allowbreak{} People are on a beach.\allowbreak{}}
\pl{0}{Explanation:\allowbreak{} Group of people are people.\allowbreak{}}
\pe
\pl{0}{Worked example 3,\allowbreak{} answer:\allowbreak{}}
\pl{0}{\char0{}\char0{}\char0{}isabelle}
\pl{0}{theory demo\_\allowbreak{}group\_\allowbreak{}beach}
\pl{0}{imports Main}
\pl{0}{begin}
\pe
\pl{0}{typedecl entity}
\pl{0}{typedecl event}
\pe
\pl{0}{consts}
\pl{2}{GroupOfPeople :\allowbreak{}:\allowbreak{} \textquotedbl{}entity \textbackslash{}\textless{}Rightarrow\textgreater{} bool\textquotedbl{}}
\pl{2}{People :\allowbreak{}:\allowbreak{} \textquotedbl{}entity \textbackslash{}\textless{}Rightarrow\textgreater{} bool\textquotedbl{}}
\pl{2}{OnBeach :\allowbreak{}:\allowbreak{} \textquotedbl{}entity \textbackslash{}\textless{}Rightarrow\textgreater{} bool\textquotedbl{}}
\pl{2}{WithCameras :\allowbreak{}:\allowbreak{} \textquotedbl{}entity \textbackslash{}\textless{}Rightarrow\textgreater{} bool\textquotedbl{}}
\pe
\pl{0}{axiomatization where}
\pl{2}{explanation\_\allowbreak{}1:\allowbreak{} \textquotedbl{}\textbackslash{}\textless{}forall\textgreater{}x.\allowbreak{} GroupOfPeople x \textbackslash{}\textless{}longrightarrow\textgreater{} People x\textquotedbl{}}
\pe
\pl{0}{theorem hypothesis:\allowbreak{}}
\pl{2}{assumes asm:\allowbreak{} \textquotedbl{}GroupOfPeople x \textbackslash{}\textless{}and\textgreater{} OnBeach x \textbackslash{}\textless{}and\textgreater{} WithCameras x\textquotedbl{}}
\pl{2}{shows \textquotedbl{}\textbackslash{}\textless{}exists\textgreater{}x.\allowbreak{} People x \textbackslash{}\textless{}and\textgreater{} OnBeach x\textquotedbl{}}
\pl{2}{sorry}
\pe
\pl{0}{end}
\pl{0}{\char0{}\char0{}\char0{}}
\pl{0}{=\allowbreak{}=\allowbreak{}=\allowbreak{}=\allowbreak{}=\allowbreak{} END OF TASK INSTRUCTIONS AND EXAMPLES =\allowbreak{}=\allowbreak{}=\allowbreak{}=\allowbreak{}=\allowbreak{}}
\pe
\pl{0}{Write every item in the format of the examples below:\allowbreak{} the line of the formal statement that renders it,\allowbreak{} one sentence of interpretation,\allowbreak{} then the Match line.\allowbreak{} Write every tag exactly as \textbackslash{}box\{\allowbreak{}Perfectly match\},\allowbreak{} \textbackslash{}box\{\allowbreak{}Minor inconsistency\} or \textbackslash{}box\{\allowbreak{}Major inconsistency\}.\allowbreak{} A tag written any other way (\allowbreak{}in bold,\allowbreak{} in plain text,\allowbreak{} or outside \textbackslash{}box\{\allowbreak{}\}) is not read.\allowbreak{}}
\pl{0}{**Stop immediately** after evaluating all pairs.\allowbreak{} Do **not** summarize or analyze further.\allowbreak{}}
\pe
\pl{0}{Output Format,\allowbreak{} example 1 (\allowbreak{}a Lean 4 statement):\allowbreak{}}
\pl{0}{Let\textquotesingle{}s compare the conditions and conclusions with the formal statement one by one:\allowbreak{}}
\pe
\pl{0}{C1.\allowbreak{} **\textbackslash{}(\allowbreak{} q \textbackslash{}) is a natural number greater than 1**:\allowbreak{}}
\pl{3}{-{} Formal:\allowbreak{} \char0{}(\allowbreak{}hq :\allowbreak{} 1 \textless{} q)\char0{}.\allowbreak{}}
\pl{3}{-{} Interpretation:\allowbreak{} The hypothesis says \textbackslash{}(\allowbreak{} q \textgreater{} 1 \textbackslash{}) for a natural number \textbackslash{}(\allowbreak{} q \textbackslash{}),\allowbreak{} which is what the condition says.\allowbreak{}}
\pl{3}{-{} Match:\allowbreak{} \textbackslash{}box\{\allowbreak{}Perfectly match\}.\allowbreak{}}
\pe
\pl{0}{C2.\allowbreak{} **\textbackslash{}(\allowbreak{} n \textbackslash{}) is a natural number greater than 1**:\allowbreak{}}
\pl{3}{-{} Formal:\allowbreak{} \char0{}(\allowbreak{}hn :\allowbreak{} 1 \textless{} n)\char0{}.\allowbreak{}}
\pl{3}{-{} Interpretation:\allowbreak{} The hypothesis says \textbackslash{}(\allowbreak{} n \textgreater{} 1 \textbackslash{}) for a natural number \textbackslash{}(\allowbreak{} n \textbackslash{}),\allowbreak{} which is what the condition says.\allowbreak{}}
\pl{3}{-{} Match:\allowbreak{} \textbackslash{}box\{\allowbreak{}Perfectly match\}.\allowbreak{}}
\pe
\pl{0}{C3.\allowbreak{} **Set \textbackslash{}(\allowbreak{} M =\allowbreak{} \textbackslash{}\{\allowbreak{}0,\allowbreak{} 1,\allowbreak{} 2,\allowbreak{} \textbackslash{}cdots,\allowbreak{} q -{} 1\textbackslash{}\} \textbackslash{})**:\allowbreak{}}
\pl{3}{-{} Formal:\allowbreak{} \char0{}(\allowbreak{}M :\allowbreak{} Finset ℕ :\allowbreak{}=\allowbreak{} Finset.\allowbreak{}range q)\char0{}.\allowbreak{}}
\pl{3}{-{} Interpretation:\allowbreak{} \char0{}Finset.\allowbreak{}range q\char0{} is \char0{}\{\allowbreak{}0,\allowbreak{} 1,\allowbreak{} .\allowbreak{}.\allowbreak{}.\allowbreak{},\allowbreak{} q -{} 1\}\char0{}.\allowbreak{}}
\pl{3}{-{} Match:\allowbreak{} \textbackslash{}box\{\allowbreak{}Perfectly match\}.\allowbreak{}}
\pe
\pl{0}{C4.\allowbreak{} **Set \textbackslash{}(\allowbreak{} A \textbackslash{}) definition**:\allowbreak{}}
\pl{3}{-{} Formal:\allowbreak{} \char0{}A :\allowbreak{} Set ℕ :\allowbreak{}=\allowbreak{} \{\allowbreak{}x \textbar{} ∃ (\allowbreak{}x\_\allowbreak{}vec :\allowbreak{} ℕ → ℕ),\allowbreak{} (\allowbreak{}∀ i,\allowbreak{} x\_\allowbreak{}vec i ∈ M) ∧ x =\allowbreak{} ∑ i in Finset.\allowbreak{}range n,\allowbreak{} x\_\allowbreak{}vec (\allowbreak{}i + 1) * q \textasciicircum{} i\}\char0{}.\allowbreak{}}
\pl{3}{-{} Interpretation:\allowbreak{} The Lean definition is technically correct but slightly more abstract than the math,\allowbreak{} since \char0{}x\_\allowbreak{}vec\char0{} is a function on all naturals of which only \char0{}x\_\allowbreak{}vec (\allowbreak{}i + 1)\char0{} for \char0{}i\char0{} in \char0{}Finset.\allowbreak{}range n\char0{} is used.\allowbreak{}}
\pl{3}{-{} Match:\allowbreak{} \textbackslash{}box\{\allowbreak{}Minor inconsistency\}.\allowbreak{}}
\pe
\pl{0}{C5.\allowbreak{} **\textbackslash{}(\allowbreak{} s,\allowbreak{} t \textbackslash{}in A \textbackslash{}) with specific expansions**:\allowbreak{}}
\pl{3}{-{} Formal:\allowbreak{} \char0{}s =\allowbreak{} ∑ i in Finset.\allowbreak{}range n,\allowbreak{} a (\allowbreak{}i + 1) * q \textasciicircum{} i\char0{},\allowbreak{} \char0{}t =\allowbreak{} ∑ i in Finset.\allowbreak{}range n,\allowbreak{} b (\allowbreak{}i + 1) * q \textasciicircum{} i\char0{},\allowbreak{} with \char0{}∀ i,\allowbreak{} a i ∈ M\char0{} and \char0{}∀ i,\allowbreak{} b i ∈ M\char0{}.\allowbreak{}}
\pl{3}{-{} Interpretation:\allowbreak{} The Lean version is correct but not a literal translation,\allowbreak{} since it uses \char0{}a i\char0{} for all \char0{}i\char0{} but only evaluates it at \char0{}i + 1\char0{}.\allowbreak{}}
\pl{3}{-{} Match:\allowbreak{} \textbackslash{}box\{\allowbreak{}Minor inconsistency\}.\allowbreak{}}
\pe
\pl{0}{C6.\allowbreak{} **\textbackslash{}(\allowbreak{} a\_\allowbreak{}n \textless{} b\_\allowbreak{}n \textbackslash{})**:\allowbreak{}}
\pl{3}{-{} Formal:\allowbreak{} \char0{}(\allowbreak{}hab :\allowbreak{} a n \textless{} b n)\char0{}.\allowbreak{}}
\pl{3}{-{} Interpretation:\allowbreak{} The hypothesis says \textbackslash{}(\allowbreak{} a\_\allowbreak{}n \textless{} b\_\allowbreak{}n \textbackslash{}),\allowbreak{} which is what the condition says.\allowbreak{}}
\pl{3}{-{} Match:\allowbreak{} \textbackslash{}box\{\allowbreak{}Perfectly match\}.\allowbreak{}}
\pe
\pl{0}{H1.\allowbreak{} **Conclusion:\allowbreak{} \textbackslash{}(\allowbreak{} s \textless{} t \textbackslash{})**:\allowbreak{}}
\pl{3}{-{} Formal:\allowbreak{} \char0{}s \textless{}=\allowbreak{} t\char0{}.\allowbreak{}}
\pl{3}{-{} Interpretation:\allowbreak{} The goal states a non-{}strict inequality where the conclusion states a strict one.\allowbreak{}}
\pl{3}{-{} Match:\allowbreak{} \textbackslash{}box\{\allowbreak{}Major inconsistency\}.\allowbreak{}}
\pe
\pl{0}{\#\#\# Check for missing conditions /\allowbreak{} implicit conditions:\allowbreak{}}
\pl{3}{-{} No missing conditions /\allowbreak{} implicit conditions}
\pl{3}{-{} Match:\allowbreak{} \textbackslash{}box\{\allowbreak{}Perfectly match\}.\allowbreak{}}
\pe
\pl{0}{Output Format,\allowbreak{} example 2 (\allowbreak{}an Isabelle/\allowbreak{}HOL formalization;\allowbreak{} the example is made up for illustration):\allowbreak{}}
\pl{0}{Let\textquotesingle{}s compare the conditions and conclusions with the formal statement one by one:\allowbreak{}}
\pe
\pl{0}{C1.\allowbreak{} **A man in a blue jacket is repairing a bicycle on the sidewalk.\allowbreak{}**}
\pl{3}{-{} Formal:\allowbreak{} \char0{}assumes asm:\allowbreak{} \textquotedbl{}Man x \textbackslash{}\textless{}and\textgreater{} InBlueJacket x \textbackslash{}\textless{}and\textgreater{} Bicycle y \textbackslash{}\textless{}and\textgreater{} Repairs x y \textbackslash{}\textless{}and\textgreater{} Sidewalk z \textbackslash{}\textless{}and\textgreater{} On x z\textquotedbl{}\char0{}.\allowbreak{}}
\pl{3}{-{} Interpretation:\allowbreak{} The assumption says that a man in a blue jacket repairs a bicycle and is on a sidewalk,\allowbreak{} which is what the premise says.\allowbreak{}}
\pl{3}{-{} Match:\allowbreak{} \textbackslash{}box\{\allowbreak{}Perfectly match\}.\allowbreak{}}
\pe
\pl{0}{C2.\allowbreak{} **A man repairing a bicycle is fixing something.\allowbreak{}**}
\pl{3}{-{} Formal:\allowbreak{} \char0{}explanation\_\allowbreak{}1:\allowbreak{} \textquotedbl{}\textbackslash{}\textless{}forall\textgreater{}x y.\allowbreak{} Repairs x y \textbackslash{}\textless{}longrightarrow\textgreater{} FixesSomething x\textquotedbl{}\char0{}.\allowbreak{}}
\pl{3}{-{} Interpretation:\allowbreak{} The axiom drops the restrictions that \char0{}x\char0{} is a man and \char0{}y\char0{} is a bicycle,\allowbreak{} so it says a little more than the sentence while keeping the same link from repairing to fixing something.\allowbreak{}}
\pl{3}{-{} Match:\allowbreak{} \textbackslash{}box\{\allowbreak{}Minor inconsistency\}.\allowbreak{}}
\pe
\pl{0}{C3.\allowbreak{} **A sidewalk is outdoors.\allowbreak{}**}
\pl{3}{-{} Formal:\allowbreak{} \char0{}explanation\_\allowbreak{}2:\allowbreak{} \textquotedbl{}\textbackslash{}\textless{}forall\textgreater{}z.\allowbreak{} Outdoors z \textbackslash{}\textless{}longrightarrow\textgreater{} Sidewalk z\textquotedbl{}\char0{}.\allowbreak{}}
\pl{3}{-{} Interpretation:\allowbreak{} The implication is reversed:\allowbreak{} the axiom says everything outdoors is a sidewalk.\allowbreak{}}
\pl{3}{-{} Match:\allowbreak{} \textbackslash{}box\{\allowbreak{}Major inconsistency\}.\allowbreak{}}
\pe
\pl{0}{H1.\allowbreak{} **Conclusion:\allowbreak{} A man is fixing something outdoors.\allowbreak{}**}
\pl{3}{-{} Formal:\allowbreak{} \char0{}shows \textquotedbl{}\textbackslash{}\textless{}exists\textgreater{}x z.\allowbreak{} Man x \textbackslash{}\textless{}and\textgreater{} FixesSomething x \textbackslash{}\textless{}and\textgreater{} Outdoors z \textbackslash{}\textless{}and\textgreater{} On x z\textquotedbl{}\char0{}.\allowbreak{}}
\pl{3}{-{} Interpretation:\allowbreak{} The goal says that some man fixes something and is on something outdoors,\allowbreak{} which is what the hypothesis says.\allowbreak{}}
\pl{3}{-{} Match:\allowbreak{} \textbackslash{}box\{\allowbreak{}Perfectly match\}.\allowbreak{}}
\pe
\pl{0}{\#\#\# Check for missing conditions /\allowbreak{} implicit conditions:\allowbreak{}}
\pl{3}{-{} No missing conditions /\allowbreak{} implicit conditions}
\pl{3}{-{} Match:\allowbreak{} \textbackslash{}box\{\allowbreak{}Perfectly match\}.\allowbreak{}}
\pe
\pl{0}{-{}-{}-{}-{}-{}-{}-{}-{}-{}-{}-{}-{}-{}-{}-{}-{}-{}}
\pe
\pl{0}{Problem:\allowbreak{}}
\pl{0}{Premise:\allowbreak{} Two women are wet while holding hands with a long building and buses in the background.\allowbreak{}}
\pl{0}{Explanation:\allowbreak{} There are women because there are two women.\allowbreak{}}
\pl{0}{Hypothesis:\allowbreak{} Women are in the city.\allowbreak{}}
\pe
\pl{0}{Conditions and conclusions:\allowbreak{}}
\pl{0}{Conditions:\allowbreak{}}
\pl{0}{C1:\allowbreak{} Two women are wet while holding hands with a long building and buses in the background.\allowbreak{}}
\pl{0}{C2:\allowbreak{} There are women because there are two women.\allowbreak{}}
\pl{0}{Conclusions:\allowbreak{}}
\pl{0}{H1:\allowbreak{} Women are in the city.\allowbreak{}}
\pe
\pl{0}{Formal statement:\allowbreak{}}
\pl{0}{theory eval\_\allowbreak{}72}
\pl{0}{imports Main}
\pl{0}{begin}
\pe
\pl{0}{typedecl entity}
\pl{0}{typedecl event}
\pe
\pl{0}{consts}
\pl{2}{TwoWomen :\allowbreak{}:\allowbreak{} \textquotedbl{}entity \textbackslash{}\textless{}Rightarrow\textgreater{} bool\textquotedbl{}}
\pl{2}{Women :\allowbreak{}:\allowbreak{} \textquotedbl{}entity \textbackslash{}\textless{}Rightarrow\textgreater{} bool\textquotedbl{}}
\pl{2}{Wet :\allowbreak{}:\allowbreak{} \textquotedbl{}entity \textbackslash{}\textless{}Rightarrow\textgreater{} bool\textquotedbl{}}
\pl{2}{HoldingHands :\allowbreak{}:\allowbreak{} \textquotedbl{}entity \textbackslash{}\textless{}Rightarrow\textgreater{} bool\textquotedbl{}}
\pl{2}{LongBuilding :\allowbreak{}:\allowbreak{} \textquotedbl{}entity \textbackslash{}\textless{}Rightarrow\textgreater{} bool\textquotedbl{}}
\pl{2}{Buses :\allowbreak{}:\allowbreak{} \textquotedbl{}entity \textbackslash{}\textless{}Rightarrow\textgreater{} bool\textquotedbl{}}
\pl{2}{InBackground :\allowbreak{}:\allowbreak{} \textquotedbl{}entity \textbackslash{}\textless{}Rightarrow\textgreater{} bool\textquotedbl{}}
\pl{2}{InCity :\allowbreak{}:\allowbreak{} \textquotedbl{}entity \textbackslash{}\textless{}Rightarrow\textgreater{} bool\textquotedbl{}}
\pe
\pl{0}{axiomatization where}
\pl{2}{explanation\_\allowbreak{}1:\allowbreak{} \textquotedbl{}\textbackslash{}\textless{}forall\textgreater{}x.\allowbreak{} TwoWomen x \textbackslash{}\textless{}longrightarrow\textgreater{} Women x\textquotedbl{}}
\pe
\pl{0}{theorem hypothesis:\allowbreak{}}
\pl{2}{assumes asm:\allowbreak{} \textquotedbl{}TwoWomen x \textbackslash{}\textless{}and\textgreater{} Wet x \textbackslash{}\textless{}and\textgreater{} HoldingHands x \textbackslash{}\textless{}and\textgreater{} LongBuilding y \textbackslash{}\textless{}and\textgreater{} Buses z \textbackslash{}\textless{}and\textgreater{} InBackground x\textquotedbl{}}
\pl{2}{shows \textquotedbl{}\textbackslash{}\textless{}exists\textgreater{}x.\allowbreak{} Women x \textbackslash{}\textless{}and\textgreater{} InCity x\textquotedbl{}}
\pl{2}{sorry}
\pe
\pl{0}{end}
\pe
\pl{0}{Output:\allowbreak{}}
\pe
\end{promptlines}

\end{promptbox}

\subsection{Compared Methods}
\begin{promptbox}{Self-Refine on Lean, critique \textmd{(example: ProofNet problem \texttt{exercise\_4\_4})}}
\noindent{\small\bfseries SYSTEM:}\par
\begin{promptlines}
\pl{0}{You are an expert reviewer of Lean 4 (\allowbreak{}Mathlib) formalizations.\allowbreak{}}
\pl{0}{You will be given a mathematical statement in natural language and a Lean 4 formalization attempt.\allowbreak{} Review the attempt purely by reading it:\allowbreak{} you have NO access to a Lean compiler,\allowbreak{} proof assistant,\allowbreak{} or any external tool.\allowbreak{}}
\pl{0}{List the concrete problems you can find (\allowbreak{}syntax errors,\allowbreak{} wrong or missing imports,\allowbreak{} type mismatches,\allowbreak{} unbound identifiers,\allowbreak{} wrong Mathlib names,\allowbreak{} mismatch between the formal statement and the natural-{}language meaning,\allowbreak{} missing hypotheses or conclusions).\allowbreak{} Be specific and actionable;\allowbreak{} do not rewrite the code yourself.\allowbreak{}}
\pl{0}{End your review with exactly one line:\allowbreak{}}
\pl{0}{VERDICT:\allowbreak{} CORRECT}
\pl{0}{or}
\pl{0}{VERDICT:\allowbreak{} NEEDS\_\allowbreak{}REVISION}
\pe
\end{promptlines}

\noindent{\small\bfseries USER:}\par
\begin{promptlines}
\pl{0}{Natural language statement:\allowbreak{} Consider a prime \$p\$ of the form \$4 t+1\$.\allowbreak{} Show that \$a\$ is a primitive root modulo \$p\$ iff \$-{}a\$ is a primitive root modulo \$p\$.\allowbreak{}}
\pe
\pl{0}{Lean 4 formalization attempt:\allowbreak{}}
\pl{0}{\char0{}\char0{}\char0{}lean}
\pl{0}{import Mathlib}
\pe
\pl{0}{open Nat}
\pe
\pl{0}{theorem exercise\_\allowbreak{}9\_\allowbreak{}18 \{\allowbreak{}p :\allowbreak{} ℕ\} (\allowbreak{}hp :\allowbreak{} p.\allowbreak{}Prime) (\allowbreak{}h :\allowbreak{} p \% 4 =\allowbreak{} 1) (\allowbreak{}a :\allowbreak{} ℕ) :\allowbreak{}}
\pl{4}{(\allowbreak{}IsPrimitiveRoot a p) ↔ (\allowbreak{}IsPrimitiveRoot (\allowbreak{}p -{} a) p) :\allowbreak{}=\allowbreak{} sorry}
\pl{0}{\char0{}\char0{}\char0{}}
\pe
\pl{0}{Review the attempt and give your verdict.\allowbreak{}}
\end{promptlines}

\end{promptbox}

\begin{promptbox}{Self-Refine on Lean, revision \textmd{(example: ProofNet problem \texttt{exercise\_4\_4})}}
\noindent{\small\bfseries SYSTEM:}\par
\begin{promptlines}
\pl{0}{You are an expert in formal language Lean4.\allowbreak{}}
\pl{0}{You will be given a mathematical statement written in natural language,\allowbreak{} a Lean 4 formalization attempt,\allowbreak{} and reviewer feedback on that attempt.\allowbreak{}}
\pl{0}{Your task is to refine the formal code according to the feedback while staying faithful to the natural language statement.\allowbreak{}}
\pl{0}{Here are some instructions for your task:\allowbreak{}}
\pl{0}{1.\allowbreak{} You should give the formal code directly without any additional explanation or any proof.\allowbreak{} If the statement is a theorem or lemma,\allowbreak{} omit the proof and use the default \textquotesingle{}sorry\textquotesingle{} mode.\allowbreak{}}
\pl{0}{2.\allowbreak{} In case that you need to import any necessary preambles,\allowbreak{} you should not import any fake (\allowbreak{}non-{}exist) preambles.\allowbreak{}}
\pl{0}{3.\allowbreak{} You should wrap the formal code in a way illustrated as the following:\allowbreak{}}
\pl{0}{\char0{}\char0{}\char0{}lean}
\pl{0}{Your Formal Code}
\pl{0}{\char0{}\char0{}\char0{}}
\pl{0}{4.\allowbreak{} Lean 4 /\allowbreak{} Mathlib notation constraint — for \char0{}Finset\char0{} summation and product notation,\allowbreak{} use the MEMBERSHIP form with \char0{}∈\char0{} (\allowbreak{}Unicode U+2208),\allowbreak{} NOT the legacy \char0{}in\char0{} keyword form (\allowbreak{}\char0{}∑ k ∈ Finset.\allowbreak{}range 12,\allowbreak{} f k\char0{},\allowbreak{} never \char0{}∑ k in .\allowbreak{}.\allowbreak{}.\allowbreak{}\char0{}).\allowbreak{} Same rule applies to \char0{}∏\char0{},\allowbreak{} \char0{}⋃\char0{},\allowbreak{} \char0{}⋂\char0{},\allowbreak{} and any other indexed-{}set binder.\allowbreak{} If the previous code uses the correct \char0{}∈\char0{} form,\allowbreak{} PRESERVE it.\allowbreak{}}
\pl{0}{Strictly follow the instructions that have been claimed.\allowbreak{}}
\pe
\end{promptlines}

\noindent{\small\bfseries USER:}\par
\begin{promptlines}
\pl{0}{Natural language statement:\allowbreak{} Consider a prime \$p\$ of the form \$4 t+1\$.\allowbreak{} Show that \$a\$ is a primitive root modulo \$p\$ iff \$-{}a\$ is a primitive root modulo \$p\$.\allowbreak{}}
\pe
\pl{0}{Here is the current .\allowbreak{}lean file:\allowbreak{}}
\pl{0}{\char0{}\char0{}\char0{}lean}
\pl{0}{import Mathlib}
\pe
\pl{0}{open Nat}
\pe
\pl{0}{theorem exercise\_\allowbreak{}9\_\allowbreak{}18 \{\allowbreak{}p :\allowbreak{} ℕ\} (\allowbreak{}hp :\allowbreak{} p.\allowbreak{}Prime) (\allowbreak{}h :\allowbreak{} p \% 4 =\allowbreak{} 1) (\allowbreak{}a :\allowbreak{} ℕ) :\allowbreak{}}
\pl{4}{(\allowbreak{}IsPrimitiveRoot a p) ↔ (\allowbreak{}IsPrimitiveRoot (\allowbreak{}p -{} a) p) :\allowbreak{}=\allowbreak{} sorry}
\pl{0}{\char0{}\char0{}\char0{}}
\pe
\pl{0}{A reviewer provided the following feedback:\allowbreak{}}
\pl{0}{The formalization has several issues:\allowbreak{}}
\pe
\pl{0}{1.\allowbreak{} **Wrong statement**:\allowbreak{} The natural language says \textquotedbl{}if and only if \char0{}-{}a\char0{} is a primitive root modulo \char0{}p\char0{}\textquotedbl{}.\allowbreak{} In modular arithmetic modulo \char0{}p\char0{},\allowbreak{} \char0{}-{}a\char0{} is \char0{}p -{} a\char0{} only when \char0{}a\char0{} is taken as an integer modulo \char0{}p\char0{}.\allowbreak{} However,\allowbreak{} the statement uses \char0{}ℕ\char0{} for \char0{}a\char0{},\allowbreak{} and \char0{}p -{} a\char0{} in \char0{}ℕ\char0{} is subtraction truncated to zero (\allowbreak{}since \char0{}Nat\char0{} subtraction is 0 when \char0{}a \textgreater{} p\char0{}).\allowbreak{} This does not correctly represent \char0{}-{}a\char0{} modulo \char0{}p\char0{}.\allowbreak{} You should work in \char0{}ZMod p\char0{} or use integer arithmetic.\allowbreak{}}
\pe
\pl{0}{2.\allowbreak{} **Missing hypothesis**:\allowbreak{} The natural language says \textquotedbl{}Consider a prime \char0{}p\char0{} of the form \char0{}4t+1\char0{}\textquotedbl{}.\allowbreak{} The formalization uses \char0{}h :\allowbreak{} p \% 4 =\allowbreak{} 1\char0{},\allowbreak{} which is correct,\allowbreak{} but it also needs \char0{}p\char0{} to be prime (\allowbreak{}\char0{}hp :\allowbreak{} p.\allowbreak{}Prime\char0{} is given).\allowbreak{} However,\allowbreak{} the condition \char0{}p \% 4 =\allowbreak{} 1\char0{} is not enough:\allowbreak{} it should be \char0{}p \% 4 =\allowbreak{} 1\char0{} **and** \char0{}p\char0{} prime.\allowbreak{} That part is fine,\allowbreak{} but note that \char0{}p\char0{} being prime is already used.\allowbreak{}}
\pe
\pl{0}{3.\allowbreak{} **Wrong type for \char0{}IsPrimitiveRoot\char0{}**:\allowbreak{} In Mathlib,\allowbreak{} \char0{}IsPrimitiveRoot\char0{} is typically defined for elements of a \char0{}CommMonoid\char0{} with respect to a base,\allowbreak{} often \char0{}IsPrimitiveRoot (\allowbreak{}a :\allowbreak{} ZMod p) (\allowbreak{}p-{}1)\char0{} (\allowbreak{}the multiplicative order).\allowbreak{} The statement as written uses \char0{}IsPrimitiveRoot a p\char0{},\allowbreak{} which would mean \char0{}a\char0{} is a primitive root of unity of order \char0{}p\char0{},\allowbreak{} not a primitive root modulo \char0{}p\char0{}.\allowbreak{} The correct notion is \char0{}IsPrimitiveRoot (\allowbreak{}a :\allowbreak{} ZMod p) (\allowbreak{}p-{}1)\char0{}.\allowbreak{}}
\pe
\pl{0}{4.\allowbreak{} **Missing \char0{}ZMod\char0{} context**:\allowbreak{} The theorem should be stated in terms of \char0{}ZMod p\char0{} or \char0{}(\allowbreak{}ZMod p)ˣ\char0{}.\allowbreak{} The current type \char0{}a :\allowbreak{} ℕ\char0{} and \char0{}IsPrimitiveRoot a p\char0{} is not the right type for primitive roots modulo \char0{}p\char0{}.\allowbreak{}}
\pe
\pl{0}{5.\allowbreak{} **\char0{}open Nat\char0{} is unnecessary** but not harmful;\allowbreak{} however,\allowbreak{} \char0{}Nat\char0{} subtraction is problematic as noted.\allowbreak{}}
\pe
\pl{0}{6.\allowbreak{} **\char0{}p -{} a\char0{} is not \char0{}-{}a\char0{} modulo p**:\allowbreak{} Even if we interpret \char0{}a\char0{} as an integer,\allowbreak{} \char0{}-{}a\char0{} modulo \char0{}p\char0{} is \char0{}(\allowbreak{}p -{} a) \% p\char0{},\allowbreak{} not \char0{}p -{} a\char0{} as a natural number.\allowbreak{} In \char0{}ℕ\char0{},\allowbreak{} \char0{}p -{} a\char0{} is zero when \char0{}a \textgreater{} p\char0{},\allowbreak{} which is wrong.\allowbreak{}}
\pe
\pl{0}{**Verdict**:\allowbreak{} The formalization does not correctly capture the mathematical statement due to type errors,\allowbreak{} misuse of \char0{}IsPrimitiveRoot\char0{},\allowbreak{} and incorrect representation of \char0{}-{}a\char0{}.\allowbreak{}}
\pe
\pl{0}{VERDICT:\allowbreak{} NEEDS\_\allowbreak{}REVISION}
\pe
\pl{0}{Refine the formal code according to the feedback while staying faithful to the natural language statement.\allowbreak{} Output the COMPLETE refined .\allowbreak{}lean file.\allowbreak{}}
\pe
\pl{0}{STRICT RULES for the imports/\allowbreak{}opens header:\allowbreak{}}
\pl{0}{-{} Copy every \char0{}import .\allowbreak{}.\allowbreak{}.\allowbreak{}\char0{} and \char0{}open .\allowbreak{}.\allowbreak{}.\allowbreak{}\char0{} line from the previous file VERBATIM,\allowbreak{} in the\ \ \ same order,\allowbreak{} unless the feedback explicitly identifies a header problem.\allowbreak{} Do NOT\ \ \ change \char0{}import Mathlib\char0{} to fine-{}grained paths like\ \ \ \char0{}Mathlib.\allowbreak{}Algebra.\allowbreak{}BigOperators.\allowbreak{}Basic\char0{};\allowbreak{} those are legacy paths and DO NOT EXIST in\ \ \ current Mathlib (\allowbreak{}the build will fail with \char0{}object file .\allowbreak{}.\allowbreak{}.\allowbreak{} does not exist\char0{}).\allowbreak{}}
\pl{0}{-{} All \char0{}import\char0{} lines must appear before any \char0{}open\char0{} or \char0{}theorem\char0{}.\allowbreak{}}
\pe
\pe
\end{promptlines}

\end{promptbox}

\begin{promptbox}{Self-Refine on Isabelle, critique \textmd{(example: e-SNLI problem 72)}}
\noindent{\small\bfseries SYSTEM:}\par
\begin{promptlines}
\pl{0}{You are an expert reviewer of Isabelle/\allowbreak{}HOL (\allowbreak{}Isabelle 2024) theory files.\allowbreak{}}
\pl{0}{You will be given a natural-{}language entailment task (\allowbreak{}premise,\allowbreak{} hypothesis,\allowbreak{} explanation) and an Isabelle theory file that attempts to formalize it.\allowbreak{} Review the file purely by reading it:\allowbreak{} you have NO access to an Isabelle process or any external tool.\allowbreak{}}
\pl{0}{List the concrete problems you can find (\allowbreak{}syntax errors,\allowbreak{} Unicode operators instead of ASCII escapes like \textbackslash{}\textless{}and\textgreater{} and \textbackslash{}\textless{}longrightarrow\textgreater{},\allowbreak{} type signature errors,\allowbreak{} unbound or undeclared constants,\allowbreak{} theory-{}name mismatch,\allowbreak{} a proof body other than \char0{}sorry\char0{},\allowbreak{} mismatch between the formalization and the premise/\allowbreak{}hypothesis/\allowbreak{}explanation,\allowbreak{} missing explanation axioms).\allowbreak{} Be specific and actionable;\allowbreak{} do not rewrite the file yourself.\allowbreak{}}
\pl{0}{End your review with exactly one line:\allowbreak{}}
\pl{0}{VERDICT:\allowbreak{} CORRECT}
\pl{0}{or}
\pl{0}{VERDICT:\allowbreak{} NEEDS\_\allowbreak{}REVISION}
\pe
\end{promptlines}

\noindent{\small\bfseries USER:}\par
\begin{promptlines}
\pl{0}{Theory name (\allowbreak{}use this exact name):\allowbreak{} eval\_\allowbreak{}72}
\pe
\pl{0}{Premise:\allowbreak{} Two women are wet while holding hands with a long building and buses in the background.\allowbreak{}}
\pl{0}{Hypothesis:\allowbreak{} Women are in the city.\allowbreak{}}
\pl{0}{Explanation:\allowbreak{}\ \ There are women because there are two women.\allowbreak{}}
\pe
\pl{0}{Isabelle theory file attempt:\allowbreak{}}
\pl{0}{\char0{}\char0{}\char0{}isabelle}
\pl{0}{theory eval\_\allowbreak{}72}
\pl{0}{imports Main}
\pl{0}{begin}
\pe
\pl{0}{typedecl entity}
\pl{0}{typedecl event}
\pe
\pl{0}{consts}
\pl{2}{TwoWomen :\allowbreak{}:\allowbreak{} \textquotedbl{}entity \textbackslash{}\textless{}Rightarrow\textgreater{} bool\textquotedbl{}}
\pl{2}{Women :\allowbreak{}:\allowbreak{} \textquotedbl{}entity \textbackslash{}\textless{}Rightarrow\textgreater{} bool\textquotedbl{}}
\pl{2}{Wet :\allowbreak{}:\allowbreak{} \textquotedbl{}entity \textbackslash{}\textless{}Rightarrow\textgreater{} bool\textquotedbl{}}
\pl{2}{HoldingHands :\allowbreak{}:\allowbreak{} \textquotedbl{}entity \textbackslash{}\textless{}Rightarrow\textgreater{} bool\textquotedbl{}}
\pl{2}{LongBuilding :\allowbreak{}:\allowbreak{} \textquotedbl{}entity \textbackslash{}\textless{}Rightarrow\textgreater{} bool\textquotedbl{}}
\pl{2}{Buses :\allowbreak{}:\allowbreak{} \textquotedbl{}entity \textbackslash{}\textless{}Rightarrow\textgreater{} bool\textquotedbl{}}
\pl{2}{InBackground :\allowbreak{}:\allowbreak{} \textquotedbl{}entity \textbackslash{}\textless{}Rightarrow\textgreater{} bool\textquotedbl{}}
\pl{2}{InCity :\allowbreak{}:\allowbreak{} \textquotedbl{}entity \textbackslash{}\textless{}Rightarrow\textgreater{} bool\textquotedbl{}}
\pe
\pl{0}{axiomatization where}
\pl{2}{explanation\_\allowbreak{}1:\allowbreak{} \textquotedbl{}\textbackslash{}\textless{}forall\textgreater{}x.\allowbreak{} TwoWomen x \textbackslash{}\textless{}longrightarrow\textgreater{} Women x\textquotedbl{}}
\pe
\pl{0}{theorem hypothesis:\allowbreak{}}
\pl{2}{assumes asm:\allowbreak{} \textquotedbl{}TwoWomen x \textbackslash{}\textless{}and\textgreater{} Wet x \textbackslash{}\textless{}and\textgreater{} HoldingHands x \textbackslash{}\textless{}and\textgreater{} LongBuilding y \textbackslash{}\textless{}and\textgreater{} Buses z \textbackslash{}\textless{}and\textgreater{} InBackground x\textquotedbl{}}
\pl{2}{shows \textquotedbl{}\textbackslash{}\textless{}exists\textgreater{}x.\allowbreak{} Women x \textbackslash{}\textless{}and\textgreater{} InCity x\textquotedbl{}}
\pl{2}{sorry}
\pe
\pl{0}{end}
\pl{0}{\char0{}\char0{}\char0{}}
\pe
\pl{0}{Review the attempt and give your verdict.\allowbreak{}}
\end{promptlines}

\end{promptbox}

\begin{promptbox}{Self-Refine on Isabelle, revision \textmd{(example: e-SNLI problem 72)}}
\noindent{\small\bfseries SYSTEM:}\par
\begin{promptlines}
\pl{0}{You are an expert in formal language Isabelle/\allowbreak{}HOL (\allowbreak{}Isabelle 2024).\allowbreak{}}
\pl{0}{You will be given an entailment task (\allowbreak{}premise,\allowbreak{} hypothesis,\allowbreak{} explanation),\allowbreak{} an Isabelle theory file that attempted to capture it,\allowbreak{} and reviewer feedback on that attempt.\allowbreak{}}
\pl{0}{Your task is to refine the theory file according to the feedback while staying faithful to the entailment task.\allowbreak{} Output the COMPLETE refined .\allowbreak{}thy file.\allowbreak{}}
\pl{0}{STRICT CONSTRAINTS:\allowbreak{}}
\pl{0}{1.\allowbreak{} ASCII Isabelle escapes ONLY (\allowbreak{}\char0{}\textbackslash{}\textless{}and\textgreater{}\char0{},\allowbreak{} \char0{}\textbackslash{}\textless{}longrightarrow\textgreater{}\char0{},\allowbreak{} \char0{}\textbackslash{}\textless{}forall\textgreater{}\char0{},\allowbreak{} \char0{}\textbackslash{}\textless{}exists\textgreater{}\char0{},\allowbreak{} \char0{}\textbackslash{}\textless{}Rightarrow\textgreater{}\char0{},\allowbreak{} etc.\allowbreak{}).\allowbreak{} Do NOT use the Unicode forms.\allowbreak{}}
\pl{0}{2.\allowbreak{} The proof body MUST remain exactly \char0{}sorry\char0{}.\allowbreak{}}
\pl{0}{3.\allowbreak{} Preserve the \char0{}theory \textless{}name\textgreater{}\char0{} header verbatim — the filename must match the theory name.\allowbreak{}}
\pl{0}{4.\allowbreak{} The file MUST have shape \char0{}theory \textless{}name\textgreater{} imports Main begin .\allowbreak{}.\allowbreak{}.\allowbreak{} end\char0{}.\allowbreak{}}
\pl{0}{5.\allowbreak{} Wrap your output in fences exactly as:\allowbreak{}}
\pl{0}{\char0{}\char0{}\char0{}isabelle}
\pl{0}{\textless{}your theory file content\textgreater{}}
\pl{0}{\char0{}\char0{}\char0{}}
\pl{0}{Output nothing outside the fences.\allowbreak{}}
\pe
\end{promptlines}

\noindent{\small\bfseries USER:}\par
\begin{promptlines}
\pl{0}{Theory name (\allowbreak{}use this exact name):\allowbreak{} eval\_\allowbreak{}72}
\pe
\pl{0}{Premise:\allowbreak{} Two women are wet while holding hands with a long building and buses in the background.\allowbreak{}}
\pl{0}{Hypothesis:\allowbreak{} Women are in the city.\allowbreak{}}
\pl{0}{Explanation:\allowbreak{}\ \ There are women because there are two women.\allowbreak{}}
\pe
\pl{0}{Here is the current .\allowbreak{}thy file:\allowbreak{}}
\pl{0}{\char0{}\char0{}\char0{}isabelle}
\pl{0}{theory eval\_\allowbreak{}72}
\pl{0}{imports Main}
\pl{0}{begin}
\pe
\pl{0}{typedecl entity}
\pl{0}{typedecl event}
\pe
\pl{0}{consts}
\pl{2}{TwoWomen :\allowbreak{}:\allowbreak{} \textquotedbl{}entity \textbackslash{}\textless{}Rightarrow\textgreater{} bool\textquotedbl{}}
\pl{2}{Women :\allowbreak{}:\allowbreak{} \textquotedbl{}entity \textbackslash{}\textless{}Rightarrow\textgreater{} bool\textquotedbl{}}
\pl{2}{Wet :\allowbreak{}:\allowbreak{} \textquotedbl{}entity \textbackslash{}\textless{}Rightarrow\textgreater{} bool\textquotedbl{}}
\pl{2}{HoldingHands :\allowbreak{}:\allowbreak{} \textquotedbl{}entity \textbackslash{}\textless{}Rightarrow\textgreater{} bool\textquotedbl{}}
\pl{2}{LongBuilding :\allowbreak{}:\allowbreak{} \textquotedbl{}entity \textbackslash{}\textless{}Rightarrow\textgreater{} bool\textquotedbl{}}
\pl{2}{Buses :\allowbreak{}:\allowbreak{} \textquotedbl{}entity \textbackslash{}\textless{}Rightarrow\textgreater{} bool\textquotedbl{}}
\pl{2}{InBackground :\allowbreak{}:\allowbreak{} \textquotedbl{}entity \textbackslash{}\textless{}Rightarrow\textgreater{} bool\textquotedbl{}}
\pl{2}{InCity :\allowbreak{}:\allowbreak{} \textquotedbl{}entity \textbackslash{}\textless{}Rightarrow\textgreater{} bool\textquotedbl{}}
\pe
\pl{0}{axiomatization where}
\pl{2}{explanation\_\allowbreak{}1:\allowbreak{} \textquotedbl{}\textbackslash{}\textless{}forall\textgreater{}x.\allowbreak{} TwoWomen x \textbackslash{}\textless{}longrightarrow\textgreater{} Women x\textquotedbl{}}
\pe
\pl{0}{theorem hypothesis:\allowbreak{}}
\pl{2}{assumes asm:\allowbreak{} \textquotedbl{}TwoWomen x \textbackslash{}\textless{}and\textgreater{} Wet x \textbackslash{}\textless{}and\textgreater{} HoldingHands x \textbackslash{}\textless{}and\textgreater{} LongBuilding y \textbackslash{}\textless{}and\textgreater{} Buses z \textbackslash{}\textless{}and\textgreater{} InBackground x\textquotedbl{}}
\pl{2}{shows \textquotedbl{}\textbackslash{}\textless{}exists\textgreater{}x.\allowbreak{} Women x \textbackslash{}\textless{}and\textgreater{} InCity x\textquotedbl{}}
\pl{2}{sorry}
\pe
\pl{0}{end}
\pl{0}{\char0{}\char0{}\char0{}}
\pe
\pl{0}{A reviewer provided the following feedback:\allowbreak{}}
\pl{0}{Here are the concrete problems:\allowbreak{}}
\pe
\pl{0}{1.\allowbreak{} **Missing explanation axioms**:\allowbreak{} The explanation states \textquotedbl{}There are women because there are two women.\allowbreak{}\textquotedbl{} This should be captured as an axiom,\allowbreak{} but only \char0{}explanation\_\allowbreak{}1\char0{} is given (\allowbreak{}which is fine),\allowbreak{} but there is no axiom linking \char0{}InCity\char0{} to the premise\textquotesingle{}s background elements (\allowbreak{}long building,\allowbreak{} buses).\allowbreak{} The hypothesis requires \char0{}InCity x\char0{},\allowbreak{} but the premise only mentions \char0{}LongBuilding\char0{},\allowbreak{} \char0{}Buses\char0{},\allowbreak{} and \char0{}InBackground\char0{} — there is no axiom stating that having a long building and buses in the background implies being in the city.\allowbreak{} The formalization is incomplete.\allowbreak{}}
\pe
\pl{0}{2.\allowbreak{} **Type signature mismatch /\allowbreak{} unbound variables**:\allowbreak{} In the theorem statement,\allowbreak{} \char0{}x\char0{},\allowbreak{} \char0{}y\char0{},\allowbreak{} \char0{}z\char0{} appear in the assumption but are not quantified.\allowbreak{} The assumption \char0{}asm\char0{} uses free variables \char0{}x\char0{},\allowbreak{} \char0{}y\char0{},\allowbreak{} \char0{}z\char0{} — these should be universally quantified (\allowbreak{}e.\allowbreak{}g.\allowbreak{},\allowbreak{} \char0{}\textbackslash{}\textless{}forall\textgreater{}x y z.\allowbreak{} .\allowbreak{}.\allowbreak{}.\allowbreak{}\char0{}) or existentially quantified depending on intent.\allowbreak{} As written,\allowbreak{} the theorem is not a closed formula and will cause a type-{}checking error.\allowbreak{}}
\pe
\pl{0}{3.\allowbreak{} **Missing \char0{}InBackground\char0{} relation**:\allowbreak{} The constant \char0{}InBackground\char0{} is typed as \char0{}entity \textbackslash{}\textless{}Rightarrow\textgreater{} bool\char0{},\allowbreak{} but in natural language \textquotedbl{}in the background\textquotedbl{} is a relation between entities (\allowbreak{}e.\allowbreak{}g.\allowbreak{},\allowbreak{} buses are in the background of the women).\allowbreak{} The formalization treats it as a property of a single entity,\allowbreak{} which is semantically wrong.\allowbreak{}}
\pe
\pl{0}{4.\allowbreak{} **\char0{}HoldingHands\char0{} as a unary predicate**:\allowbreak{} Holding hands is a relation between two entities,\allowbreak{} not a property of one.\allowbreak{} The formalization should use a binary predicate.\allowbreak{}}
\pe
\pl{0}{5.\allowbreak{} **\char0{}Wet\char0{} applied to women**:\allowbreak{} The premise says \textquotedbl{}two women are wet\textquotedbl{},\allowbreak{} but \char0{}Wet\char0{} is applied to the same entity \char0{}x\char0{} that is \char0{}TwoWomen\char0{}.\allowbreak{} That is acceptable,\allowbreak{} but note that \char0{}TwoWomen\char0{} is a property of a single entity,\allowbreak{} which is unusual — typically \textquotedbl{}two women\textquotedbl{} would be a plural or a count predicate.\allowbreak{} This is a minor modeling issue but not a syntax error.\allowbreak{}}
\pe
\pl{0}{**Verdict**:\allowbreak{} NEEDS\_\allowbreak{}REVISION}
\pe
\pl{0}{Refine the formal code according to the feedback while staying faithful to the entailment task.\allowbreak{} Output the COMPLETE refined .\allowbreak{}thy file with the \char0{}theory\char0{} header,\allowbreak{} \char0{}imports\char0{},\allowbreak{} \char0{}begin\char0{},\allowbreak{} body,\allowbreak{} and \char0{}end\char0{} lines all present and the proof body still \char0{}sorry\char0{}.\allowbreak{}}
\pe
\end{promptlines}

\end{promptbox}

\begin{promptbox}{Reflexion on Lean, reflection \textmd{(example: ProofNet problem \texttt{exercise\_2\_29})}}
{\small\itshape The prover-error excerpt in this example is taken from the prover errors recorded for the same candidate in our main run.}

\noindent{\small\bfseries SYSTEM:}\par
\begin{promptlines}
\pl{0}{You are an expert in formal language Lean4 doing a post-{}mortem on your own failed formalization attempt.\allowbreak{}}
\pl{0}{You will be given a mathematical statement in natural language,\allowbreak{} your Lean 4 formalization attempt,\allowbreak{} and the prover\textquotesingle{}s error output (\allowbreak{}the attempt FAILED to compile/\allowbreak{}elaborate).\allowbreak{}}
\pl{0}{Write a short reflection (\allowbreak{}at most 150 words):\allowbreak{} diagnose the most likely root cause of the failure and state a concrete plan for what to do differently in the next attempt.\allowbreak{} Do NOT write any code;\allowbreak{} write prose only.\allowbreak{}}
\end{promptlines}

\noindent{\small\bfseries USER:}\par
\begin{promptlines}
\pl{0}{Prove that every open set in \$\textbackslash{}mathbb\{\allowbreak{}R\}\$ is the union of an at most countable collection of disjoint segments.\allowbreak{}}
\pe
\pl{0}{Your failed attempt:\allowbreak{}}
\pl{0}{\char0{}\char0{}\char0{}}
\pl{0}{import Mathlib}
\pe
\pl{0}{open Set Real}
\pl{0}{open scoped Topology}
\pe
\pl{0}{theorem exercise\_\allowbreak{}2\_\allowbreak{}6 \{\allowbreak{}U :\allowbreak{} Set ℝ\} (\allowbreak{}hU :\allowbreak{} IsOpen U) :\allowbreak{} }
\pl{2}{∃ (\allowbreak{}I :\allowbreak{} Set (\allowbreak{}Set ℝ)),\allowbreak{} (\allowbreak{}∀ s ∈ I,\allowbreak{} ∃ a b :\allowbreak{} ℝ,\allowbreak{} s =\allowbreak{} Set.\allowbreak{}Ioo a b ∨ s =\allowbreak{} Set.\allowbreak{}Ioc a b ∨ s =\allowbreak{} Set.\allowbreak{}Ico a b ∨ s =\allowbreak{} Set.\allowbreak{}Icc a b) ∧}
\pl{2}{(\allowbreak{}∀ s t ∈ I,\allowbreak{} s ≠ t → s ∩ t =\allowbreak{} ∅) ∧ (\allowbreak{}⋃₀ I =\allowbreak{} U) ∧ Countable I :\allowbreak{}=\allowbreak{} sorry}
\pl{0}{\char0{}\char0{}\char0{}}
\pe
\pl{0}{Prover output (\allowbreak{}status:\allowbreak{} failed):\allowbreak{}}
\pl{0}{unexpected token \textquotesingle{}∈\textquotesingle{};\allowbreak{} expected \textquotesingle{},\allowbreak{}\textquotesingle{}}
\pe
\pl{0}{Write your reflection.\allowbreak{}}
\end{promptlines}

\end{promptbox}

\begin{promptbox}{Reflexion on Lean, retry \textmd{(example: ProofNet problem \texttt{exercise\_2\_29})}}
\noindent{\small\bfseries SYSTEM:}\par
\begin{promptlines}
\pl{0}{You are an expert in formal language Lean4.\allowbreak{}}
\pl{0}{You will be given a mathematical statement written in natural language and LaTeX symbols.\allowbreak{}}
\pl{0}{Your task is to provide the formal code of the given natural language mathematical statement in Lean4 with the following instructions:\allowbreak{}}
\pl{0}{1.\allowbreak{} You should give the formal code directly without any additional explanation.\allowbreak{} If the given mathematical statement is a theorem or lemma,\allowbreak{} omit the proof and use the default \textquotesingle{}sorry\textquotesingle{} mode in the formal code.\allowbreak{}2.\allowbreak{} In case that you need to import any necessary preambles,\allowbreak{} you should not import any fake (\allowbreak{}non-{}exist) preambles.\allowbreak{}}
\pl{0}{3.\allowbreak{} You should wrap the formal code in a way illustrated as the following:\allowbreak{}}
\pl{0}{\char0{}\char0{}\char0{}lean}
\pl{0}{Your Formal Code}
\pl{0}{\char0{}\char0{}\char0{}}
\pl{0}{4.\allowbreak{} Lean 4 /\allowbreak{} Mathlib notation constraint — for \char0{}Finset\char0{} summation and product notation,\allowbreak{} use the MEMBERSHIP form with \char0{}∈\char0{} (\allowbreak{}Unicode U+2208),\allowbreak{} NOT the legacy \char0{}in\char0{} keyword form.\allowbreak{} Mathlib\textquotesingle{}s current BigOperators notation requires:\allowbreak{}}
\pl{3}{✓ \char0{}∑ k ∈ Finset.\allowbreak{}range 12,\allowbreak{} f k\char0{}\ \ \ \ \ -{}-{} correct}
\pl{3}{✗ \char0{}∑ k in Finset.\allowbreak{}range 12,\allowbreak{} f k\char0{}\ \ \ \ -{}-{} INVALID,\allowbreak{} fails to parse with \char0{}unexpected token \textquotesingle{}in\textquotesingle{}\char0{}}
\pl{0}{Same rule applies to \char0{}∏\char0{},\allowbreak{} \char0{}⋃\char0{},\allowbreak{} \char0{}⋂\char0{},\allowbreak{} and any other indexed-{}set binder:\allowbreak{} use \char0{}∈\char0{},\allowbreak{} not \char0{}in\char0{}.\allowbreak{}}
\pl{0}{Strictly follow the instructions that have been claimed.\allowbreak{}}
\pe
\end{promptlines}

\noindent{\small\bfseries USER:}\par
\begin{promptlines}
\pl{0}{Natural language statement:\allowbreak{} Prove that there is no rational number whose square is \$12\$.\allowbreak{}}
\pl{0}{Give me the Lean4 formal code of the statement:\allowbreak{}}
\pe
\end{promptlines}

\noindent{\small\bfseries ASSISTANT:}\par
\begin{promptlines}
\pl{0}{\char0{}\char0{}\char0{}lean}
\pl{0}{import Mathlib}
\pe
\pl{0}{open Topology Filter Real Complex TopologicalSpace Finset}
\pl{0}{open scoped BigOperators}
\pe
\pl{0}{theorem exercise\_\allowbreak{}1\_\allowbreak{}2 :\allowbreak{} ¬ ∃ (\allowbreak{}x :\allowbreak{} ℚ),\allowbreak{} (\allowbreak{} x \textasciicircum{} 2 =\allowbreak{} 12 ) :\allowbreak{}=\allowbreak{} sorry}
\pl{0}{\char0{}\char0{}\char0{}}
\pe
\end{promptlines}

\noindent{\small\bfseries USER:}\par
\begin{promptlines}
\pl{0}{Natural language statement:\allowbreak{} For all odd \$n\$ show that \$8 \textbackslash{}mid n\textasciicircum{}\{\allowbreak{}2\}-{}1\$.\allowbreak{}}
\pl{0}{Give me the Lean4 formal code of the statement:\allowbreak{}}
\pe
\end{promptlines}

\noindent{\small\bfseries ASSISTANT:}\par
\begin{promptlines}
\pl{0}{\char0{}\char0{}\char0{}lean}
\pl{0}{import Mathlib}
\pe
\pl{0}{open Real}
\pl{0}{open scoped BigOperators}
\pe
\pl{0}{theorem exercise\_\allowbreak{}1\_\allowbreak{}27 \{\allowbreak{}n :\allowbreak{} ℕ\} (\allowbreak{}hn :\allowbreak{} Odd n) :\allowbreak{} 8 ∣ (\allowbreak{}n\textasciicircum{}2 -{} 1) :\allowbreak{}=\allowbreak{} sorry}
\pl{0}{\char0{}\char0{}\char0{}}
\pe
\end{promptlines}

\noindent{\small\bfseries USER:}\par
\begin{promptlines}
\pl{0}{Natural language statement:\allowbreak{} Prove that \$x\textasciicircum{}3 -{} 3x -{} 1\$ is irreducible over \$\textbackslash{}mathbb\{\allowbreak{}Q\}\$.\allowbreak{}}
\pl{0}{Give me the Lean4 formal code of the statement:\allowbreak{}}
\pe
\end{promptlines}

\noindent{\small\bfseries ASSISTANT:}\par
\begin{promptlines}
\pl{0}{\char0{}\char0{}\char0{}lean}
\pl{0}{import Mathlib}
\pe
\pl{0}{open Fintype Set Real Ideal Polynomial}
\pl{0}{open scoped BigOperators}
\pe
\pl{0}{theorem exercise\_\allowbreak{}5\_\allowbreak{}5\_\allowbreak{}2 :\allowbreak{} Irreducible (\allowbreak{}X\textasciicircum{}3 -{} 3*X -{} 1 :\allowbreak{} Polynomial ℚ) :\allowbreak{}=\allowbreak{} sorry}
\pl{0}{\char0{}\char0{}\char0{}}
\pe
\end{promptlines}

\noindent{\small\bfseries USER:}\par
\begin{promptlines}
\pl{0}{Natural language statement:\allowbreak{} Prove that every open set in \$\textbackslash{}mathbb\{\allowbreak{}R\}\$ is the union of an at most countable collection of disjoint segments.\allowbreak{}}
\pl{0}{Give me the Lean4 formal code of the statement:\allowbreak{}}
\pe
\pl{0}{Your previous attempt FAILED the prover check.\allowbreak{} Here is the failed attempt:\allowbreak{}}
\pl{0}{\char0{}\char0{}\char0{}}
\pl{0}{import Mathlib}
\pe
\pl{0}{open Set Real}
\pl{0}{open scoped Topology}
\pe
\pl{0}{theorem exercise\_\allowbreak{}2\_\allowbreak{}6 \{\allowbreak{}U :\allowbreak{} Set ℝ\} (\allowbreak{}hU :\allowbreak{} IsOpen U) :\allowbreak{} }
\pl{2}{∃ (\allowbreak{}I :\allowbreak{} Set (\allowbreak{}Set ℝ)),\allowbreak{} (\allowbreak{}∀ s ∈ I,\allowbreak{} ∃ a b :\allowbreak{} ℝ,\allowbreak{} s =\allowbreak{} Set.\allowbreak{}Ioo a b ∨ s =\allowbreak{} Set.\allowbreak{}Ioc a b ∨ s =\allowbreak{} Set.\allowbreak{}Ico a b ∨ s =\allowbreak{} Set.\allowbreak{}Icc a b) ∧}
\pl{2}{(\allowbreak{}∀ s t ∈ I,\allowbreak{} s ≠ t → s ∩ t =\allowbreak{} ∅) ∧ (\allowbreak{}⋃₀ I =\allowbreak{} U) ∧ Countable I :\allowbreak{}=\allowbreak{} sorry}
\pl{0}{\char0{}\char0{}\char0{}}
\pe
\pl{0}{Your reflections on the past failures (\allowbreak{}most recent last):\allowbreak{}}
\pl{0}{1.\allowbreak{} The error occurs because \char0{}∀ s t ∈ I,\allowbreak{} .\allowbreak{}.\allowbreak{}.\allowbreak{}\char0{} is not valid syntax;\allowbreak{} you need \char0{}∀ s ∈ I,\allowbreak{} ∀ t ∈ I,\allowbreak{} .\allowbreak{}.\allowbreak{}.\allowbreak{}\char0{}.\allowbreak{} The root cause is a syntax mistake in the quantifier pattern.\allowbreak{} Next time,\allowbreak{} carefully review binder syntax:\allowbreak{} use \char0{}∀ s ∈ I,\allowbreak{} ∀ t ∈ I,\allowbreak{} s ≠ t → .\allowbreak{}.\allowbreak{}.\allowbreak{}\char0{} or \char0{}∀ s t,\allowbreak{} s ∈ I → t ∈ I → .\allowbreak{}.\allowbreak{}.\allowbreak{}\char0{}.\allowbreak{}}
\pe
\pl{0}{Using these reflections,\allowbreak{} produce a corrected formalization.\allowbreak{} Follow the same output format rules as before (\allowbreak{}wrap the code in a fenced code block).\allowbreak{}}
\pe
\pl{0}{STRICT RULES for the imports/\allowbreak{}opens header:\allowbreak{}}
\pl{0}{-{} Copy every \char0{}import .\allowbreak{}.\allowbreak{}.\allowbreak{}\char0{} and \char0{}open .\allowbreak{}.\allowbreak{}.\allowbreak{}\char0{} line from the previous file VERBATIM,\allowbreak{} in the\ \ \ same order,\allowbreak{} unless the feedback explicitly identifies a header problem.\allowbreak{} Do NOT\ \ \ change \char0{}import Mathlib\char0{} to fine-{}grained paths like\ \ \ \char0{}Mathlib.\allowbreak{}Algebra.\allowbreak{}BigOperators.\allowbreak{}Basic\char0{};\allowbreak{} those are legacy paths and DO NOT EXIST in\ \ \ current Mathlib (\allowbreak{}the build will fail with \char0{}object file .\allowbreak{}.\allowbreak{}.\allowbreak{} does not exist\char0{}).\allowbreak{}}
\pl{0}{-{} All \char0{}import\char0{} lines must appear before any \char0{}open\char0{} or \char0{}theorem\char0{}.\allowbreak{}}
\pe
\end{promptlines}

\end{promptbox}

\begin{promptbox}{Reflexion on Isabelle, reflection \textmd{(example: e-SNLI problem 70)}}
{\small\itshape The prover-error excerpt in this example is taken from the prover errors recorded for the same candidate in our main run.}

\noindent{\small\bfseries SYSTEM:}\par
\begin{promptlines}
\pl{0}{You are an expert in formal language Isabelle/\allowbreak{}HOL (\allowbreak{}Isabelle 2024) doing a post-{}mortem on your own failed formalization attempt.\allowbreak{}}
\pl{0}{You will be given an entailment task (\allowbreak{}premise,\allowbreak{} hypothesis,\allowbreak{} explanation),\allowbreak{} your Isabelle theory file attempt,\allowbreak{} and the prover\textquotesingle{}s error output (\allowbreak{}the attempt FAILED Isabelle\textquotesingle{}s check).\allowbreak{}}
\pl{0}{Write a short reflection (\allowbreak{}at most 150 words):\allowbreak{} diagnose the most likely root cause of the failure and state a concrete plan for what to do differently in the next attempt.\allowbreak{} Do NOT write any code;\allowbreak{} write prose only.\allowbreak{}}
\end{promptlines}

\noindent{\small\bfseries USER:}\par
\begin{promptlines}
\pl{0}{Theory name (\allowbreak{}use this exact name):\allowbreak{} eval\_\allowbreak{}70}
\pe
\pl{0}{Premise:\allowbreak{} A man and a woman stare at the skyline.\allowbreak{}}
\pl{0}{Hypothesis:\allowbreak{} The two people are looking outside.\allowbreak{}}
\pl{0}{Explanation:\allowbreak{}\ \ A man and a woman are two people.\allowbreak{} \textquotedbl{}Looking outside\textquotedbl{} is a rephrasing of \textquotedbl{}stare at the skyline.\allowbreak{}\textquotedbl{}}
\pe
\pl{0}{Your failed attempt:\allowbreak{}}
\pl{0}{\char0{}\char0{}\char0{}}
\pl{0}{theory eval\_\allowbreak{}70}
\pl{0}{imports Main}
\pl{0}{begin}
\pe
\pl{0}{typedecl entity}
\pl{0}{typedecl event}
\pe
\pl{0}{consts}
\pl{2}{Man :\allowbreak{}:\allowbreak{} \textquotedbl{}entity \textbackslash{}\textless{}Rightarrow\textgreater{} bool\textquotedbl{}}
\pl{2}{Woman :\allowbreak{}:\allowbreak{} \textquotedbl{}entity \textbackslash{}\textless{}Rightarrow\textgreater{} bool\textquotedbl{}}
\pl{2}{StareAtSkyline :\allowbreak{}:\allowbreak{} \textquotedbl{}entity \textbackslash{}\textless{}Rightarrow\textgreater{} bool\textquotedbl{}}
\pl{2}{TwoPeople :\allowbreak{}:\allowbreak{} \textquotedbl{}entity \textbackslash{}\textless{}Rightarrow\textgreater{} bool\textquotedbl{}}
\pl{2}{LookingOutside :\allowbreak{}:\allowbreak{} \textquotedbl{}entity \textbackslash{}\textless{}Rightarrow\textgreater{} bool\textquotedbl{}}
\pe
\pl{0}{axiomatization where}
\pl{2}{explanation\_\allowbreak{}1:\allowbreak{} \textquotedbl{}\textbackslash{}\textless{}forall\textgreater{}x.\allowbreak{} (\allowbreak{}Man x \textbackslash{}\textless{}and\textgreater{} Woman x) \textbackslash{}\textless{}longrightarrow\textgreater{} TwoPeople x\textquotedbl{}}
\pl{2}{explanation\_\allowbreak{}2:\allowbreak{} \textquotedbl{}\textbackslash{}\textless{}forall\textgreater{}x.\allowbreak{} StareAtSkyline x \textbackslash{}\textless{}longleftrightarrow\textgreater{} LookingOutside x\textquotedbl{}}
\pe
\pl{0}{theorem hypothesis:\allowbreak{}}
\pl{2}{assumes asm:\allowbreak{} \textquotedbl{}Man x \textbackslash{}\textless{}and\textgreater{} Woman x \textbackslash{}\textless{}and\textgreater{} StareAtSkyline x\textquotedbl{}}
\pl{2}{shows \textquotedbl{}TwoPeople x \textbackslash{}\textless{}and\textgreater{} LookingOutside x\textquotedbl{}}
\pl{2}{sorry}
\pe
\pl{0}{end}
\pl{0}{\char0{}\char0{}\char0{}}
\pe
\pl{0}{Prover output (\allowbreak{}status:\allowbreak{} failed):\allowbreak{}}
\pl{0}{Outer syntax error (\allowbreak{}line 17 of \textquotedbl{}/\allowbreak{}tmp/\allowbreak{}isabelle\_\allowbreak{}eval/\allowbreak{}eval\_\allowbreak{}70.\allowbreak{}thy\textquotedbl{}):\allowbreak{} command expected,\allowbreak{} but identifier explanation\_\allowbreak{}2 (\allowbreak{}line 17 of \textquotedbl{}/\allowbreak{}tmp/\allowbreak{}isabelle\_\allowbreak{}eval/\allowbreak{}eval\_\allowbreak{}70.\allowbreak{}thy\textquotedbl{}) was found At command \textquotedbl{}\textless{}malformed\textgreater{}\textquotedbl{} (\allowbreak{}line 15 of \textquotedbl{}/\allowbreak{}tmp/\allowbreak{}isabelle\_\allowbreak{}eval/\allowbreak{}eval\_\allowbreak{}70.\allowbreak{}thy\textquotedbl{})}
\pe
\pl{0}{Write your reflection.\allowbreak{}}
\end{promptlines}

\end{promptbox}

\begin{promptbox}{Reflexion on Isabelle, retry \textmd{(example: e-SNLI problem 70)}}
\noindent{\small\bfseries SYSTEM:}\par
\begin{promptlines}
\pl{0}{You are an expert in formal language Isabelle/\allowbreak{}HOL (\allowbreak{}Isabelle 2024).\allowbreak{}}
\pl{0}{You will be given a natural-{}language entailment task:\allowbreak{} a premise,\allowbreak{} a hypothesis,\allowbreak{} and one or more explanations bridging them.\allowbreak{}}
\pl{0}{Your task is to produce a complete Isabelle 2024 theory file that formalizes the premise (\allowbreak{}as the \char0{}assumes\char0{} clause),\allowbreak{} the explanations (\allowbreak{}as \char0{}axiomatization\char0{} blocks),\allowbreak{} and the hypothesis (\allowbreak{}as the \char0{}shows\char0{} clause).\allowbreak{} The proof body MUST be exactly \char0{}sorry\char0{}.\allowbreak{} Proof completion is out of scope;\allowbreak{} we are validating the formalization only.\allowbreak{}}
\pe
\pl{0}{STRICT CONSTRAINTS (\allowbreak{}Isabelle 2024 source format):\allowbreak{}}
\pl{0}{1.\allowbreak{} Use ASCII Isabelle escapes for logical operators,\allowbreak{} NOT Unicode.\allowbreak{} Required substitutions:\allowbreak{}}
\pl{3}{∧\ \ →\ \ \textbackslash{}\textless{}and\textgreater{}}
\pl{3}{∨\ \ →\ \ \textbackslash{}\textless{}or\textgreater{}}
\pl{3}{¬\ \ →\ \ \textbackslash{}\textless{}not\textgreater{}}
\pl{3}{⟹\ \ →\ \ \textbackslash{}\textless{}Longrightarrow\textgreater{}}
\pl{3}{⟶\ \ →\ \ \textbackslash{}\textless{}longrightarrow\textgreater{}}
\pl{3}{⟷\ \ →\ \ \textbackslash{}\textless{}longleftrightarrow\textgreater{}}
\pl{3}{↔\ \ →\ \ \textbackslash{}\textless{}longleftrightarrow\textgreater{}}
\pl{3}{∀\ \ →\ \ \textbackslash{}\textless{}forall\textgreater{}}
\pl{3}{∃\ \ →\ \ \textbackslash{}\textless{}exists\textgreater{}}
\pl{3}{⋀\ \ →\ \ \textbackslash{}\textless{}And\textgreater{}}
\pl{3}{⇒\ \ →\ \ \textbackslash{}\textless{}Rightarrow\textgreater{}}
\pl{3}{→\ \ →\ \ \textbackslash{}\textless{}longrightarrow\textgreater{}}
\pl{0}{2.\allowbreak{} The proof body MUST be the single token \char0{}sorry\char0{} (\allowbreak{}no \char0{}proof -{} qed\char0{} block,\allowbreak{} no \char0{}by .\allowbreak{}.\allowbreak{}.\allowbreak{}\char0{} tactic,\allowbreak{} no \char0{}oops\char0{},\allowbreak{} no \char0{}admit\char0{}).\allowbreak{} The file must compile under Isabelle\textquotesingle{}s \char0{}quick\_\allowbreak{}and\_\allowbreak{}dirty\char0{} mode.\allowbreak{}}
\pl{0}{3.\allowbreak{} Structure the file as:\allowbreak{}}
\pl{3}{theory \textless{}theory\_\allowbreak{}name\textgreater{}}
\pl{3}{imports Main}
\pl{3}{begin}
\pe
\pl{3}{typedecl entity}
\pl{3}{typedecl event}
\pe
\pl{3}{consts}
\pl{5}{\textless{}predicate\_\allowbreak{}1\textgreater{} :\allowbreak{}:\allowbreak{} \textquotedbl{}\textless{}type signature\textgreater{}\textquotedbl{}}
\pl{5}{\textless{}predicate\_\allowbreak{}2\textgreater{} :\allowbreak{}:\allowbreak{} \textquotedbl{}\textless{}type signature\textgreater{}\textquotedbl{}}
\pl{5}{.\allowbreak{}.\allowbreak{}.\allowbreak{}}
\pe
\pl{3}{axiomatization where}
\pl{5}{explanation\_\allowbreak{}1:\allowbreak{} \textquotedbl{}\textless{}formalized explanation 1\textgreater{}\textquotedbl{}}
\pe
\pl{3}{theorem hypothesis:\allowbreak{}}
\pl{5}{assumes asm:\allowbreak{} \textquotedbl{}\textless{}formalized premise\textgreater{}\textquotedbl{}}
\pl{5}{shows \textquotedbl{}\textless{}formalized hypothesis\textgreater{}\textquotedbl{}}
\pl{5}{sorry}
\pe
\pl{3}{end}
\pl{0}{FAITHFULNESS RULES (\allowbreak{}the file is checked against these after it compiles):\allowbreak{}}
\pl{0}{-{} One rule per explanation sentence:\allowbreak{} every \char0{}axiomatization\char0{} entry must render exactly one explanation sentence,\allowbreak{} saying what that sentence says and nothing more.\allowbreak{} Do NOT add any other axioms or bridge rules.\allowbreak{}}
\pl{0}{-{} The premise goes into \char0{}assumes\char0{} and the hypothesis into \char0{}shows\char0{},\allowbreak{} with every qualifier kept and nothing added,\allowbreak{} dropped,\allowbreak{} strengthened or weakened.\allowbreak{}}
\pl{0}{-{} The statements must connect:\allowbreak{} write the premise in \char0{}assumes\char0{} at a fine enough grain that it introduces the same entities and predicates the explanation rule uses in its antecedent,\allowbreak{} so that the rule applies to the assumptions directly and yields what \char0{}shows\char0{} needs.\allowbreak{} Reuse one predicate name for one concept across premise,\allowbreak{} explanation and hypothesis.\allowbreak{}}
\pl{0}{-{} Name every symbol after words that occur in the text.\allowbreak{}}
\pl{0}{4.\allowbreak{} The \char0{}\textless{}theory\_\allowbreak{}name\textgreater{}\char0{} MUST be EXACTLY the value provided in the user prompt — Isabelle requires that the theory name match the filename.\allowbreak{} Do not invent or modify it.\allowbreak{}}
\pl{0}{5.\allowbreak{} Wrap your output in fences exactly as:\allowbreak{}}
\pl{0}{\char0{}\char0{}\char0{}isabelle}
\pl{0}{\textless{}your theory file content\textgreater{}}
\pl{0}{\char0{}\char0{}\char0{}}
\pl{0}{Output nothing outside the fences.\allowbreak{}}
\pe
\end{promptlines}

\noindent{\small\bfseries USER:}\par
\begin{promptlines}
\pl{0}{Theory name (\allowbreak{}use this exact name):\allowbreak{} demo\_\allowbreak{}javelin\_\allowbreak{}competition}
\pe
\pl{0}{Premise:\allowbreak{} Javelin competition}
\pl{0}{Hypothesis:\allowbreak{} There is a competition.\allowbreak{}}
\pl{0}{Explanation:\allowbreak{} There must be a competition for a Javelin competition to take place.\allowbreak{}}
\end{promptlines}

\noindent{\small\bfseries ASSISTANT:}\par
\begin{promptlines}
\pl{0}{\char0{}\char0{}\char0{}isabelle}
\pl{0}{theory demo\_\allowbreak{}javelin\_\allowbreak{}competition}
\pl{0}{imports Main}
\pl{0}{begin}
\pe
\pl{0}{typedecl entity}
\pl{0}{typedecl event}
\pe
\pl{0}{consts}
\pl{2}{JavelinCompetition :\allowbreak{}:\allowbreak{} \textquotedbl{}event \textbackslash{}\textless{}Rightarrow\textgreater{} bool\textquotedbl{}}
\pl{2}{Competition :\allowbreak{}:\allowbreak{} \textquotedbl{}event \textbackslash{}\textless{}Rightarrow\textgreater{} bool\textquotedbl{}}
\pe
\pl{0}{axiomatization where}
\pl{2}{explanation\_\allowbreak{}1:\allowbreak{} \textquotedbl{}\textbackslash{}\textless{}forall\textgreater{}e.\allowbreak{} JavelinCompetition e \textbackslash{}\textless{}longrightarrow\textgreater{} Competition e\textquotedbl{}}
\pe
\pl{0}{theorem hypothesis:\allowbreak{}}
\pl{2}{assumes asm:\allowbreak{} \textquotedbl{}JavelinCompetition e\textquotedbl{}}
\pl{2}{shows \textquotedbl{}\textbackslash{}\textless{}exists\textgreater{}e.\allowbreak{} Competition e\textquotedbl{}}
\pl{2}{sorry}
\pe
\pl{0}{end}
\pl{0}{\char0{}\char0{}\char0{}}
\pe
\end{promptlines}

\noindent{\small\bfseries USER:}\par
\begin{promptlines}
\pl{0}{Theory name (\allowbreak{}use this exact name):\allowbreak{} demo\_\allowbreak{}group\_\allowbreak{}costumes}
\pe
\pl{0}{Premise:\allowbreak{} a group of people are wearing costumes}
\pl{0}{Hypothesis:\allowbreak{} People are wearing costumes.\allowbreak{}}
\pl{0}{Explanation:\allowbreak{} A group means that there are people.\allowbreak{}}
\end{promptlines}

\noindent{\small\bfseries ASSISTANT:}\par
\begin{promptlines}
\pl{0}{\char0{}\char0{}\char0{}isabelle}
\pl{0}{theory demo\_\allowbreak{}group\_\allowbreak{}costumes}
\pl{0}{imports Main}
\pl{0}{begin}
\pe
\pl{0}{typedecl entity}
\pl{0}{typedecl event}
\pe
\pl{0}{consts}
\pl{2}{Group :\allowbreak{}:\allowbreak{} \textquotedbl{}entity \textbackslash{}\textless{}Rightarrow\textgreater{} bool\textquotedbl{}}
\pl{2}{People :\allowbreak{}:\allowbreak{} \textquotedbl{}entity \textbackslash{}\textless{}Rightarrow\textgreater{} bool\textquotedbl{}}
\pl{2}{WearingCostumes :\allowbreak{}:\allowbreak{} \textquotedbl{}entity \textbackslash{}\textless{}Rightarrow\textgreater{} bool\textquotedbl{}}
\pe
\pl{0}{axiomatization where}
\pl{2}{explanation\_\allowbreak{}1:\allowbreak{} \textquotedbl{}\textbackslash{}\textless{}forall\textgreater{}x.\allowbreak{} Group x \textbackslash{}\textless{}longrightarrow\textgreater{} People x\textquotedbl{}}
\pe
\pl{0}{theorem hypothesis:\allowbreak{}}
\pl{2}{assumes asm:\allowbreak{} \textquotedbl{}Group x \textbackslash{}\textless{}and\textgreater{} WearingCostumes x\textquotedbl{}}
\pl{2}{shows \textquotedbl{}\textbackslash{}\textless{}exists\textgreater{}x.\allowbreak{} People x \textbackslash{}\textless{}and\textgreater{} WearingCostumes x\textquotedbl{}}
\pl{2}{sorry}
\pe
\pl{0}{end}
\pl{0}{\char0{}\char0{}\char0{}}
\pe
\end{promptlines}

\noindent{\small\bfseries USER:}\par
\begin{promptlines}
\pl{0}{Theory name (\allowbreak{}use this exact name):\allowbreak{} demo\_\allowbreak{}group\_\allowbreak{}beach}
\pe
\pl{0}{Premise:\allowbreak{} A group of people on the beach with cameras.\allowbreak{}}
\pl{0}{Hypothesis:\allowbreak{} People are on a beach.\allowbreak{}}
\pl{0}{Explanation:\allowbreak{} Group of people are people.\allowbreak{}}
\end{promptlines}

\noindent{\small\bfseries ASSISTANT:}\par
\begin{promptlines}
\pl{0}{\char0{}\char0{}\char0{}isabelle}
\pl{0}{theory demo\_\allowbreak{}group\_\allowbreak{}beach}
\pl{0}{imports Main}
\pl{0}{begin}
\pe
\pl{0}{typedecl entity}
\pl{0}{typedecl event}
\pe
\pl{0}{consts}
\pl{2}{GroupOfPeople :\allowbreak{}:\allowbreak{} \textquotedbl{}entity \textbackslash{}\textless{}Rightarrow\textgreater{} bool\textquotedbl{}}
\pl{2}{People :\allowbreak{}:\allowbreak{} \textquotedbl{}entity \textbackslash{}\textless{}Rightarrow\textgreater{} bool\textquotedbl{}}
\pl{2}{OnBeach :\allowbreak{}:\allowbreak{} \textquotedbl{}entity \textbackslash{}\textless{}Rightarrow\textgreater{} bool\textquotedbl{}}
\pl{2}{WithCameras :\allowbreak{}:\allowbreak{} \textquotedbl{}entity \textbackslash{}\textless{}Rightarrow\textgreater{} bool\textquotedbl{}}
\pe
\pl{0}{axiomatization where}
\pl{2}{explanation\_\allowbreak{}1:\allowbreak{} \textquotedbl{}\textbackslash{}\textless{}forall\textgreater{}x.\allowbreak{} GroupOfPeople x \textbackslash{}\textless{}longrightarrow\textgreater{} People x\textquotedbl{}}
\pe
\pl{0}{theorem hypothesis:\allowbreak{}}
\pl{2}{assumes asm:\allowbreak{} \textquotedbl{}GroupOfPeople x \textbackslash{}\textless{}and\textgreater{} OnBeach x \textbackslash{}\textless{}and\textgreater{} WithCameras x\textquotedbl{}}
\pl{2}{shows \textquotedbl{}\textbackslash{}\textless{}exists\textgreater{}x.\allowbreak{} People x \textbackslash{}\textless{}and\textgreater{} OnBeach x\textquotedbl{}}
\pl{2}{sorry}
\pe
\pl{0}{end}
\pl{0}{\char0{}\char0{}\char0{}}
\pe
\end{promptlines}

\noindent{\small\bfseries USER:}\par
\begin{promptlines}
\pl{0}{Theory name (\allowbreak{}use this exact name):\allowbreak{} eval\_\allowbreak{}70}
\pe
\pl{0}{Premise:\allowbreak{} A man and a woman stare at the skyline.\allowbreak{}}
\pl{0}{Hypothesis:\allowbreak{} The two people are looking outside.\allowbreak{}}
\pl{0}{Explanation:\allowbreak{}\ \ A man and a woman are two people.\allowbreak{} \textquotedbl{}Looking outside\textquotedbl{} is a rephrasing of \textquotedbl{}stare at the skyline.\allowbreak{}\textquotedbl{}}
\pe
\pl{0}{Produce the complete Isabelle 2024 theory file with a \char0{}sorry\char0{} proof body.\allowbreak{}}
\pl{0}{Your previous attempt FAILED the prover check.\allowbreak{} Here is the failed attempt:\allowbreak{}}
\pl{0}{\char0{}\char0{}\char0{}}
\pl{0}{theory eval\_\allowbreak{}70}
\pl{0}{imports Main}
\pl{0}{begin}
\pe
\pl{0}{typedecl entity}
\pl{0}{typedecl event}
\pe
\pl{0}{consts}
\pl{2}{Man :\allowbreak{}:\allowbreak{} \textquotedbl{}entity \textbackslash{}\textless{}Rightarrow\textgreater{} bool\textquotedbl{}}
\pl{2}{Woman :\allowbreak{}:\allowbreak{} \textquotedbl{}entity \textbackslash{}\textless{}Rightarrow\textgreater{} bool\textquotedbl{}}
\pl{2}{StareAtSkyline :\allowbreak{}:\allowbreak{} \textquotedbl{}entity \textbackslash{}\textless{}Rightarrow\textgreater{} bool\textquotedbl{}}
\pl{2}{TwoPeople :\allowbreak{}:\allowbreak{} \textquotedbl{}entity \textbackslash{}\textless{}Rightarrow\textgreater{} bool\textquotedbl{}}
\pl{2}{LookingOutside :\allowbreak{}:\allowbreak{} \textquotedbl{}entity \textbackslash{}\textless{}Rightarrow\textgreater{} bool\textquotedbl{}}
\pe
\pl{0}{axiomatization where}
\pl{2}{explanation\_\allowbreak{}1:\allowbreak{} \textquotedbl{}\textbackslash{}\textless{}forall\textgreater{}x.\allowbreak{} (\allowbreak{}Man x \textbackslash{}\textless{}and\textgreater{} Woman x) \textbackslash{}\textless{}longrightarrow\textgreater{} TwoPeople x\textquotedbl{}}
\pl{2}{explanation\_\allowbreak{}2:\allowbreak{} \textquotedbl{}\textbackslash{}\textless{}forall\textgreater{}x.\allowbreak{} StareAtSkyline x \textbackslash{}\textless{}longleftrightarrow\textgreater{} LookingOutside x\textquotedbl{}}
\pe
\pl{0}{theorem hypothesis:\allowbreak{}}
\pl{2}{assumes asm:\allowbreak{} \textquotedbl{}Man x \textbackslash{}\textless{}and\textgreater{} Woman x \textbackslash{}\textless{}and\textgreater{} StareAtSkyline x\textquotedbl{}}
\pl{2}{shows \textquotedbl{}TwoPeople x \textbackslash{}\textless{}and\textgreater{} LookingOutside x\textquotedbl{}}
\pl{2}{sorry}
\pe
\pl{0}{end}
\pl{0}{\char0{}\char0{}\char0{}}
\pe
\pl{0}{Your reflections on the past failures (\allowbreak{}most recent last):\allowbreak{}}
\pl{0}{1.\allowbreak{} The error indicates a syntax issue:\allowbreak{} the \char0{}axiomatization\char0{} block is missing the \char0{}axioms\char0{} keyword or is improperly closed.\allowbreak{} The parser sees \char0{}explanation\_\allowbreak{}2\char0{} as an unexpected identifier because the block ended prematurely.\allowbreak{} In the next attempt,\allowbreak{} I will use the correct \char0{}axiomatization\char0{} syntax with \char0{}where\char0{} followed by a single \char0{}axiom\char0{} keyword or properly separate each axiom with \char0{}and\char0{}.\allowbreak{} I will also ensure the block is correctly terminated before the \char0{}theorem\char0{} command.\allowbreak{}}
\pe
\pl{0}{Using these reflections,\allowbreak{} produce a corrected formalization.\allowbreak{} Follow the same output format rules as before (\allowbreak{}wrap the code in a fenced code block).\allowbreak{}}
\pe
\end{promptlines}

\end{promptbox}

\end{document}